\documentclass[10pt,twocolumn,letterpaper]{article}

\usepackage[pagenumbers]{cvpr} 

\usepackage{graphicx}
\usepackage{amsmath}
\usepackage{amssymb}
\usepackage{booktabs}

\usepackage[pagebackref,breaklinks,colorlinks]{hyperref}

\usepackage[capitalize]{cleveref}
\crefname{section}{Sec.}{Secs.}
\Crefname{section}{Section}{Sections}
\Crefname{table}{Table}{Tables}
\crefname{table}{Tab.}{Tabs.}

\usepackage{enumitem}                               
\usepackage{extarrows}                              
\usepackage{amsmath,amssymb,amsfonts,amsthm,dsfont} 
\usepackage{algorithm,algorithmicx,listings}        
\usepackage[noend]{algpseudocode}			        
\usepackage{graphicx,tabularx,subcaption}
\usepackage{multirow}
\usepackage[export]{adjustbox}
\usepackage{siunitx} 

\makeatletter
\@namedef{ver@everyshi.sty}{}
\makeatother
\usepackage{tikz}
\usetikzlibrary{shapes.geometric}

\usepackage[titletoc,title]{appendix}

\def\bfz{\boldsymbol{z}}
\def\bfR{\boldsymbol{R}}
\def\bfP{\boldsymbol{P}}
\def\bfI{\boldsymbol{I}}
\def\bfx{\boldsymbol{x}}
\def\bfe{\boldsymbol{e}}
\def\bfu{\boldsymbol{u}}

\def\bfy{\boldsymbol{y}}
\def\bfp{\boldsymbol{p}}
\def\bfq{\boldsymbol{q}}

\def\bfeta{\boldsymbol{\eta}}
\def\bftheta{\boldsymbol{\theta}}

\def\calO{\mathcal{O}}

\def\calF{\mathcal{F}}
\def\calI{\mathcal{I}}

\DeclareMathOperator{\erf}{erf}
\newcommand{\scaleMathLine}[2][1]{\resizebox{#1\linewidth}{!}{$\displaystyle{#2}$}}
\newcommand{\prl}[1]{\left(#1\right)}

\newcommand{\crl}[1]{\left\{#1\right\}}

\newtheorem{proposition}{Proposition}

\newtheorem{lemma}{Lemma}
\theoremstyle{definition}
\newtheorem{definition}{Definition}

\newtheorem*{assumption*}{Assumption}
\newtheorem*{problem*}{Problem}
\newtheorem{problem}{Problem}
\theoremstyle{remark}

\newtheorem*{solution*}{Solution}

\def\confYear{2022}

\begin{document}

\title{A Deep Signed Directional Distance Function for Object Shape Representation}

\author{Ehsan Zobeidi\quad\quad Nikolay Atanasov\\
University of California, San Diego\\
{\tt\small \{ezobeidi,natanasov\}@ucsd.edu}
}
\maketitle


\begin{abstract}
Neural networks that map 3D coordinates to signed distance function (SDF) or occupancy values have enabled high-fidelity implicit representations of object shape. This paper develops a new shape model that allows synthesizing novel distance views by optimizing a continuous signed directional distance function (SDDF). Similar to deep SDF models, our SDDF formulation can represent whole categories of shapes and complete or interpolate across shapes from partial input data. Unlike an SDF, which measures distance to the nearest surface in any direction, an SDDF measures distance in a given direction. This allows training an SDDF model without 3D shape supervision, using only distance measurements, readily available from depth camera or Lidar sensors. Our model also removes post-processing steps like surface extraction or rendering by directly predicting distance at arbitrary locations and viewing directions. Unlike deep view-synthesis techniques, such as Neural Radiance Fields, which train high-capacity black-box models, our model encodes by construction the property that SDDF values decrease linearly along the viewing direction. This structure constraint not only results in dimensionality reduction but also provides analytical confidence about the accuracy of SDDF predictions, regardless of the distance to the object surface.
\end{abstract}


\section{Introduction}

Geometric understanding of object shape is a central problem for enabling task specification, environment interaction, and safe navigation for autonomous systems. Various models of object shape have been proposed to facilitate recognition, classification, rendering, reconstruction, etc. There is no universal 3D shape representation because different models offer different advantages. For example, explicit shape models based on polygonal meshes allow accurate representation of surfaces and texture and lighting properties. Generative mesh modeling \cite{tulsiani2017learning,cmrKanazawa18}, however, is very challenging as it requires predicting the mesh topology and the number of vertices. Impressive results have been achieved recently with implicit shape models, representing surfaces as the zero level set of a deep neural network approximation of signed distance function (SDF) \cite{park2019deepsdf} or occupancy field \cite{mescheder2019occupancy}. Many implicit surface techniques, however, require 3D shape supervision and post-processing in the form of surface extraction and distance computation \cite{han2019fiesta}. Deep view synthesis models \cite{mildenhall2020nerf} offer an alternative to directly synthesize texture, lighting, and distance, avoiding surface extraction or differentiable rendering.

This paper enables implicit shape description by learning a model capable of novel distance view synthesis. We represent an object shape as a continuous function $h(\bfp,\bfeta)$ which measures the (signed) distance to the object surface from a given 3D position $\bfp$ and unit-norm viewing direction $\bfeta$. We refer to such a function as a signed directional distance function (SDDF). Compared to an SDF, approximating an SDDF with a neural network appears more challenging due to the additional two-degree-of-freedom input $\bfeta$. With a naive model, there is no guarantee that the parameter optimization converges to the correct object geometry as the input degrees of freedom increase. This challenge is evident even in SDF models, where learned distances are only accurate close to the surface and surface extraction is necessary to predict distances far away from the object boundary. Inspired by Gropp et al.~\cite{gropp2020implicit} who observe that a valid SDF must satisfy an Eikonal differential equation, we obtain a differential equation capturing the fact that SDDF values decrease linearly along the viewing direction. We design a neural network architecture for learning SDDFs that ensures \emph{by construction} that the SDDF gradient property is satisfied. This not only results in dimensionality reduction but also provides analytical confidence that the accuracy of an SDDF model is independent of the distance of the training or testing points to the object surface. More precisely, training the model to accurately predict SDDF values everywhere does not need dense sampling of the domain. The particular distance of the training samples to the object surface is not important, leading to a reduction in the number of necessary training samples. Because an SDDF outputs distance to the object surfaces directly, our model can be trained without 3D supervision, using distance measurements from depth camera or Lidar sensors.

Nonetheless, training an SDDF model requires distance data from different orientations $\bfeta$ (unlike an SDF model). To avoid the need for a large training set with distance data from many views, 
we develop a data augmentation technique for distance data synthesis from new positions and orientations.
Given a point cloud observation of the object surface, we decide which points would be visible from a desired view using spherical projection and convex hull approximation. Our data augmentation technique ensures that we can train a multi-view consistency SDDF model even from a small training set with a few distance views.
%
%
%
%
%
%
%


Inspired by DeepSDF \cite{park2019deepsdf}, we extend the SDDF model to enable category-level shape description. Instances from the same object category posses similar geometric structure. Training a different shape model for each instance is inefficient and impractical. We introduce a category-level SDDF auto-decoder and an instance-level latent shape code to represent the geometry of a class of shapes. We show that a trained SDDF model is capable of interpolating between the latent shape codes of different instances, while generating valid shapes and distance views at intermediate points. Optimizing the latent code at test time also allows shape completion of previously unseen instances from a small set of distance samples.


In summary, we make the following contributions.
\begin{itemize}[nosep]
  \item We propose a new signed directional distance function to model continuous distance view synthesis.
  \item We derive structural properties satisfied by SDDFs and encode them in the design of a neural network architecture for SDDF learning.
  \item We propose a data augmentation technique to ensure that a multi-view consistent SDDF model can be trained from a small dataset.
  \item We demonstrate that an SDDF model, augmented with a latent shape code, is capable of representing a category of shapes, enabling shape completion and shape interpolation without 3D supervision.
\end{itemize}
Our model is demonstrated in qualitative and quantitative experiments using the ShapeNet dataset~\cite{chang2015shapenet}.

\section{Related Works}

This section reviews 3D shape modeling techniques.

%
 
{\bf \noindent Mesh models: } Several memory efficient explicit mesh representations of shape have been proposed \cite{gao2019sparse, kobbelt1998interactive, sorkine2004laplacian, yu2004mesh, au2006dual, zhou2007large, tan2018variational}. Object surfaces may be viewed as a collection of connected charts, parameterized by a neural network \cite{sinha2017surfnet, maron2017convolutional, williams2019deep}. AtlasNet \cite{groueix2018papier} parametrizes each chart with a multi-layer perceptron that maps a flat square to the real chart.
Deep geometric prior (DGP)~\cite{williams2019deep} improves the results using Wasserstein distance and enforcing a consistency condition to fit the charts.

{\bf \noindent Geometric primitive models: } PointNet \cite{qi2017pointnet} proposes a new architecture for point cloud feature extraction that respects the permutation invariance of point clouds.
Shape completion from partial point cloud data is investigated by \cite{xie2020grnet,groueix2018papier,yuan2018pcn,yang2018foldingnet,tchapmi2019topnet, liu2020morphing} using an encoder-decoder structure to estimate the point cloud of unseen shapes. Generative adversarial networks and adversarial auto-encoders have been employed recently for point cloud shape synthesis \cite{wu2016learning, achlioptas2018learning, shu20193d, li2019pu, yang2019pointflow, yu2020point}. Point cloud models, however, do not provide continuous shape representations. In cases where a coarse model is sufficient, 3D volumetric primitives can be used \cite{tulsiani2017learning}, including cuboids \cite{yang2019cubeslam} or quadrics \cite{nicholson2018quadricslam, paschalidou2019superquadrics}.


%

{\bf \noindent Grid-based models: } Discretizing 3D space into a regular or adaptive grid to store occupancy \cite{choy20163d,tatarchenko2017octree} is another popular representation. OctNet \cite{riegler2017octnet,riegler2017octnetfusion} defines convolution directly over octrees, exploiting the sparsity and hierarchical partitioning of 3D space. Octrees may be used to store a truncated signed distance function whose zero level set corresponds with the object surface \cite{curless1996volumetric, zeng20173dmatch, chen2019learning, mescheder2019occupancy}. Choosing voxels close to the surface and using a kernel to predict continuous truncated SDF values improves the accuracy \cite{Zobeidi_GPMapping_IROS20}. 
{\bf \noindent Deep signed distance models: } DeepSDF \cite{park2019deepsdf} develops an auto-decoder model for approximating continuous SDF values and enables learning category-level shape through a latent shape code. This work demonstrated that various object topologies can be captured as differentiable implicit functions, inspiring interest in learned SDF representations \cite{Genova_2020_CVPR,Jiang_2020_CVPR,lin2020sdfsrn,periodic,sitzmann2019srns,disn}. IGR \cite{gropp2020implicit} improves the method by incorporating a unit-norm gradient constraint on the SDF values in the training loss function. IDR \cite{idr} extends the SDF model to simultaneously learn geometry, camera parameters, and a neural renderer that approximates the light reflected towards the camera. MVSDF \cite{Zhang_2021_ICCV} optimizes an SDF and a light field appearance model jointly, supervised by image features and depth from a multi-view stereo network. A-SDF \cite{mu2021sdf} represents articulated shapes with a disentangled latent space, including separate codes for encoding shape and articulation.

{\bf \noindent View synthesis models: }
Niemeyer et al. \cite{niemeyer2020differentiable} enable differentiable rendering of implicit shape and texture representations by deriving the gradients of the predicted depth map with respect to the network parameters. NeRF \cite{mildenhall2020nerf,martin2020nerf, graf} learns to predict the volume density and radiance of a scene at arbitrary positions and viewing directions using RGB images as input. 
NeRF is trained as a high-capactiy black-box model and does not capture the property that distances decrease linearly along the viewing direction. IBRNet \cite{wang2021ibrnet} uses a multilayer perceptron and a ray transformer to estimate the radiance and volume density at continuous position and view locations from a sparse set of nearby views. GRF~\cite{grf2020} is a neural network model for implicit radiance field representation and rendering of 3D objects, trained by aggregating pixel features from multiple 2D views. MVSNeRF \cite{chen2021mvsnerf} extends deep multi-view stereo methods to reason about both scene geometry and appearance and output a neural radiance field. This radiance field model can be fine-tuned on novel test scenes significantly faster than a NeRF model. DietNeRF \cite{Jain_2021_ICCV} introduces an auxiliary semantic consistency loss that encourages realistic renderings at novel poses. This allows supervising DietNeRF from arbitrary poses leading to high-quality scene reconstruction with as few as 8 training views.

\section{Problem Statement}
\label{sec:problem}

We focus on learning shape representations for object instances from a known category, e.g., car, airplane, chair, etc. In contrast with most existing work for shape modeling which relies on 3D CAD models for training, we only consider distance measurement data, e.g., obtained from a depth camera or a Lidar scanner. We model a distance sensor measurement as a collection of rays (e.g, corresponding to depth camera pixels or Lidar beams) along which the distance from the sensor position (e.g., depth camera optical center or Lidar sensor frame origin) to the nearest surface is measured. Let $\bfeta_i \in S^{n-1} := \crl{\bfeta \in \mathbb{R}^n \mid \|\bfeta\|_2=1}$ denote a unit-vector in the direction of ray $i$ with associated distance measurement $d_i \in (\underline{d},\overline{d}) \cup \{\infty\}$ obtained from sensor position $\bfp_i \in \mathbb{R}^n$. In practice, the dimension $n$ is $2$ or $3$ and the measurements are limited by a minimum distance $\underline{d} > 0$ and a maximum distance $\overline{d} < \infty$. The measurements of rays that do not hit a surface are set to $\infty$. We consider the following shape representation problem.

\begin{problem} \label{problem}
Let $\mathcal{D}_l := \crl{(\bfp_{i,l}, \bfeta_{i,l}, d_{i,l})}_i$ be sets of distance measurements obtained from different instances $l$ from the same object category. Learn a latent shape encoding $\bfz_l \in \mathbb{R}^m$ for each instance $l$ and a function $h(\bfp,\bfeta,\bfz)$ that can predict the distance from any point $\bfp$ along any direction $\bfeta$ to the surface of any instance with shape $\bfz$.
\end{problem}



\section{Method}
\label{sec:method}

This section proposes a new signed directional distance representation of object shape (Sec.~\ref{sec:SDDF}), studies its properties (Sec.~\ref{sec:sddf_structure}, Sec.~\ref{sec:infinite_sddf}), and proposes a neural network architecture, cost function, and data augmentation technique for learning such shape representations (Sec.~\ref{sec:training_and_inference}, Sec.~\ref{sec:multiview}). 


\subsection{Signed Directional Distance Function}
\label{sec:SDDF}

We propose a signed directional distance function to model the data generated by distance sensors. 

\begin{definition}\label{SDDF}
The \emph{signed directional distance function} (SDDF) $h: \mathbb{R}^n \times S^{n-1} \mapsto \mathbb{R}$ of a set $\calO \subset \mathbb{R}^n$ measures the signed distance from a point $\bfp \in \mathbf{R}^n$ to the set boundary $\partial\calO$ in direction $\bfeta\in S^{n-1}$:
	\begin{equation}
		\begin{aligned} \label{eq:H_signed_distance}
			h(\bfp,\bfeta) &:= d_{\bfeta}(\bfp,\partial\calO),\\
			d_{\bfeta}(\bfp,\partial\calO) &:= \min \left\{ d \in \mathbb{R} \;\big\vert\; \bfp + d \bfeta \in  \partial\calO \right\}. 
		\end{aligned}
	\end{equation}
\end{definition}

Unlike an SDF, which measures the distance to the nearest surface in \emph{any} direction, an SDDF measures the distance to the nearest surface in a \emph{specific} direction. Also, unlike an SDF, which is negative inside the surface that it models, and SDDF is negative behind the observer's point of view. A key property is that, if the SDDF of a set is known, we can generate arbitrary distance views to the set boundary. In other words, we can image what a distance sensor would see from any point $\bfp$ in any viewing direction $\bfeta$.

We focus on learning SDDF representaions using distance measurements as in Problem~\ref{problem}. We propose a neural network architecture that, by design, captures the structure of an SDDF. Note that for a fixed viewing direction $\bfeta$, an SDDF satisfies $h(\bfp_1,\bfeta) - h(\bfp_2,\bfeta) = (\bfp_2 - \bfp_1)^\top \bfeta$ for points $\bfp_1$, $\bfp_2$ along the ray $\bfeta$ that are close to each other, in the sense that they see the same nearest point on the set surface. This property is formalized below.

\begin{lemma}\label{cond}
The gradient of an SDDF $h(\bfp,\bfeta)$ with respect to $\bfp$ projected to the viewing direction $\bfeta$ satisfies:
\begin{equation}\label{eq:cond}
\nabla_{\bfp} h(\bfp,\bfeta)^\top \bfeta = -1.
\end{equation}
\end{lemma}



\subsection{SDDF Structure}
\label{sec:sddf_structure}

In this section, we propose a neural network parameterization of a function $h(\bfp,\bfeta)$ that satisfies the condition in \eqref{eq:cond} by construction. First, we simplify the requirement that the gradient in \eqref{eq:cond} is non-zero by defining a function $g(\bfp, \bfeta) := h(\bfp, \bfeta) + \bfp^\top \bfeta$. Note that \eqref{eq:cond} is equivalent to:
\begin{equation}\label{eq:g-cond}
\nabla_{\bfp} g(\bfp,\bfeta)^\top \bfeta = 0.
\end{equation}
Next, we show that \eqref{eq:g-cond} implies that one degree of freedom should be removed from the domain of $g(\bfp, \bfeta)$. Our idea is to rotate $\bfp$ and $\bfeta$ so that viewing direction $\bfeta$ becomes the unit vector $\bfe_n = [0, \ldots, 0, 1]^\top$ along the last coordinate axis in the sensor frame. This rotation will show that the third element of the gradient of $g(\bfp, \bfeta)$ should be zero, implying that $g(\bfp, \bfeta)$ is constant along the third dimension in the rotated reference frame. The rotation matrix $\bfR \in SO(n)$ that maps a unit vector $\bfx \in S^{n-1}$ to another unit vector $\bfy \in S^{n-1}$ with $\bfy \neq - \bfx$ along the sphere geodesic (shortest path) is \cite{CodesidoRotation}:
\begin{equation}\label{eq:codesido-rotation}
	\bfR = \bfI + \bfy\bfx^\top - \bfx\bfy^\top + \frac{1}{1+\bfx^\top \bfy}(\bfy\bfx^\top - \bfx\bfy^\top)^2.
\end{equation}
Using~\eqref{eq:codesido-rotation}, we can obtain an explicit expression for the rotation matrix $\bfR_{\bfeta}$ that maps $\bfeta$ to $\bfe_n$.

\begin{lemma}
A vector $\bfeta = [a, b]^\top \in S^1$ can be mapped to $\bfe_2 \in S^1$ via the rotation matrix $\bfR_{\bfeta} := \begin{bmatrix} b & -a\\ a & b \end{bmatrix} \in SO(2)$.
\end{lemma}

\begin{lemma}
\label{lemma:R3}
A vector $\bfeta = [a, b, c]^\top \in S^2$ can be mapped to $\bfe_3 \in S^2$ via the rotation matrix $\bfR_{\bfeta} \in SO(3)$ below:
\begin{equation}
	\bfR_\eta := 
	\begin{cases}
		\begin{bmatrix}
			1 & 0 & 0\\
			0 & 1 & 0\\
			0 & 0 & -1
		\end{bmatrix}& \text{if } \bfeta = -\bfe_3,\\
		\begin{bmatrix}
			1-\frac{a^2}{1+c} & -\frac{ab}{1+c} & -a\\
			-\frac{ab}{1+c} & 1-\frac{b^2}{1+c} & -b\\
			a & b & c
		\end{bmatrix}& \text{otherwise}.
	\end{cases}
\end{equation}
\end{lemma}

Using $\bfR_{\bfeta}$, we can express the condition in \eqref{eq:g-cond} in a rotated coordinate frame where $\bfq = \bfR_{\bfeta}\bfp$. By the chain rule:
\begin{equation}\label{eq:f-cond}
0 = \frac{dg}{d\bfp} \bfeta = \frac{dg}{d \bfq} \frac{d\bfq}{d\bfp} \bfeta = \frac{dg}{d \bfq} \bfR_{\bfeta} \bfeta = \frac{dg}{d \bfq} \bfe_n = \frac{d g}{d q_n}.
\end{equation}
The set of functions that satisfy \eqref{eq:f-cond} do not depend on the last element of $\bfq$ or, in other words, can be expressed as $g(\bfp,\bfeta) = f(\bfP \bfR_{\bfeta} \bfp, \bfeta)$ for a projection matrix $\bfP := [\bfI\; \mathbf{0}] \in \mathbb{R}^{(n-1) \times n}$ and some function $f$. This elucidates the structure of signed directional distance functions.

\begin{proposition} \label{prop:SDDF-structure}
Learning a function $f: \mathbb{R}^{n-1} \times S^{n-1} \mapsto \mathbb{R}$ guarantees that $h(\bfp,\bfeta) := f(\bfP \bfR_{\bfeta} \bfp, \bfeta) - \bfp^\top\bfeta$ is an SDDF (Def.~\ref{SDDF}) and satisfies \eqref{eq:cond} in Lemma~\ref{cond}. 
\end{proposition}

\subsection{Infinite SDDF Values}
\label{sec:infinite_sddf}

Proposition~\ref{prop:SDDF-structure} allows learning SDDF representations of object shape from distance measurements without the need to enforce structure constrains explicitly. An additional challenge, however, is that real distance sensors have a limited field of view and, hence, the SDDF values $h(\bfp,\bfeta)$ at some sensor positions $\bfp$ and viewing directions $\bfeta$ (e.g., not directly looking toward the object) will be infinite. We cannot expect a regression model to predict infinite values directly. We introduce an invertible function $\phi$ to condition the distance data by squashing the values to a finite range.

\begin{lemma}\label{lem:infSol}
Let $\phi:\mathbb{R}\mapsto \mathbb{R}$ be a function with non-zero derivative, $\phi'(x) \neq 0$, for all $x \in \mathbb{R}$. Then, for any function $g:\mathbb{R}^n \times S^{n-1}\mapsto \mathbb{R}$ and vector $\bfeta \in S^{n-1}$, we have:
\begin{equation}
\nabla_{\bfp} g(\bfp,\bfeta)^\top \bfeta = 0 \quad\text{iff}\quad  \nabla_{\bfp} \phi(g(\bfp,\bfeta))^\top \bfeta = 0.
\end{equation}
\end{lemma}

\begin{proof}
The claim is concluded by the chain rule, $0 = \nabla_{\bfp} \phi(g(\bfp,\bfeta))^\top \bfeta = \phi'(g(\bfp,\bfeta)) \nabla_{\bfp} g(\bfp,\bfeta)^\top \bfeta$ and since $\phi'(g(\bfp,\bfeta))$ is never zero.
\end{proof}

Since $\phi'$ is never zero, $\phi$ is either strictly increasing or strictly decreasing by the mean value theorem. In both cases, it has an inverse $\phi^{-1}$. Useful examples of such functions, which can be used to squash the distance values to a finite range, include logistic sigmoid $\sigma(x) := (1+\exp(-x))^{-1}$, hyperbolic tangent $\tanh(x)$, and the Gaussian error function $\erf(x)$. Hereafter we assume $\phi$ is strictly increasing and define $q(\bfp,\bfeta) := \phi(f(\bfP\bfR_{\bfeta}\bfp, \bfeta))$ such that as in Proposition~\ref{prop:SDDF-structure}:
\begin{equation}\label{eq:main}
h(\bfp, \bfeta) = \phi^{-1}(q(\bfp, \bfeta)) - \bfp^\top \bfeta.
\end{equation}
This formulation allows training of and inference with a neural network parameterization of $q(\bfp,\bfeta)$ with possibly infinite distance values. Due to Lemma~\ref{lem:infSol}, \eqref{eq:main} is still guaranteed to satisfy the SDDF property $\nabla_{\bfp} h(\bfp,\bfeta)^\top \bfeta = -1$.

\subsection{SDDF Learning}
\label{sec:training_and_inference}

Fig.~\ref{fig:method} shows a neural network model for learning an SDDF representation.
%
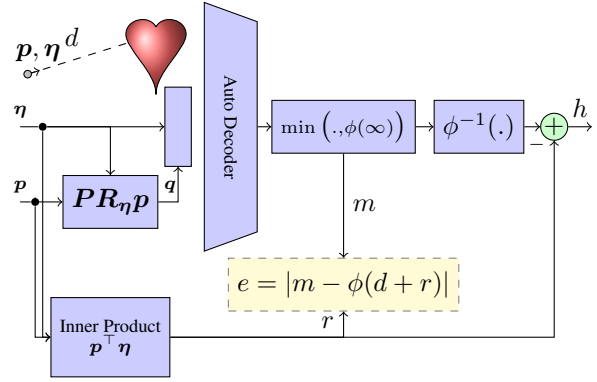
\begin{figure}[t]
\tikzstyle{innerproduct} = [draw, fill=blue!20, rectangle, 
    minimum height=3em, minimum width=3em]
\tikzstyle{dimreduction} = [draw, fill=blue!20, rectangle, 
    minimum height=2em, minimum width=2em]
\tikzstyle{decoderinput} = [draw, fill=blue!20, rectangle, 
    minimum height=3em, minimum width=1em]
\tikzstyle{decoder} = [draw, fill=blue!20, trapezium, trapezium angle=70, rotate =-90, 
    minimum height=2em, minimum width=4em]
\tikzstyle{phiInv} = [draw, fill=blue!20, rectangle, 
    minimum height=2em, minimum width=2em]
\tikzstyle{min} = [draw, fill=blue!20, rectangle, 
    minimum height=2em, minimum width=2em]
\tikzstyle{loss} = [draw=gray,dashed, fill=yellow!20, rectangle, 
    minimum height=2em, minimum width=3em]
\tikzstyle{input} = [coordinate]
\tikzstyle{output} = [coordinate]
\tikzstyle{branch}=[fill,shape=circle,minimum size=3pt,inner sep=0pt]
\tikzstyle{block} = [draw, fill=blue!20, rectangle, 
    minimum height=3em, minimum width=6em]
\tikzstyle{sum} = [draw, fill=green!20, circle, node distance=1cm, minimum size=3pt, inner sep=0pt]
\tikzstyle{input} = [coordinate]
\tikzstyle{output} = [coordinate]
\tikzstyle{pinstyle} = [pin edge={to-,thin,black}]

\begin{tikzpicture}[auto, node distance=1.8cm]
    \node [input, name=input_eta_pre] {};
    \node [input, below of=input_eta_pre, name=input_p_pre, node distance=10mm] {};
    \node [branch, right of=input_eta_pre,name=input_eta, node distance=3mm] {};
    \node [branch, right of=input_p_pre, node distance=2mm, name=input_p] {};
    \node [dimreduction, right of=input_p, node distance=10mm] (dimreduction) {$\bfP \bfR_{\bfeta} \bfp$};
    \node [innerproduct, below of=dimreduction] (innerproduct) {$\substack{\text{Inner Product}\\ \text{$\bfp^\top\bfeta$}}$};
    
    \node [decoderinput, right of=input_eta, node distance=18mm] (decoderinput) {};
     \node [decoder, above of=decoderinput, node distance=7mm] (decoder) {$\substack{\text{Auto Decoder}}$};
     \node [min, right of=decoder, node distance=15mm] (min) {$\scriptstyle\min\big(.,\scriptstyle{\phi(\infty)}\big)$};
     \node [phiInv, right of=min] (phiInv) {$\phi^{-1}(.)$};
     \node [sum, right of=phiInv] (sum) {$+$};
     \node [loss, below of=min, node distance=21mm] (loss) {$e = |m-\phi(d+r)|$};
     \node [output, right of=phiInv, name=out, node distance=15mm] {};
     \draw [draw,-, ] (input_eta_pre) |- node {$\scriptstyle\bfeta$} (input_eta);
     \draw [draw,-, ] (input_p_pre) |- node {$\scriptstyle\bfp$} (input_p);
     \draw [draw,->, ] (input_eta) |- node {} (innerproduct);
     \draw [draw,->] (input_p) |- node {} (innerproduct);
     \draw [draw,->, ] (input_eta) -| node {} (dimreduction);
     \draw [draw,->] (input_p) |- node {} (dimreduction);
     \draw [draw,->, ] (input_eta) -- node {} (decoderinput);
     \draw [draw,->, ] (dimreduction) -| node [pos=0.3] {$\scriptstyle \bfq$} (decoderinput);
     \draw [draw,->] (innerproduct) -| node[left, pos=0.99] {$\scriptstyle-$} (sum);
     \draw [draw,->] (decoder) -- node {} (min);
     \draw [draw,->] (min) -- node {} (phiInv);
     \draw [draw,->] (phiInv) -- node {} (sum);
     \draw [draw,->] (sum) -- node[above] {$h$} (out);
     \draw [draw,->] (min) -- node {$m$} (loss);
     \draw [draw,->] (innerproduct) -| node[above left] {$r$} (loss);
     \def\xpoi{0.1}
     \def\ypoi{0.7}
     \def\xobj{1.7}
     \def\yobj{1.2}
     \def\r{1}
     \def\xraydif{\xobj-\xpoi}
     \def\yraydif{\yobj-\ypoi}
     \def\xraydifs{\xraydif*\r}
     \def\yraydifs{\yraydif*\r}
     \def\xray{\xraydifs+\xpoi}
     \def\yray{\yraydifs+\ypoi}
     \draw [dashed,-] (\xpoi,\ypoi) -- node {$d$} (\xobj,\yobj);
     \draw [draw, ->] (\xpoi,\ypoi) -- node[above, pos=0.90] {$\bfp,\bfeta$} (0.26,0.75);
     \draw [draw, fill=gray!50, ellipse, minimum height=1em, minimum width=3em] (\xpoi,\ypoi) circle (0.5mm);
     \draw[ball color=red!60 , shading=ball, scale=0.1, yshift=35 0pt, xshift=40 0pt] (4,1) ..controls +(120:2cm)
        and +(90:2cm) .. (0,0) .. controls  +(-90:2cm) and +(90:3cm) ..
        (4,-8) .. controls +(90:3cm) and +(-90:2cm) ..(8,0)  .. controls
        +(90:2cm) and  +(60:2cm) .. (4,1);
     %
\end{tikzpicture}
\vspace{-0.7em}
\caption{Neural network model for SDDF approximation. Given a position $\bfp$, viewing direction $\bfeta$ and measured distance $d$, the model rotates $\bfp$ to new coordinates $\bfR_{\bfeta}\bfp$, whose last component does not effect the SDDF value. The projected input $\bfq$ is processed by an autodecoder to predict a squashed distance value $m$, which may be converted to an SDDF value $h$ or compared to a modified distance $d+\bfp^\top\bfeta$ in the error function.}
\vspace{-1.5em}
\label{fig:method}
\end{figure}

{\bf \noindent Single-Instance SDDF Training: }
Given distance measurements $\mathcal{D}_l$, as in Problem~\ref{problem}, from a single object instance $l$, we can learn an SDDF representation $h(\bfp,\bfeta)$ in \eqref{eq:main} of the instance shape by optimizing the parameters of a neural network model $q_{\bftheta}(\bfp,\bfeta)$ with structure described in Sec.~\ref{sec:exp}.

We split the training data $\mathcal{D}_l$ into two sets, distinguishing whether the distance measurements are finite or infinite:
\begin{equation}
\begin{aligned}
\calF_l &:= \crl{ (\bfp,\bfeta,d) \in \mathcal{D}_l \mid d < \infty}, \\
\calI_l &:= \crl{(\bfp,\bfeta,d) \in \mathcal{D}_l \mid d = \infty},
\end{aligned}
\end{equation}
and define an error function for training the parameters $\bftheta$:
\begin{align}
e&(\bftheta; \calF, \calI) := \frac{\alpha}{|\calF|}\sum_{(\bfp,\bfeta,d) \in \calF} \!\!|\phi(d+\bfp^\top\bfeta) - q_{\bftheta}(\bfp,\bfeta)|^p \notag\\
&+\frac{\beta}{|\calI|} \sum_{(\bfp,\bfeta,d) \in \calF} \!\!r\prl{\phi(\infty) - q_{\bftheta}(\bfp,\bfeta)}^p + \gamma \|\bftheta\|_p^p, \label{eq:error_function}
\end{align}
where $\alpha,\beta,\gamma > 0$ are weights, $p \geq 1$, and $r$ is a rectifier, such as ReLU $r(x) = \max\crl{0,x}$, GELU $r(x) = x\Phi(x)$, or softplus $r(x) = \log(1+\exp(x))$. In the experiments, we use $p = 1$ and $r(x)= \max\crl{0,x}$. The last term in \eqref{eq:error_function} regularizes the network parameters $\bftheta$, but in all experiments we set the $\gamma$ to zero. The first term encourages $q_{\bftheta}(\bfp,\bfeta)$ to predict the squashed distance values accurately. We introduced a rectifier $r$ in the second term in \eqref{eq:error_function} to allow the output of $q_{\bftheta}(\bfp,\bfeta)$ to exceed $\phi(\infty)$, which we observed empirically leads to faster convergence. To address that $q_{\bftheta}(\bfp,\bfeta)$ may exceed $\phi(\infty)$, we modify its conversion to an SDDF as:
\begin{equation}
h(\bfp, \bfeta) = \phi^{-1}(\min\crl{q_{\bftheta}(\bfp, \bfeta),\phi(\infty)}) - \bfp^\top \bfeta.
\end{equation}

{\bf \noindent Multi-Instance SDDF Training: }
Next, we consider learning an SDDF shape model for multiple instances $l$ from the same category with common parameters $\bftheta$. Inspired by DeepSDF~\cite{park2019deepsdf}, we introduce a latent code $\bfz_l \in \mathbb{R}^m$ to model the shape of each instance $l$ and learn it as part of the neural network parameters with structure $q_{\bftheta}(\bfp,\bfeta,\bfz_l)$ described in Sec.~\ref{sec:exp}. Given distance measurements $\calF_l$ and $\calI_l$, we optimize $\bfz_l$ independently, for each instance $l$, and $\bftheta$ jointly, across all instances using the same error as in \eqref{eq:error_function}:
\begin{equation*}
\scaleMathLine{\begin{aligned}
&\min_{\bftheta,\{\bfz_l\}_l} \frac{\alpha}{\sum_l |\calF_l|} \sum_l\!\!\sum_{(\bfp,\bfeta,d)\in \calF_l}\!\!|\phi(d + \bfp^\top\bfeta) - q_{\bftheta}(\bfp, \bfeta, \bfz_l)|^p\\ 
&\;\;+\frac{1}{\sum_l |\calI_l|} \sum_l\!\!\sum_{(\bfp,\bfeta,d)\in \calI_l} \!\!\beta  r(\phi(\infty)-q_{\bftheta}(\bfp,\bfeta,\bfz_l))^p + \sigma \|\bfz_l\|_p^p + \gamma \|\bftheta\|_p^p.
\end{aligned}}
\end{equation*}

{\bf \noindent Online Shape Optimization: }
Finally, we consider a shape completion task, where we predict the SDDF shape of a previously unseen instance from partial distance measurements $\calF$, $\calI$. In this case, we assume that the category-level neural network parameters $\bftheta$ are already trained offline and we have an average category-level shape encoding $\bar{\bfz} \in \mathbb{R}^m$ (e.g., can be obtained by using a fixed $\bfz$ for all instances $l$ during training or simply as the mean of $\crl{\bfz_l}_l$). We initialize the shape code for the new instance with $\bar{\bfz}$ and optimize it using $\calF$ and $\calI$ and the same error function as before:
\begin{equation}
\label{eq:codefinder}
\begin{aligned}
&\min_{\bfz} \frac{\alpha}{|\calF|} \sum_{(\bfp,\bfeta,d)\in \calF}\!\!|\phi(d + \bfp^\top\bfeta) - q_{\bftheta}(\bfp, \bfeta, \bfz)|^p\\ 
&\;\;+\frac{\beta}{|\calI|} \sum_{(\bfp,\bfeta,d)\in \calI} \!\!r(\phi(\infty)-q_{\bftheta}(\bfp,\bfeta,\bfz))^p + \sigma \|\bfz\|_p^p.
\end{aligned}
\end{equation}
The optimized latent shape code $\bfz^*$ captures all geometric information about the object and can be used to synthesize novel distance views $h(\bfp, \bfeta) = \phi^{-1}(\min\crl{q_{\bftheta}(\bfp, \bfeta,\bfz^*),\phi(\infty)}) - \bfp^\top \bfeta$ from any point $\bfp$ in any viewing direction $\bfeta$.

\subsection{Multi-view Consistency}
\label{sec:multiview}




Proposition~\ref{prop:SDDF-structure} reduces the input dimension of an SDDF function $h(\bfp,\bfeta)$ from $2n-1$ to $2n-2$. To model 3D shape, we need to represent a 4D SDDF function. In contrast, an SDF model \cite{park2019deepsdf} has a 3D input, which may even be reduced to a 2D surface using an Eikonal constraint \cite{gropp2020implicit}. Hence, training a multi-view consistent SDDF model might require a larger data set with distance measurements from many positions $\bfp$ and directions $\bfeta$. To reduce the necessary data, we develop an approach to synthesize additional data from the initial training set $\mathcal{D}_l := \crl{(\bfp_{i,l}, \bfeta_{i,l}, d_{i,l})}_i$. Given an arbitrary position $\hat{\bfp} \in \mathbb{R}^n$, we describe to how to synthesize both finite and infinite (no surface hit) distance measurements $\hat{d}$ along different view rays $\hat{\bfeta}$ originating at $\hat{\bfp}$. Let $\mathcal{P}_l := \crl{ \bfp + d \bfeta \mid (\bfp,\bfeta,d) \in \mathcal{D}_l, d < \infty}$ be a point cloud representation of the training data.

\begin{figure}[t]
\includegraphics[width=0.32\linewidth]{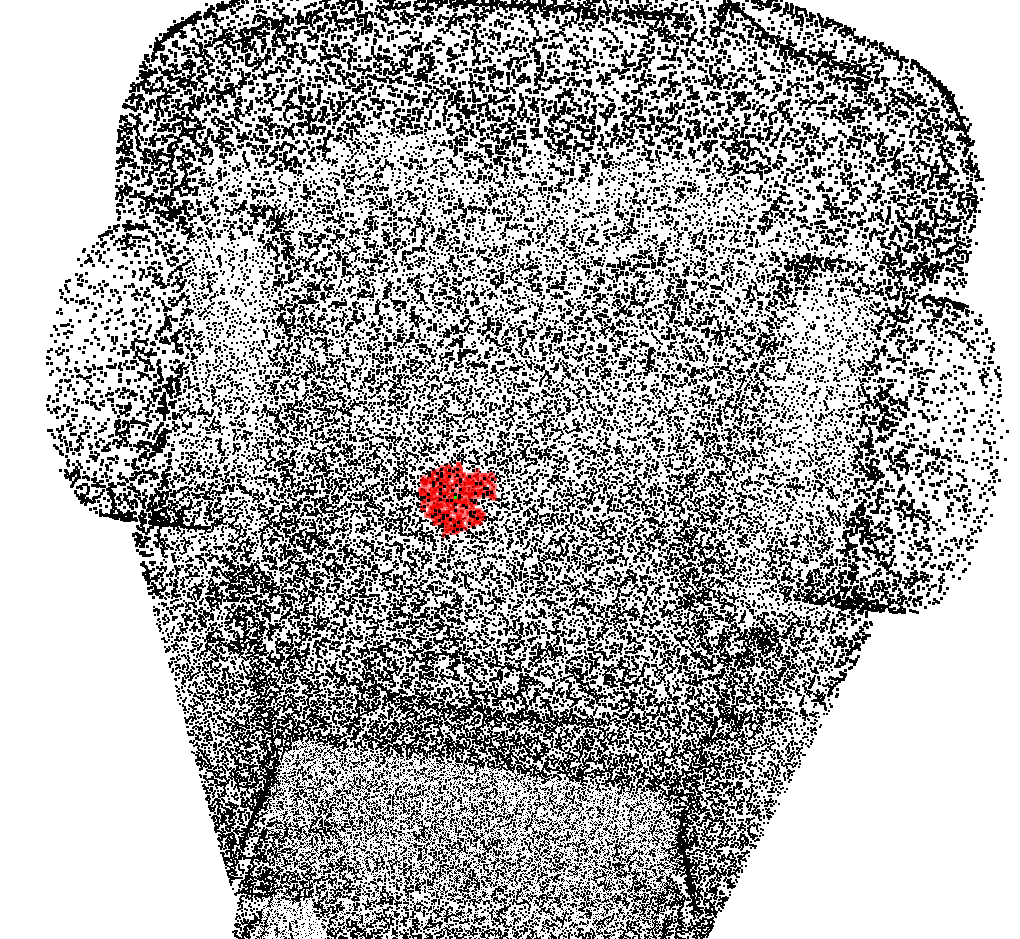}%
\hfill%
\includegraphics[width=0.32\linewidth]{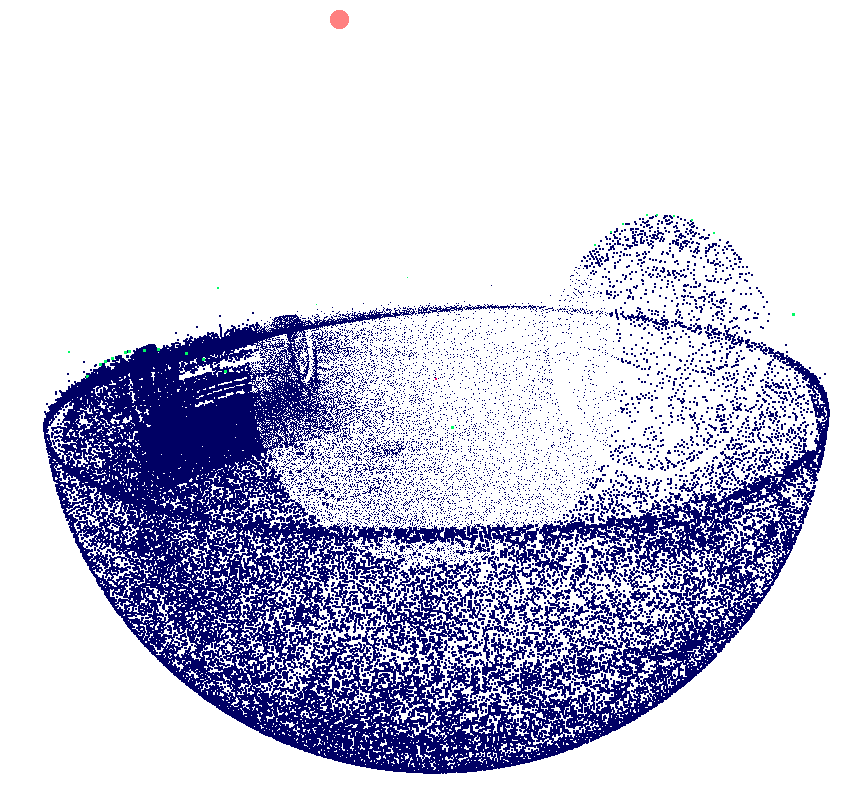}%
\hfill%
\includegraphics[width=0.35\linewidth]{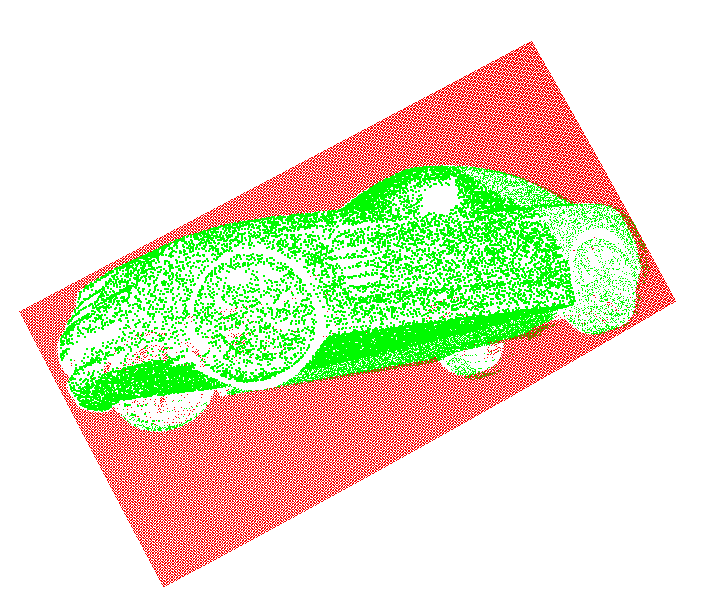}
\vspace{-1em}
\caption{A new distance view (right) is synthesized from a point cloud $\mathcal{P}_l$ (left) by deciding whether each point $\bfq$ (left, red) is visible from the new view $\hat{\bfp}$. The point cloud is projected on a sphere around $\bfq$ (middle) to judge the visibility from $\hat{\bfp}$ (middle, red).}
\label{fig:multiview}
 \vspace{-1em}
\end{figure}

{\bf \noindent Infinite Ray Synthesis: }
To synthesize infinite rays, we project the point cloud $\mathcal{P}_l$ to the desired image frame and select the directions $\hat{\bfeta}$ of all pixels that do not contain a projected point. These directions correspond to views with infinite distance. The points may be inflated with a finite radius to handle sparse point cloud data.

{\bf \noindent Finite Ray Synthesis: }
If a point $\bfq \in \mathcal{P}_l$ is observable from $\hat{\bfp}$, we can obtain a synthetic measurement with distance $\hat{d} = \|\bfq - \hat{\bfp}\|_2$ in direction $\hat{\bfeta} = \frac{1}{\hat{d}} \prl{\bfq - \hat{\bfp}}$. The challenge is to decide which points in $\mathcal{P}_l$ are visible from $\hat{\bfp}$.
For $\bfq \in \mathcal{P}_l$, let $\bfp$ be the start point of the ray that observed $\bfq$ originally. We know that $\bfq$ is observable from $\bfp$. In contrast, for all $\bfu \in \mathcal{P}_l\setminus \{\bfq\}$ and all $\epsilon > 0$, $\bfq$ is not observable from $\bfu - \epsilon (\bfq - \bfu)$, since $\bfu$ is in the way. Hence, $\bfq$ is observable when we look at it in the direction $\frac{\bfq - \bfp}{\|\bfq - \bfp\|_2}$ and unobservable in the direction $\frac{\bfq - \bfu}{\|\bfq - \bfu\|_2}$ for all $\bfu \in \mathcal{P}_l\setminus \{\bfq\}$. For convenience of representation, translate all points such that $\bfq$ is at the origin, project all points on a unit sphere around the origin, and rotate all of the points such that $\bfp$ maps to $\bfe_3 = [0,0,1]^\top$. Formally, this can be achieved with the transformation:
\begin{equation}
	T_{\bfq}(\bfx) := 
	\begin{cases}
		[0, 0, 0]^\top,& \text{if } \bfx = \bfq,\\
		\bfR_{\bfeta}\frac{\bfx - \bfq}{\|\bfx - \bfq\|_2}, & \text{otherwise},
	\end{cases}
\end{equation}
where $\bfeta = \frac{\bf p - \bf q}{\|\bf p - \bf q\|_2}$. To decide whether $\bfq$ is observable from $\hat{\bfp}$, equivalently we should decide whether the origin is observable from $T_{\bfq}(\hat{\bfp})$. Let $\mathcal{P}_{\bfq} := T_{\bfq}(\mathcal{P}_l \setminus \{\bfq\})$. The origin is observable from $\bfe_3$ and a region around it and unobservable from all $\bfu \in \mathcal{P}_{\bfq}$. See Fig.~\ref{fig:multiview} for an illustration.


%

\begin{figure*}[t]
\includegraphics[width=0.33\linewidth, trim=20mm 0mm 10mm 0mm, clip]{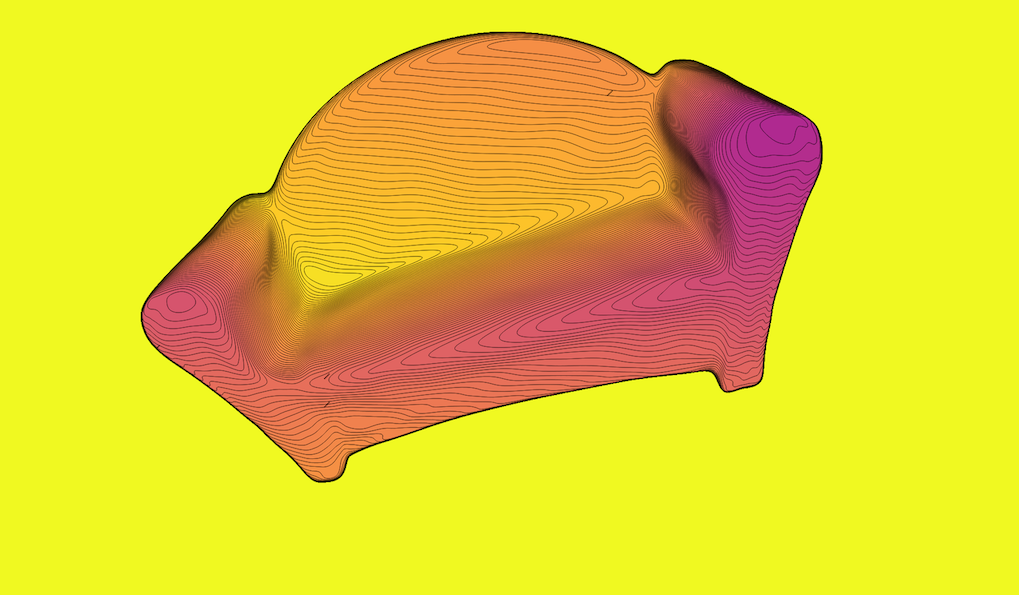}%
\hfill%
\includegraphics[width=0.33\linewidth, trim=20mm 0mm 10mm 0mm, clip]{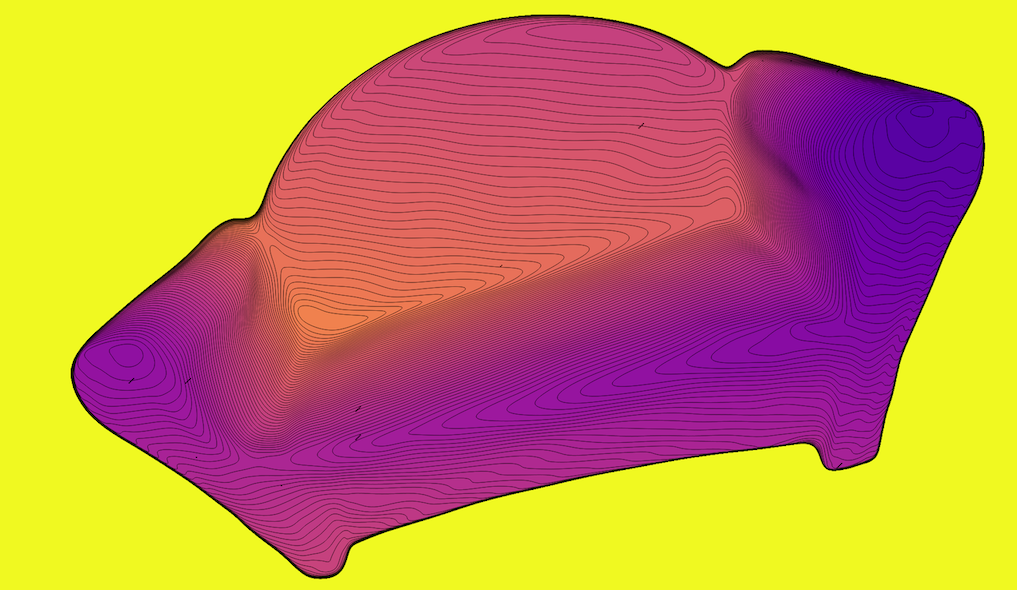}%
\hfill%
\includegraphics[width=0.33\linewidth, trim=0mm 50mm 0mm 60mm, clip]{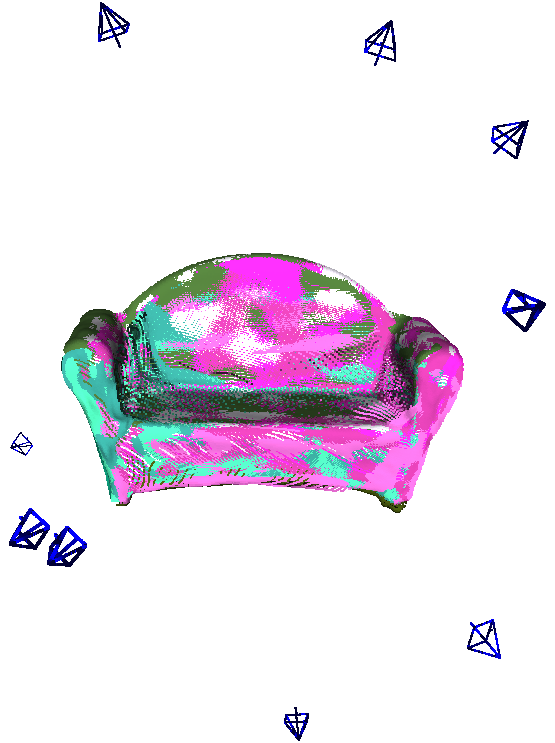}
%
\vspace{-1em}
\caption{SDDF shape representation of a sofa. Distance images synthesized by SDDF along the same view direction but different distance from the object are shown (left, middle). Point clouds synthesized from arbitrary camera views are multi-view consistent (right).}
\label{fig:single}
\vspace{-1em}
\end{figure*}

\begin{figure*}
\includegraphics[width=0.24\linewidth, trim=0mm 0mm 0mm 25mm, clip]{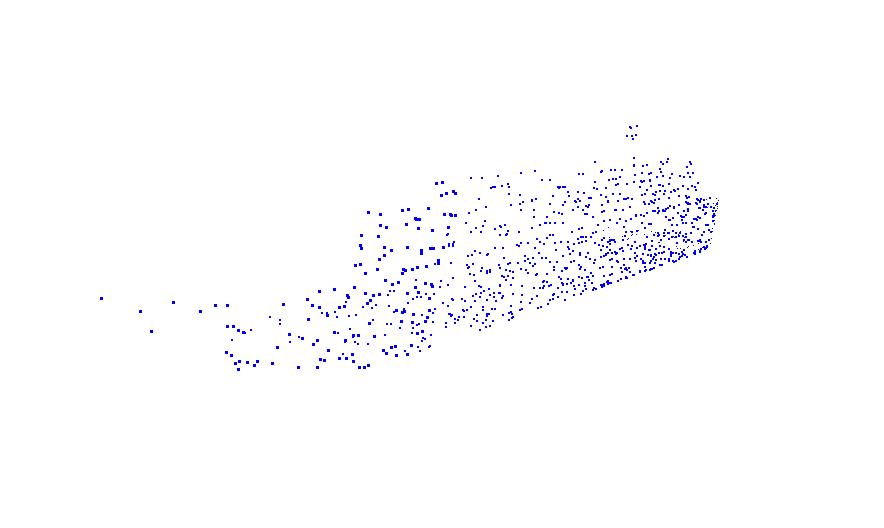}%
\hfill%
\includegraphics[width=0.24\linewidth, trim=10mm 20mm 40mm 40mm, clip]{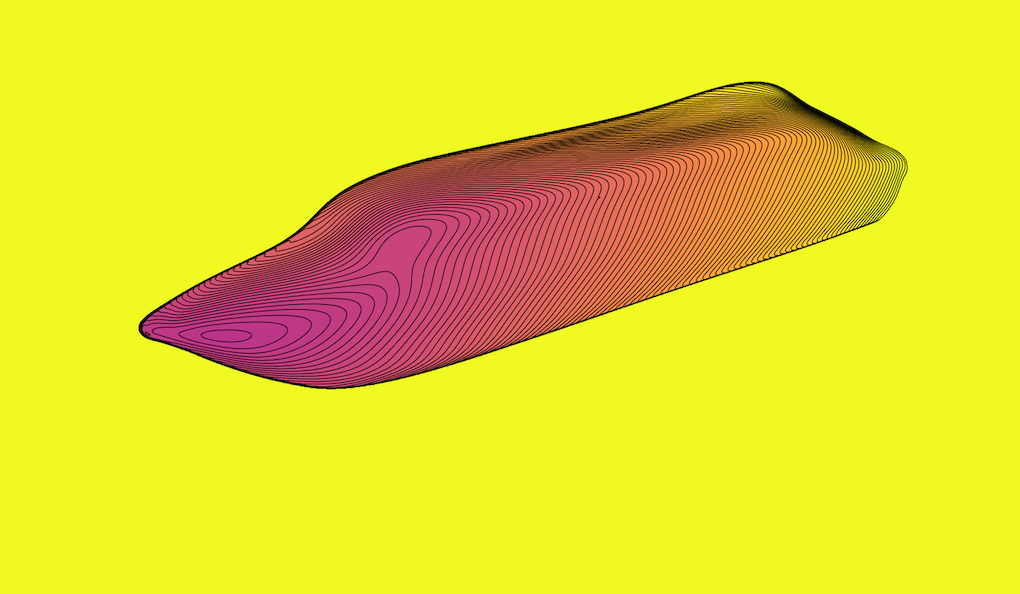}%
\hfill%
\includegraphics[width=0.24\linewidth, trim=10mm 20mm 40mm 40mm, clip]{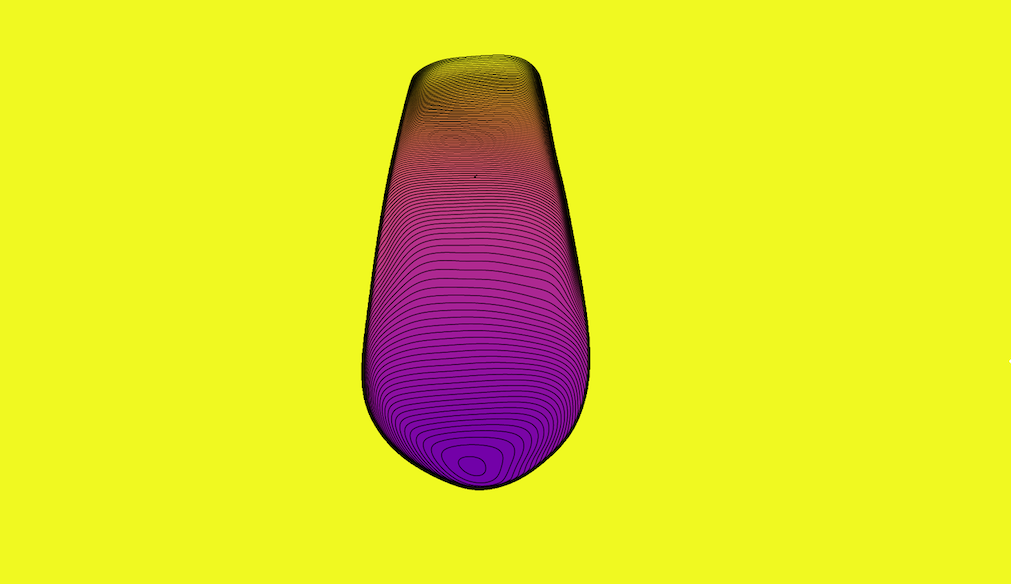}%
\hfill%
\includegraphics[width=0.24\linewidth, trim=0mm 0mm 0mm 80mm, clip]{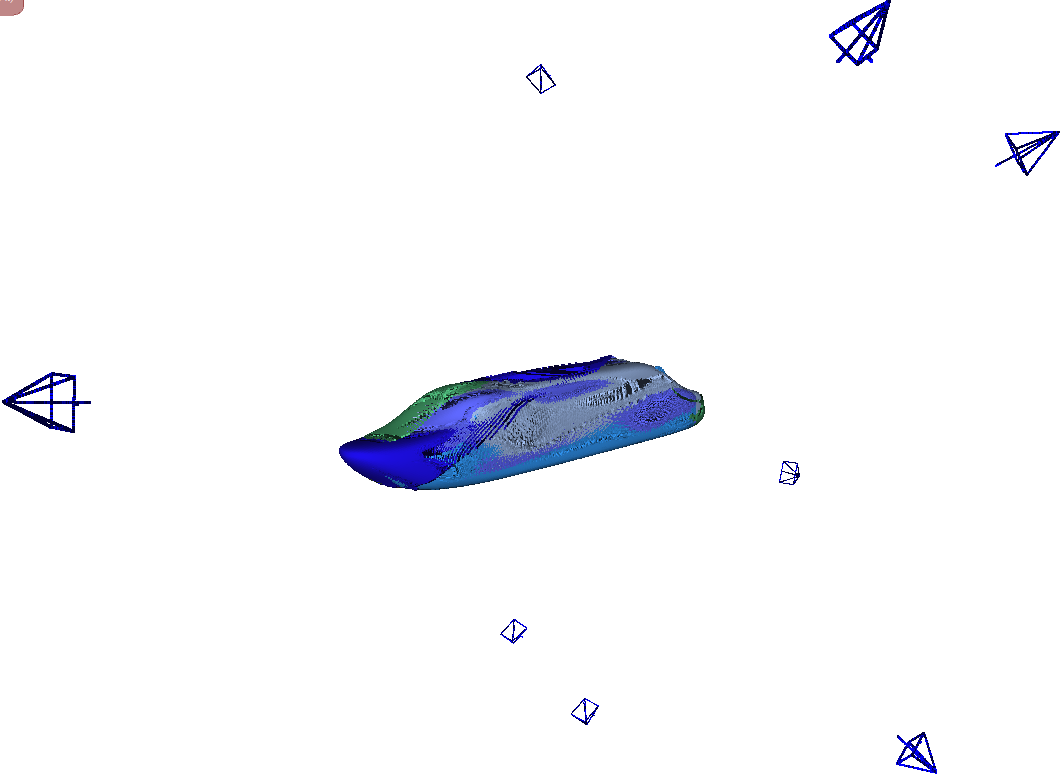}\\[-1.5ex]
\includegraphics[width=0.24\linewidth, trim=0mm 30mm 0mm 0mm, clip]{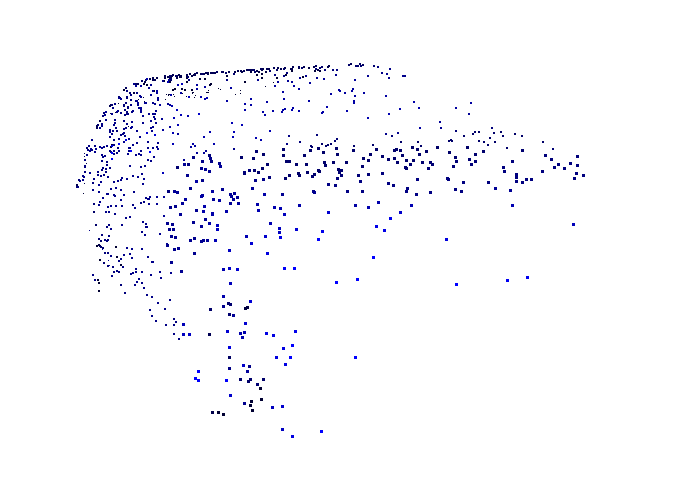}%
\hfill%
\includegraphics[width=0.24\linewidth, trim=10mm 0mm 40mm 10mm, clip]{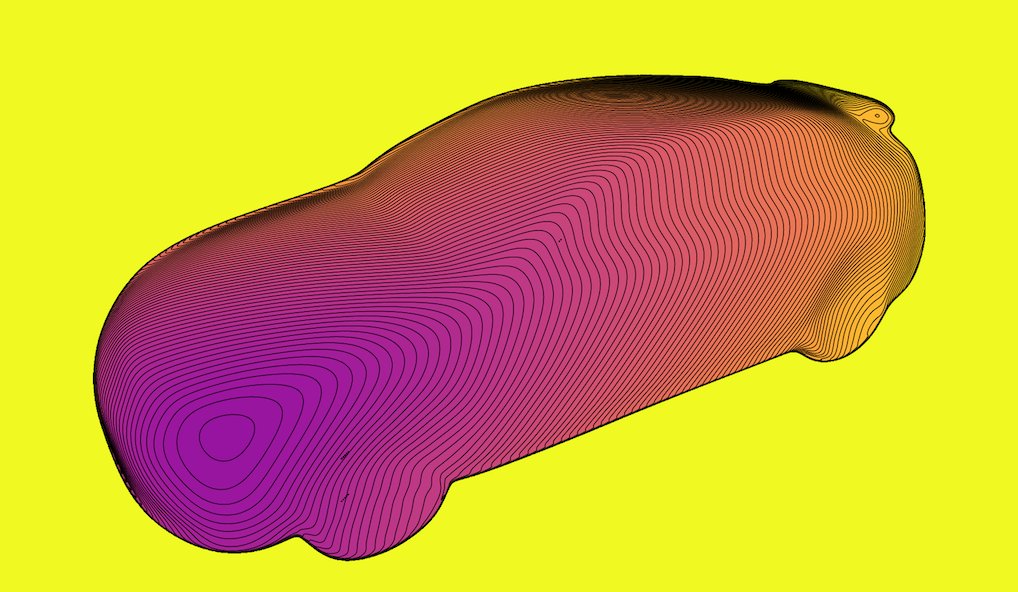}%
\hfill%
\includegraphics[width=0.24\linewidth, trim=10mm 10mm 40mm 0mm, clip]{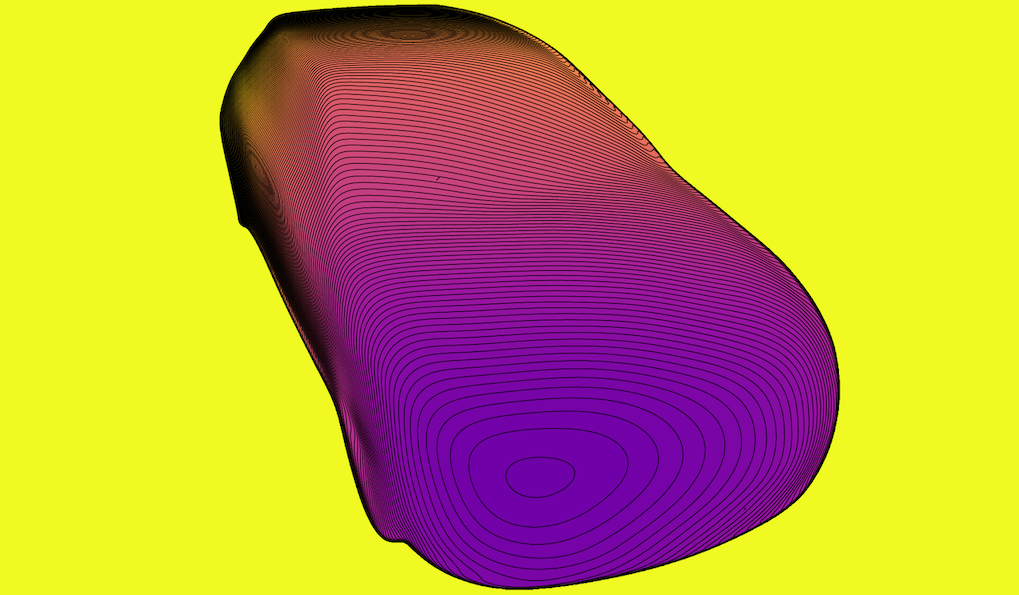}%
\hfill%
\includegraphics[width=0.24\linewidth, trim=10mm 50mm 60mm 30mm, clip]{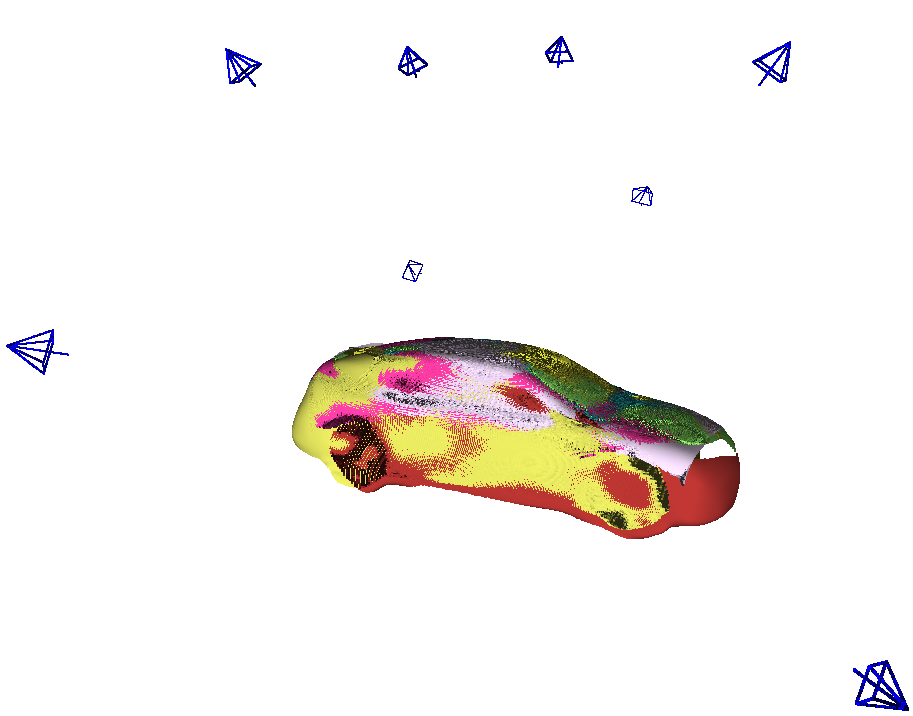}
\vspace{-1em}
\caption{SDDF shape completion using distance measurements (first column) from unseen boat (first row) and car (second row) instances. The trained SDDF model can synthesize novel distance views (second and third columns) or point clouds from arbitrary camera views which are multi-view consistent (fourth column).}
\label{fig:completion}
\vspace{-1em}
\end{figure*}

{\bf \noindent Spherical Convex Hull: } 
We approximate the region of points around $\bfe_3$ that can observe the origin. 
The vertices adjacent to $\bfe_3$ in the convex hull of $\mathcal{P}_{\bfq} \cup \{\bfe_3\}$ represent the boundary. We sort the boundary points based on their azimuth so that the geodesic among them represents the boundary. The origin is observable from the part of sphere that contains the $\bfe_3$. We provide a lemma that allows the convex hull computation to be performed in 2D.



%
%
\begin{lemma}\label{lem:cnv3D2D}
Let $\mathcal{P}$ be a set of points on the unit sphere. The boundary points of $\mathcal{P}$ with respect to $\bfe_3$ are points that are adjacent vertices to $\bfe_3$ in the convex hull of $\mathcal{P} \cup \{\bfe_3\}$. Let $m$ be a function that maps a point on the unit sphere to the plane $z=0$ with center $\bfe_3$, i.e., $m([x,y,z]^\top) := [\frac{x}{1-z}, \frac{y}{1-z}]^\top$. A point $\bfu \in \mathcal{P}$ is a boundary point if and only if $m(\bfu)$ is a vertex of the convex hull of $m(\mathcal{P})$.
\end{lemma}
\begin{proof}
See the Supplementary Material.
\end{proof}
%
%
%
%
To accelerate the convex hull computation further, we propose an approximation using discretization. We discretize the azimuth of the sphere into $N$ segments. Let $E_i = \max\{el({\bf u}) \mid {\bf u}\in {\cal{P}}_{\bf{q}}, az({\bf u}) \in [\frac{2\pi i}{N}, \frac{2\pi(i+1)}{N})\}$, $0\leq i<N$ be the maximum elevation of points in $\mathcal{P}_{\bfq}$ with azimuth in $[\frac{2\pi i}{N}, \frac{2\pi(i+1)}{N})$. Then, let $i$ determine the interval $[\frac{2\pi i}{N}, \frac{2\pi(i+1)}{N})$ that contains the azimuth of $T_{\bfq}(\hat{\bfp})$. We consider $\bfq$ observable from $\hat{\bfp}$ if the elevation of $T_{\bfq}(\hat{\bfp})$ is larger than $E_i$. In the experiments, we accelerate the computation further by sub-sampling $\mathcal{P}_{\bfq}$.

\section{Evaluation}
\label{sec:exp}

This section presents qualitative and quantitative evaluation of the SDDF model. Sec.~\ref{sec:instance_evaluation} presents results for single-instance shape modeling in comparison to the deep geometric prior (DGP) model~\cite{williams2019deep}. Sec.~\ref{sec:category_evaluation} applies the SDDF model to a class of shapes, demonstrating shape completion from a single distance view and shape interpolation between different instances. The accuracy of SDDF for shape completion is compared against the decoder-only deep SDF model, IGR~\cite{gropp2020implicit}. Both our model and IGR capture structural constraints for SDDF and SDF, respectively, making a quantitative comparison interesting. We also compare the results against a group of category-level shape modeling methods \cite{xie2020grnet,groueix2018papier,yuan2018pcn,yang2018foldingnet,tchapmi2019topnet, liu2020morphing} that utilize point-cloud data. Finally, we present SDDF shape interpolation results to demonstrate that our model captures the latent space of an object category shape continuously and meaningfully.

\begin{table}[t]
\centering
\caption{Quantitative comparison between DGP~\cite{williams2019deep} and SDDF with exact (SDDF (conv)) and approximate (SDDF (disc)) convex hull data augmentation on $5$ object instances from ShapeNet~\cite{chang2015shapenet}. At test time, the metrics from \cite{mescheder2019occupancy} are computed after the reconstructed meshes are normalized to a unit-length bounding box.}
\label{tab:single}
\vspace{-1em}
\resizebox{\linewidth}{!}{\begin{tabular}{l|c|c|c|c|c}\toprule
Class & Method & Chamfer-$L_2$ & Chamfer-$L_1$ & Completeness & Accuracy\\\midrule
\multirow{4}{0em}{Truck}& SDDF (conv) & 3.918e-05 & 3.475e-03 & 3.077e-03 & 3.874e-03 \\
&  SDDF (disc)      & 4.421e-05 & 3.681e-03 & 3.251e-03 & 4.111e-03 \\
&  DGP      & 8.842e-05 & 7.388e-03 & 6.842e-03 & 7.935e-03 \\\midrule
\multirow{4}{4em}{Airplane}& SDDF (conv) & 3.088e-05 & 2.135e-03 & 1.713e-03 & 2.557e-03 \\
&  SDDF (disc)      & 2.332e-05 & 2.741e-03 & 2.077e-03 & 3.406e-03 \\
&  DGP      & 4.900e-05 & 5.141e-03 & 4.512e-03 & 5.770e-03 \\\midrule
\multirow{4}{0em}{Sofa}& SDDF (conv) & 1.562e-05 & 2.428e-03 & 1.725e-03 & 3.130e-03 \\
&  SDDF (disc)      & 2.563e-05 & 2.610e-03 & 1.767e-03 & 3.454e-03 \\
&  DGP      & 18.822e-05 & 10.747e-03 & 9.546e-03 & 11.949e-03 \\\midrule
\multirow{4}{0em}{Boat}& SDDF (conv) & 0.425e-05 & 1.725e-03 & 1.446e-03 & 2.005e-03 \\
&  SDDF (disc)      & 0.489e-05 & 1.778e-03 & 1.481e-03 & 2.075e-03 \\
&  DGP      & 2.399e-05 & 4.015e-03 & 3.848e-03 & 4.182e-03 \\\midrule
\multirow{4}{0em}{Car}& SDDF (conv) & 2.889e-05 & 2.892e-03 & 2.858e-03 & 2.925e-03 \\
&  SDDF (disc)      & 2.765e-05 & 2.921e-03 & 2.824e-03 & 3.019e-03 \\
&  DGP      & 5.988e-05 & 5.927e-03 & 5.809e-03 & 6.046e-03 \\\bottomrule
\end{tabular}}
\end{table}

{\bf \noindent Network Architecture: }
We use an autodecoder with $16$ layers, $512$ hidden units per layer, and a skip connection from the input to layers $4,8,12$ to represent the SDDF model $q_{\bftheta}(\bfp,\bfeta,\bfz)$ introduced in Sec.~\ref{sec:training_and_inference}. The dimension of the latent shape code $\bfz$ is set to $256$ for category-level shape completion and interpolation and to $0$ for single instance shape modeling.


{\bf \noindent Data Preparation: }
We use the ShapeNet dataset \cite{chang2015shapenet}. Distance images with resolution $512\times512$ are generated as training data from $8$ camera views facing the object from azimuth $\frac{k\pi}{4}$ and elevation $\frac{(-1)^k\pi}{4}$ for $k = 0,\ldots,7$ on a sphere. Each distance image is subsampled to contain at most $100k$ finite and $100k$ infinite distance measurements. Both DGP and IGR were trained using the point cloud obtained from all points with finite distance measurements and augmented with normals obtained using the method of \cite{Zhou2018}. The results of the remaining baseline methods \cite{xie2020grnet,groueix2018papier,yuan2018pcn,yang2018foldingnet,tchapmi2019topnet, liu2020morphing} were obtained from the GRNet paper \cite{xie2020grnet}. 

\begin{table}[t]
\centering
\caption{Comparison between SDDF and IGR~\cite{gropp2020implicit}, over $5$ classes from ShapeNet \cite{chang2015shapenet}, using the metrics from \cite{mescheder2019occupancy}. The metrics are computed after the reconstructed meshes are normalized in a unit-length bounding box. Two versions of IGR are evaluated: IGR(1), using the same test points as SDDF, and IGR(2), producing a uniform point cloud from the reconstructed mesh.}
\label{tab:IGR}
\vspace{-1em}
\resizebox{\linewidth}{!}{\begin{tabular}{l|c|c|c|c|c}\toprule
Class & Method & Chamfer-$L_2$ & Chamfer-$L_1$ & Completeness & Accuracy\\\midrule
\multirow{4}{0em}{Car}& SDDF & 2.688e-04 & 8.365e-03 & 7.833e-03 & 8.897e-03 \\
&  IGR(1)      & 37.873e-04 & 39.682e-03 & 15.971e-03 & 63.394e-03 \\
&  IGR(2)      & 27.165e-04 & 30.465e-03 & 11.291e-03 & 49.639e-03 \\\midrule
\multirow{4}{3em}{Airplane}& SDDF & 3.539e-04 & 7.735e-03 & 6.790e-03 & 8.680e-03 \\
&  IGR(1)      & 182.028e-04 & 89.677e-03 & 21.411e-03 & 157.942e-03 \\
&  IGR(2)      & 139.607e-04 & 70.804e-03 & 10.752e-03 & 130.856e-03 \\\midrule
\multirow{4}{4em}{Watercraft}& SDDF & 7.869e-04 & 12.762e-03 & 10.772e-03 & 14.752e-03 \\
&  IGR(1)      & 72.551e-04 & 55.594e-03 & 25.620e-03 & 85.567e-03 \\
&  IGR(2)      & 69.437e-04 & 52.774e-03 & 22.409e-03 & 83.139e-03 \\\midrule
\multirow{4}{0em}{Sofa}& SDDF & 3.952e-04 & 12.245e-03 & 10.940e-03 & 13.551e-03 \\
&  IGR(1)      & 119.493e-04 & 71.506e-03 & 40.964e-03 & 102.048e-03 \\
&  IGR(2)      & 114.114e-04 & 67.189e-03 & 33.903e-03 & 100.475e-03 \\\midrule
\multirow{4}{0em}{Display}& SDDF & 9.318e-04 & 17.089e-03 & 14.051e-03 & 20.127e-03 \\
&  IGR(1)      & 93.326e-04 & 53.326e-03 & 35.082e-03 & 71.571e-03 \\
&  IGR(2)      & 86.551e-04 & 53.734e-03 & 36.470e-03 & 70.998e-03 \\\bottomrule
\end{tabular}}
\end{table}

\begin{table}[t]
\caption{Quantitative comparison between SDDF and several baseline methods reported in GRNet~\cite{xie2020grnet}, over $4$ classes from ShapeNet \cite{chang2015shapenet}. The errors are scaled by $\times 10^{-3}$.}
\label{tab:grnet}
\vspace{-1em}
\resizebox{\linewidth}{!}{\begin{tabular}{l|c|c|c|c|c}\toprule
Class & Metric & Car & Airplane & Watercraft & Sofa\\\midrule
\multirow{2}{0em}{AtlasNet}& Chamfer-$L_2$ & 0.3237 & 0.1753 & 0.4177 & 0.5990 \\
&  Chamfer-$L_1$      & 10.105 & 6.366 & 10.607 & 12.990 \\\midrule
\multirow{2}{3em}{PCN}& Chamfer-$L_2$ & 0.2445 & 0.1400 & 0.4062 & 0.5129 \\
&  Chamfer-$L_1$      & 8.696 & 5.502 & 9.665 & 11.676 \\\midrule
\multirow{2}{4em}{FoldingNet}& Chamfer-$L_2$ & 0.4676 & 0.3151 & 0.7325 & 0.8895 \\
&  Chamfer-$L_1$      & 12.611 & 9.491 & 14.987 & 15.969 \\\midrule
\multirow{2}{4em}{TopNet}& Chamfer-$L_2$ & 0.3513 & 0.2152 & 0.4359 & 0.6949 \\
&  Chamfer-$L_1$      & 10.898 & 7.614 & 11.124 & 14.779 \\\midrule
\multirow{2}{4em}{MSN}& Chamfer-$L_2$ & 0.4711 & 0.1543 & 0.3853 & 0.5894 \\
&  Chamfer-$L_1$      & 10.776 & 5.596 & 9.485 & 11.895 \\\midrule
\multirow{2}{4em}{GRNet}& Chamfer-$L_2$ & 0.2752 & 0.1531 & 0.2122 & 0.3613 \\
&  Chamfer-$L_1$      & 9.447 & 6.450 & 8.039 & 10.512 \\\midrule
\multirow{2}{0em}{SDDF}& Chamfer-$L_2$ & 0.17351 & 0.26172 & 0.39942 & 0.23767 \\
&  Chamfer-$L_1$      & 8.31331 & 5.90701 & 9.80779 & 9.8593 \\\bottomrule
\end{tabular}}
 \vspace{-1em}
\end{table}
\subsection{Geometric Model Evaluation}
\label{sec:instance_evaluation}

We first evaluate the performance of our SDDF model and DGP~\cite{williams2019deep} for single-instance shape reconstruction using $5$ objects from ShapeNet~\cite{chang2015shapenet}. The models are trained using point cloud data from $8$ views described above. The data is augmented with normals for DGP and the DGP radius is set to guarantee at least $128$ patches for each object. SDDF is trained with learning rate $0.005$ that decreases every $1k$ epoch by factor of $0.5$ for $10k$ epochs using additional synthesized data with exact (SDDF (conv)) and discretized (SDDF (disc)) convex hull computation (see Sec.~\ref{sec:multiview}). The results are presented in Table~\ref{tab:single}.

Given the trained SDDF model, we can generate a distance image at an arbitrary view by calculating an SDDF prediction for the rays corresponding to each pixel in the desired image. To visualize the learned model, we choose camera locations facing the object and show the predictions in Fig.~\ref{fig:single}. All generated distance images recognize the object shape and free space precisely. By inspecting the distance level sets, we can also conclude that the SDDF model successfully captures the shape details. The images get brighter as the view moves further away from the object because the measured distances at each pixel increase. On the other hand, the level sets in the distant view remain parallel to the close-up view. The two-view comparison shows that the directional condition of the SDDF model in \eqref{eq:cond} holds.

\begin{figure*}[t]
%
\includegraphics[width=0.198\linewidth, trim={100mm 80mm 40mm 85mm}, clip]{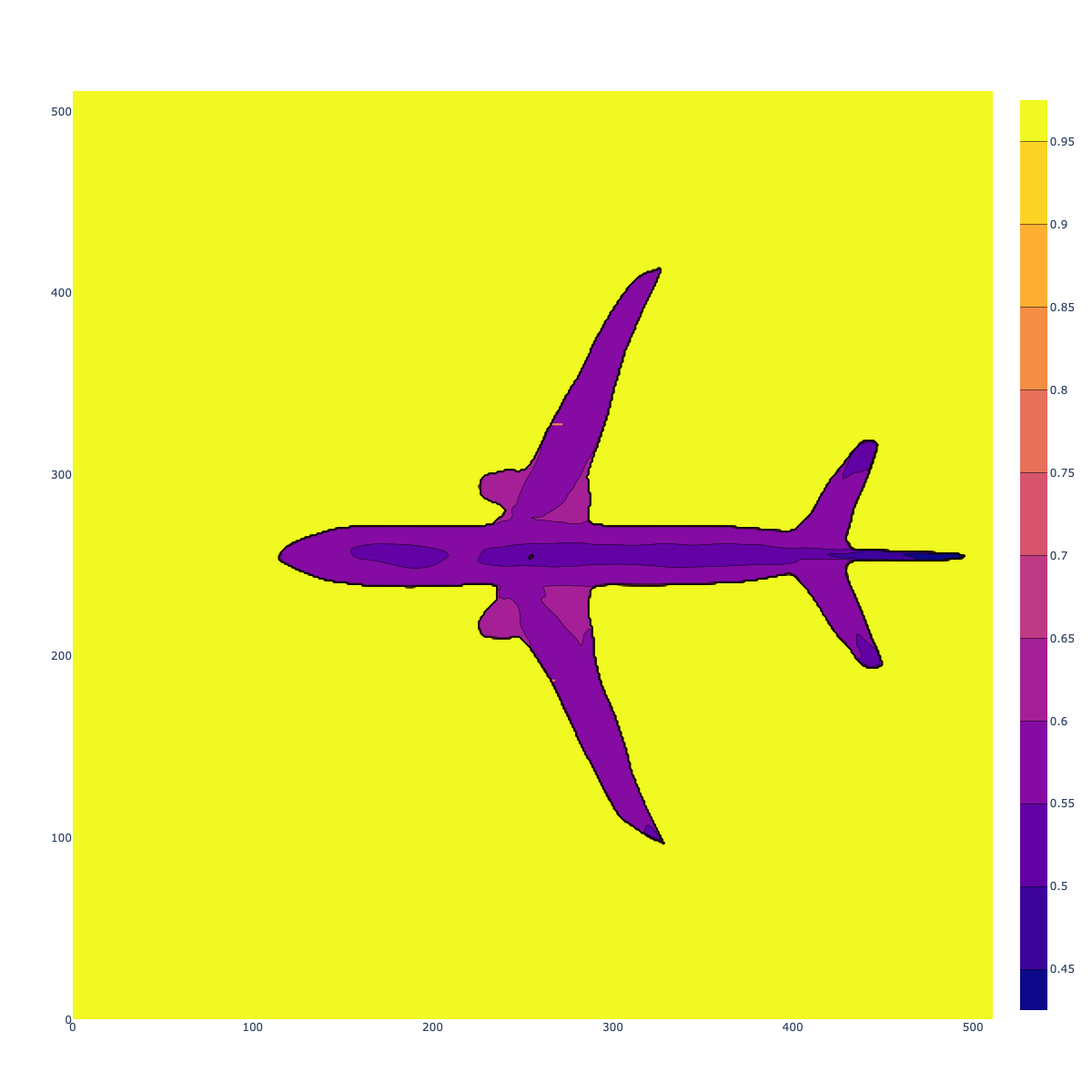}%
\hfill%
\includegraphics[width=0.198\linewidth, trim=100mm 80mm 40mm 85mm, clip]{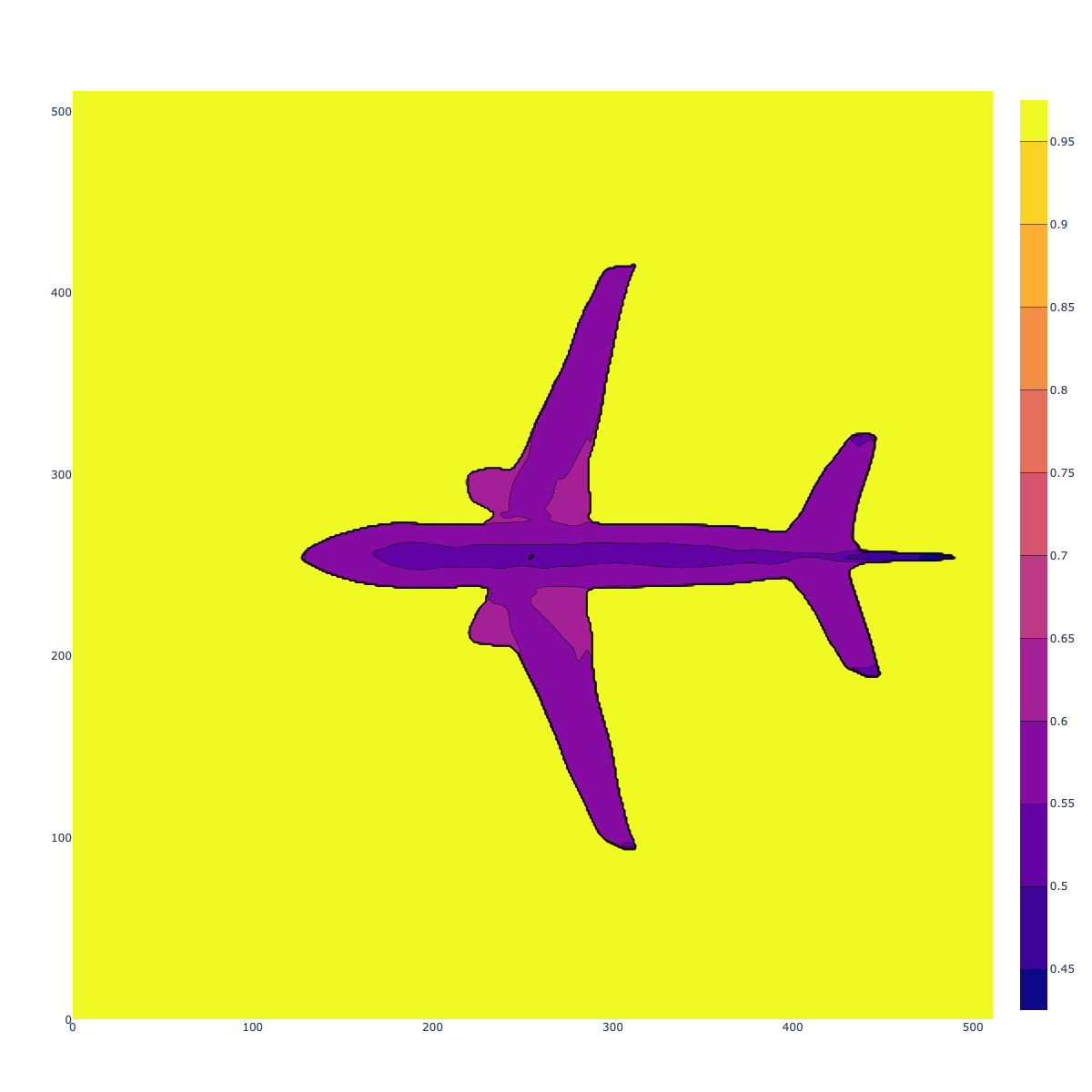}%
\hfill%
\includegraphics[width=0.198\linewidth, trim=100mm 80mm 40mm 85mm, clip]{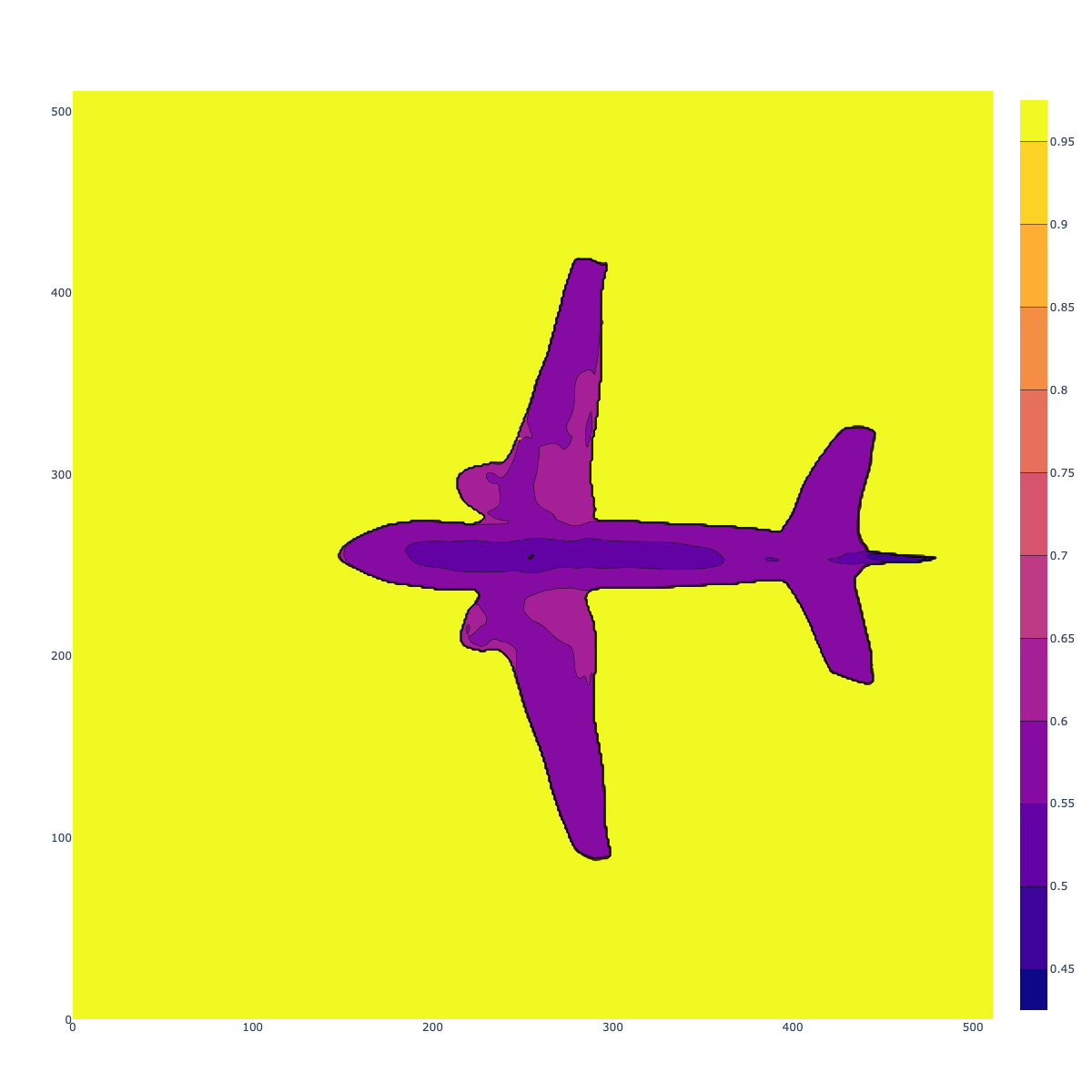}%
\hfill%
\includegraphics[width=0.198\linewidth, trim=100mm 80mm 40mm 85mm, clip]{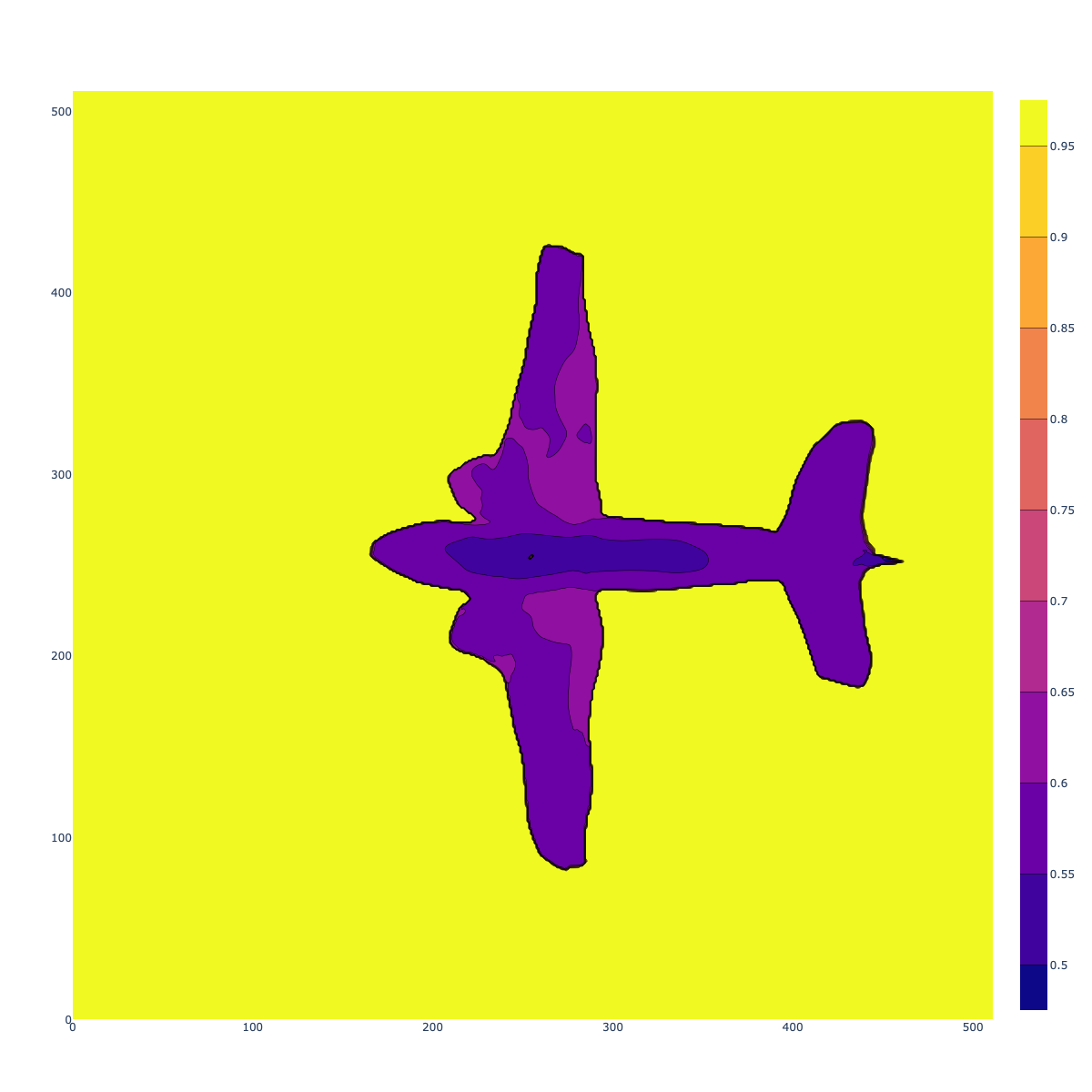}%
\hfill%
\includegraphics[width=0.198\linewidth, trim={100mm 80mm 40mm 85mm}, clip]{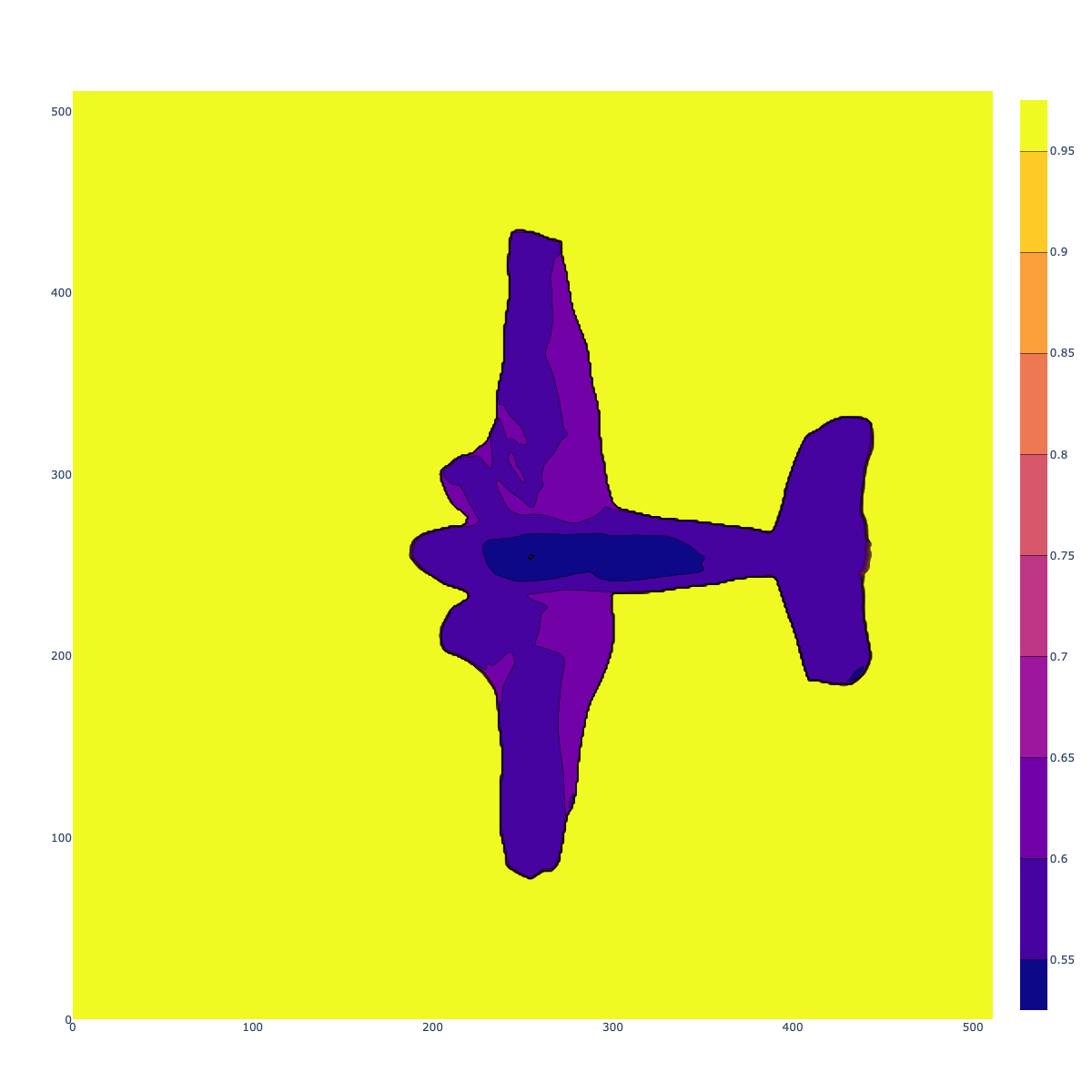}\\
\scalebox{1}{\includegraphics[width=0.245\linewidth, trim={10mm 10mm 10mm 10mm}, clip]{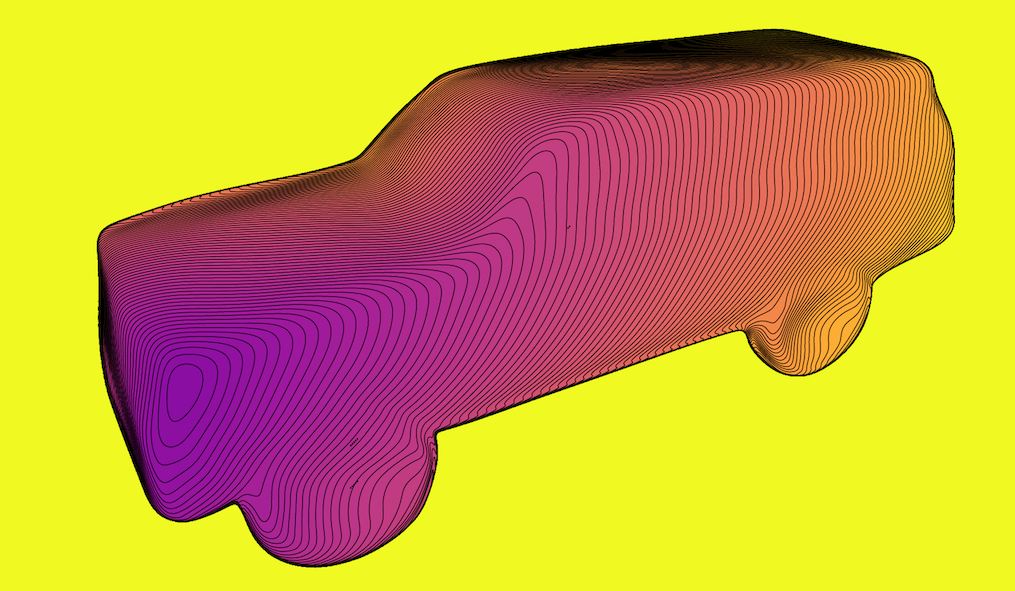}}%
\hfill%
\includegraphics[width=0.245\linewidth, trim=10mm 10mm 10mm 10mm, clip]{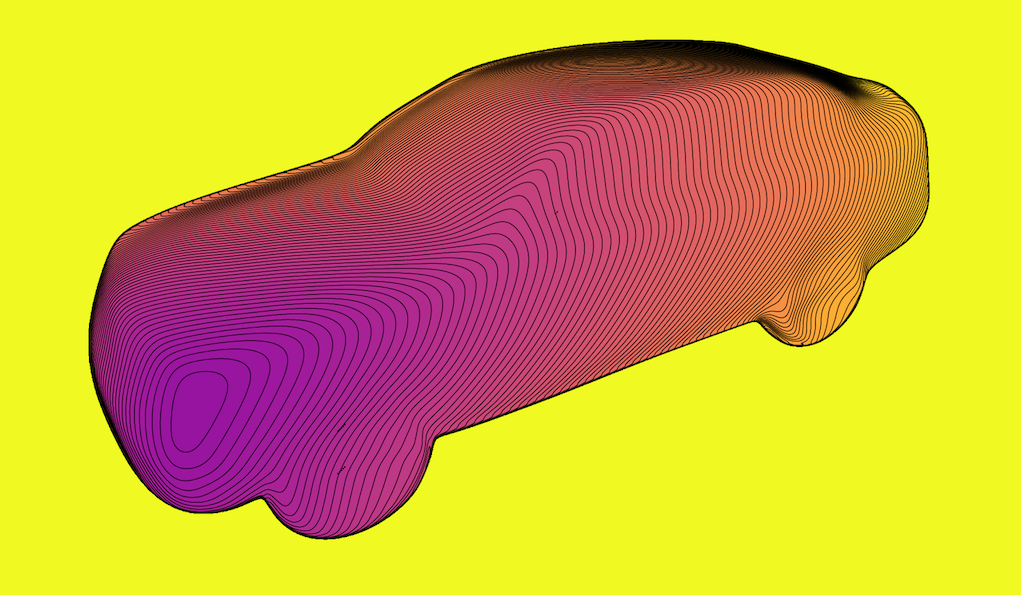}%
\hfill%
\includegraphics[width=0.245\linewidth, trim=10mm 10mm 10mm 10mm, clip]{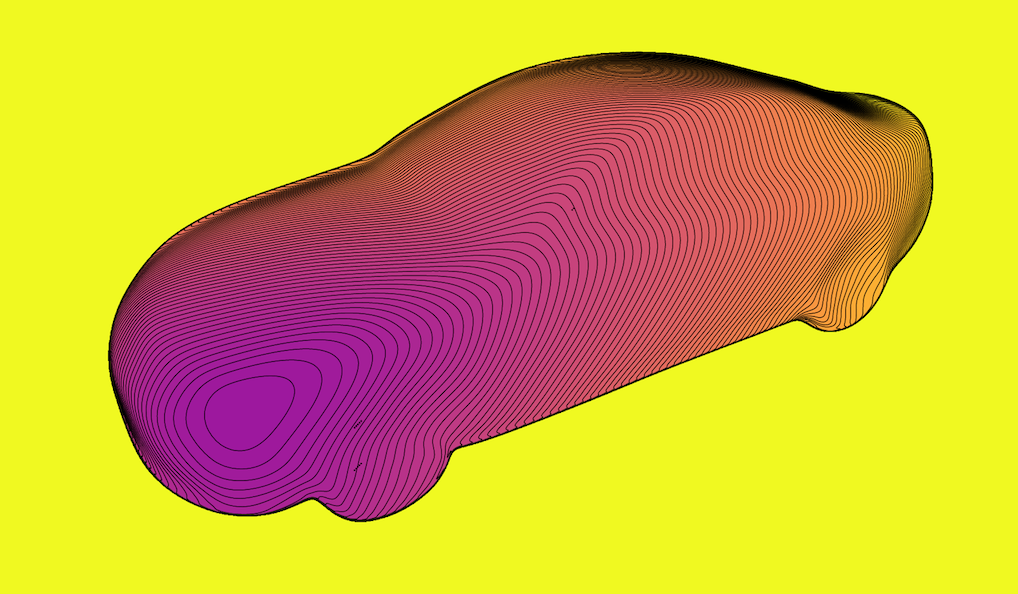}%
\hfill%
\includegraphics[width=0.245\linewidth, trim=10mm 10mm 0mm 100mm, clip]{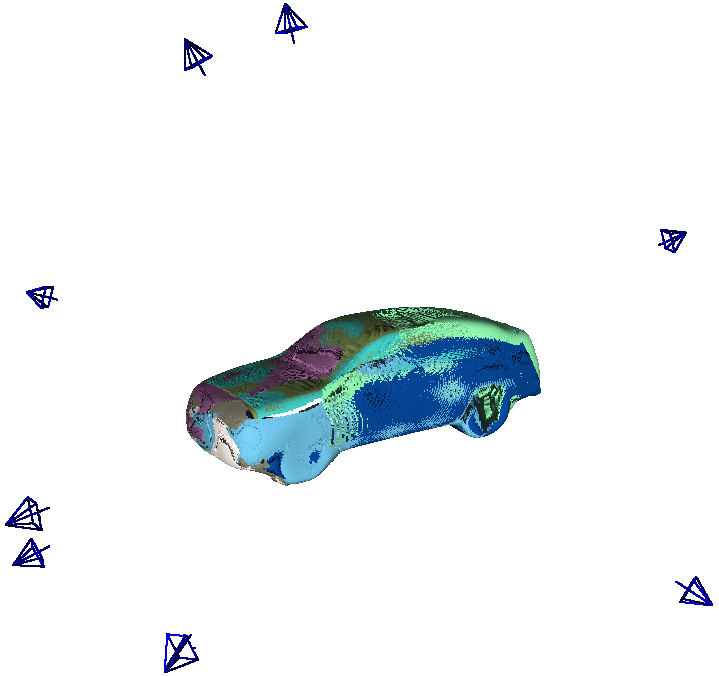}
\vspace{-1em}
\caption{SDDF shape interpolation between two instances.  The left-most and right-most column in the first row show the SDDF output from the same view for two different airplane instances from the training set. The three columns in the middle are generated by using a weighted average of the latent codes of the left-most and right-most instances as an input to the SDDF network. In each row, from left to right, the latent code weights with respect to the left-most instance are $1$, $0.75$, $0.5$, $0.25$, $0$, respectively. Note how the shapes transform smoothly from the left-most to the right-most instance with intermediate shapes looking like valid airplanes. The second row shows interpolation (in column $2$) between two learned car shapes (in columns $1$, $3$). The last column shows the point cloud reconstruction of the interpolated car instance from several different views.}
\label{fig:interpolation}
\vspace{-1em}
\end{figure*}

\subsection{Latent Space Learning and Shape Completion}
\label{sec:category_evaluation}

In this section, we explore the capability of our method to represent a whole category of object shapes.



{\bf \noindent Shape Completion: } First, we focus on recognizing the shape of an unseen instance based on a single distance image or point cloud observation.
Given a trained category auto-decoder $q_{\bftheta}$, we optimize the latent shape code $\bfz \in \mathbb{R}^{256}$ for the unseen instance using \eqref{eq:codefinder}. We train SDDF and IGR~\cite{gropp2020implicit} for $1000$ epochs, with $2000$ random samples, and with learning rate $0.0005, 0.0001$ for the network weights and latent code weights. The learning rate decreases by factor of $2$ every $500$ epochs for IGR (as suggested in the paper~\cite{gropp2020implicit}) and every $200$ epochs for SDDF. At test time, for each instance we take a distance image from one view and down sample it such that we have $1000$ finite rays and $1000$ infinite rays. We use this data for our method and the associated point cloud, augmented with normals, for IGR to optimize the latent code. 
The shape reconstruction accuracy is evaluated at $8$ views different from the one used to obtain first distance image. The SDDF model can directly generate point clouds for these query views, as shown in Fig.~\ref{fig:completion}. To obtain point cloud predictions from IGR, we used the Marching cubes algorithm~\cite{newman2006survey} to extract a mesh from the predicted SDF and used the same $8$ views to generate noiseless point clouds (IGR(1)). Additionally, we generated a uniform point cloud from the IGR mesh (IGR(2)) and compared it with ground truth point cloud. The results are presented in Table~\ref{tab:IGR}.


We also compare the SDDF reconstruction accuracy versus GRNet~\cite{xie2020grnet}, AtlasNet~\cite{groueix2018papier}, PCN~\cite{yuan2018pcn}, FoldingNet~\cite{yang2018foldingnet}, TopNet~\cite{tchapmi2019topnet}, and MSN~\cite{liu2020morphing} in Table~\ref{tab:grnet}. The reconstructed objects are not normalized to a unit-length bounding box in these experiments since the baseline methods did not do this. Our method learns a higher dimensional representation than these methods, so it needs more training data. The training data used by the baseline methods was insufficient to obtain reliable convex hull approximation results. To make the comparison as fair as possible, our method was trained on the same categories with the same train/test splits and was evaluated at view rays that collide with the points used for testing of the baseline methods.

{\bf \noindent Shape Interpolation: } Finally, we demonstrate that the SDDF model represents the latent shape space of an object category continuously and meaningfully. Fig.~\ref{fig:interpolation} presents results for linear interpolation between the latent shape codes of two object instances from the training set. 


\subsection{Limitations}
\label{sec:limitations}

Our method models the distance to an object from any location and orientation. This comes at a price of increased dimension compared to SDF models that only represent the object surface. Our result in Proposition~\ref{prop:SDDF-structure} reduces the SDDF input dimension from $5$ to $4$ for modeling 3D shapes. However, our method still requires more training data compared to SDF models to achieve multi-view consistency. We introduced a data augmentation technique to synthesize data from novel views and alleviate the data requirements. This requires a spherical convex hull computation, which increases the training time, but we introduced a reasonable approximation method using discretization. Our method currently does not utilize additional geometric information such as normals, which may improve the performance. It also cannot currently be trained from RGB images only due to its reliance on distance data.

\section{Conclusion}
This work proposed a signed directional distance function as an implicit representation of object shape. Any valid SDDF was shown to satisfy a gradient condition, which should be respected by neural network approximations. We designed an auto-decoder model that guarantees the gradient condition by construction and can be trained efficiently without 3D supervision using distance measurements from depth camera or Lidar sensors. The SDDF model offers a promising approach for scene modeling in applications requiring efficient visibility or collision checking. Future work will focus on extending the SDDF model to capture texture, color, and lighting and represent complete scenes.

\clearpage

{\small
\bibliographystyle{ieee_fullname}
\bibliography{egbib}
}

\clearpage

\section{Supplementary Material}
\label{sec:supplementary}
\subsection{Network Architecture and Training Details}

We present additional details about the network architecture for $q_{\bftheta}(\bfp,\bfeta,\bfz)$ and the training procedure. The results in Sec.~5 are generated with an $16$ layer fully connected network with 512 hidden units per layer and a skip connection from the input to layers $4$, $8$, $12$ (every $4$ layers). We use a soft-plus activation function $\frac{1}{\beta}\ln(1+\exp(\beta x))$ with $\beta=100$. The inputs are positions $\bfp \in \mathbb{R}^3$, view direction $\bfeta \in S^2$, and latent code $\bfz \in \mathbb{R}^{256}$. The third component of $\bfeta$ ($c$ in Lemma~3) may be very close to $-1$. To avoid numerical problems, we let $\bfeta = [\sin(\theta)\cos(\phi), \sin(\theta)\sin(\phi), \cos(\theta)]^\top$. Using $s_\theta := \sin(\theta)$, $c_\theta := \cos(\theta)$, $s_{\phi}:= \sin(\phi)$, $c_{\phi}:=\cos(\phi)$, and applying the fact that $\frac{{s_\theta}^2}{1+c_\theta} = \frac{1-{c_\theta}^2}{1+c_\theta} = 1-c_\theta$, the rotation matrix $\bfR_{\bfeta}$ in Lemma~3 used to map $\bfeta$ to the standard basis vector $\bfe_3$ becomes:
%
\begin{equation*}
\bfR_{\bfeta} = \begin{bmatrix}
1 - (1-c_\theta)c_\phi^2 & -(1-c_\theta)s_\phi c_\phi & -s_\theta c_\phi\\
-(1-c_\theta)s_\phi c_\phi & 1 -(1-c_\theta)s_\phi^2 & -s_\theta s_\phi\\
s_\theta c_\phi & s_\theta s_\phi & c_\theta
\end{bmatrix}.
\end{equation*}
Using the dimension reduction in Proposition~1, the  final network inputs are $\bfP \bfR_{\bfeta} \bfp \in \mathbb{R}^2$, $\bfeta \in S^2$, and $\bfz \in \mathbb{R}^{256}$.

%
%

%

All experiments are done on a single GTX $1080$ Ti GPU with the PyTorch deep learning framework \cite{paszke2017automatic} and the ADAM optimizer \cite{kingma2014adam}. For single shape estimation, the network is trained with initial learning rate of $0.005$, decreasing by a factor of $2$ every $1000$ steps for $10k$ iterations. In each iteration, we pick a batch of $100k$ samples randomly from the synthesized training data and use $\alpha = 1$, $\beta = 0.5$, $\gamma = 0$, $p = 1$, $r(x) = \max\{0,x\}$ in the error function in (10). For category-level shape estimation, the network is trained with initial learning rate of $0.0005$ for the network parameters $\bftheta$ and $0.0001$ for latent code $\bfz$, both decreasing by a factor of $2$ every $200$ steps for $1k$ iterations. In each iteration, for each object we pick a batch of $2k$ samples randomly from the union of the synthesized and original samples. We use $\alpha = 1$, $\beta = 1$, $p = 1$, $r(x) = \max\{0,x\}$, $\sigma = 0.001$ and Euclidean norm regularization $\|\bfz\|_2^2$ for the latent shape code in the error function (12).

To train the IGR network \cite{gropp2020implicit}, for each object we picked a batch of $2k$ samples randomly from the original point cloud and augmented them with normals. Note that for our method we pick a total of $2k$ samples for each object, including the original data and synthesized data as well as finite rays and infinite rays. 
We use the default training parameters for IGR as provided in the open-source implementation \cite{gropp2020implicit}. The only parameter we adjusted was the initial learning rate because there was a discrepancy between the open source code and the IGR paper. We chose the setting that provided better results, namely initial learning rate of $0.0005$ for network parameters $\bftheta$ and $0.0001$ for latent code $\bfz$, both decreasing by a factor of $2$ every $500$ steps for $1k$ iterations.

We use the default training parameters for the DGP network \cite{williams2019deep}, except increasing the upsamples-per-patch parameter from $8$ to $20$ and adjusting the radius parameter to guarantee at least 128 patches for each object. 

To produce training data for the category-level experiments, we normalize each instance to a unit box and generate $8$ distance images with resolution $512 \times 512$ using PyRenderer~\cite{pyrender}. Each distance image is subsampled to have at most $12500$ infinite rays and $12500$ finite rays. Hence, there are at most $100k$ finite and $100k$ infinite rays for each object. To accelerate the data augmentation procedure, we further subsample the point cloud produced by the finite rays in each view from $12500$ to $1250$ points, when computing the convex hull approximation in Sec. 4.5. This provides a subsampled point cloud across all views with at most $10k$ points, which was further subsampled to $2k$ points using the diversipy python package \cite{salomon2013psa,hardin2004discretizing}. To generate synthetic data, we choose $1k$ random azimuth and elevation views on a sphere around the point cloud. Infinite rays are obtained by projecting the original point cloud (with $100k$ points) to an image plane with resolution $128 \times 128$ for each imaginary camera and selecting the unoccupied pixel directions. Finite rays are generated using the procedure described in Sec. 4.5. 



Additional qualitative results for shape completion are presented in Fig.~\ref{fig:supp_comp} and for shape interpolation in Fig.~\ref{fig:sofainterpolation}, Fig.~\ref{fig:airinterpolation}, Fig.~\ref{fig:carinterpolation}, Fig.~\ref{fig:watercraft_display_interpolation}.




\begin{figure*}[h!]
    \centering
\begin{minipage}{\linewidth}
  \centering
\includegraphics[width=0.25\linewidth, trim={0mm 30mm 0mm 40mm}, clip]{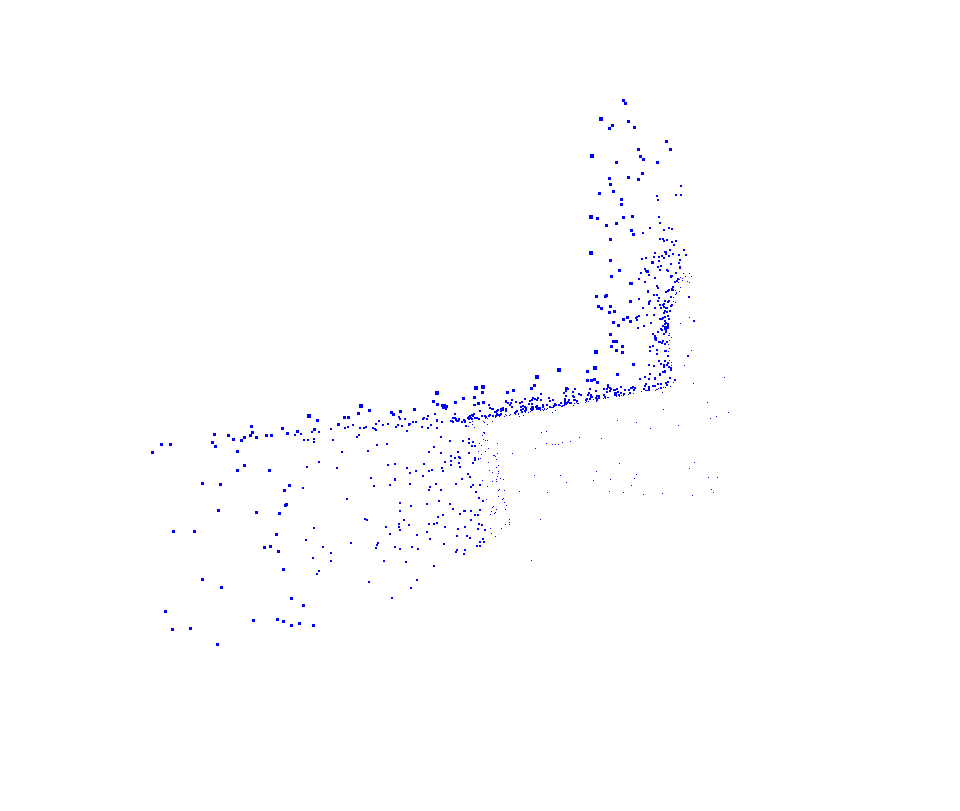}
\includegraphics[width=0.25\linewidth, trim={0mm 30mm 0mm 40mm}, clip]{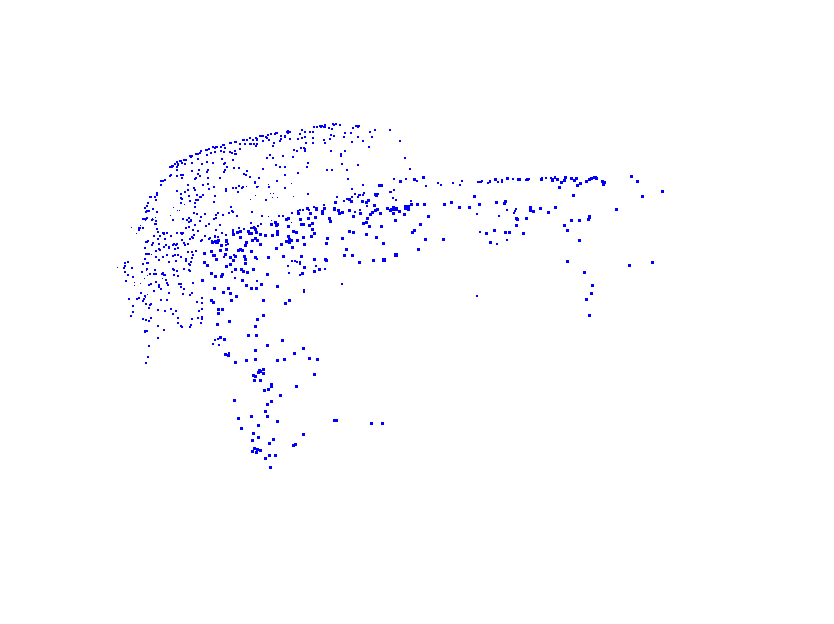}
\includegraphics[width=0.25\linewidth, trim={0mm 30mm 0mm 40mm}, clip]{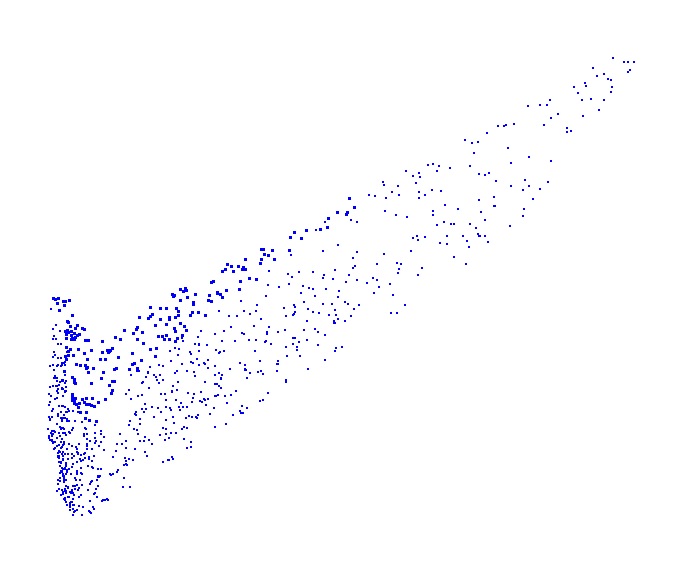}
\includegraphics[width=0.23\linewidth, trim={0mm 10mm 0mm 20mm}, clip]{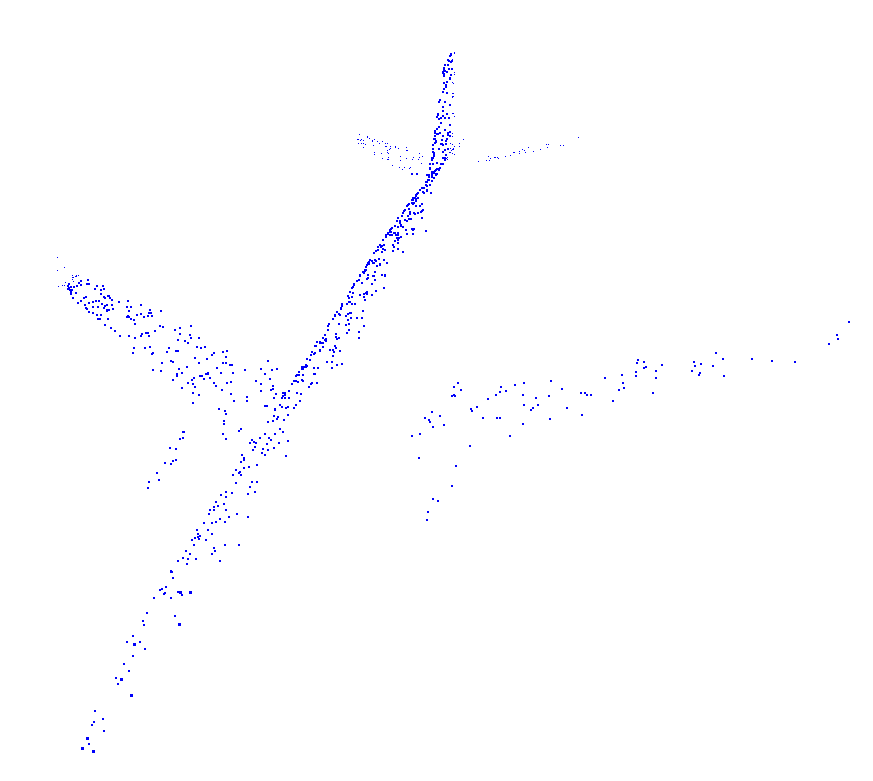}
\includegraphics[width=0.255\linewidth, trim={40mm 30mm 20mm 40mm}, clip]{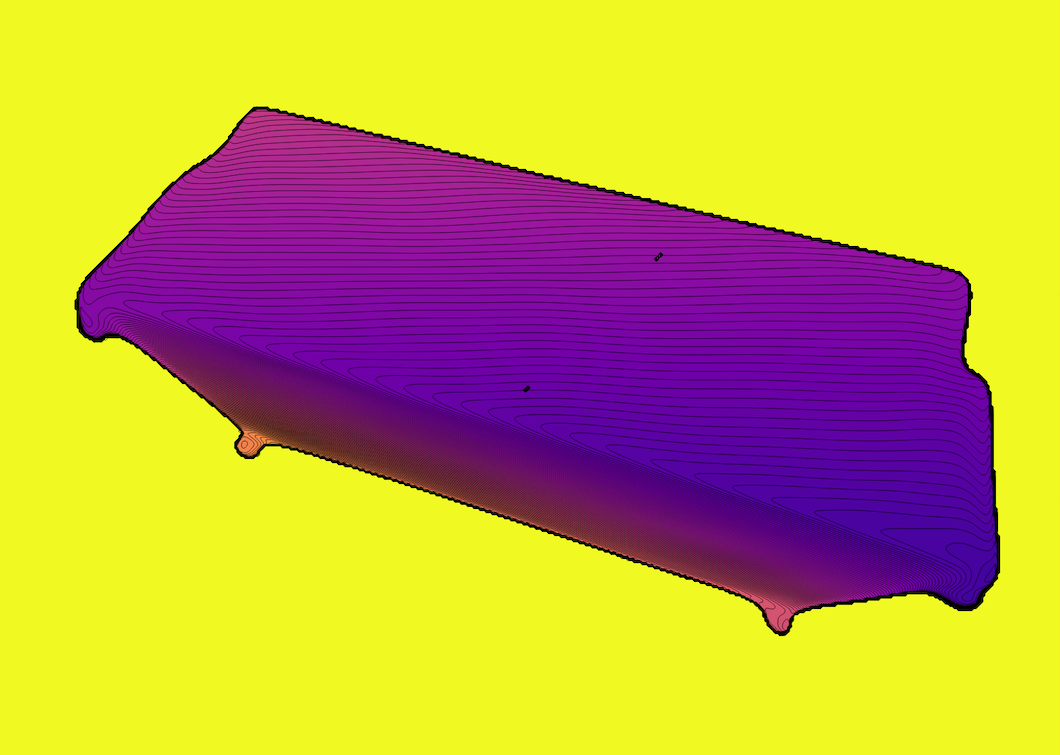}
\includegraphics[width=0.255\linewidth, trim={0mm 30mm 0mm 67mm}, clip]{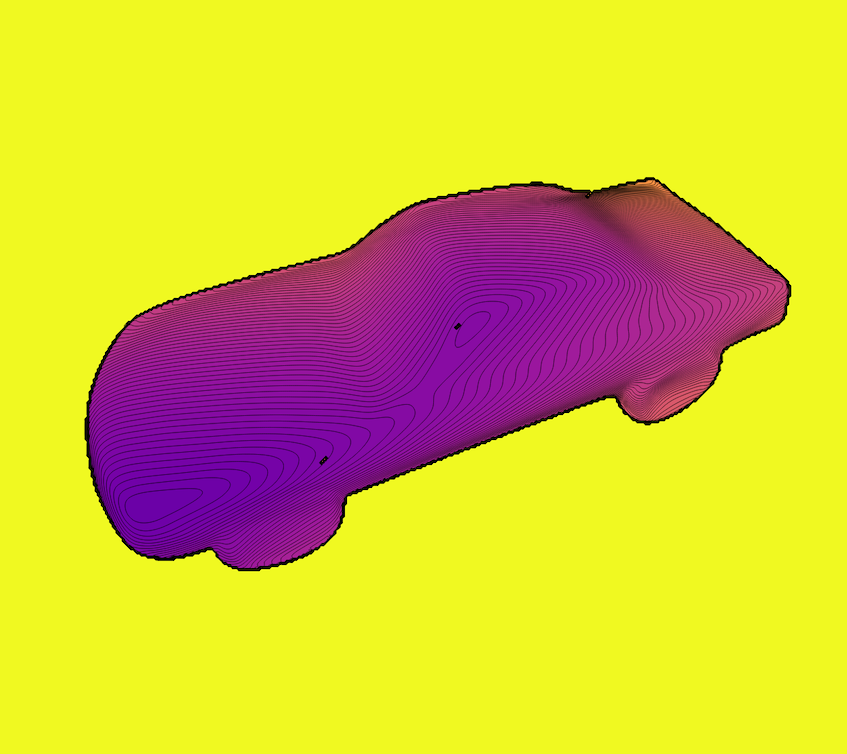}
\includegraphics[width=0.255\linewidth, trim={40mm 80mm 60mm 50mm}, clip]{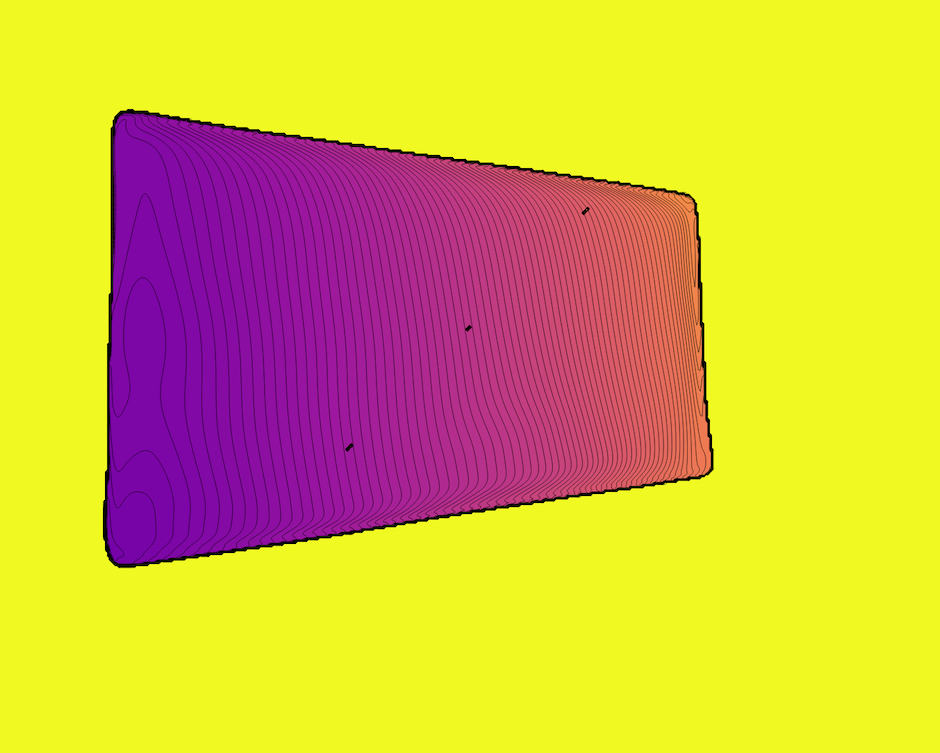}
\includegraphics[width=0.22\linewidth, trim={40mm 10mm 20mm 40mm}, clip]{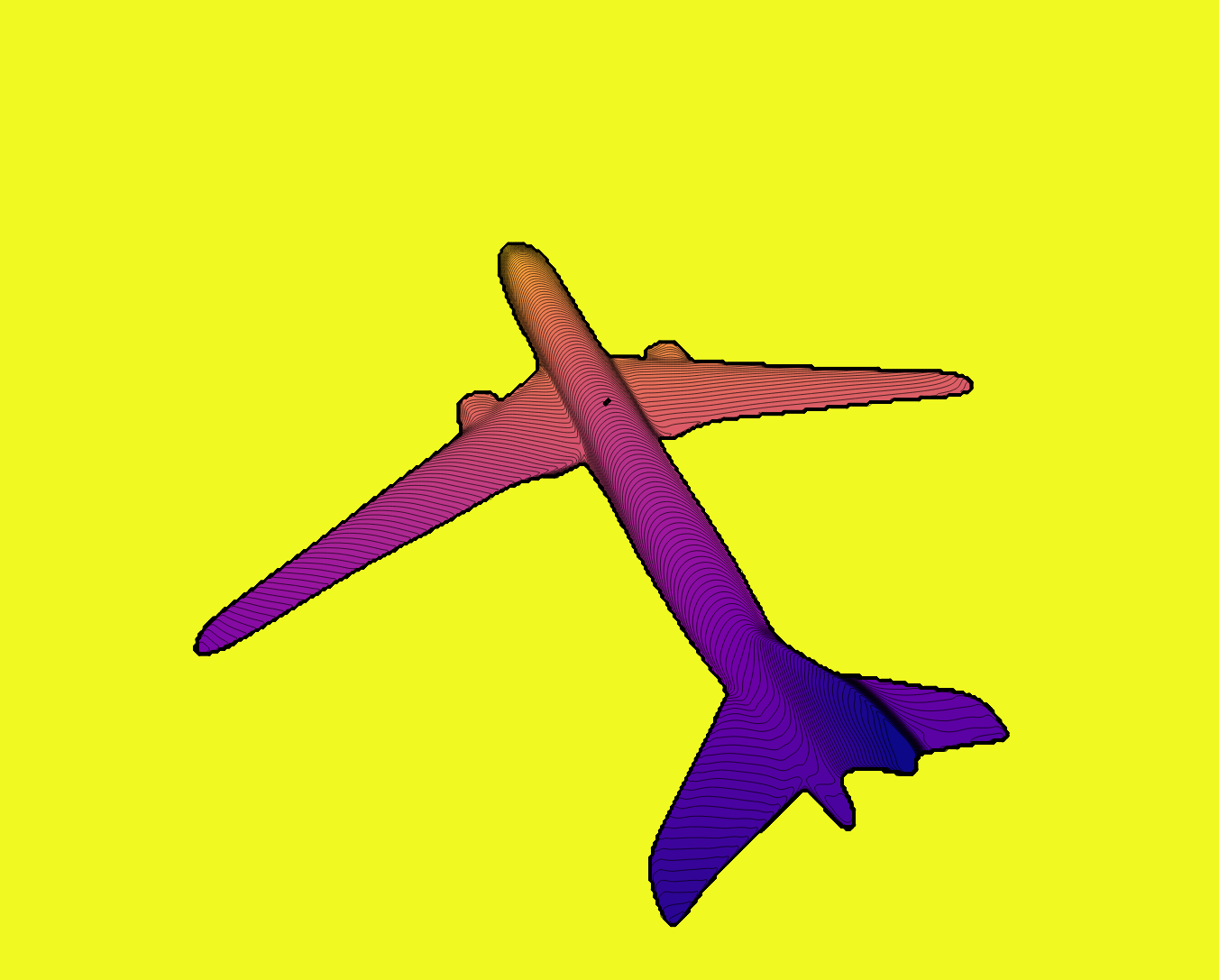}
\includegraphics[width=0.255\linewidth, trim={40mm 30mm 20mm 40mm}, clip]{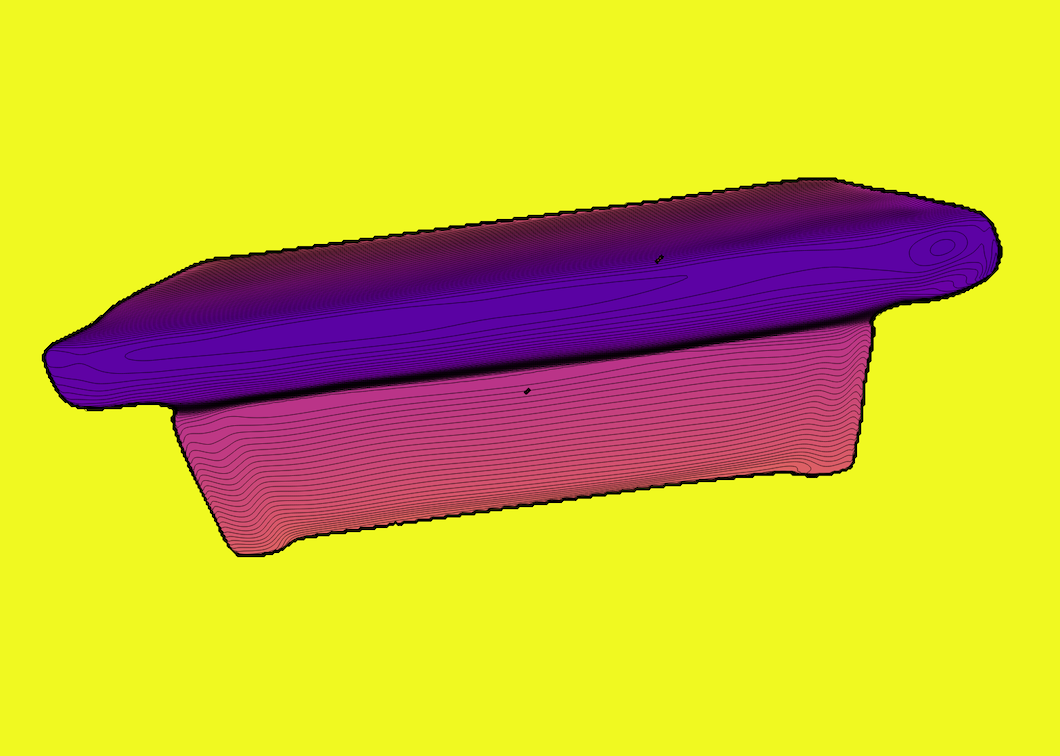}
\includegraphics[width=0.255\linewidth, trim={0mm 30mm 0mm 67mm}, clip]{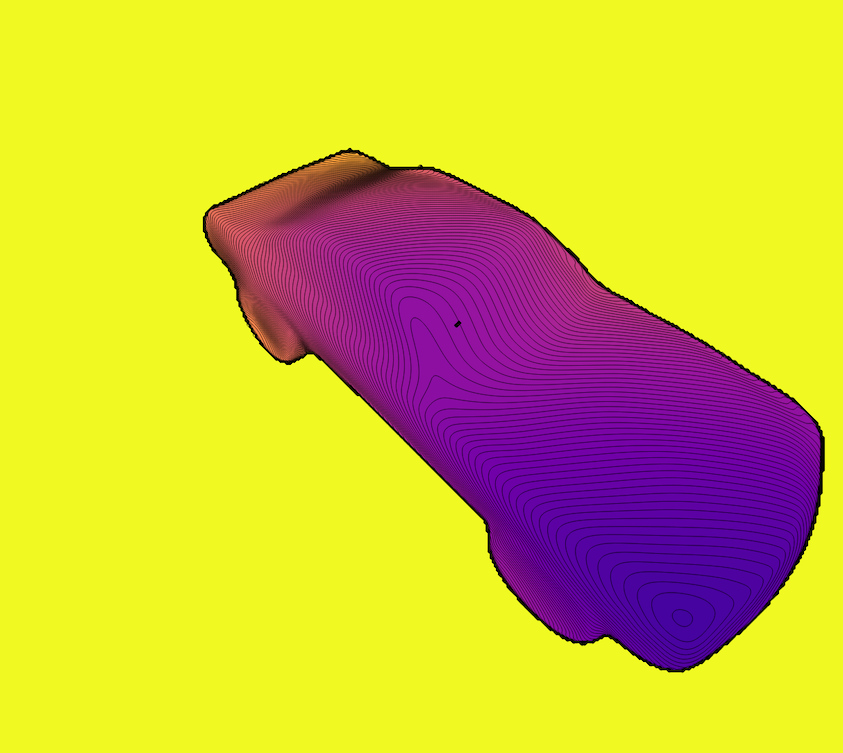}
\includegraphics[width=0.255\linewidth, trim={40mm 80mm 60mm 50mm}, clip]{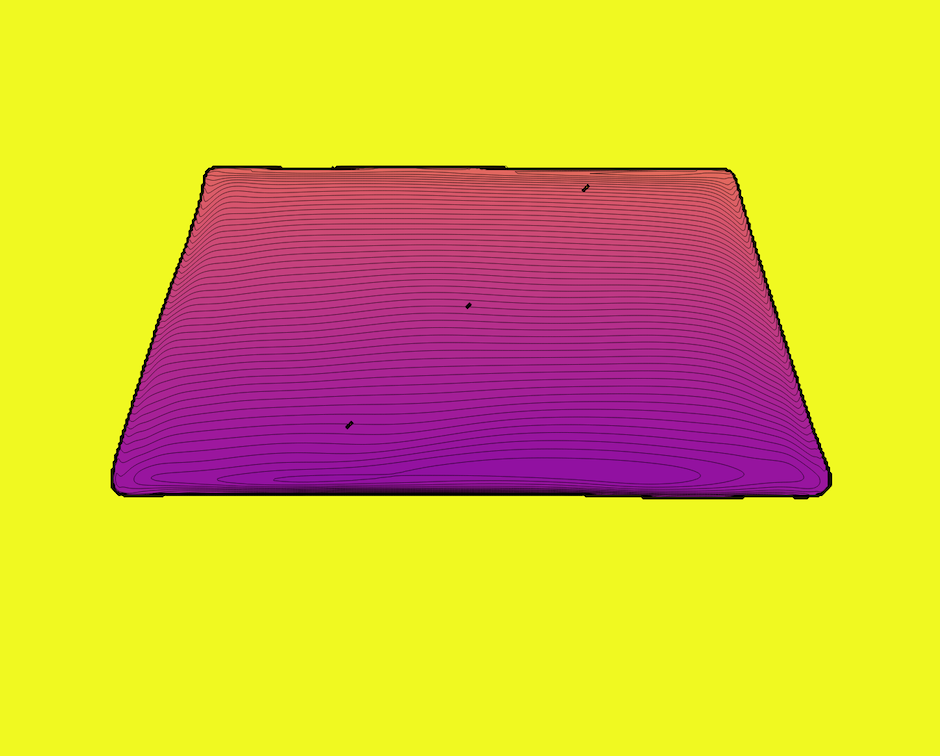}
\includegraphics[width=0.22\linewidth, trim={40mm 20mm 20mm 30mm}, clip]{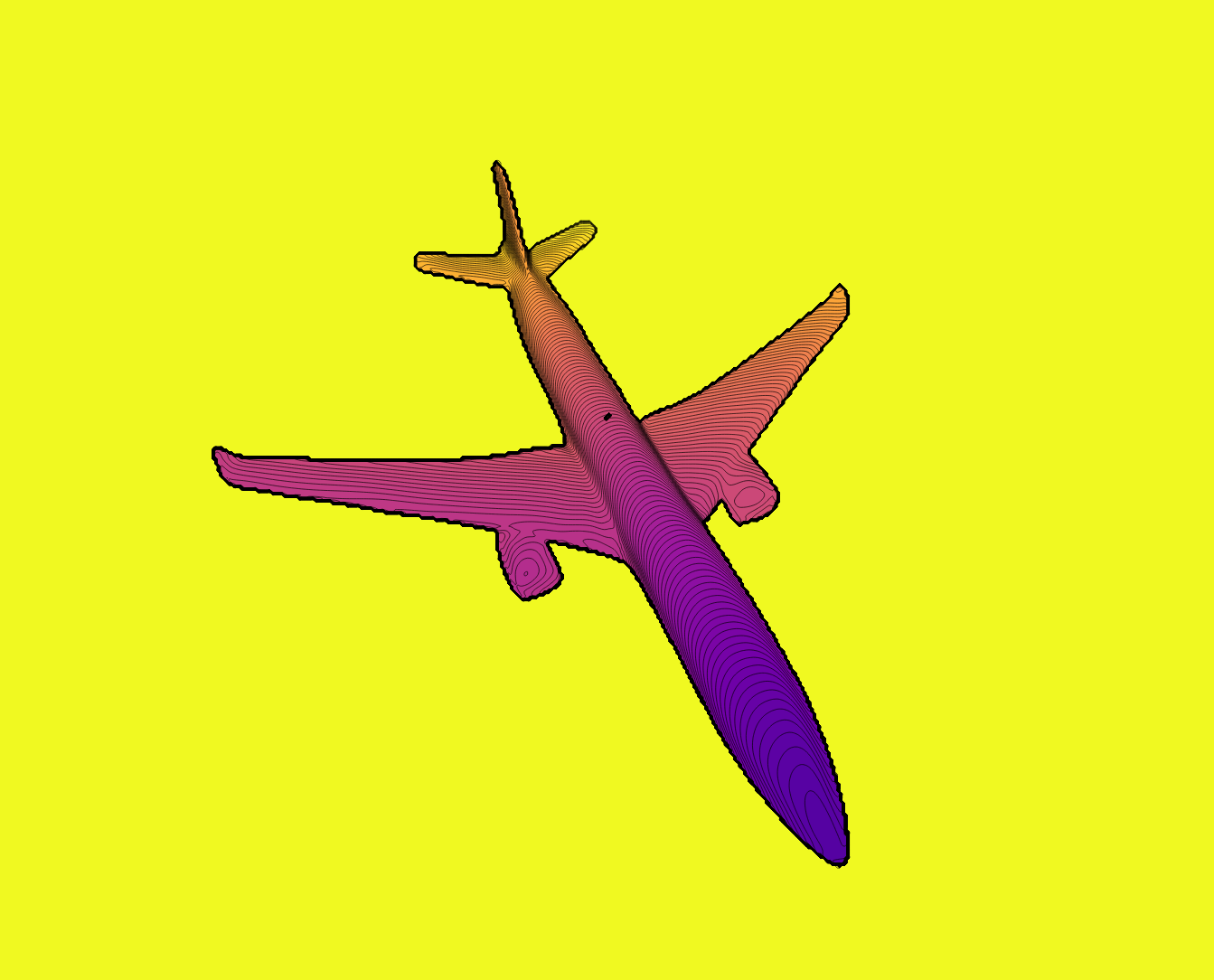}
\includegraphics[width=0.255\linewidth, trim={40mm 30mm 20mm 40mm}, clip]{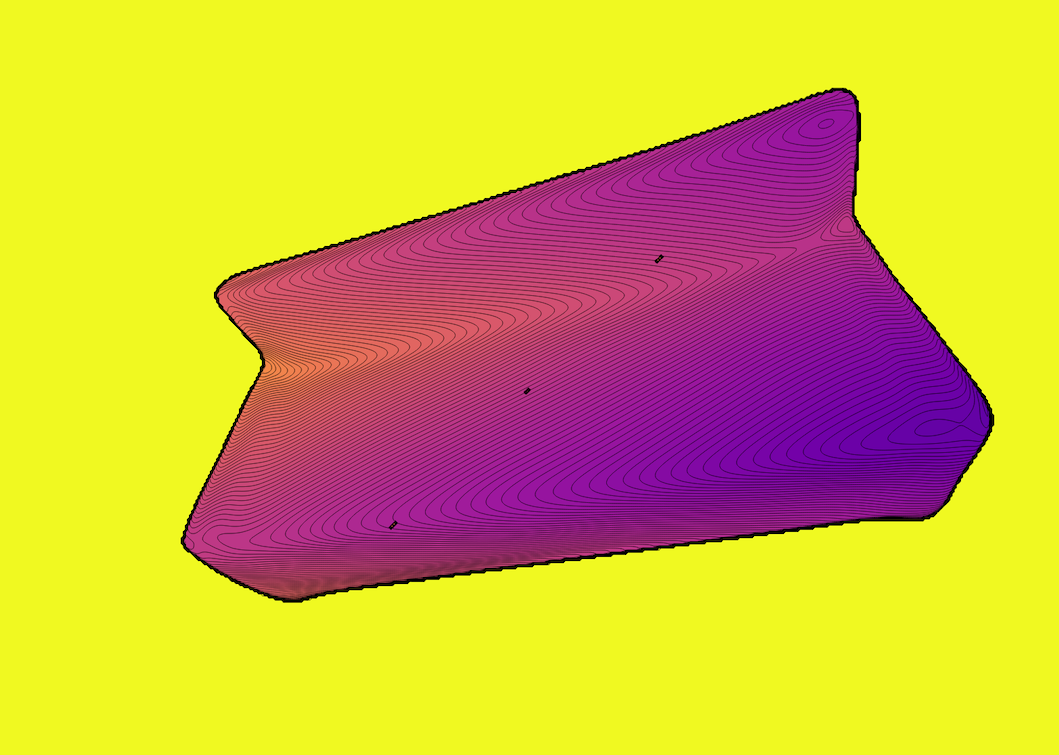}
\includegraphics[width=0.255\linewidth, trim={0mm 30mm 0mm 67mm}, clip]{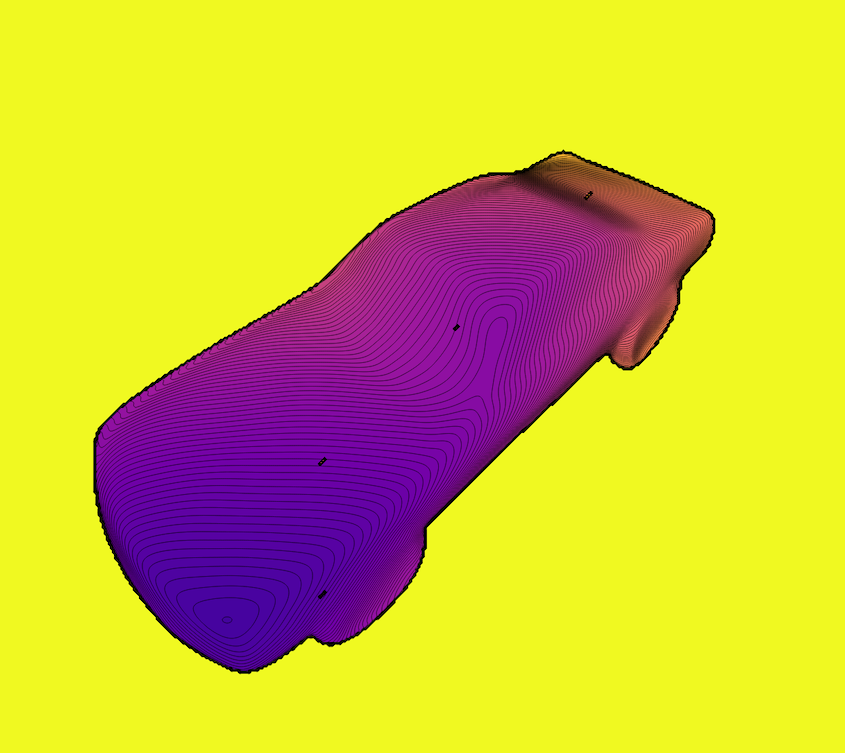}
\includegraphics[width=0.255\linewidth, trim={40mm 80mm 60mm 50mm}, clip]{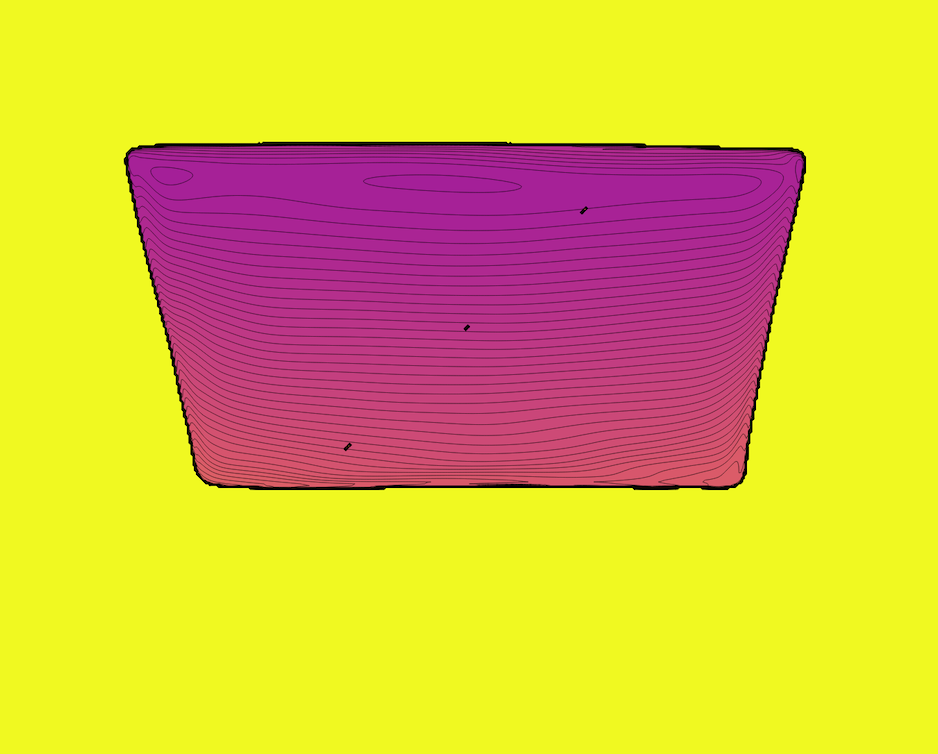}
\includegraphics[width=0.22\linewidth, trim={40mm 10mm 20mm 40mm}, clip]{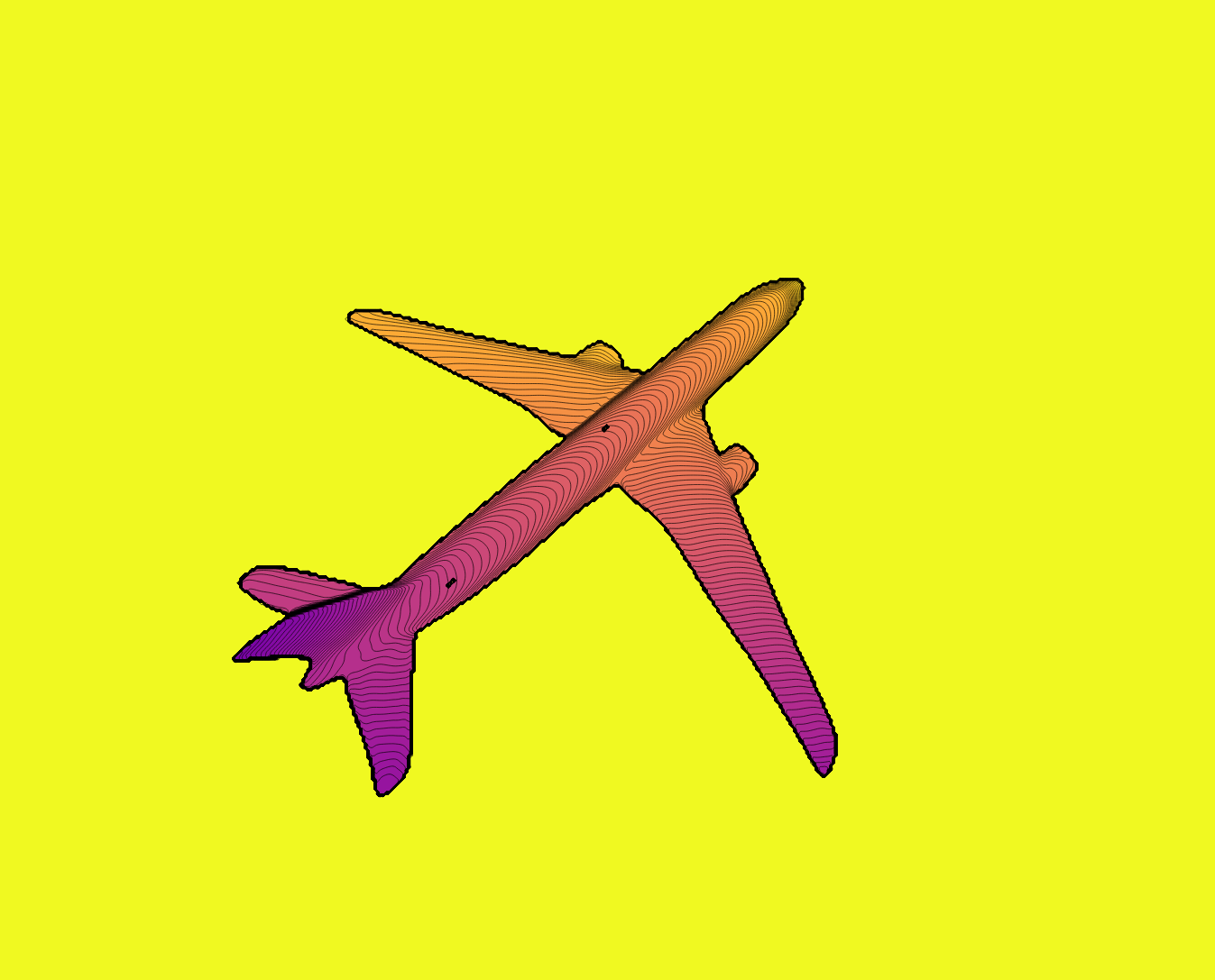}
\includegraphics[width=0.255\linewidth, trim={40mm 30mm 20mm 40mm}, clip]{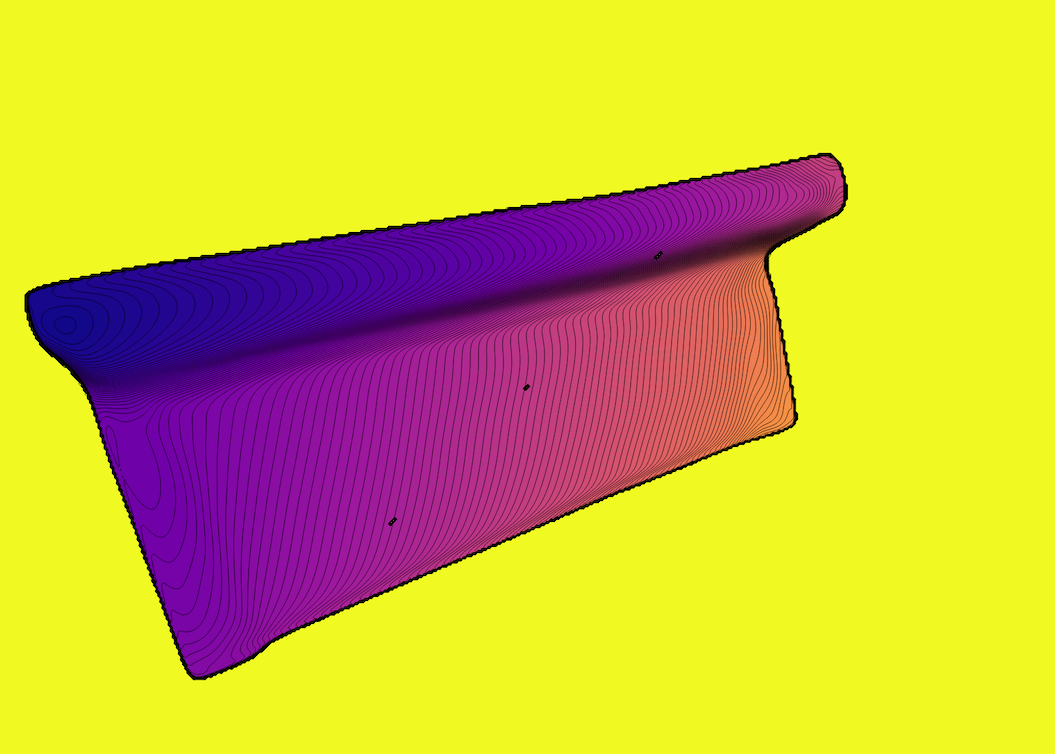}
\includegraphics[width=0.255\linewidth, trim={0mm 30mm 0mm 67mm}, clip]{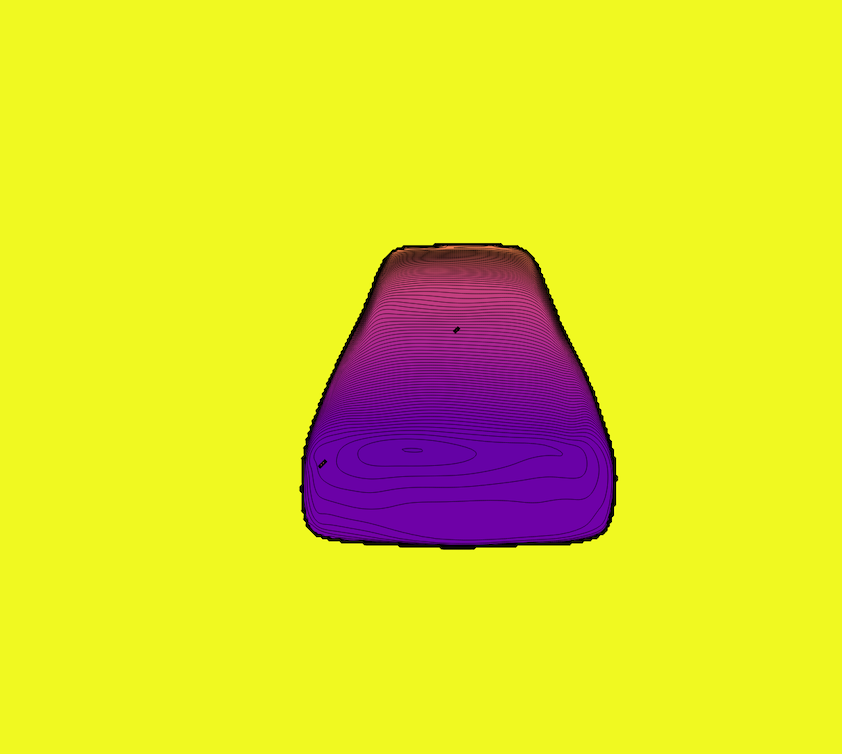}
\includegraphics[width=0.255\linewidth, trim={40mm 80mm 60mm 50mm}, clip]{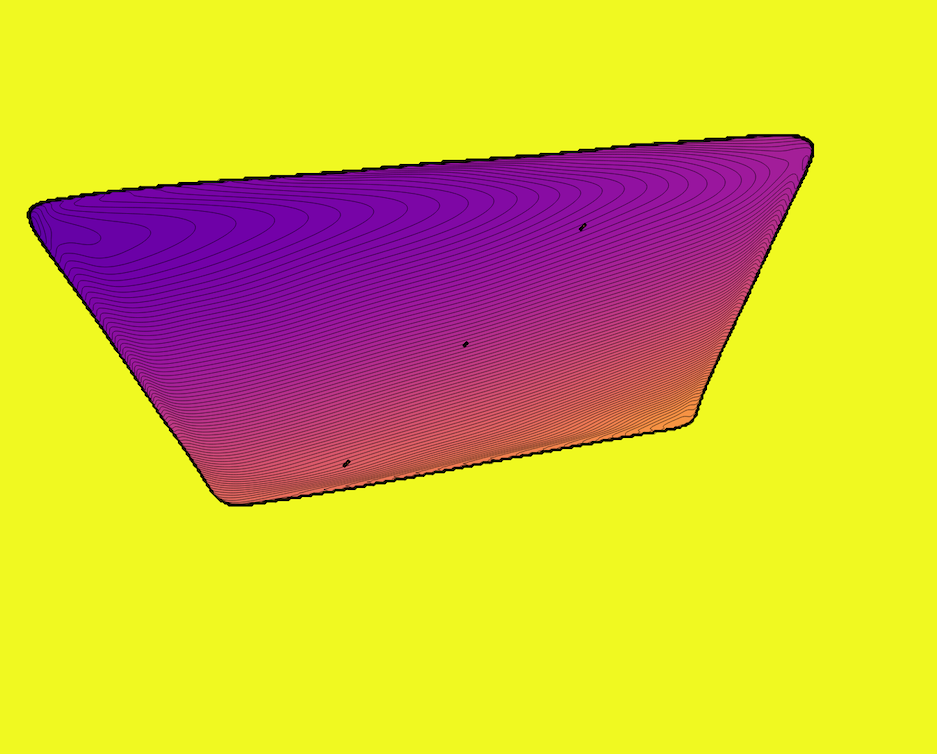}
\includegraphics[width=0.22\linewidth, trim={40mm 20mm 20mm 20mm}, clip]{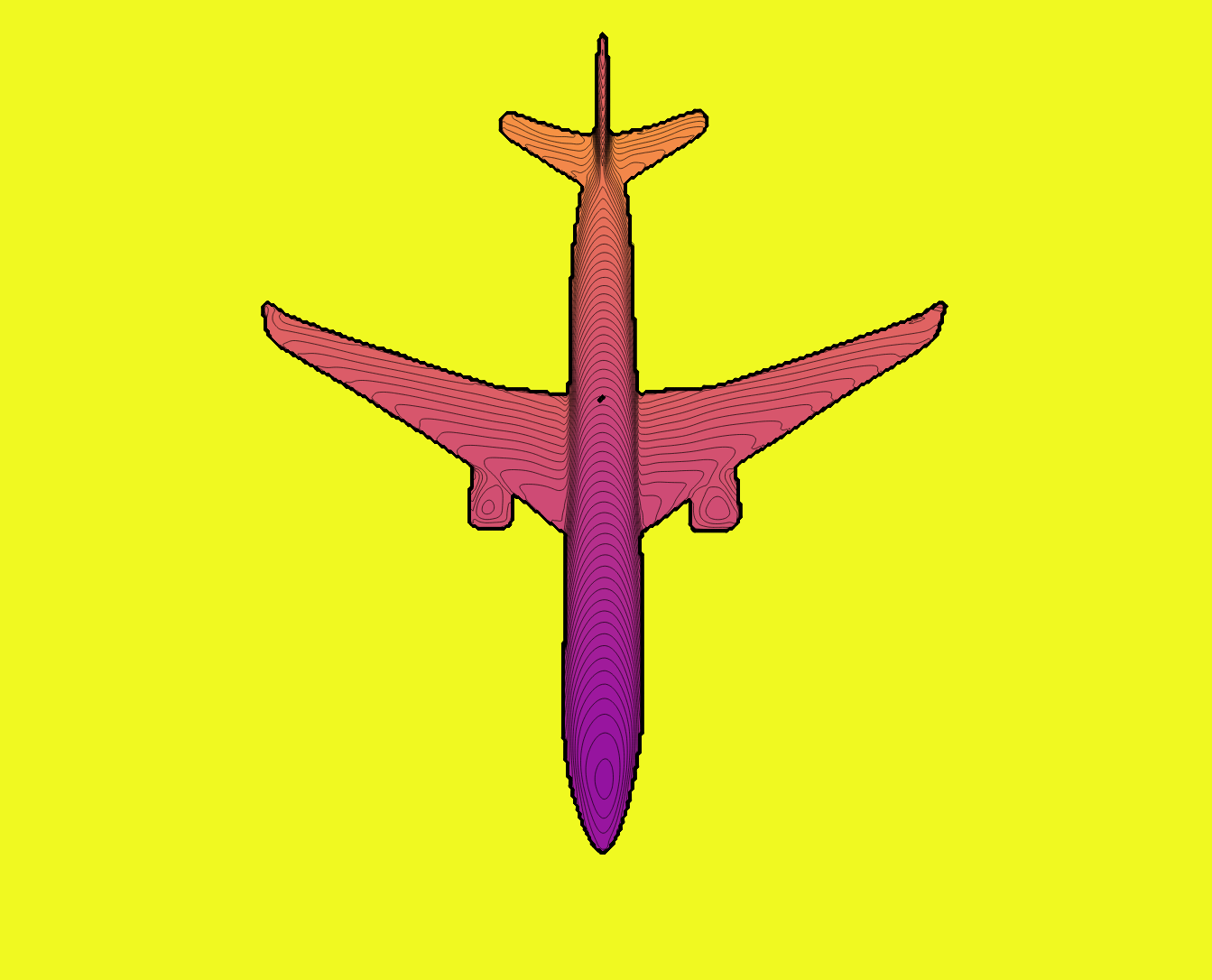}
\includegraphics[width=0.24\linewidth, trim={60mm 30mm 60mm 70mm}, clip]{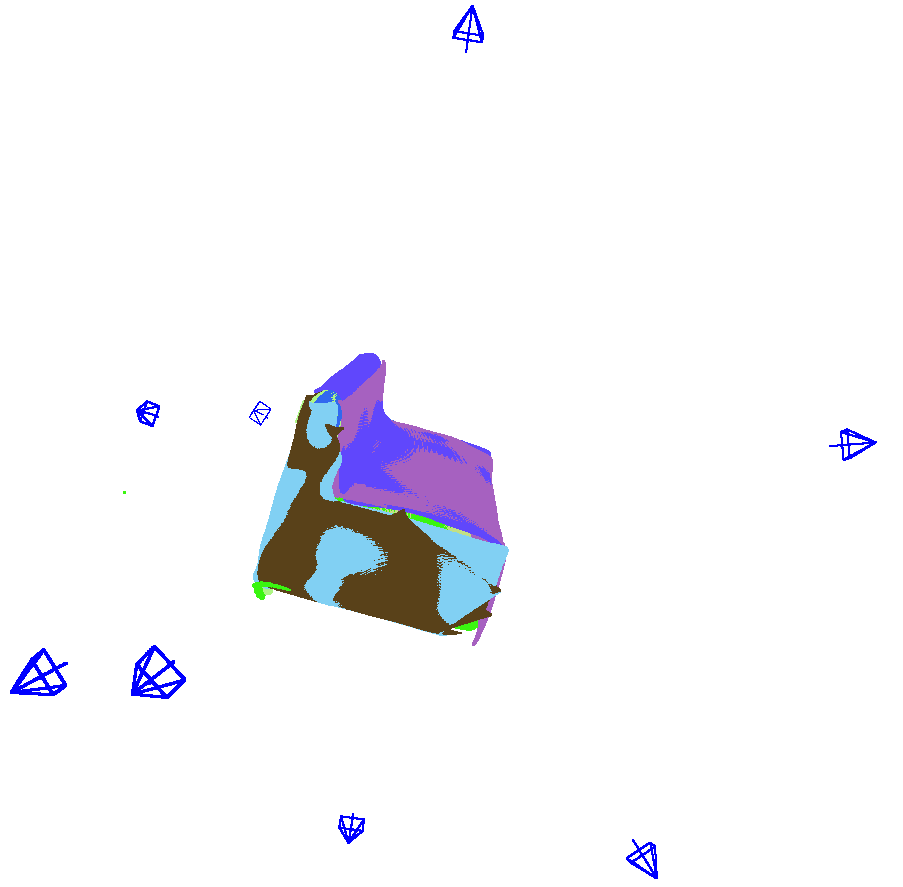}
\includegraphics[width=0.24\linewidth, trim={80mm 30mm 60mm 70mm}, clip]{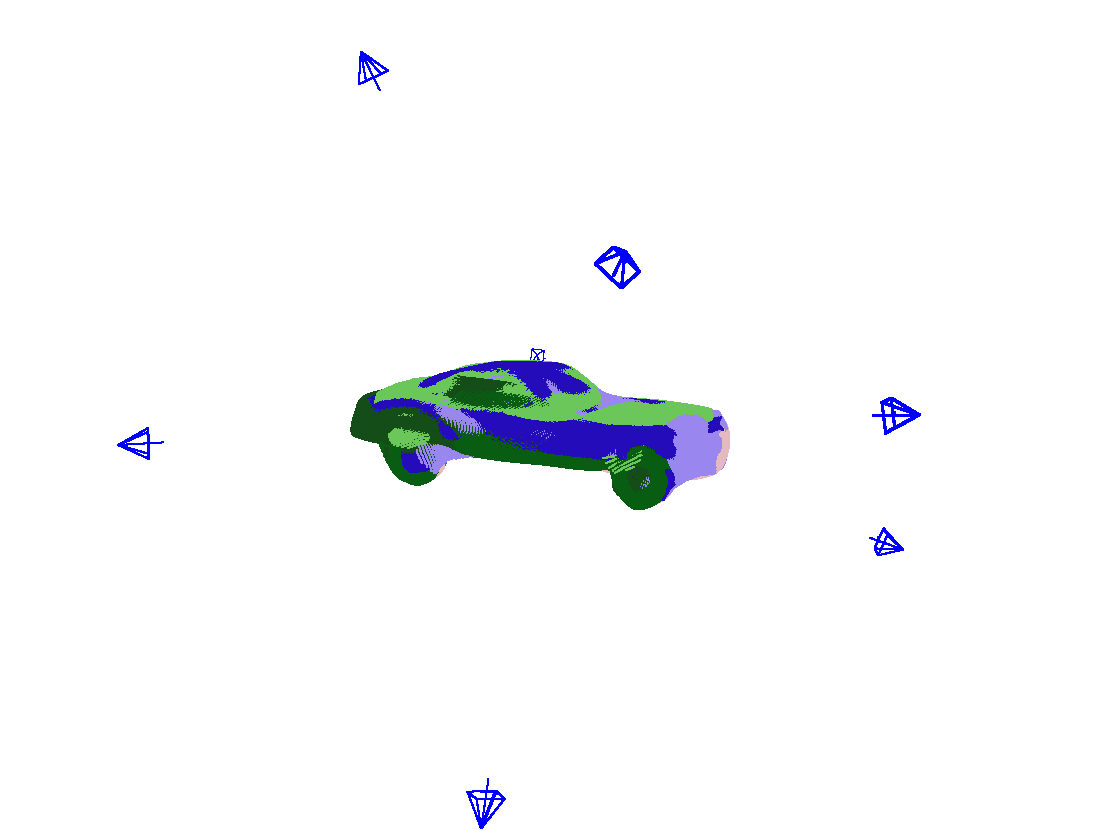}
\includegraphics[width=0.24\linewidth, trim={40mm 30mm 40mm 70mm}, clip]{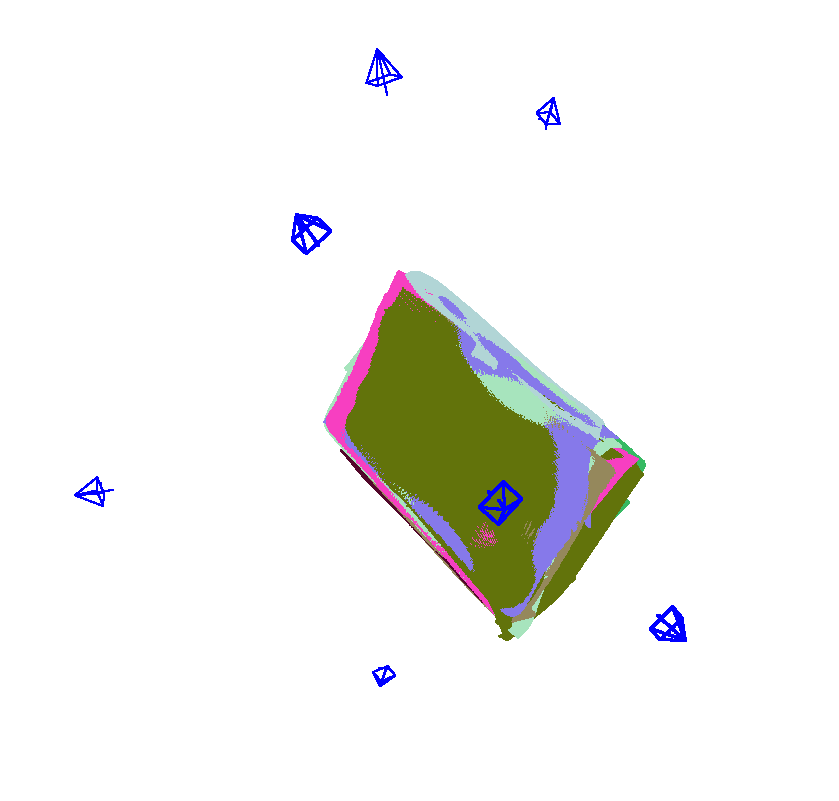}
\includegraphics[width=0.24\linewidth, trim={40mm 10mm 40mm 70mm}, clip]{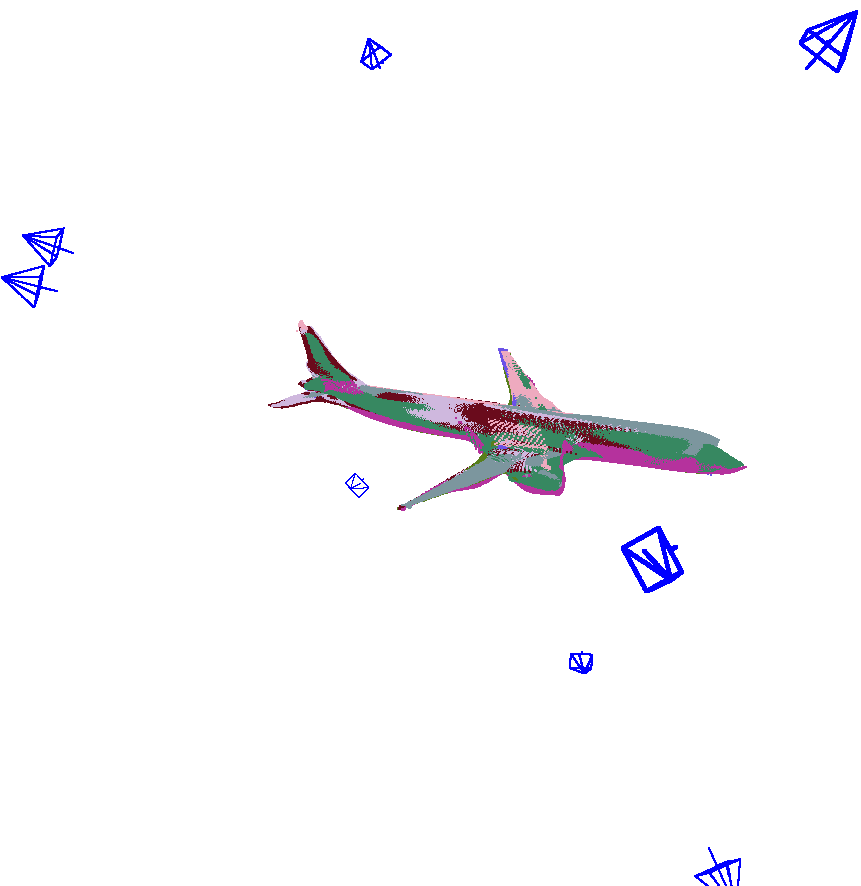}
%
%
%
\end{minipage}%
\caption{SDDF shape completion using $1k$ finite and $1k$ infinite rays from a single distance view and the corresponding point cloud (upper row) from an unseen object instance. After latent code optimization, the SDDF model can synthesize novel distance views (the four middle rows), and novel point clouds from arbitrary views (last row).}
\label{fig:supp_comp}
\end{figure*}
\begin{figure*}[h!]
    \centering
\begin{minipage}{\linewidth}
  \centering
\includegraphics[width=0.24\linewidth, trim={60mm 20mm 20mm 20mm}, clip]{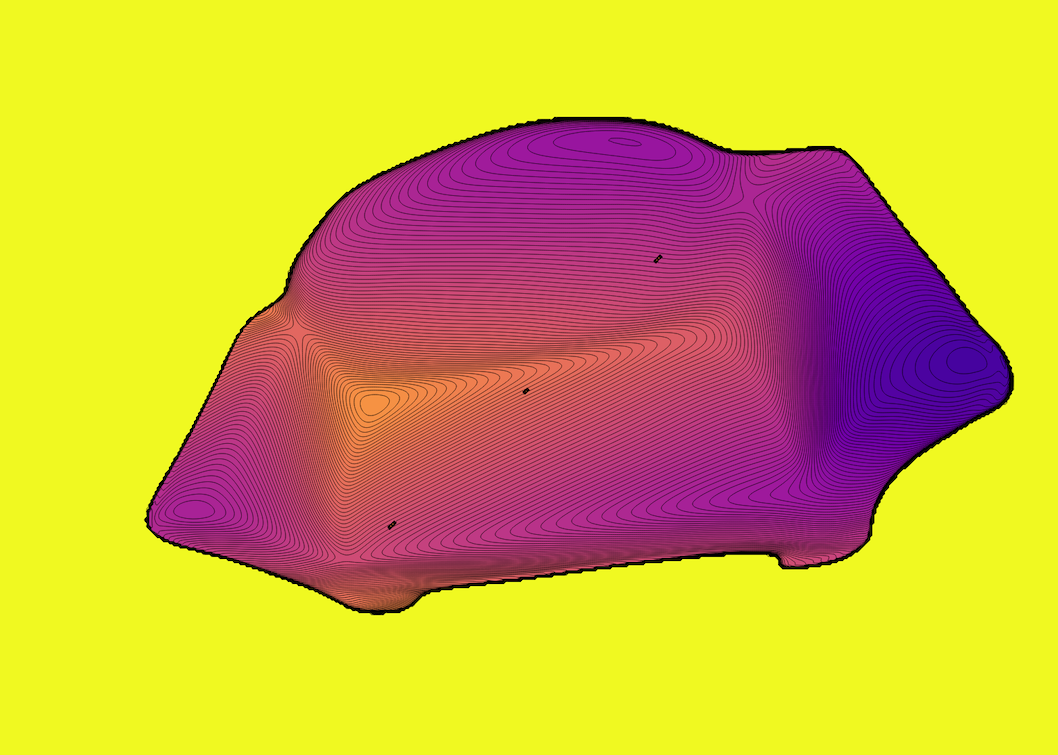}
\includegraphics[width=0.24\linewidth, trim={60mm 20mm 20mm 20mm}, clip]{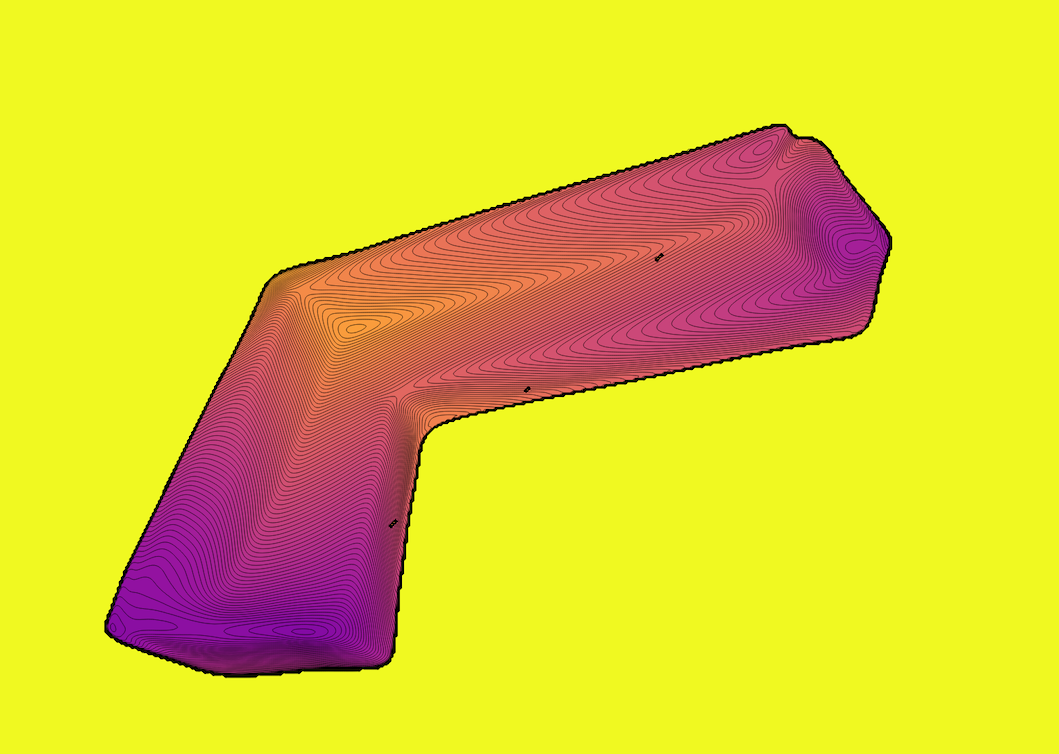}
\includegraphics[width=0.24\linewidth, trim={60mm 20mm 20mm 20mm}, clip]{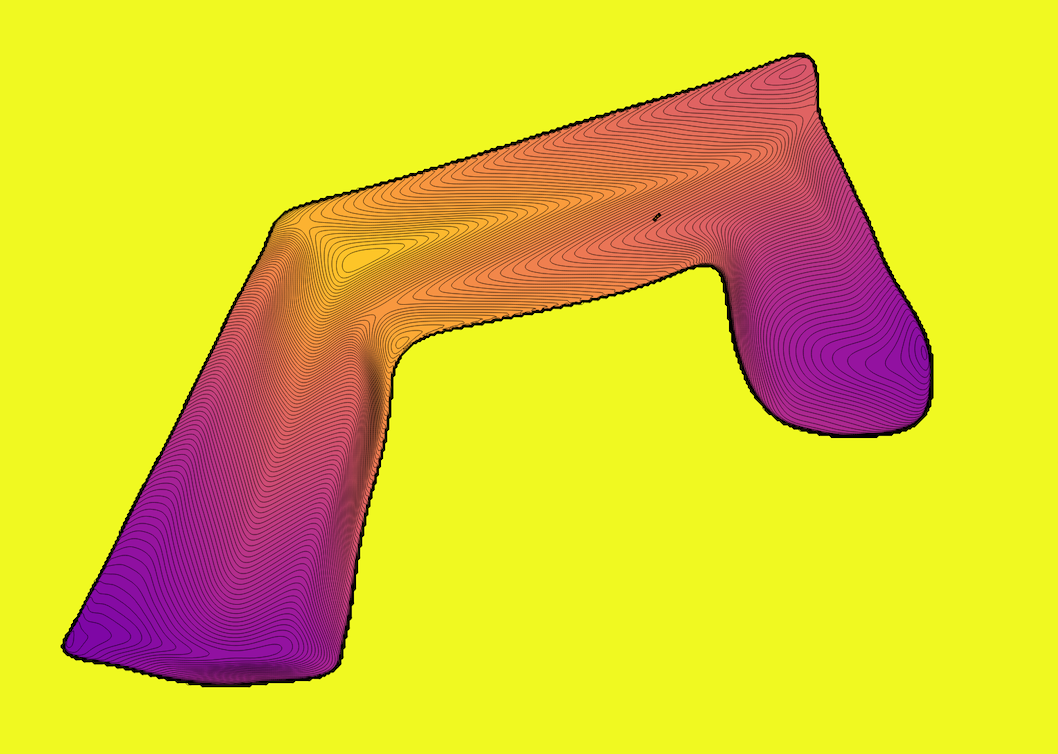}
\includegraphics[width=0.24\linewidth, trim={60mm 20mm 20mm 20mm}, clip]{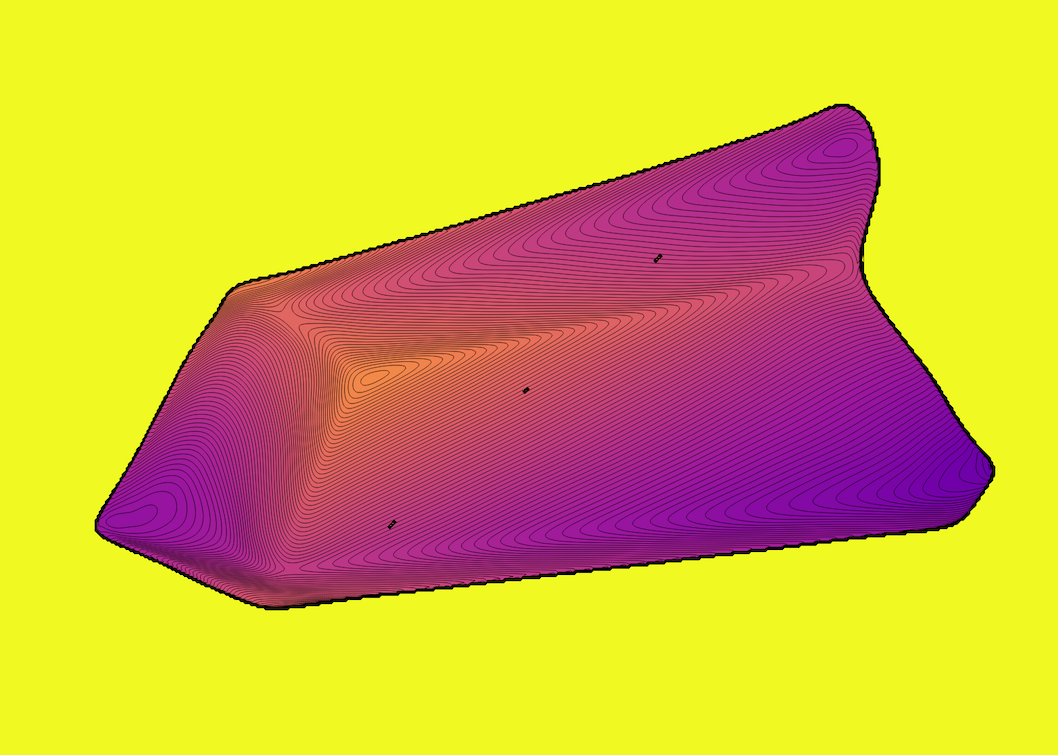}
\includegraphics[width=0.24\linewidth, trim={60mm 20mm 20mm 20mm}, clip]{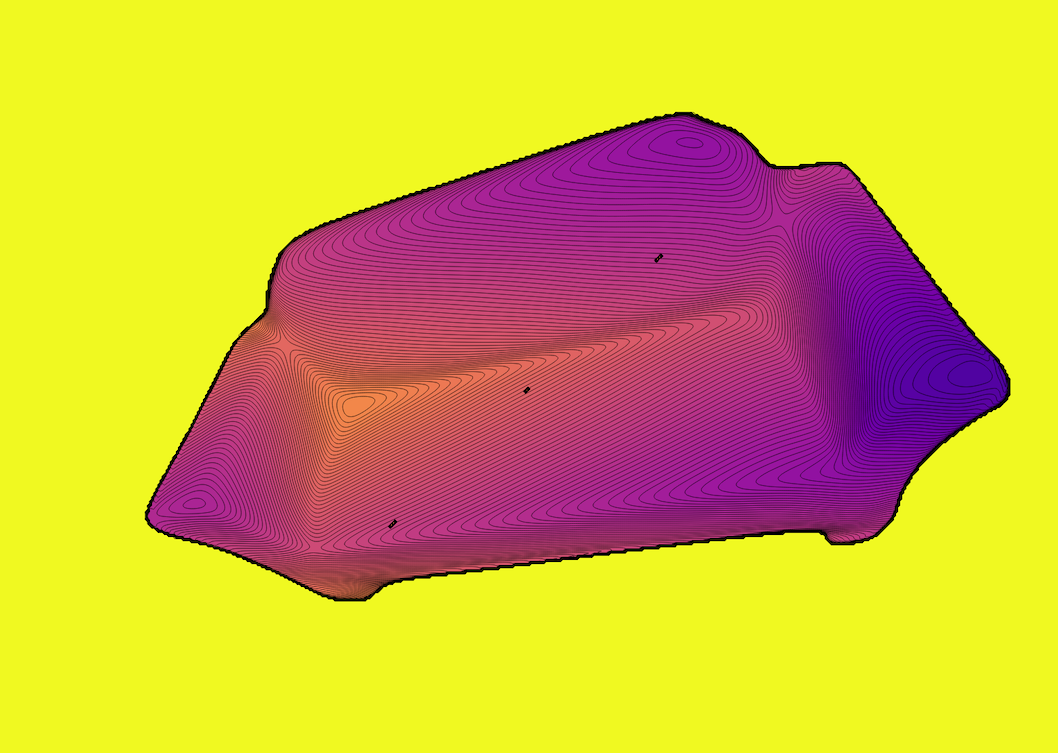}
\includegraphics[width=0.24\linewidth, trim={60mm 20mm 20mm 20mm}, clip]{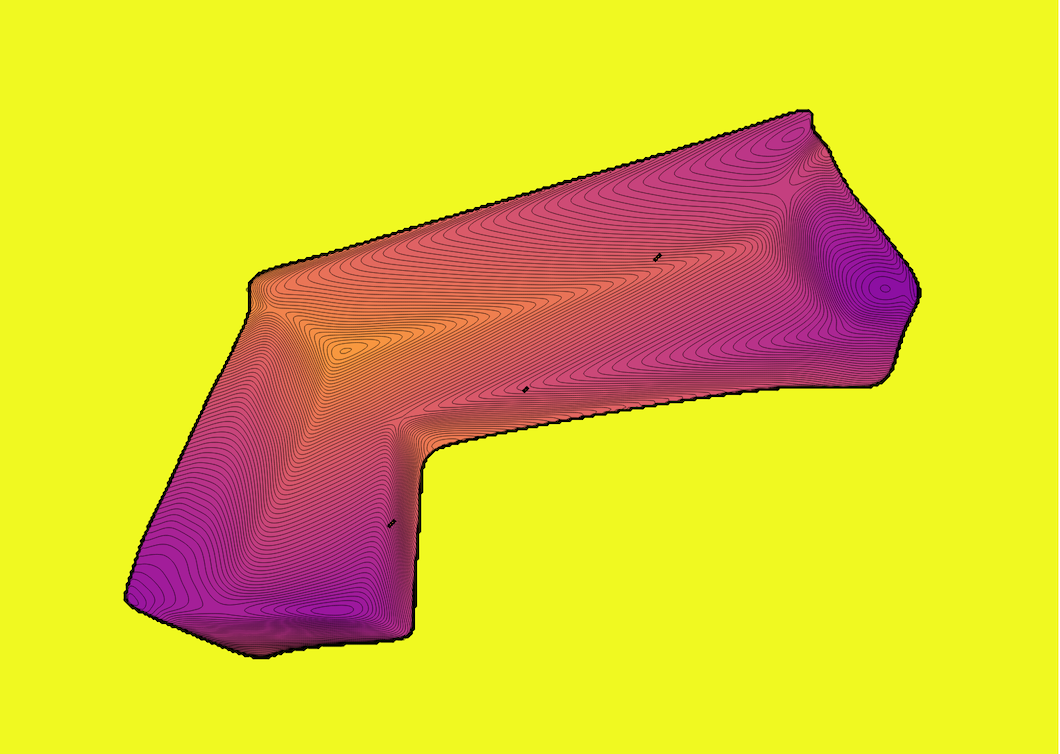}
\includegraphics[width=0.24\linewidth, trim={60mm 20mm 20mm 20mm}, clip]{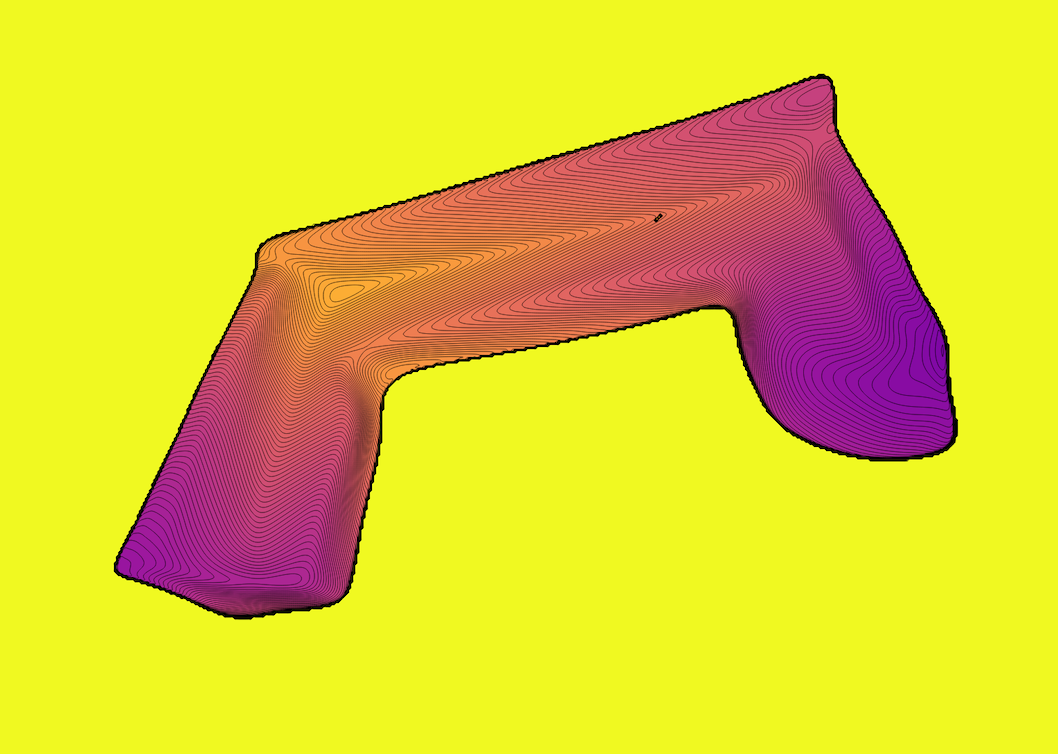}
\includegraphics[width=0.24\linewidth, trim={60mm 20mm 20mm 20mm}, clip]{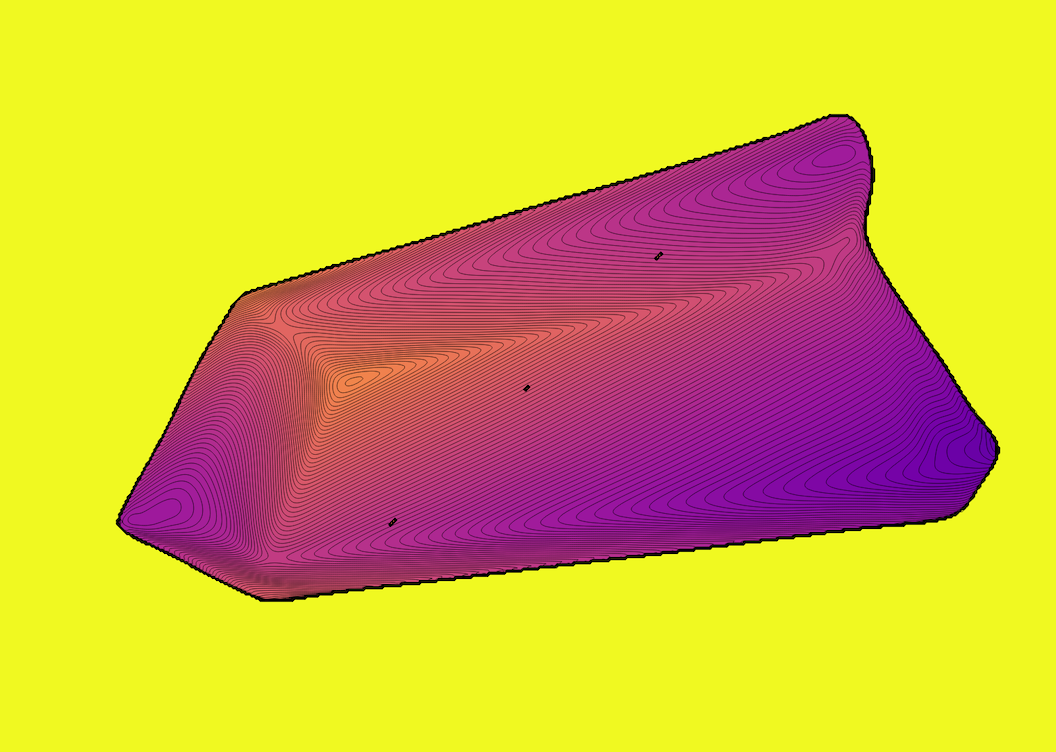}
\includegraphics[width=0.24\linewidth, trim={60mm 20mm 20mm 20mm}, clip]{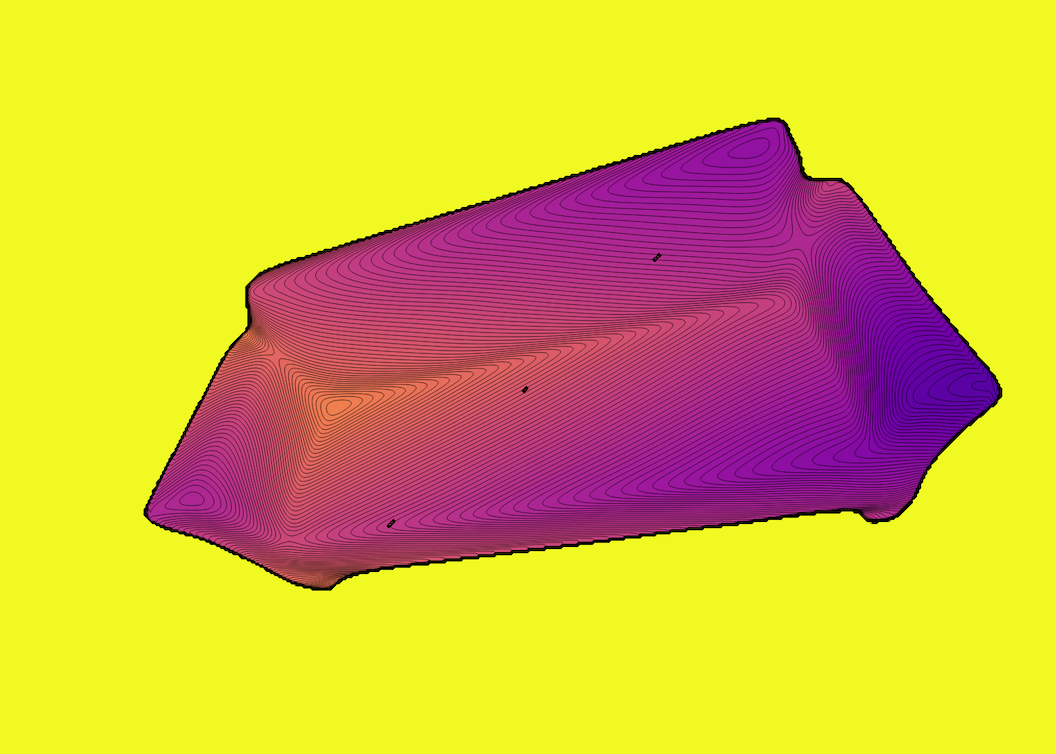}
\includegraphics[width=0.24\linewidth, trim={60mm 20mm 20mm 20mm}, clip]{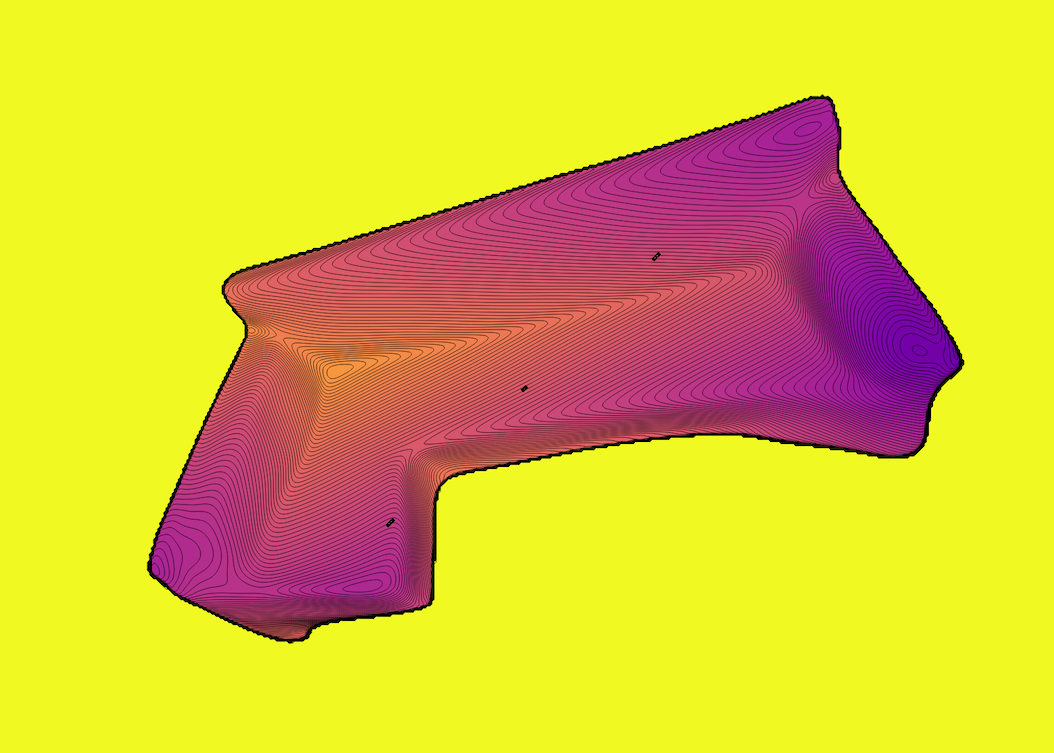}
\includegraphics[width=0.24\linewidth, trim={60mm 20mm 20mm 20mm}, clip]{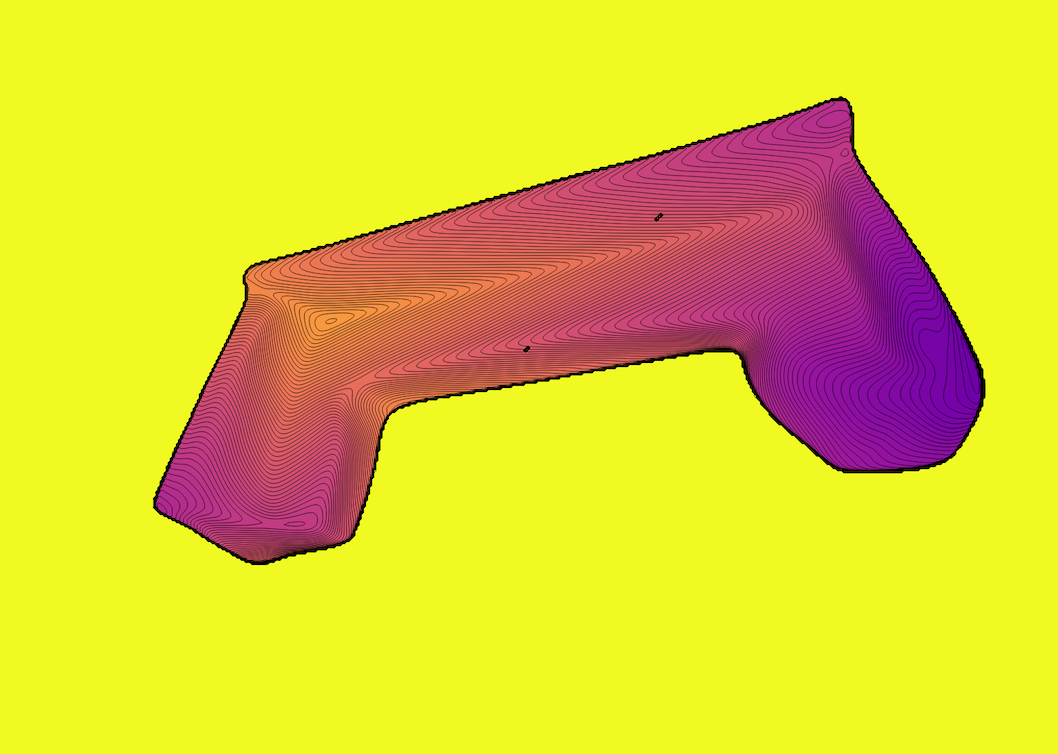}
\includegraphics[width=0.24\linewidth, trim={60mm 20mm 20mm 20mm}, clip]{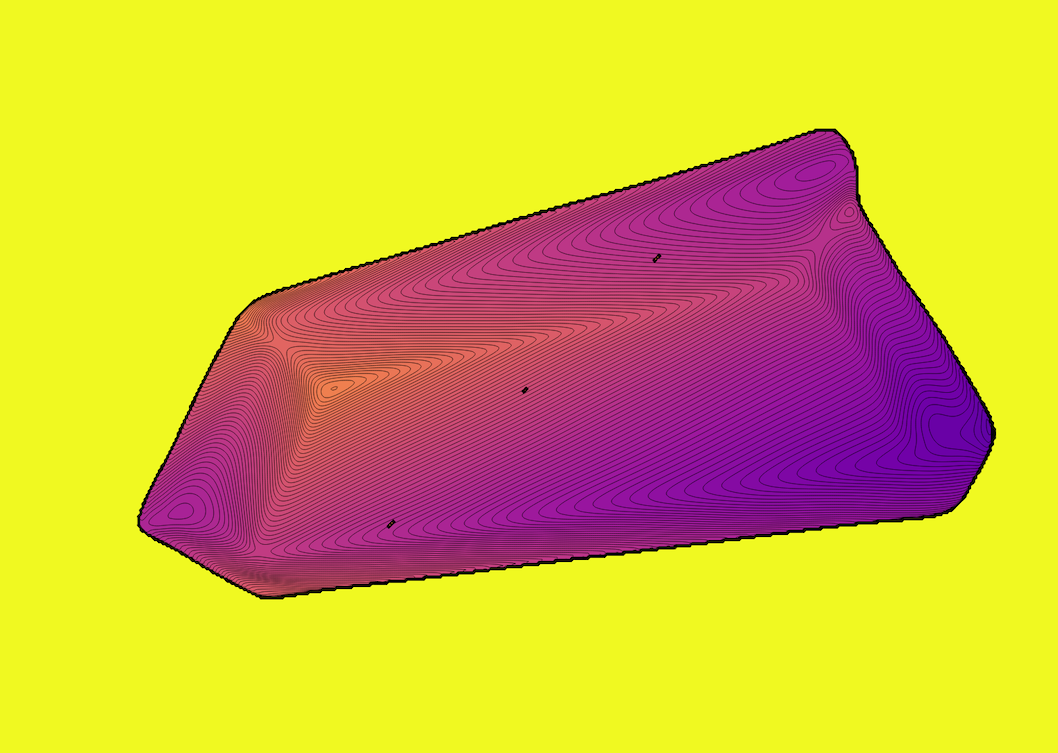}
\includegraphics[width=0.24\linewidth, trim={60mm 20mm 20mm 20mm}, clip]{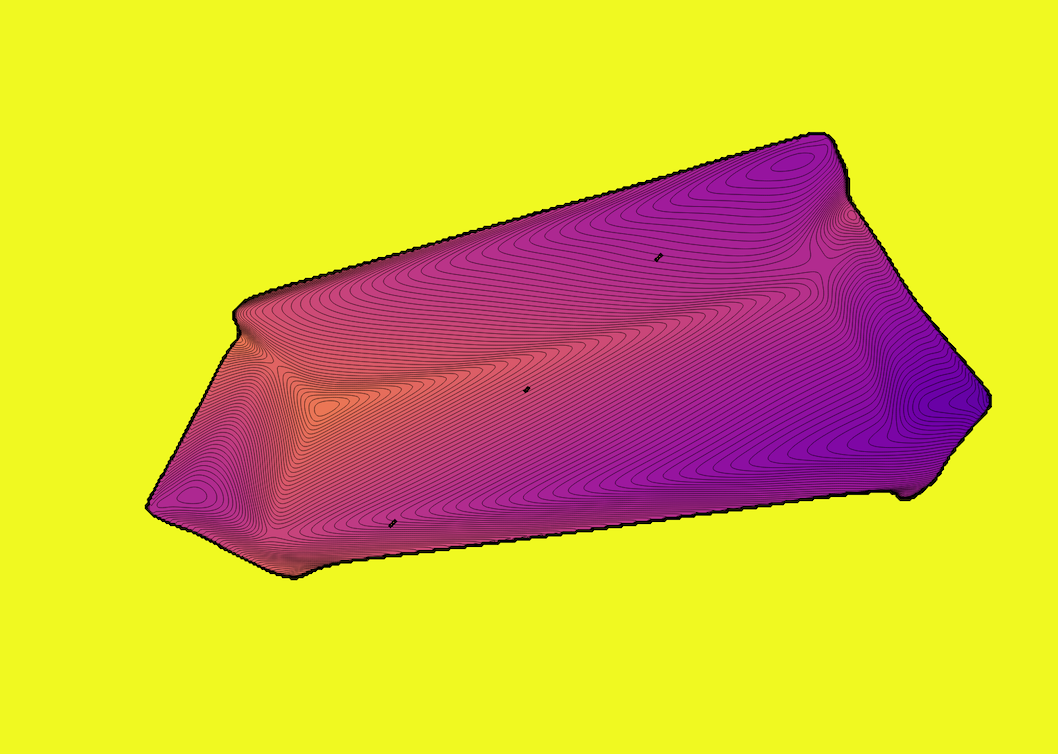}
\includegraphics[width=0.24\linewidth, trim={60mm 20mm 20mm 20mm}, clip]{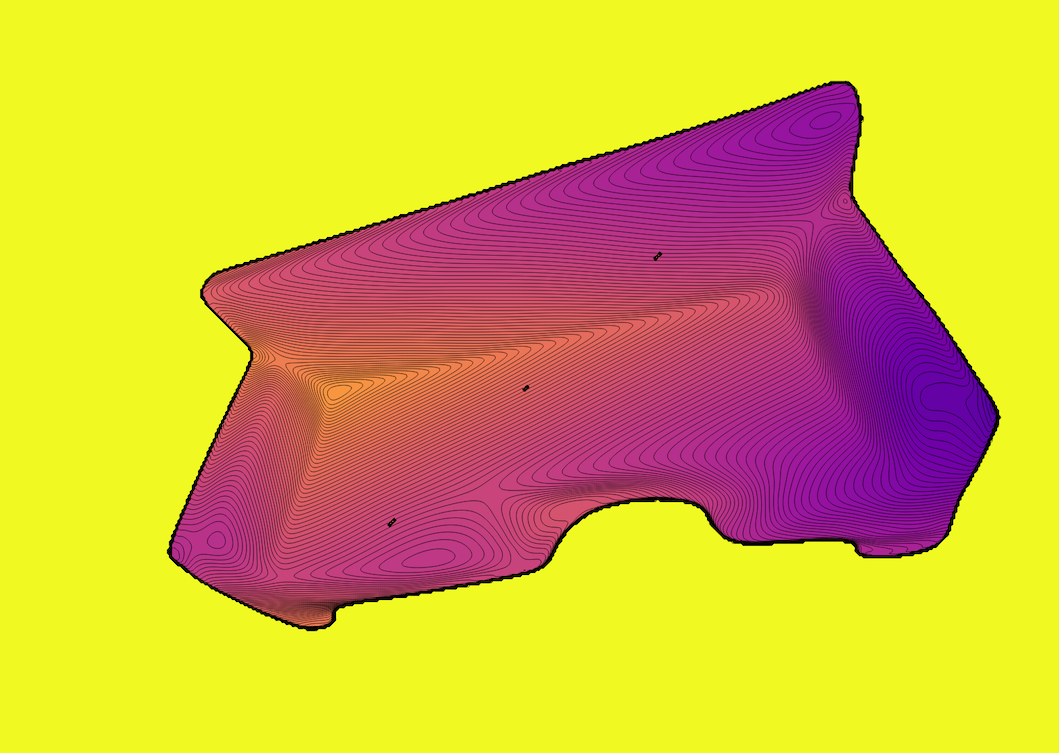}
\includegraphics[width=0.24\linewidth, trim={60mm 20mm 20mm 20mm}, clip]{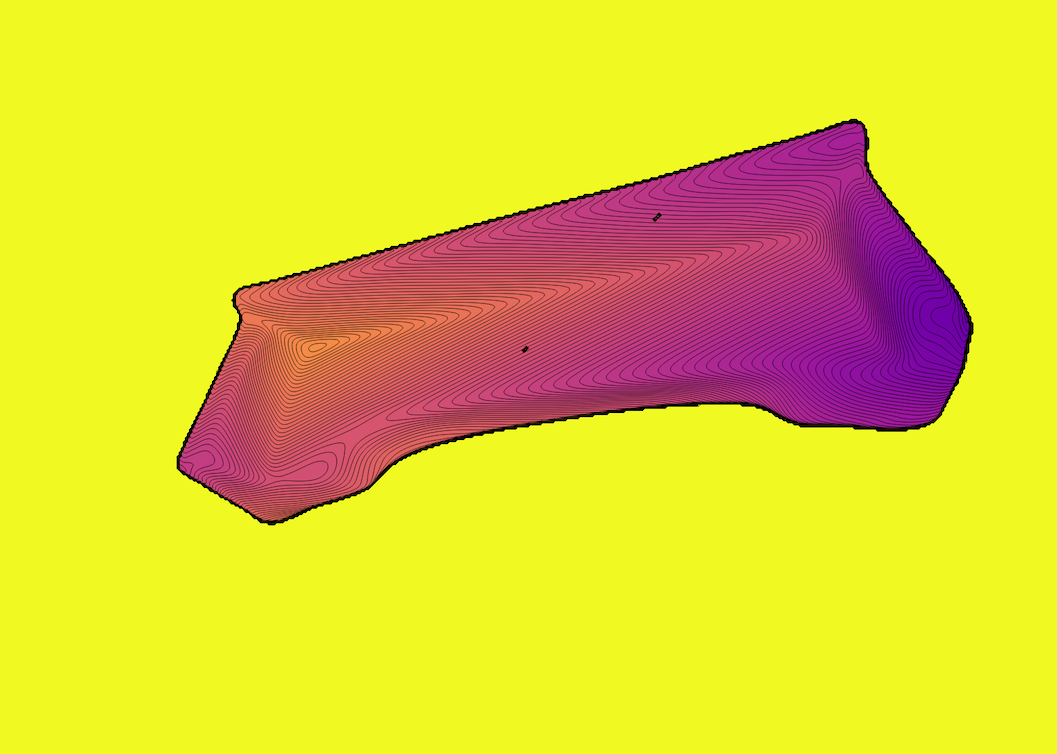}
\includegraphics[width=0.24\linewidth, trim={60mm 20mm 20mm 20mm}, clip]{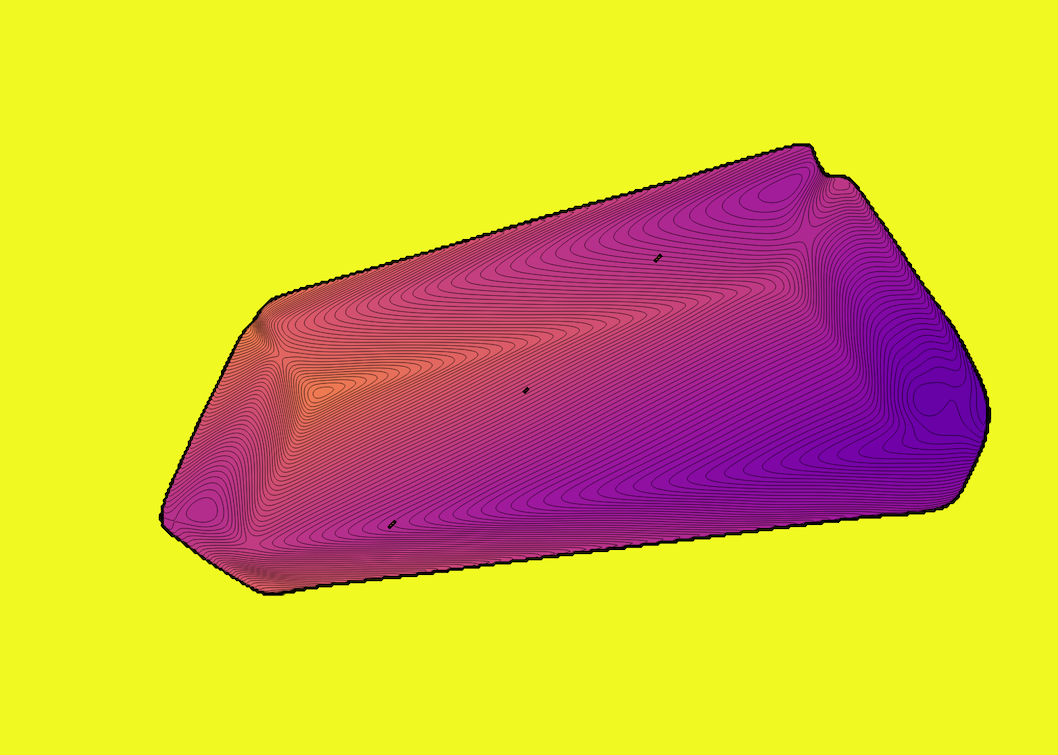}
\includegraphics[width=0.24\linewidth, trim={60mm 20mm 20mm 20mm}, clip]{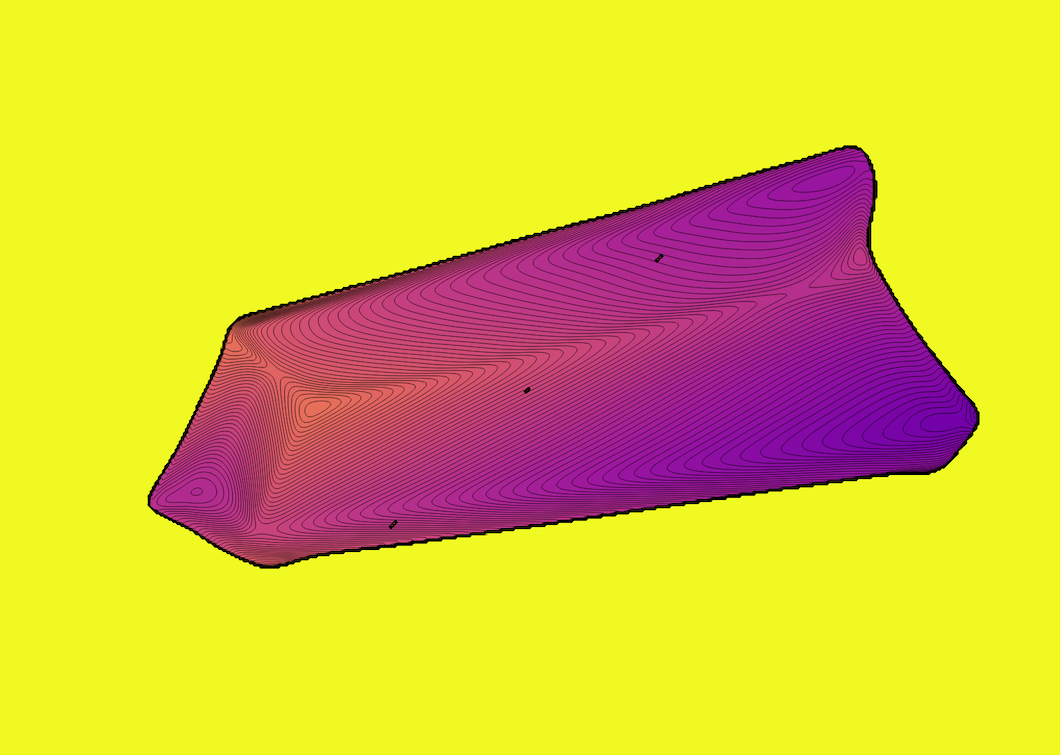}
\includegraphics[width=0.24\linewidth, trim={60mm 20mm 20mm 20mm}, clip]{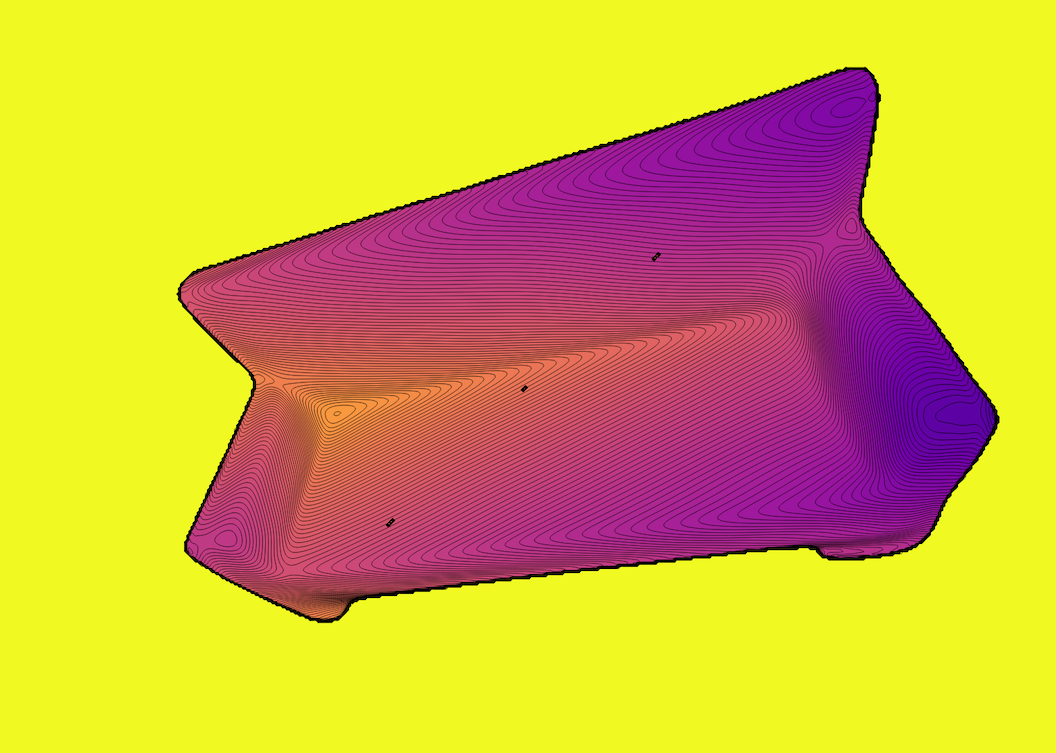}
\includegraphics[width=0.24\linewidth, trim={60mm 20mm 20mm 20mm}, clip]{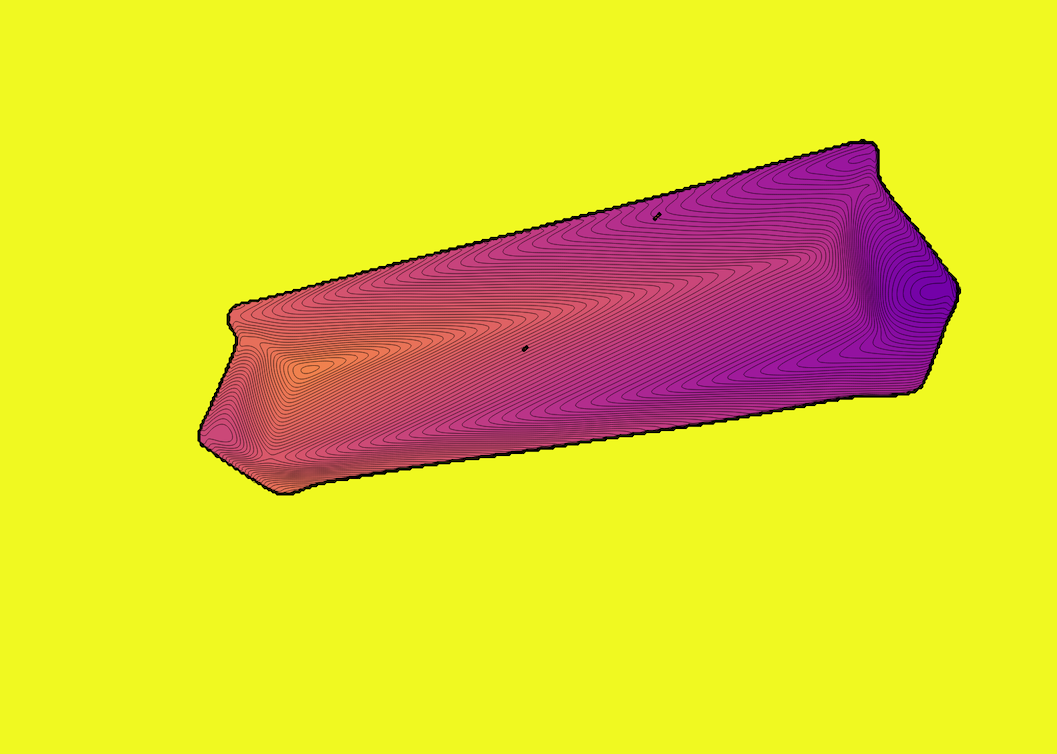}
\includegraphics[width=0.24\linewidth, trim={60mm 20mm 20mm 20mm}, clip]{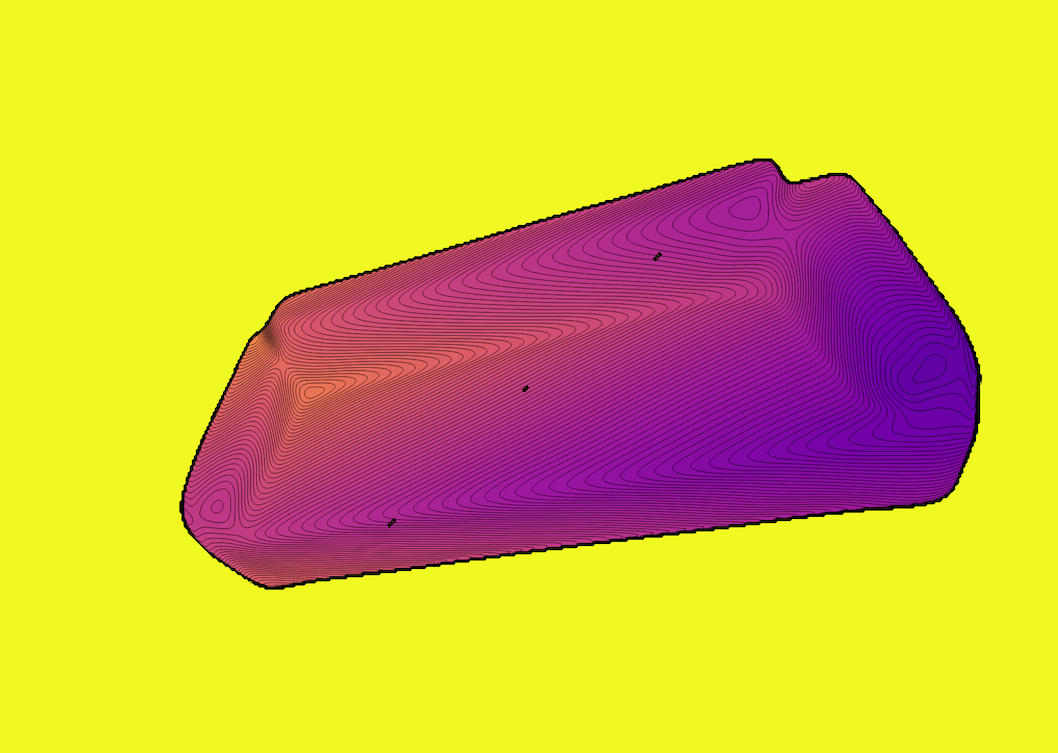}
%
\end{minipage}%
\caption{SDDF shape interpolation between two sofa instances. The first and last row show the SDDF output from the same view for two different instances from the training set. The rows in the middle are generated by using a weighted average of the latent codes of the upper-most and down-most instances as an input to the SDDF network. In each column, from top to bottom, the latent code weights with respect to the upper-most instance are $1$, $0.75$, $0.5$, $0.25$, $0$, respectively. Note how the shapes transform smoothly from top to bottom with intermediate shapes looking like valid sofas. This demonstrates that the SDDF model represents the latent shape space continuously and meaningfully.}
\label{fig:sofainterpolation}
\end{figure*}
\begin{figure*}[h!]
    \centering
\begin{minipage}{\linewidth}
  \centering
\includegraphics[width=0.24\linewidth, trim={75mm 20mm 75mm 0mm}, clip]{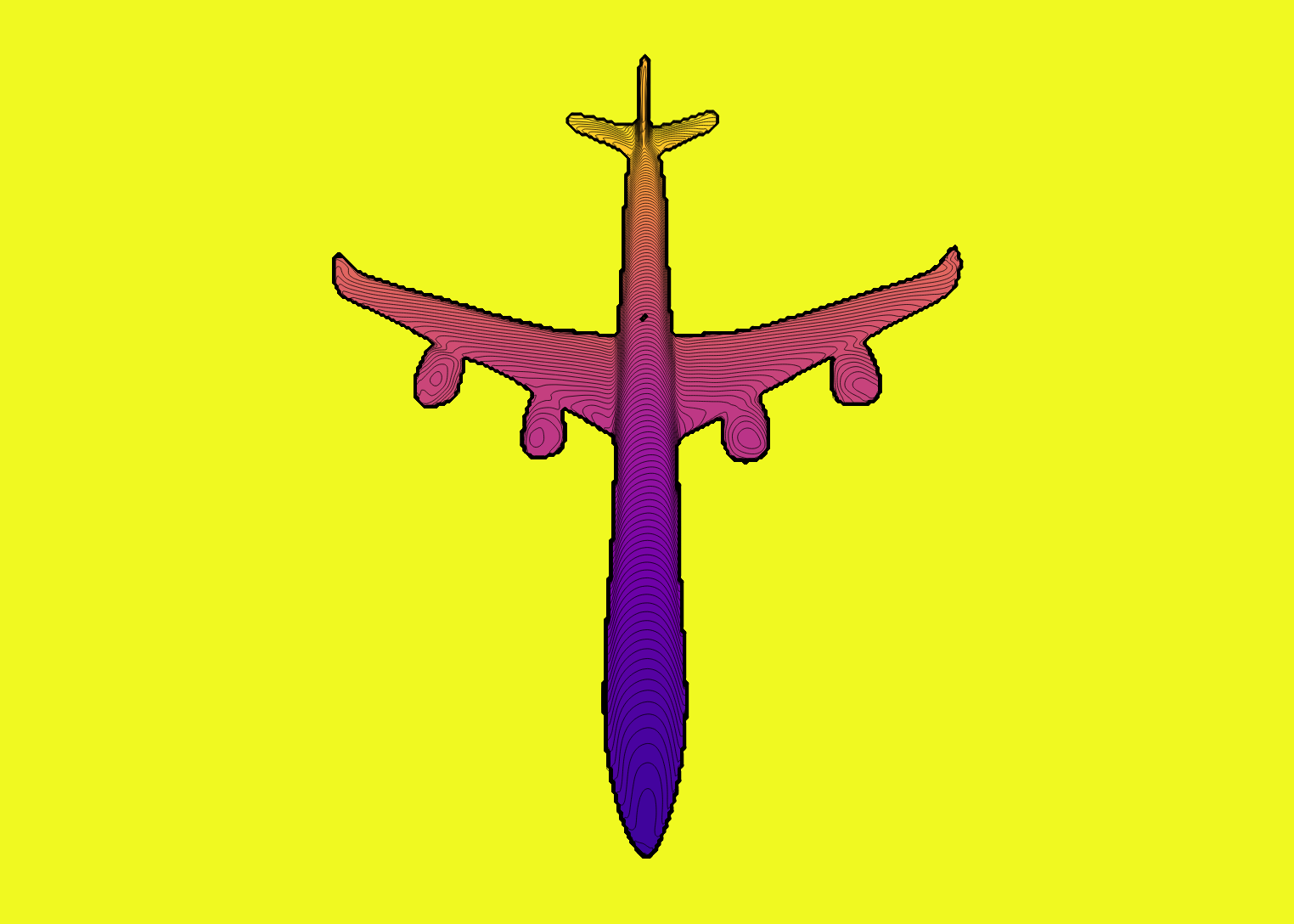}
\includegraphics[width=0.24\linewidth, trim={75mm 20mm 75mm 0mm}, clip]{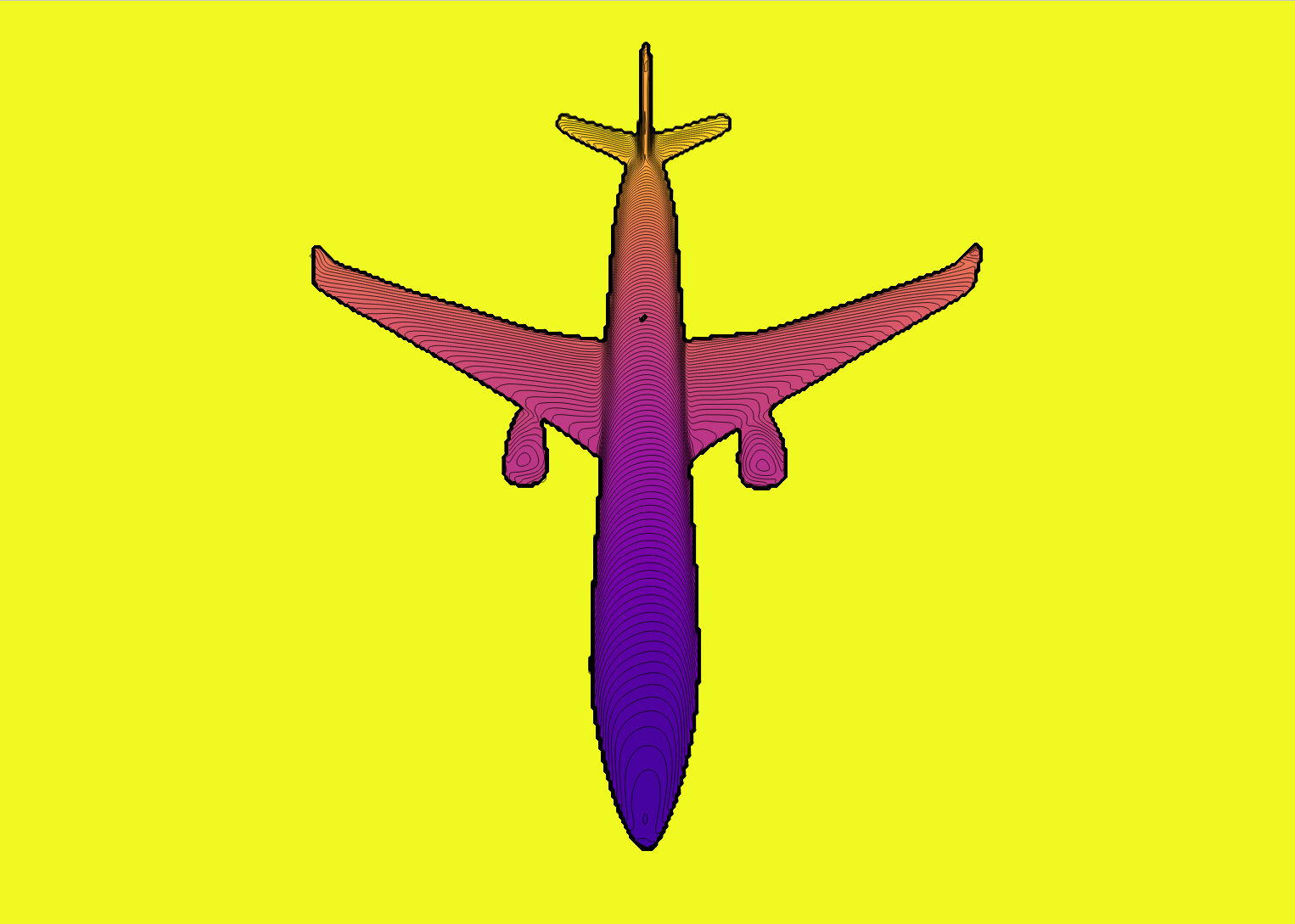}
\includegraphics[width=0.24\linewidth, trim={75mm 20mm 75mm 0mm}, clip]{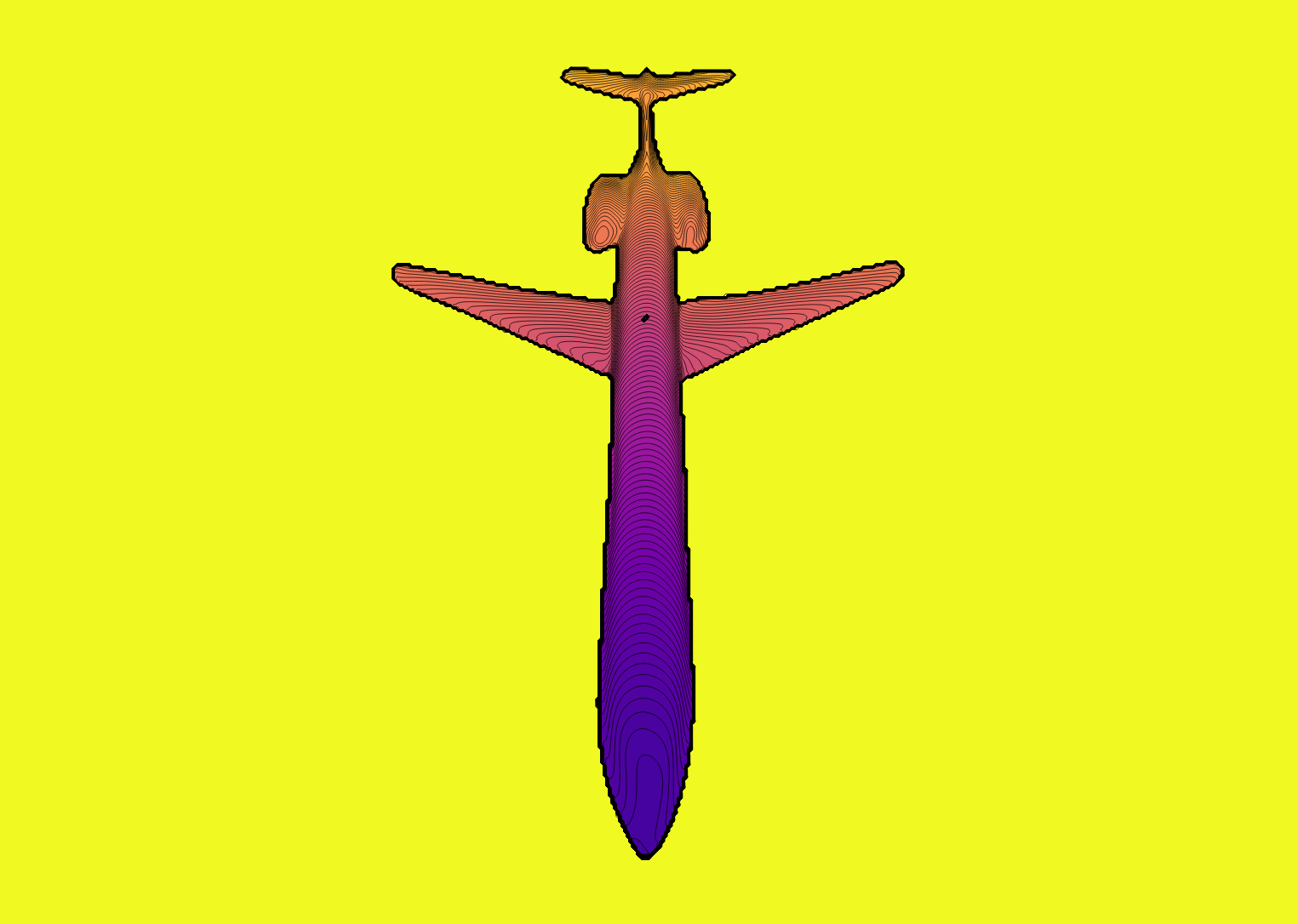}
\includegraphics[width=0.24\linewidth, trim={75mm 20mm 75mm 0mm}, clip]{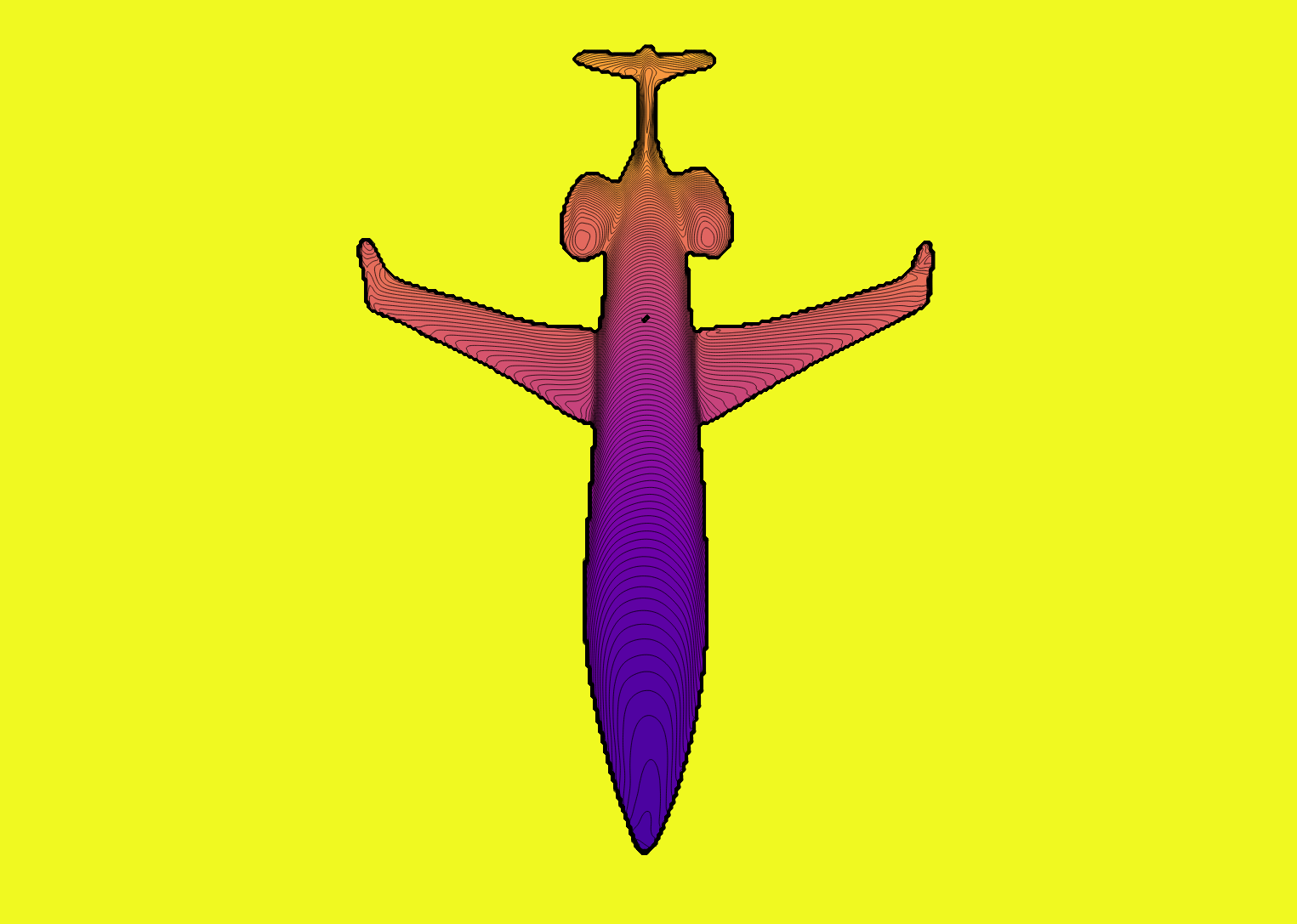}
\includegraphics[width=0.24\linewidth, trim={75mm 20mm 75mm 0mm}, clip]{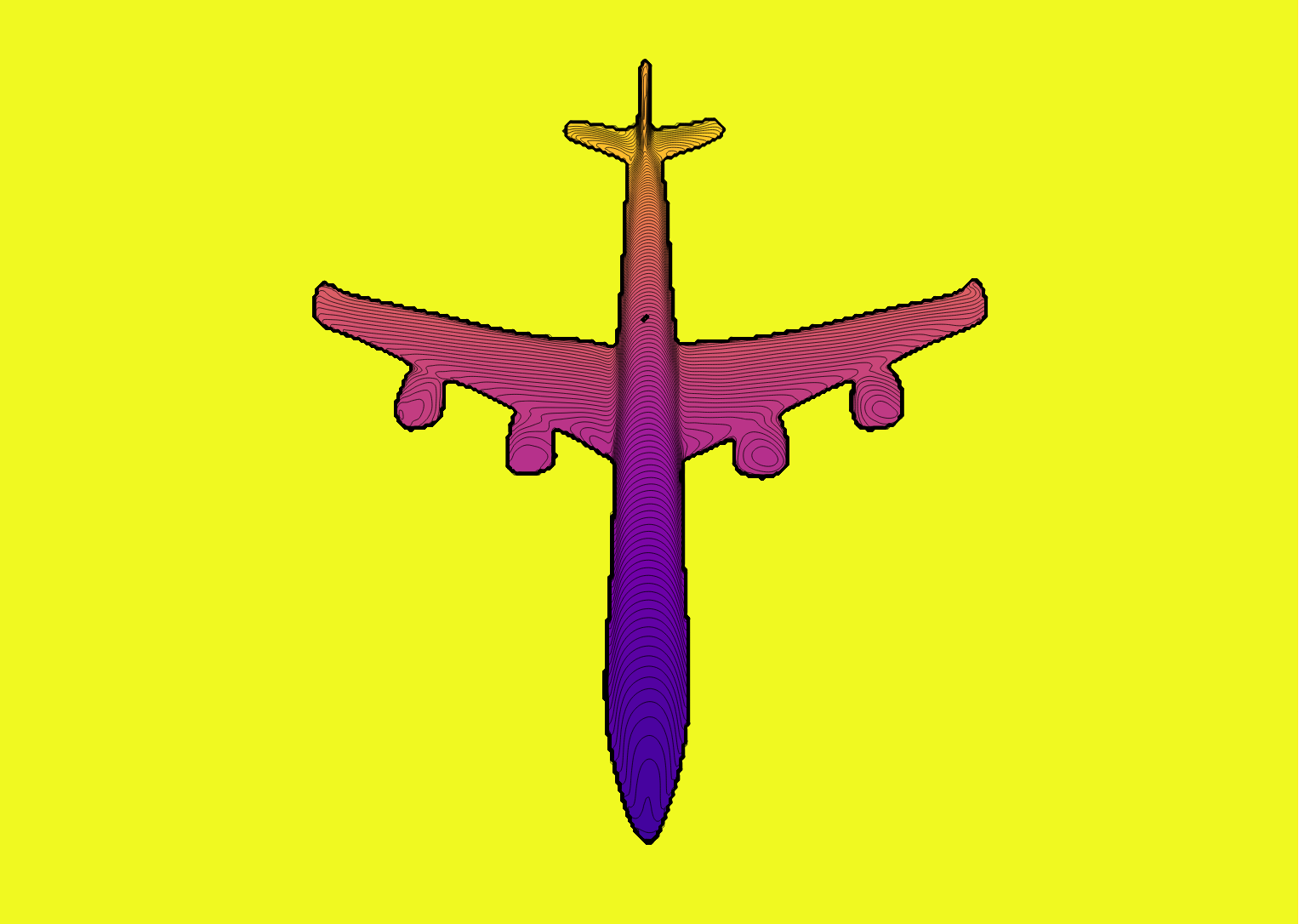}
\includegraphics[width=0.24\linewidth, trim={75mm 20mm 75mm 0mm}, clip]{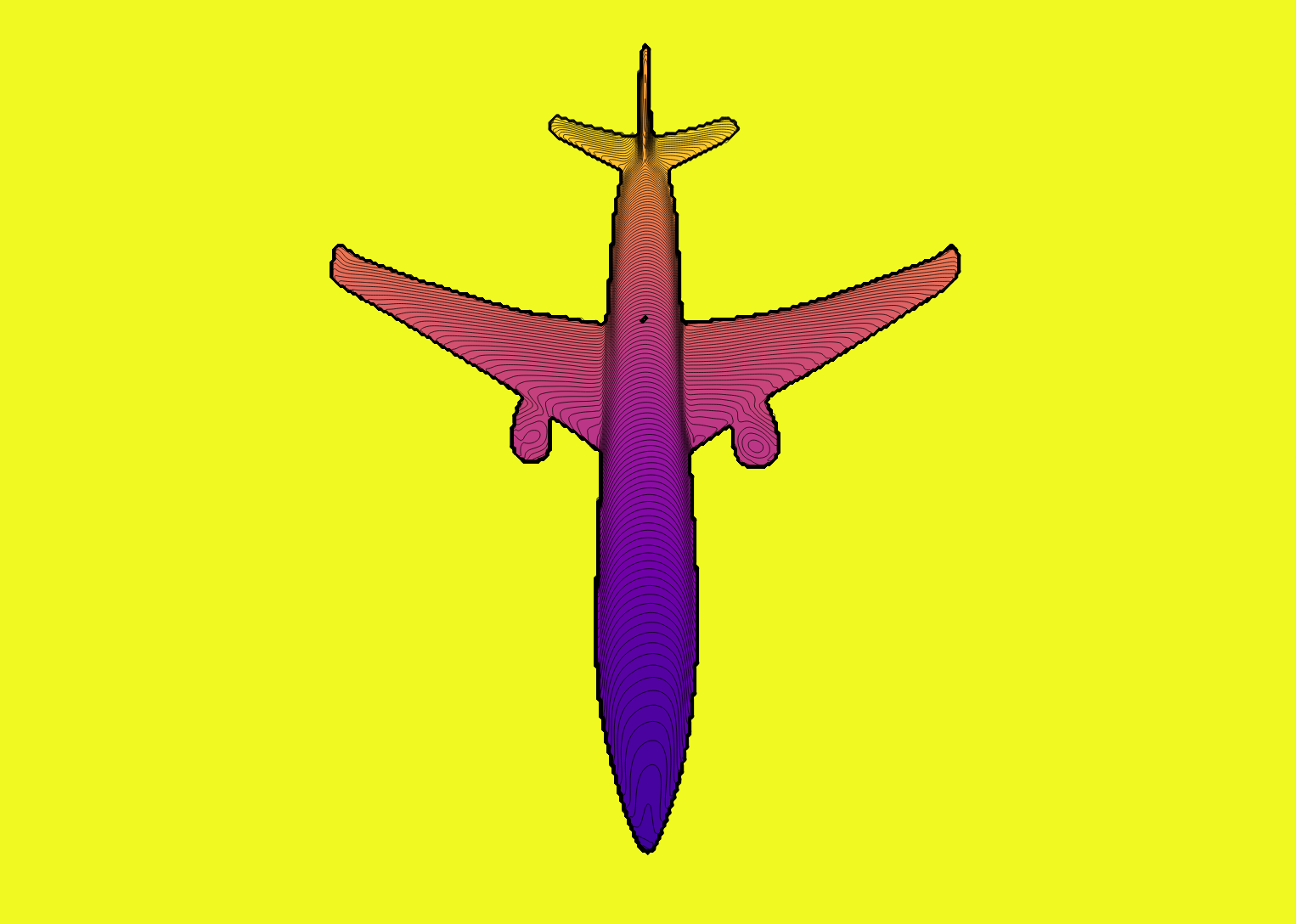}
\includegraphics[width=0.24\linewidth, trim={75mm 20mm 75mm 0mm}, clip]{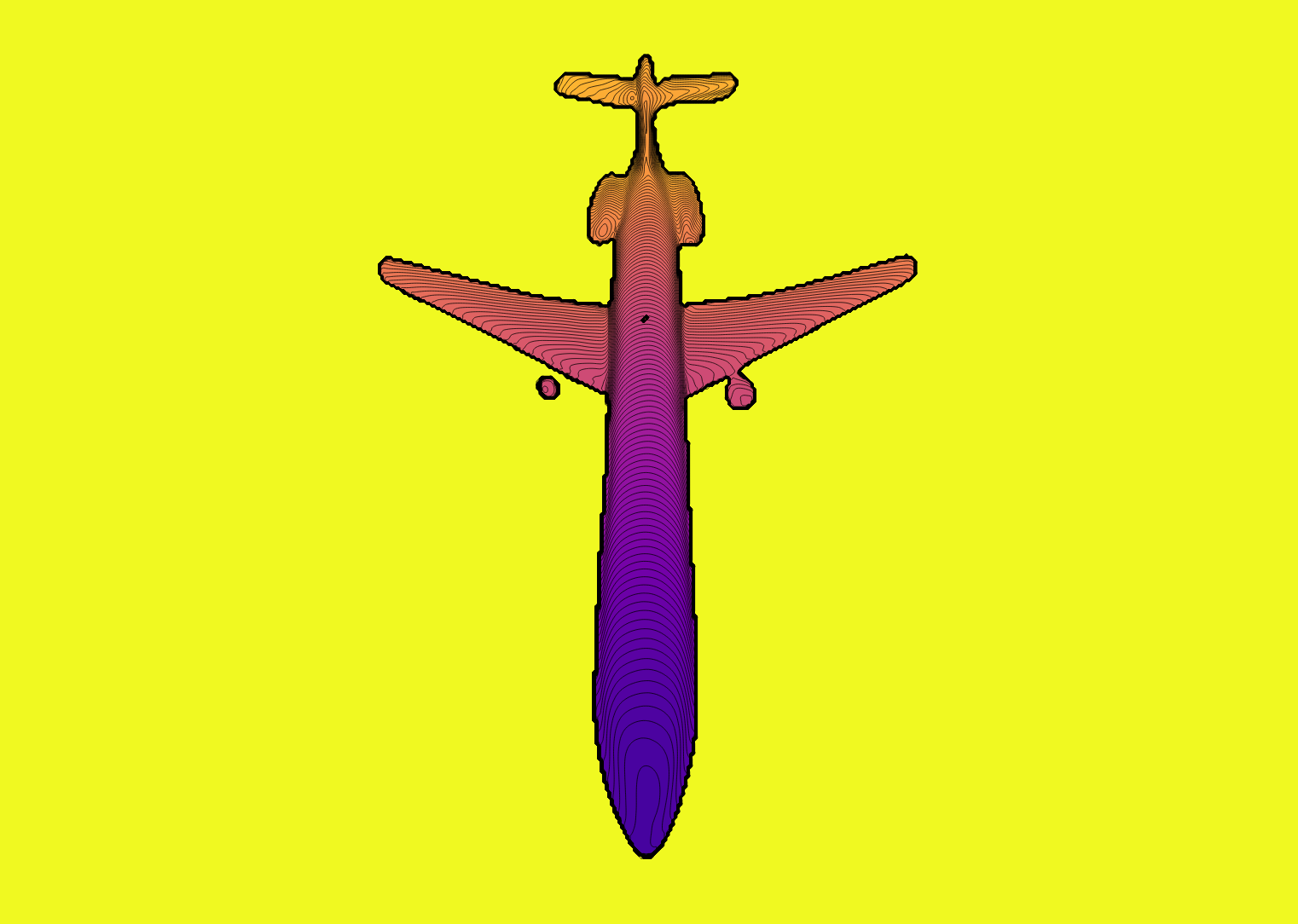}
\includegraphics[width=0.24\linewidth, trim={75mm 20mm 75mm 0mm}, clip]{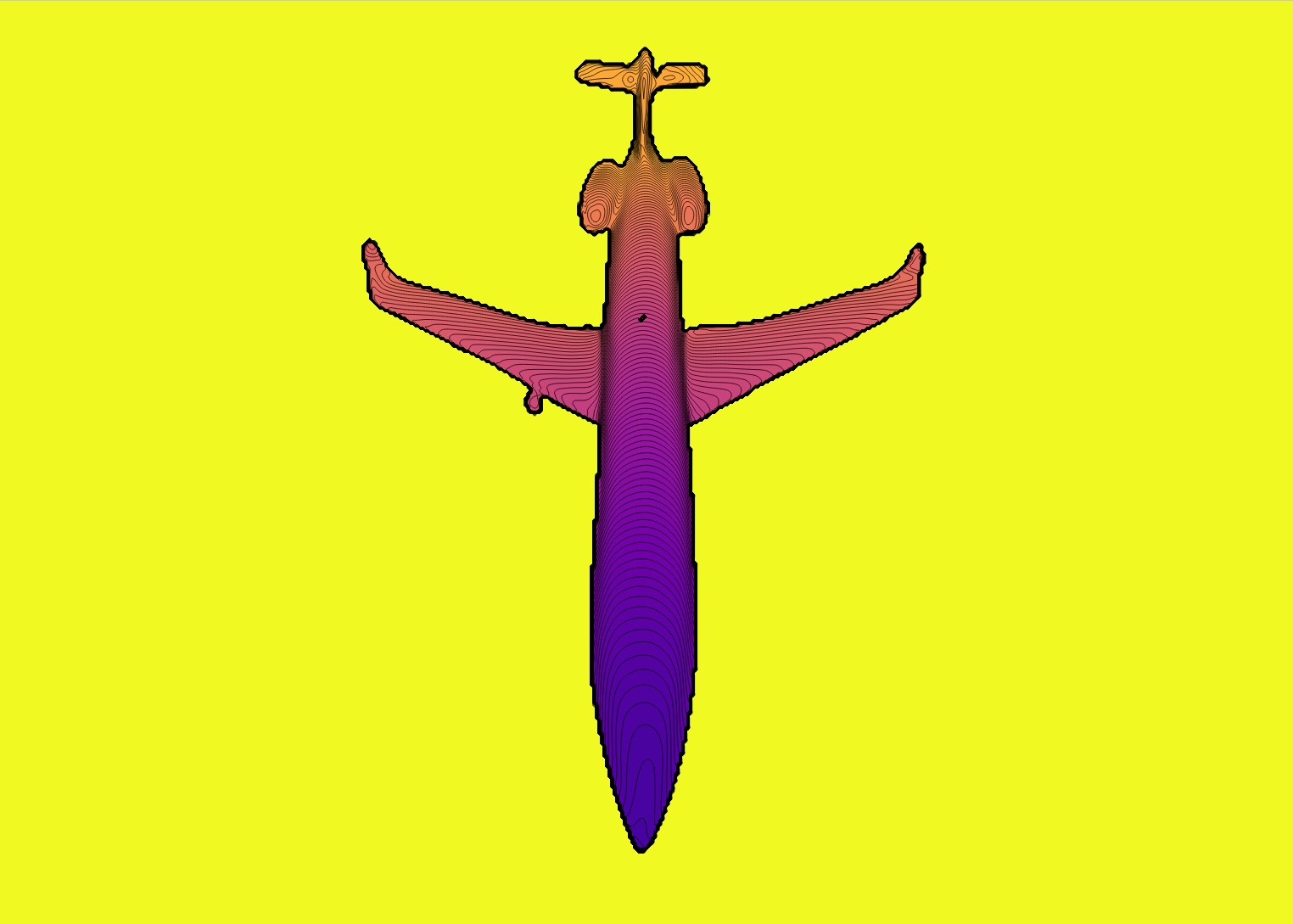}
\includegraphics[width=0.24\linewidth, trim={75mm 20mm 75mm 0mm}, clip]{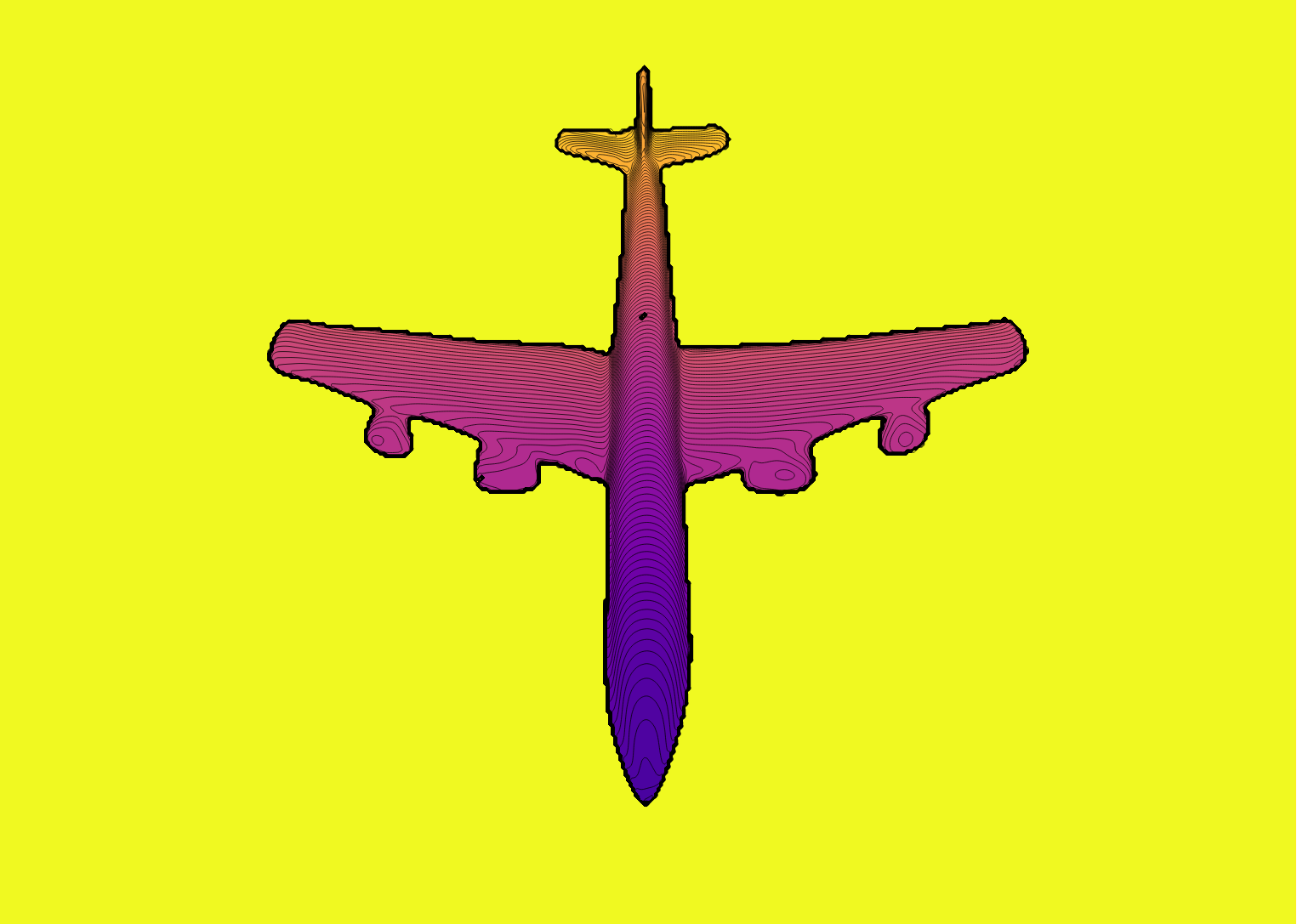}
\includegraphics[width=0.24\linewidth, trim={75mm 20mm 75mm 0mm}, clip]{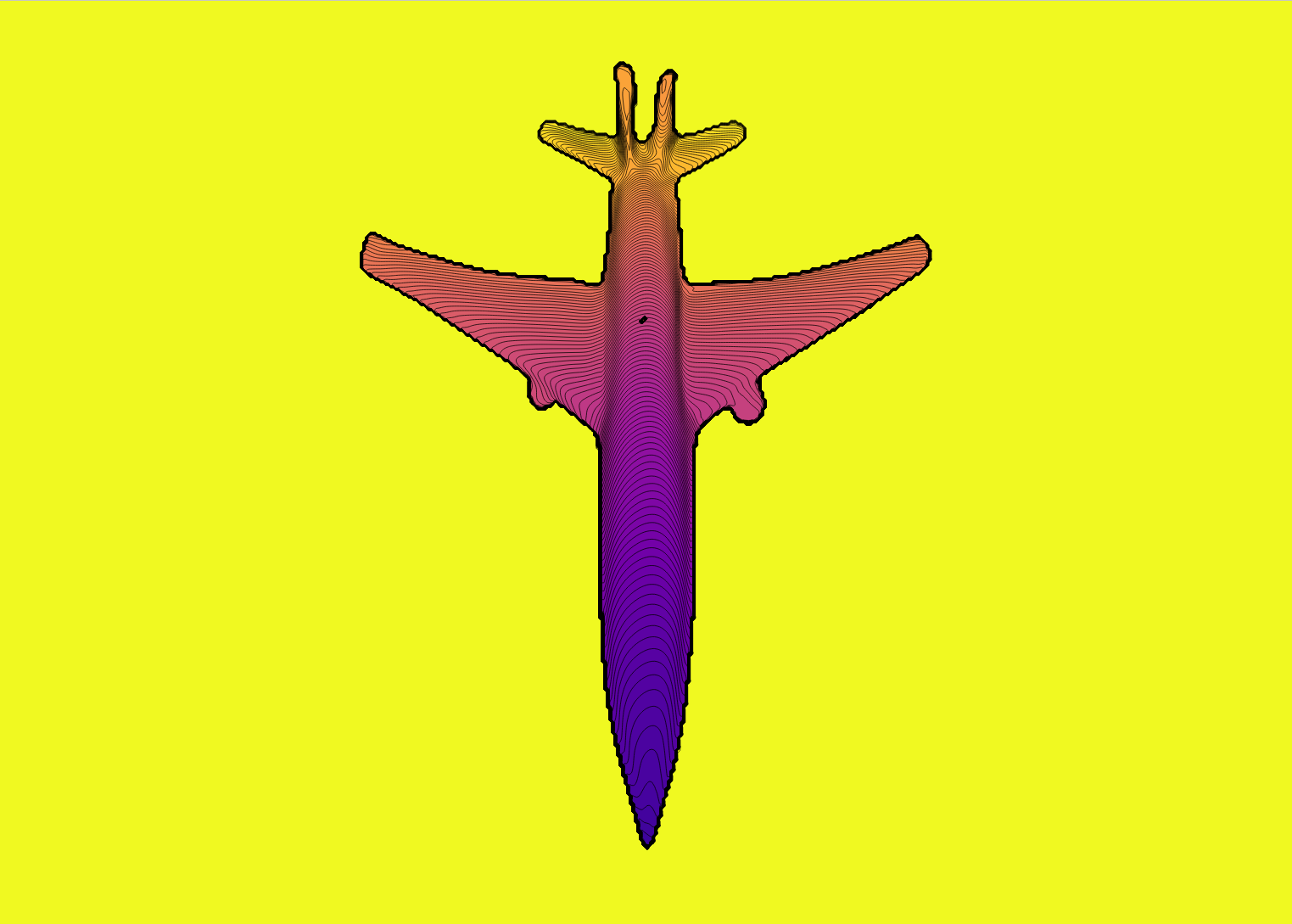}
\includegraphics[width=0.24\linewidth, trim={75mm 20mm 75mm 0mm}, clip]{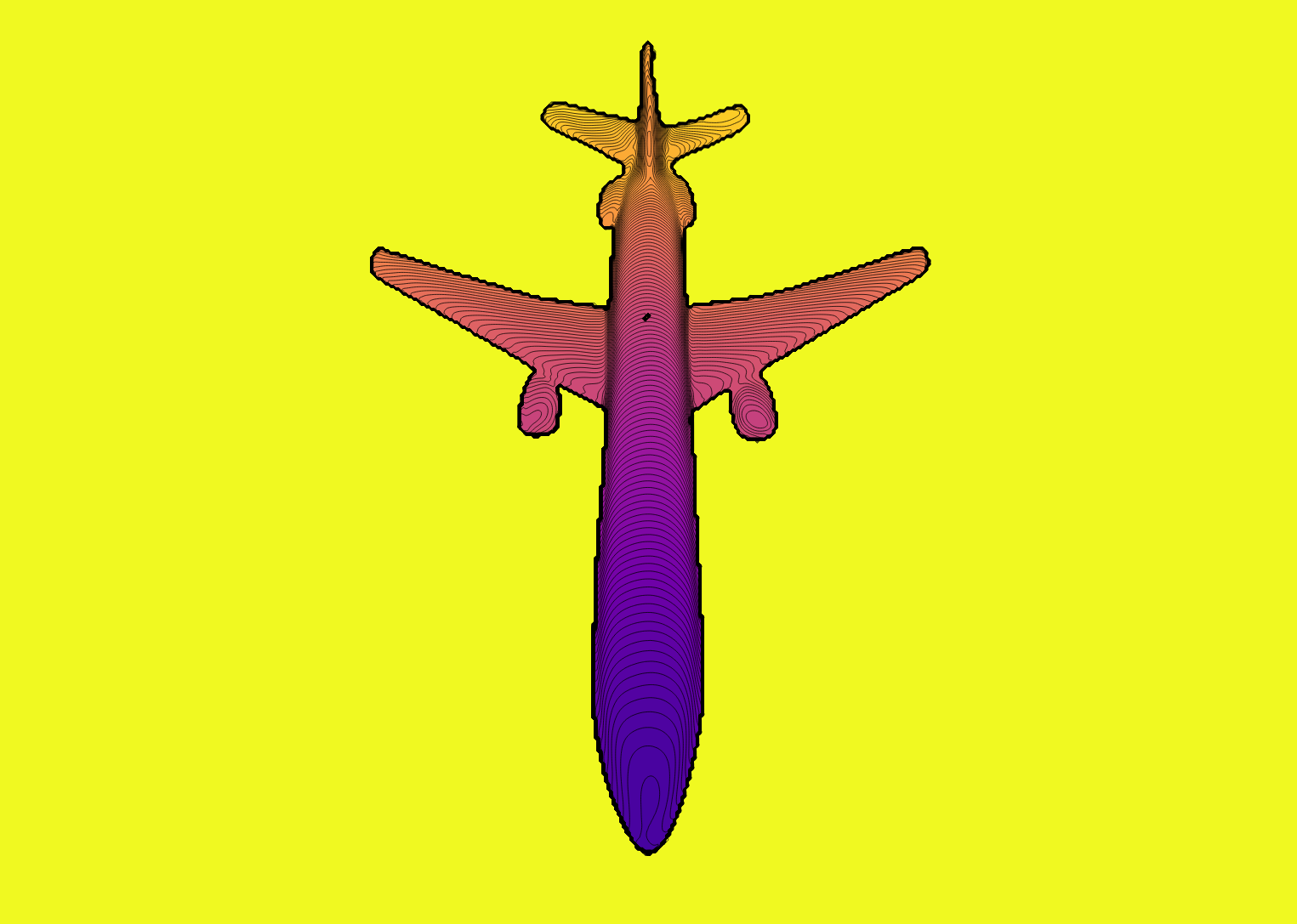}
\includegraphics[width=0.24\linewidth, trim={75mm 20mm 75mm 0mm}, clip]{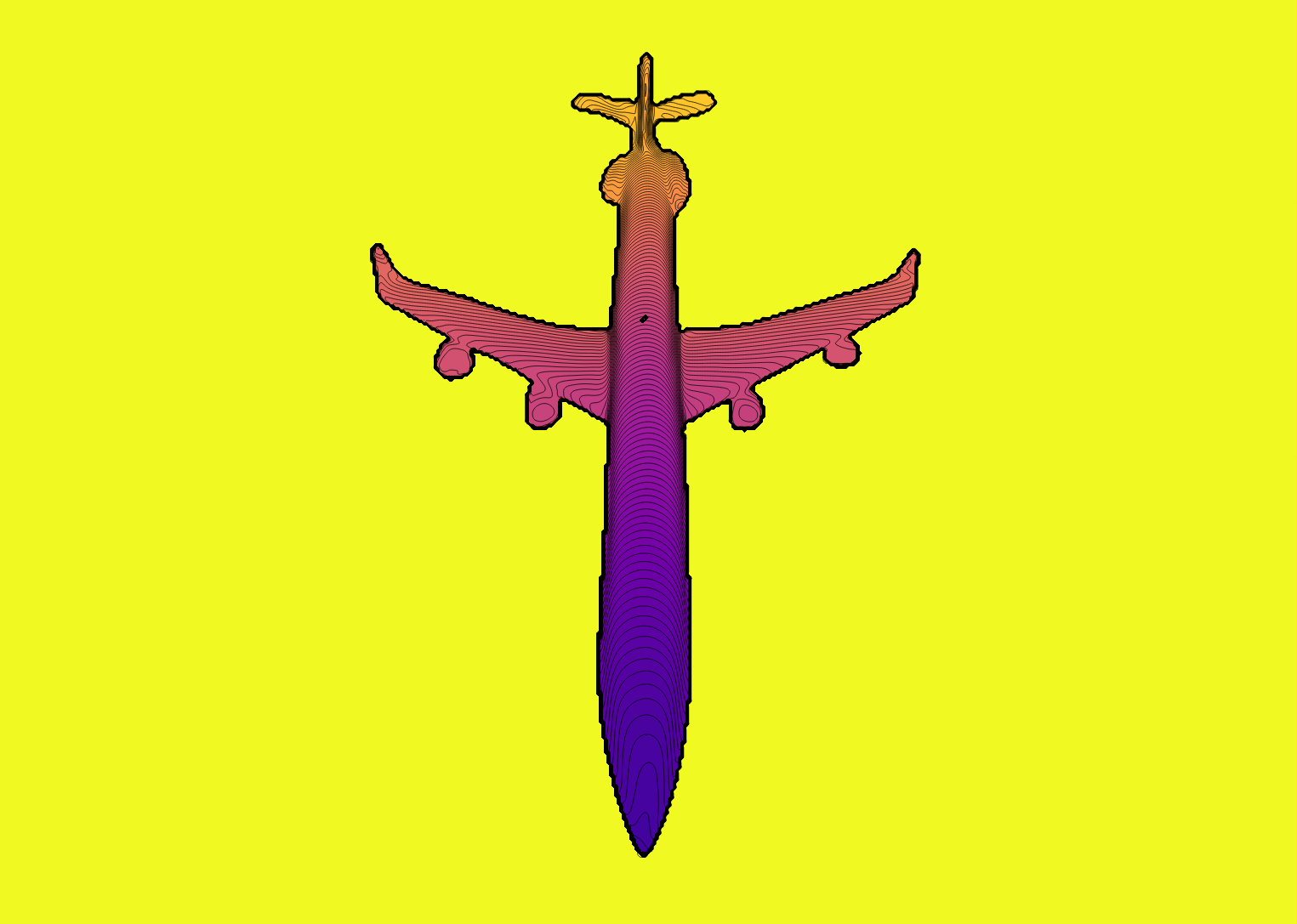}
\includegraphics[width=0.24\linewidth, trim={75mm 20mm 75mm 0mm}, clip]{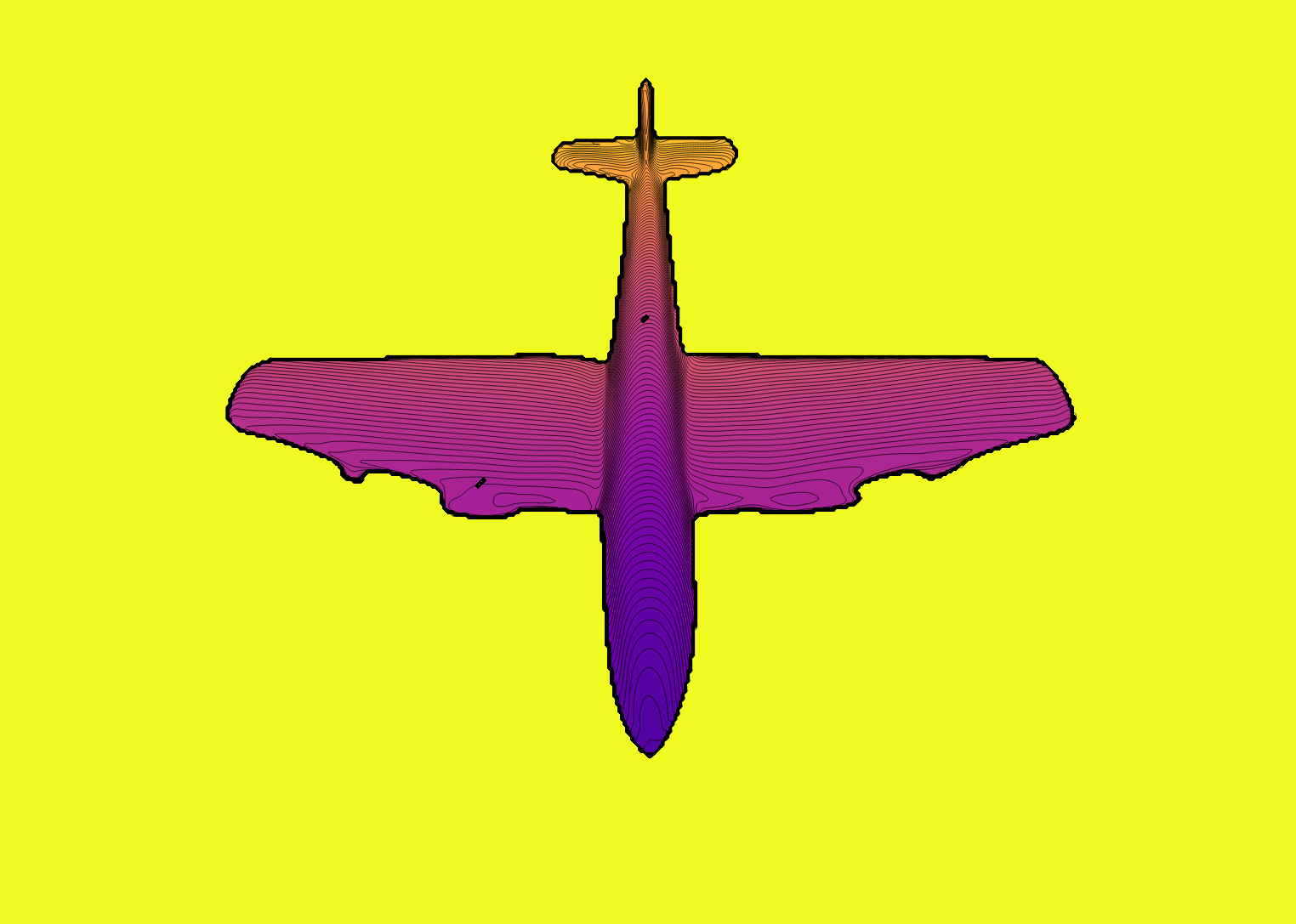}
\includegraphics[width=0.24\linewidth, trim={75mm 20mm 75mm 0mm}, clip]{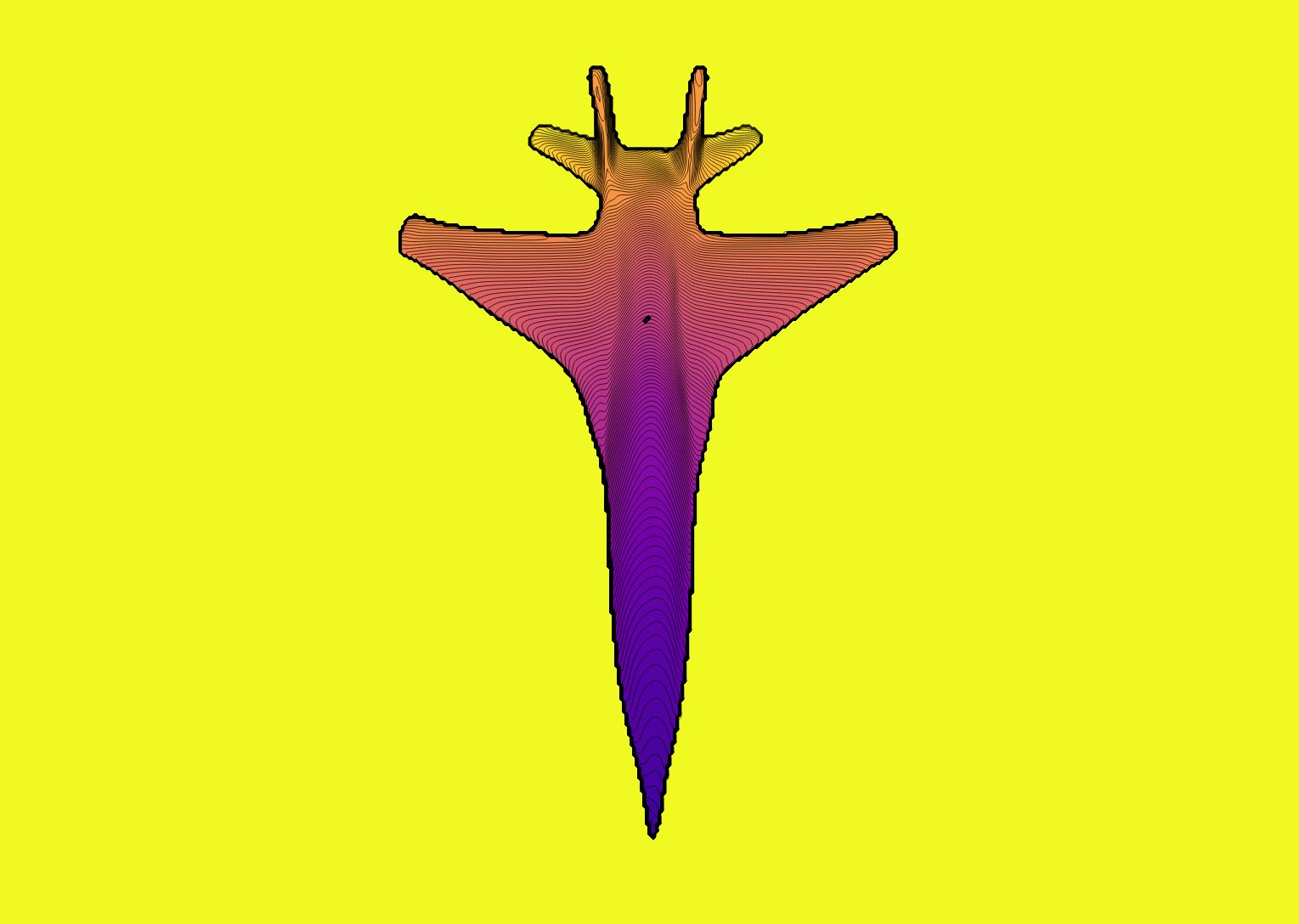}
\includegraphics[width=0.24\linewidth, trim={75mm 20mm 75mm 0mm}, clip]{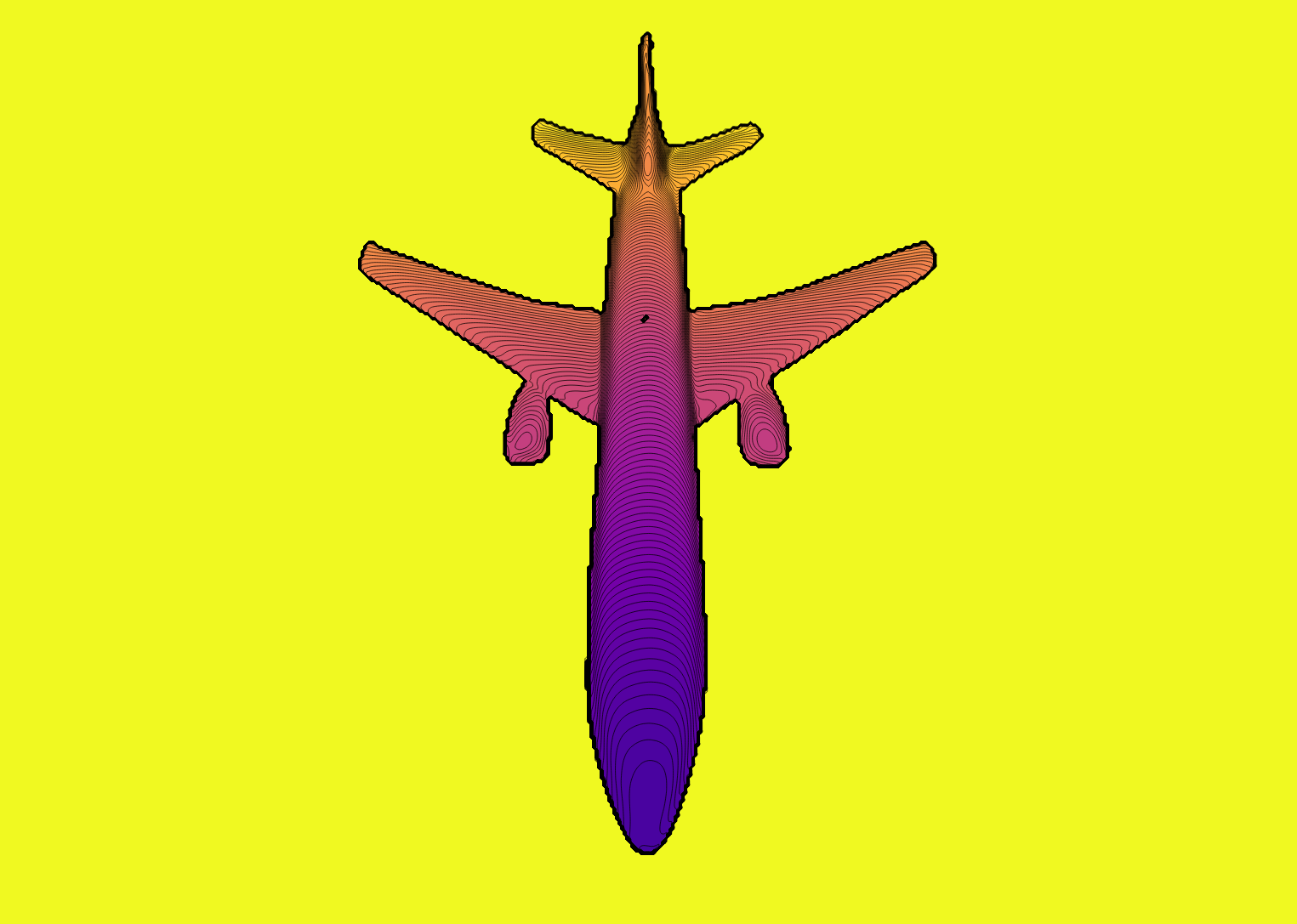}
\includegraphics[width=0.24\linewidth, trim={75mm 20mm 75mm 0mm}, clip]{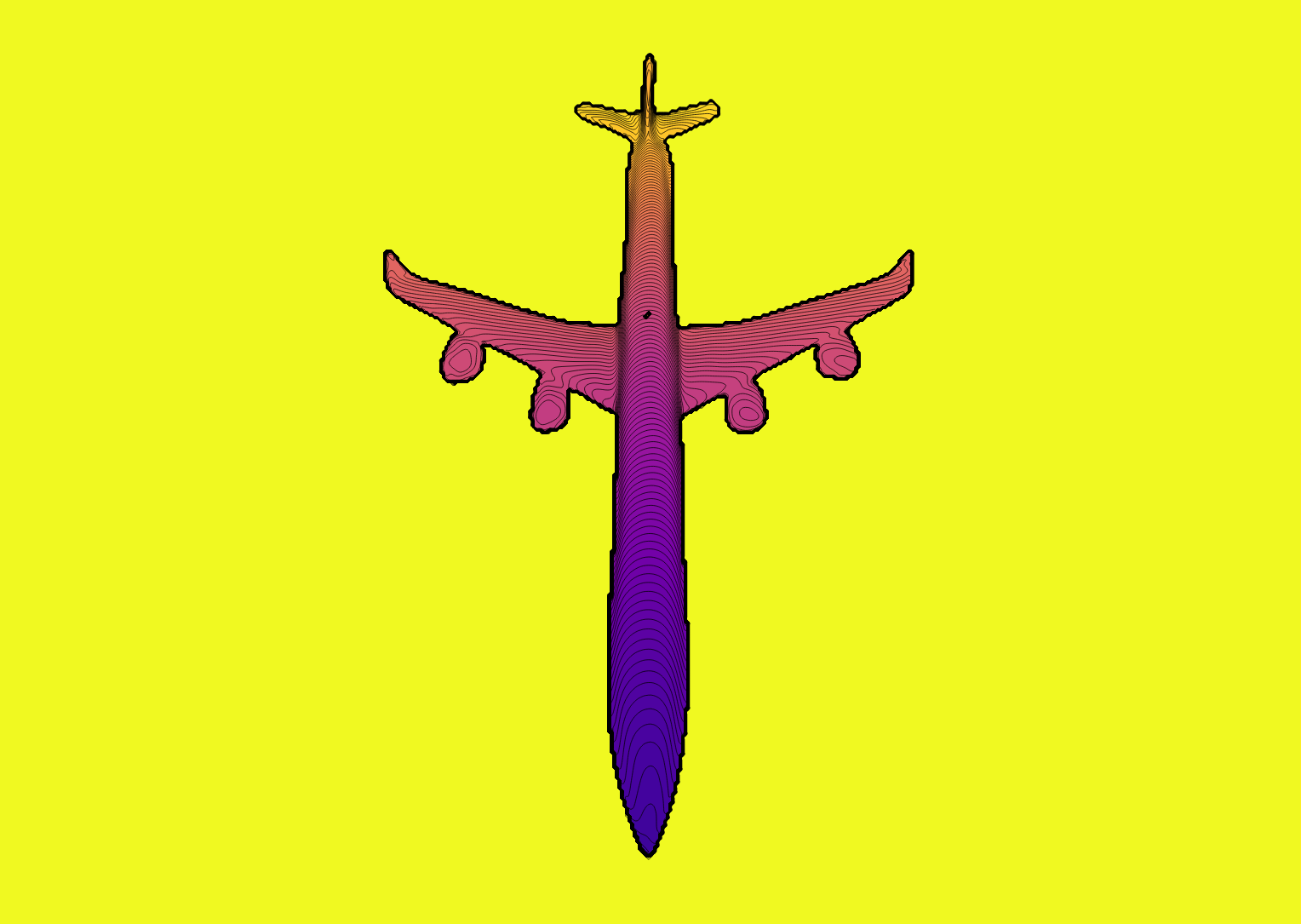}
\includegraphics[width=0.24\linewidth, trim={75mm 20mm 75mm 0mm}, clip]{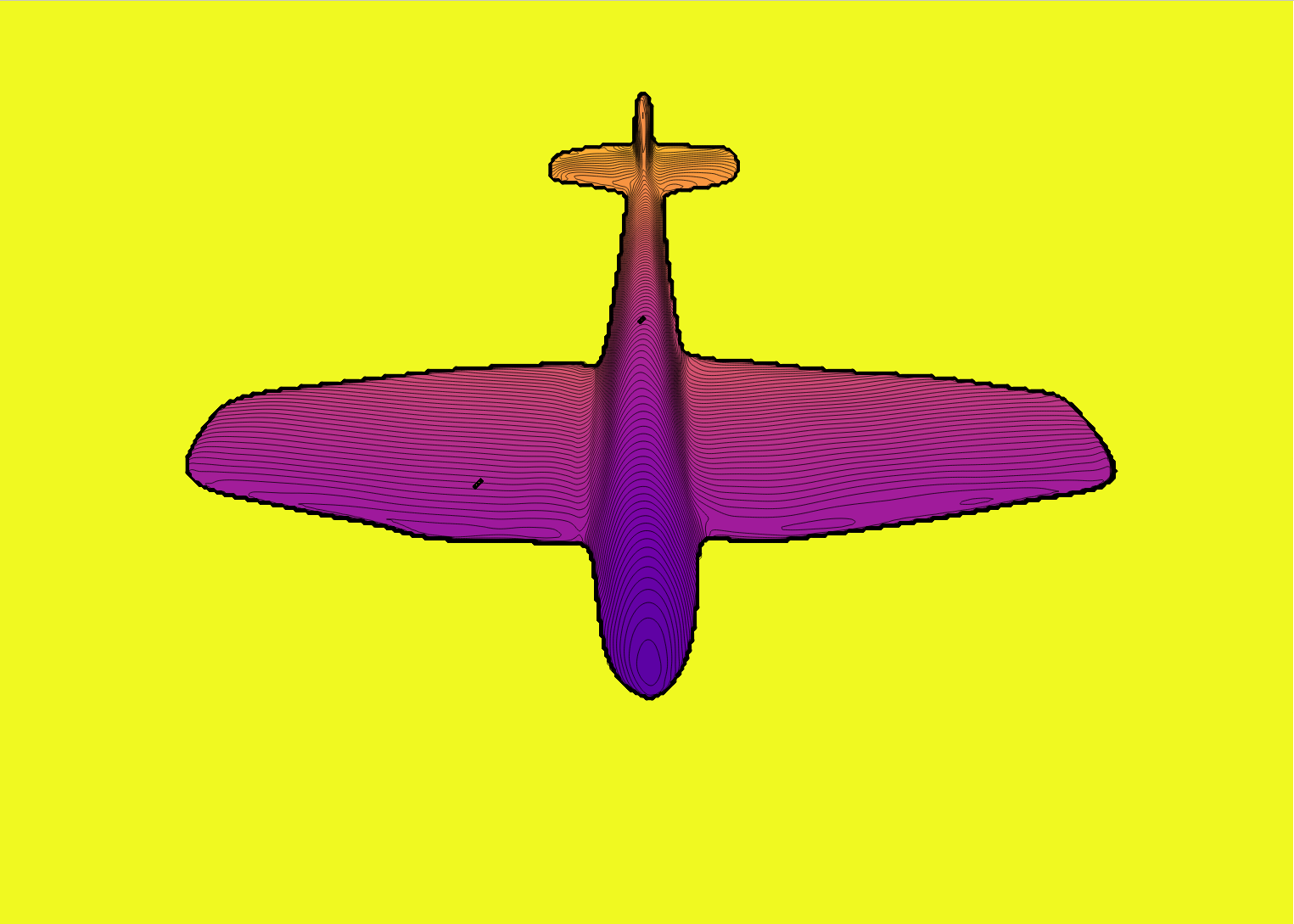}
\includegraphics[width=0.24\linewidth, trim={75mm 20mm 75mm 0mm}, clip]{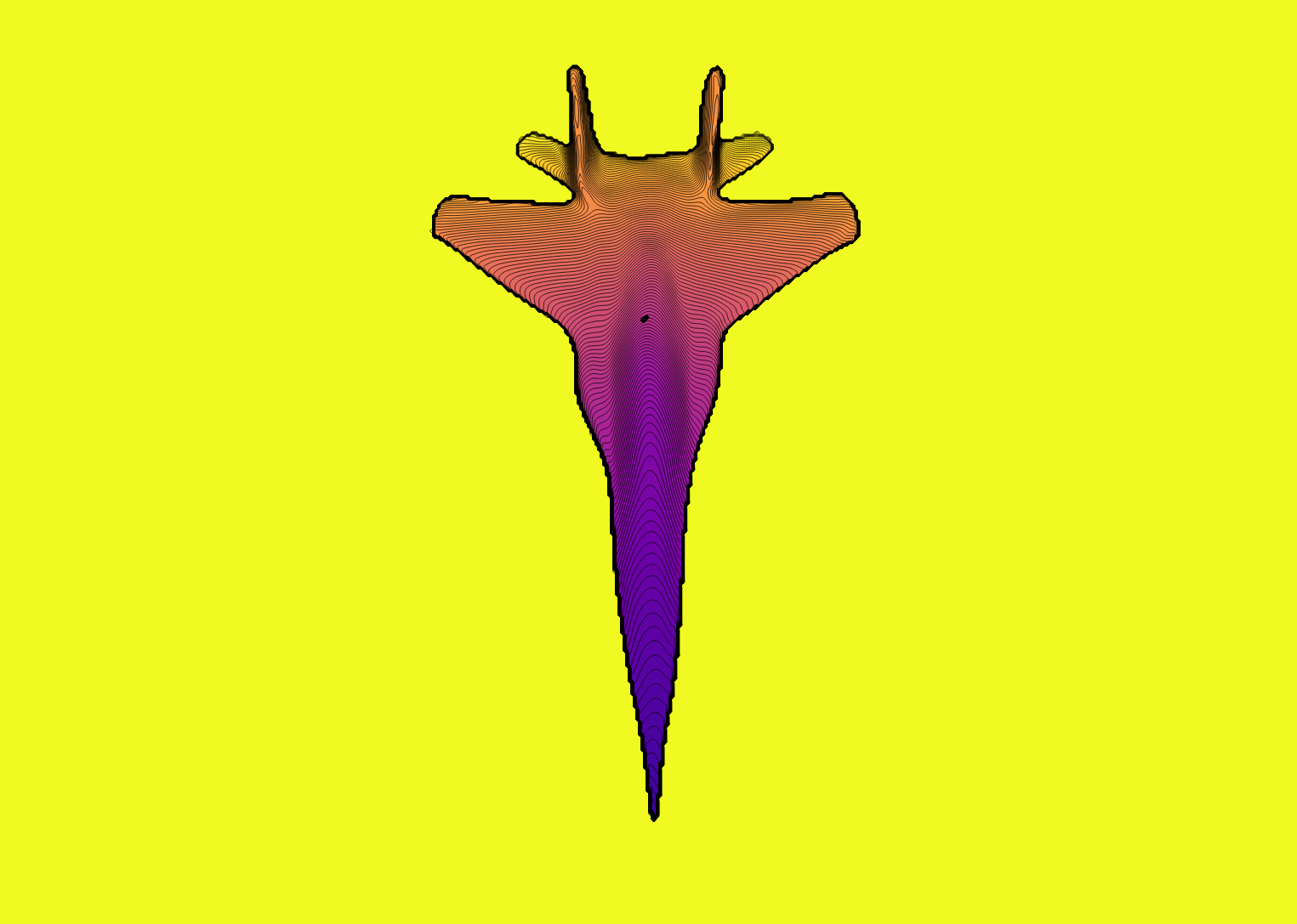}
\includegraphics[width=0.24\linewidth, trim={75mm 20mm 75mm 0mm}, clip]{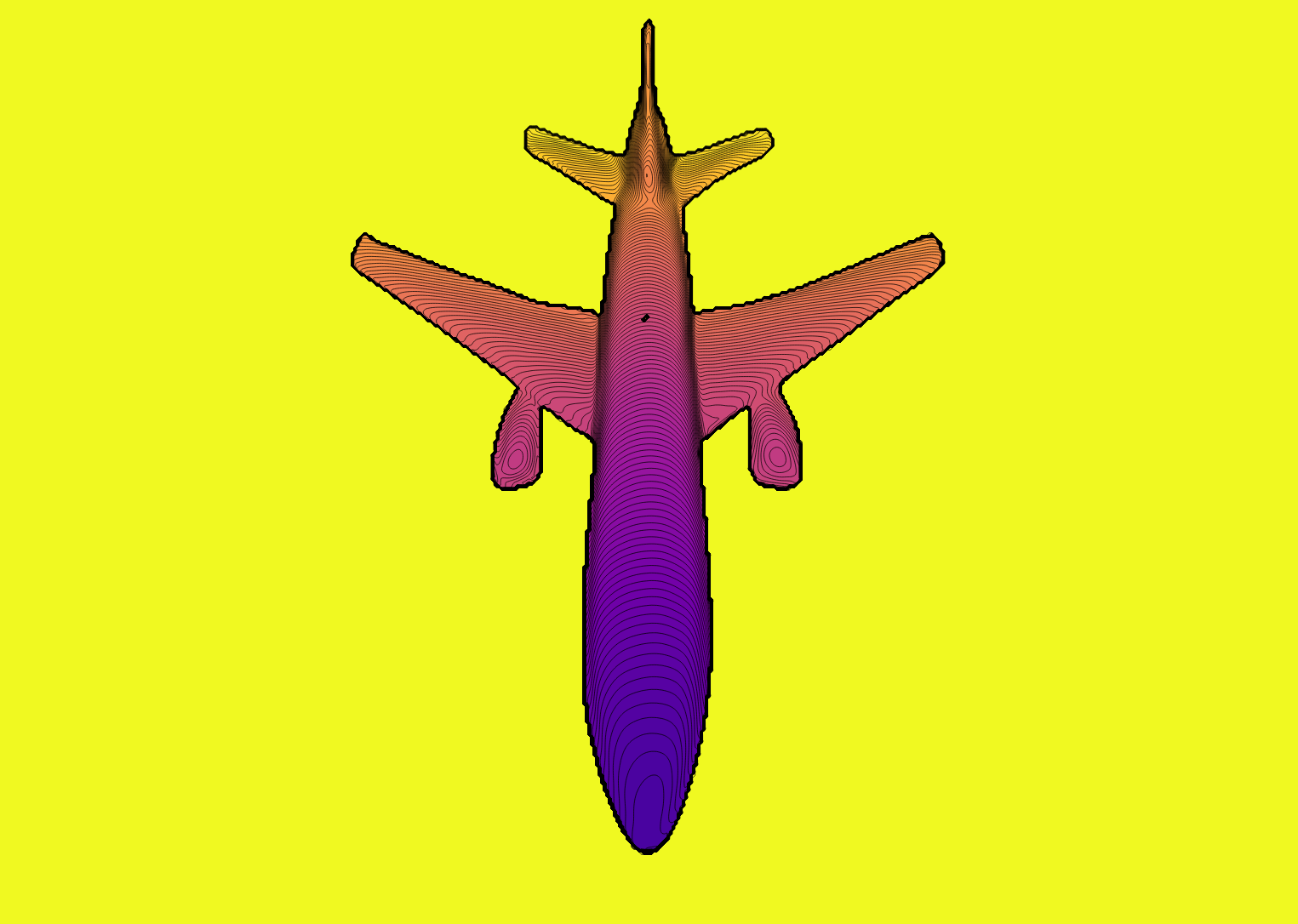}
\includegraphics[width=0.24\linewidth, trim={75mm 20mm 75mm 0mm}, clip]{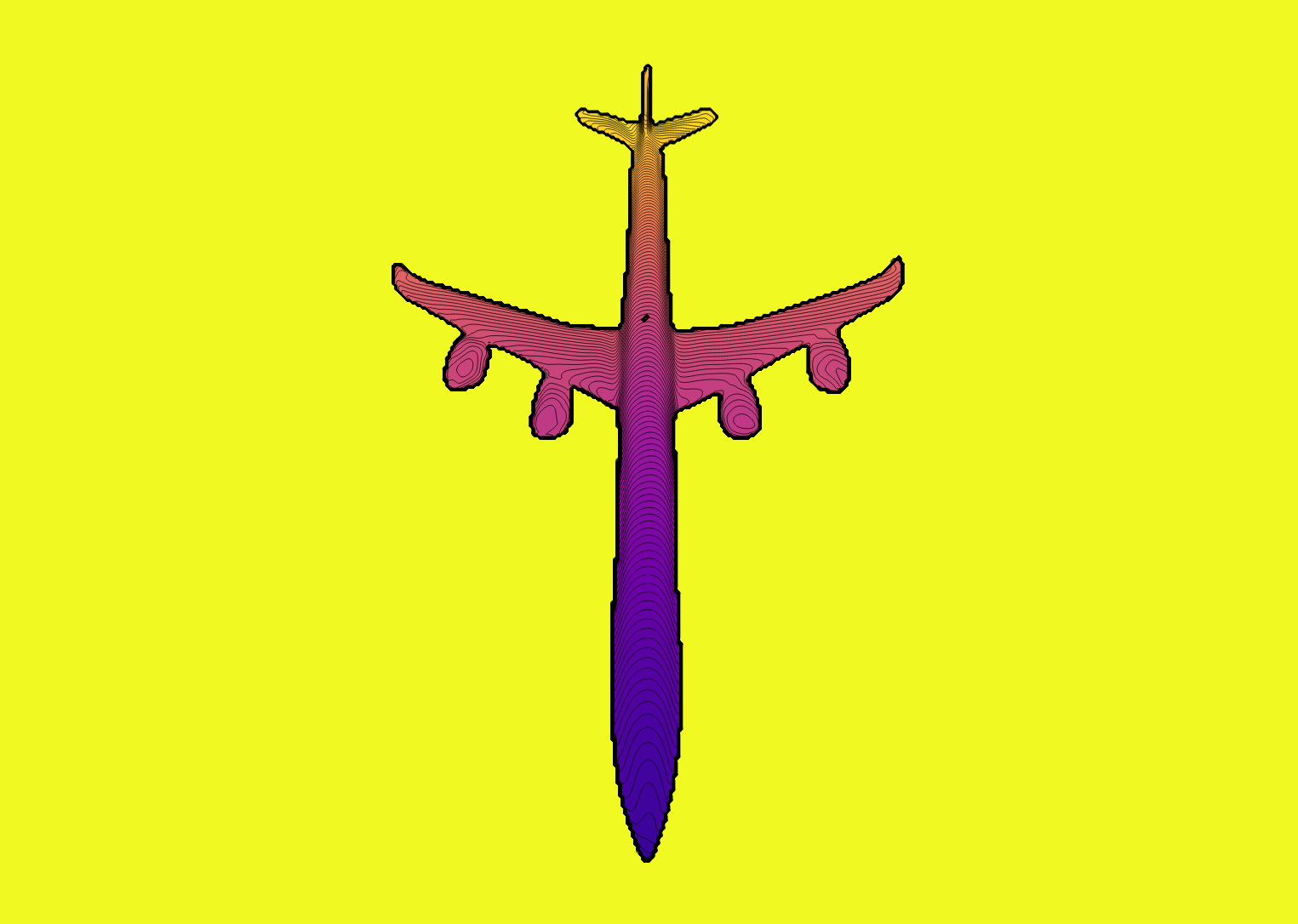}
%
\end{minipage}%
\caption{SDDF shape interpolation between two airplane instances. The first and last row show the SDDF output from the same view for two different instances from the training set. The rows in the middle are generated by using a weighted average of the latent codes of the upper-most and down-most instances as an input to the SDDF network. In each column, from top to bottom, the latent code weights with respect to the upper-most instance are $1$, $0.75$, $0.5$, $0.25$, $0$, respectively. Note how the shapes transform smoothly from top to bottom with intermediate shapes looking like valid airplanes. This demonstrates that the SDDF model represents the latent shape space continuously and meaningfully.}
\label{fig:airinterpolation}
\end{figure*}%
\begin{figure*}[h!]
    \centering
\begin{minipage}{\linewidth}
  \centering
\includegraphics[width=0.24\linewidth, trim={40mm 90mm 20mm 20mm}, clip]{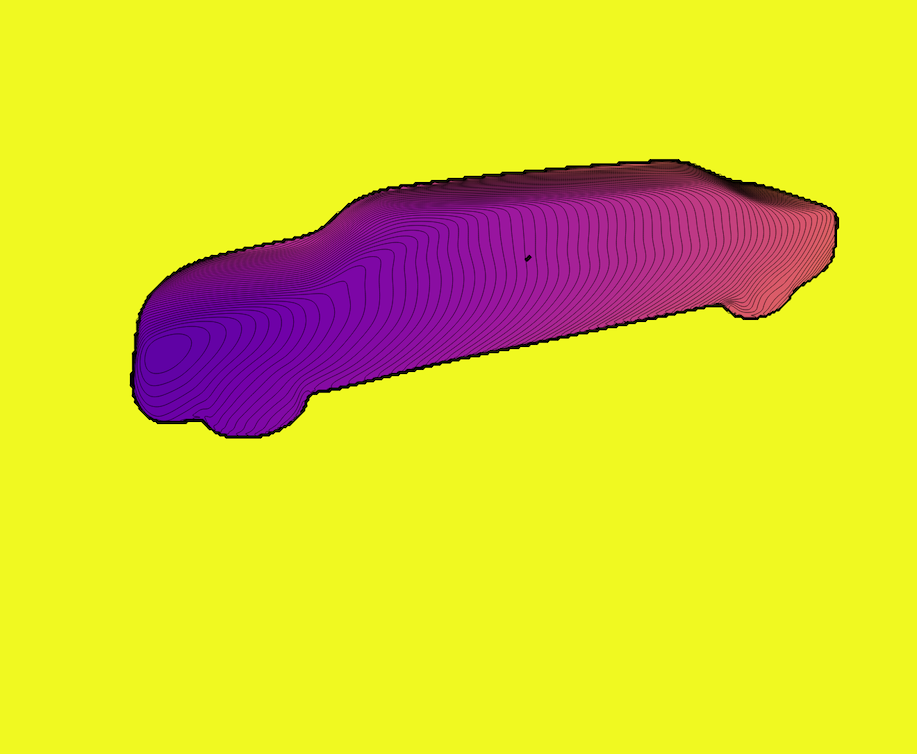}
\includegraphics[width=0.24\linewidth, trim={40mm 90mm 20mm 20mm}, clip]{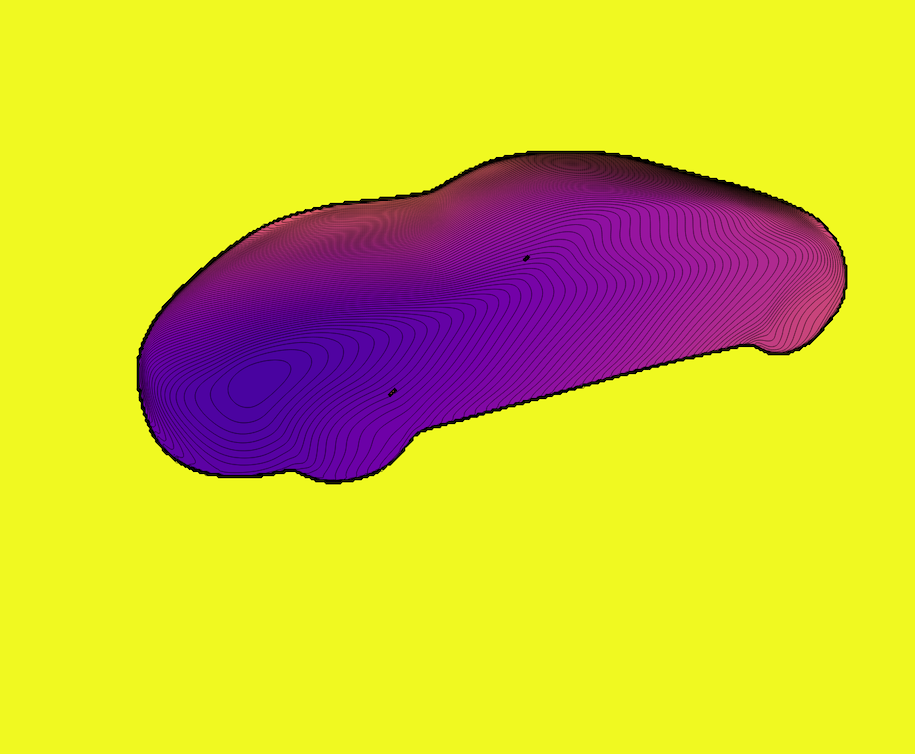}
\includegraphics[width=0.24\linewidth, trim={40mm 90mm 20mm 20mm}, clip]{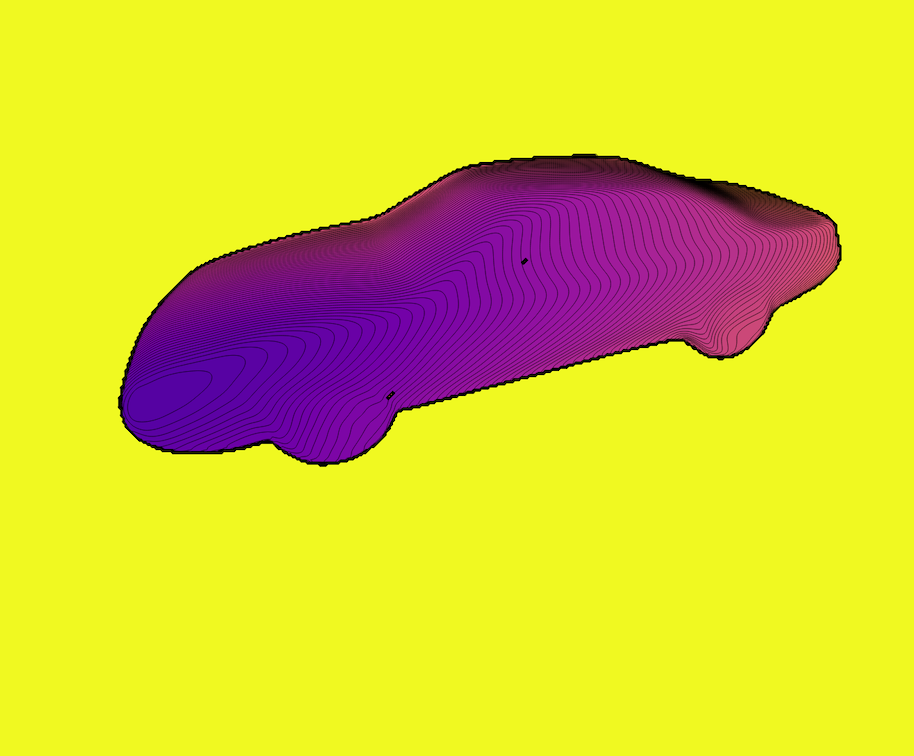}
\includegraphics[width=0.24\linewidth, trim={40mm 90mm 20mm 20mm}, clip]{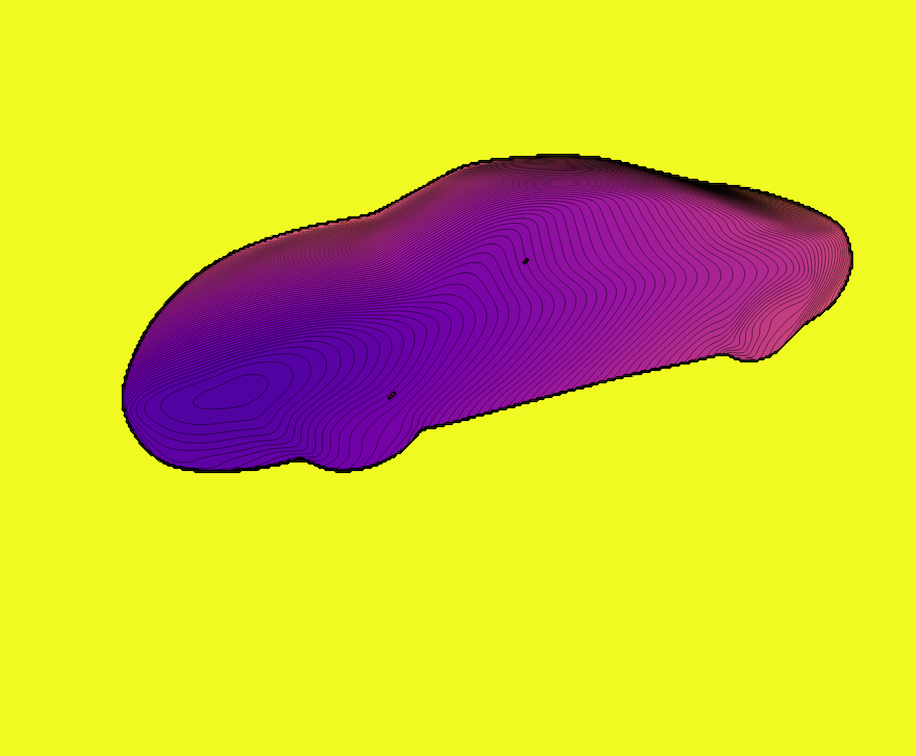}
\includegraphics[width=0.24\linewidth, trim={40mm 90mm 20mm 20mm}, clip]{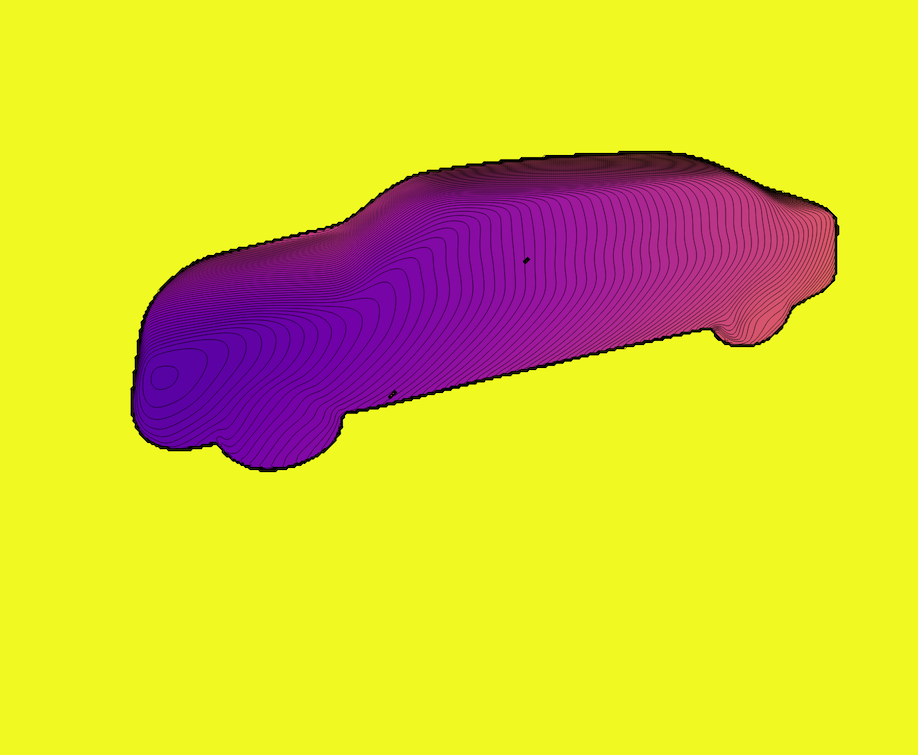}
\includegraphics[width=0.24\linewidth, trim={40mm 90mm 20mm 20mm}, clip]{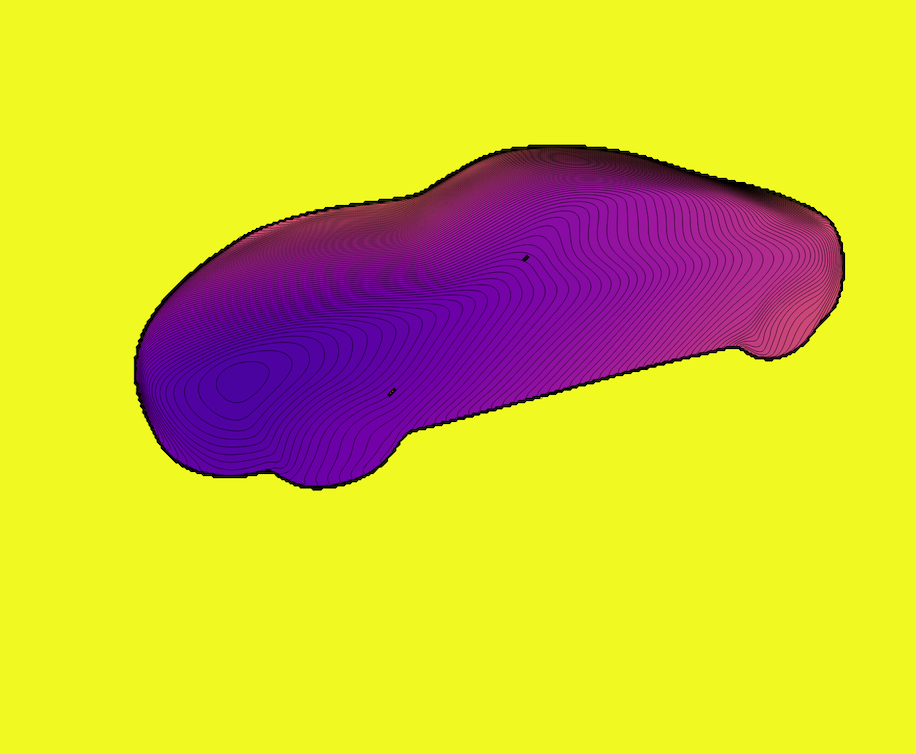}
\includegraphics[width=0.24\linewidth, trim={40mm 90mm 20mm 20mm}, clip]{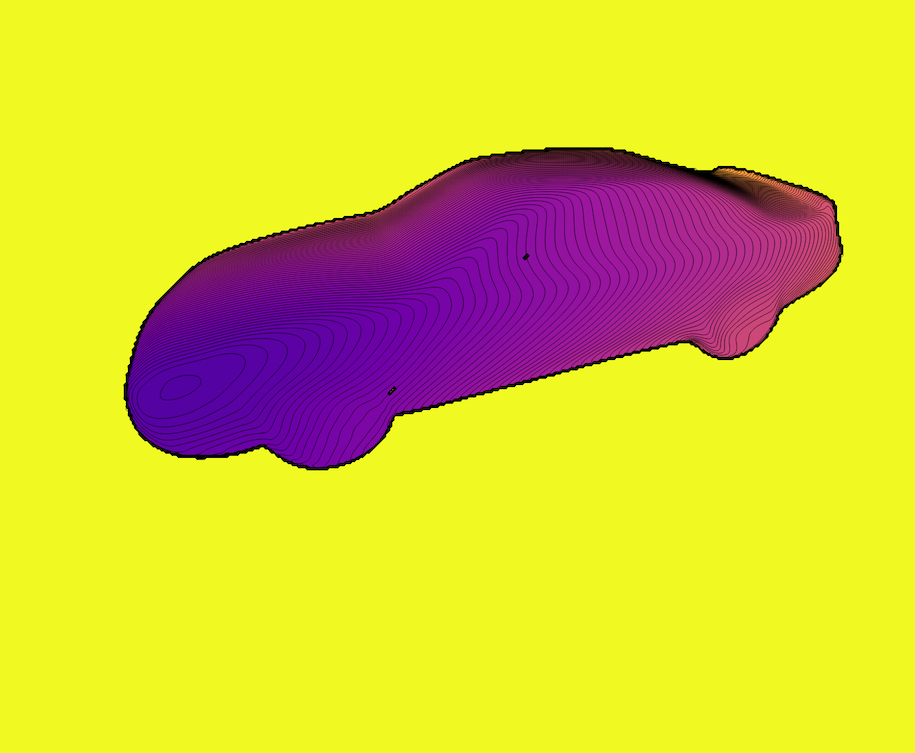}
\includegraphics[width=0.24\linewidth, trim={40mm 90mm 20mm 20mm}, clip]{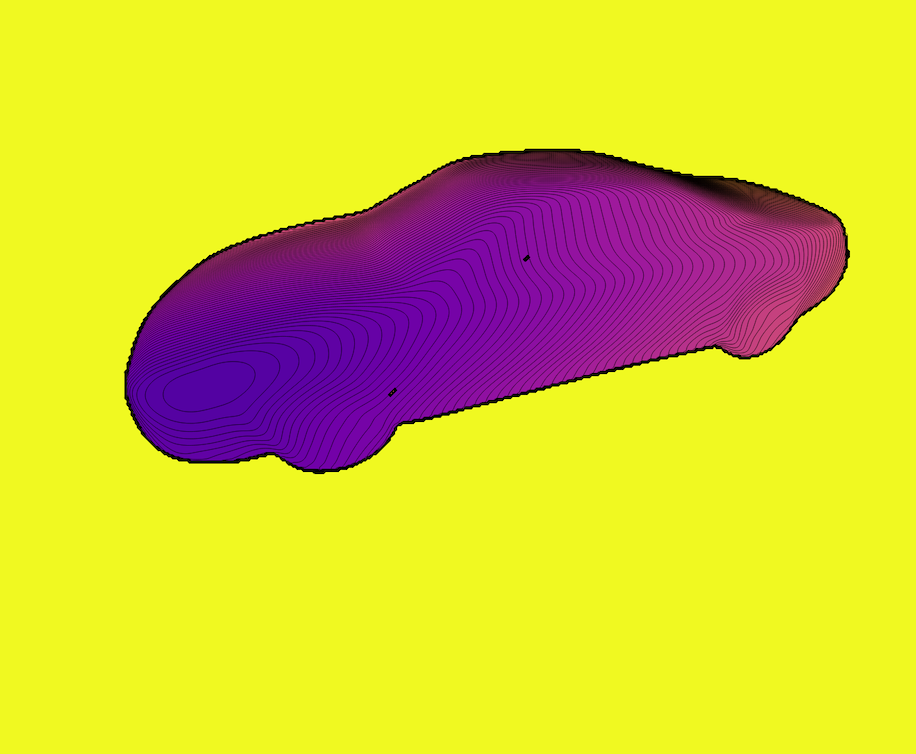}
\includegraphics[width=0.24\linewidth, trim={40mm 90mm 20mm 20mm}, clip]{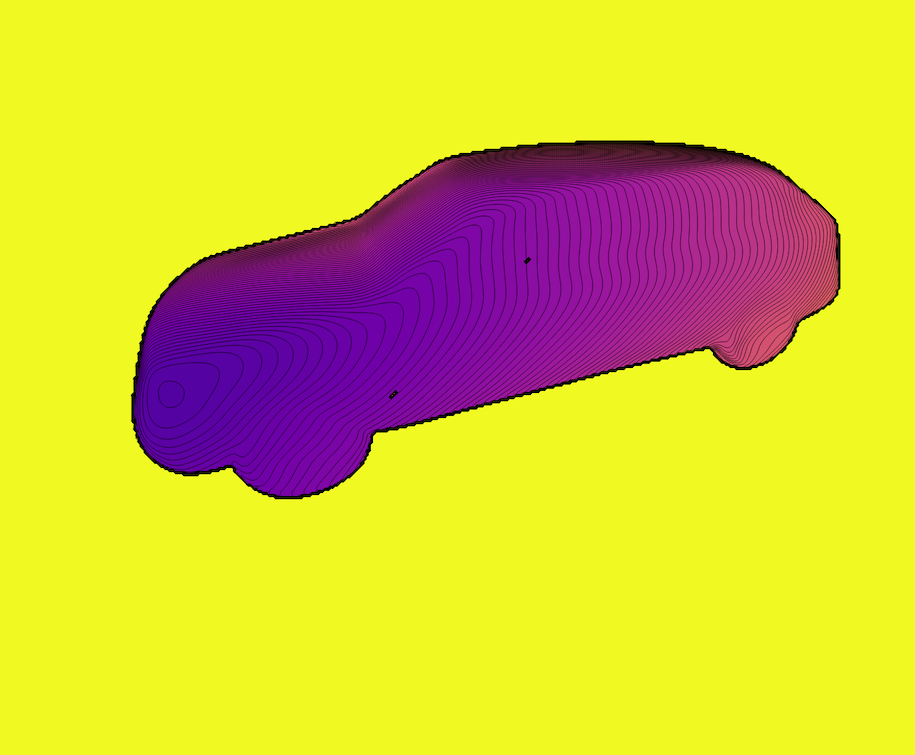}
\includegraphics[width=0.24\linewidth, trim={40mm 90mm 20mm 20mm}, clip]{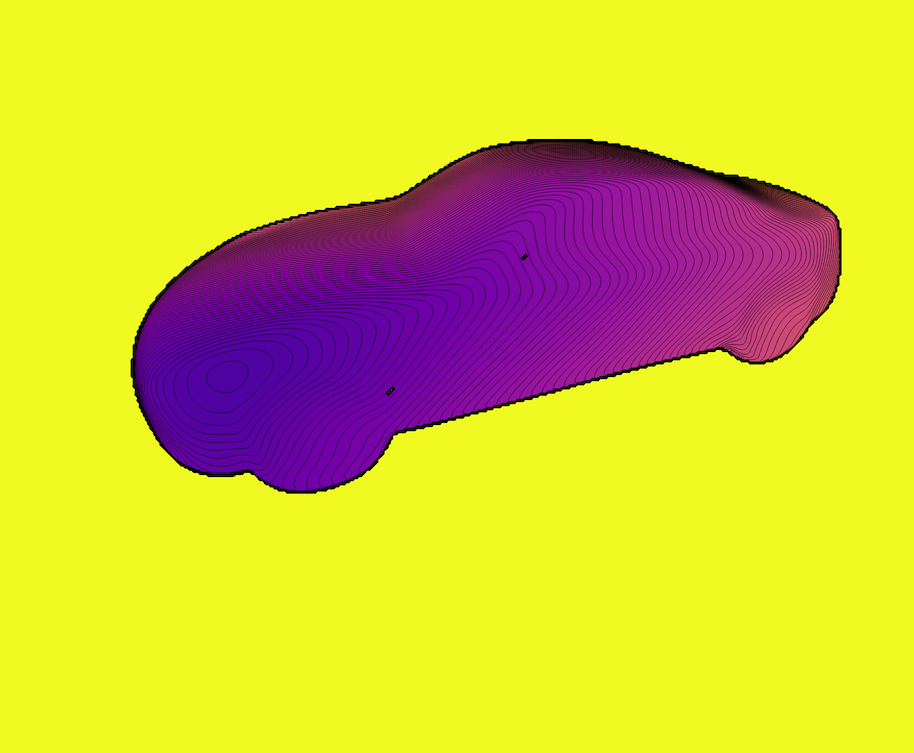}
\includegraphics[width=0.24\linewidth, trim={40mm 90mm 20mm 20mm}, clip]{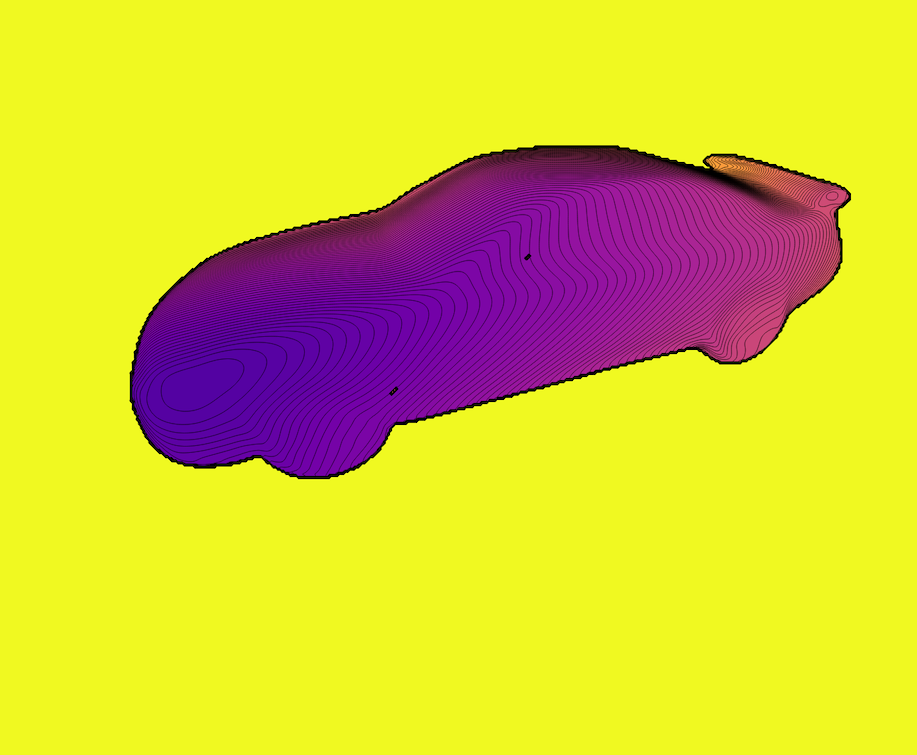}
\includegraphics[width=0.24\linewidth, trim={40mm 90mm 20mm 20mm}, clip]{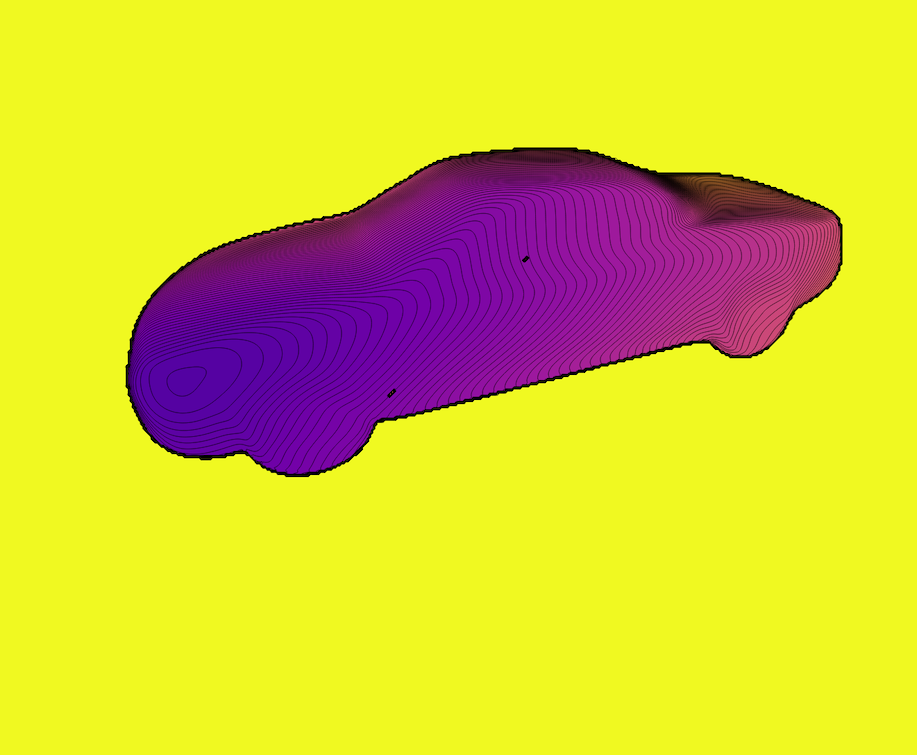}
\includegraphics[width=0.24\linewidth, trim={40mm 90mm 20mm 20mm}, clip]{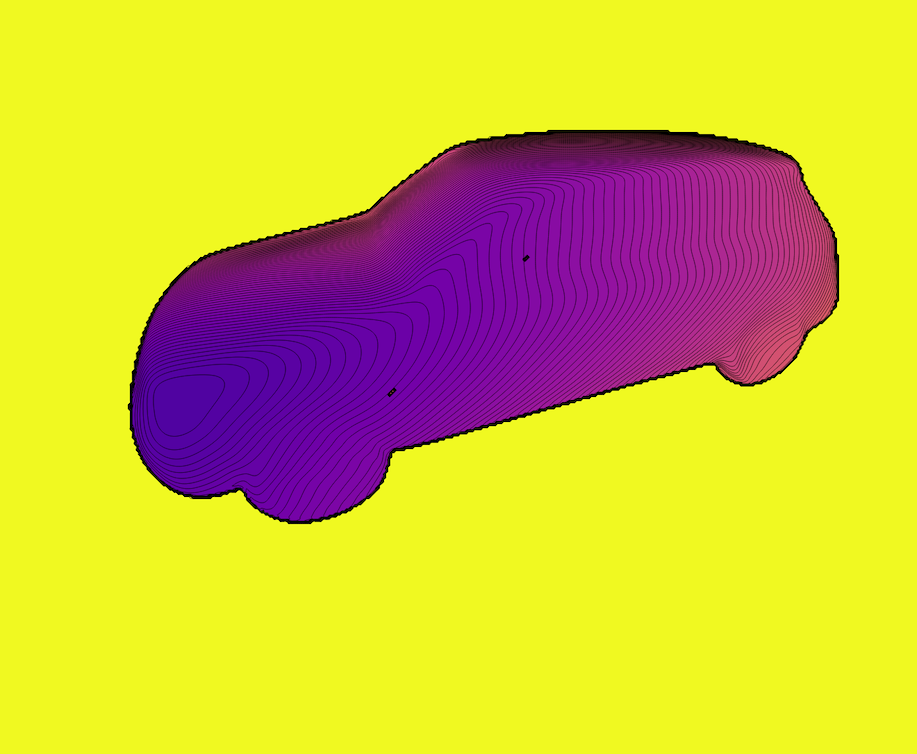}
\includegraphics[width=0.24\linewidth, trim={40mm 90mm 20mm 20mm}, clip]{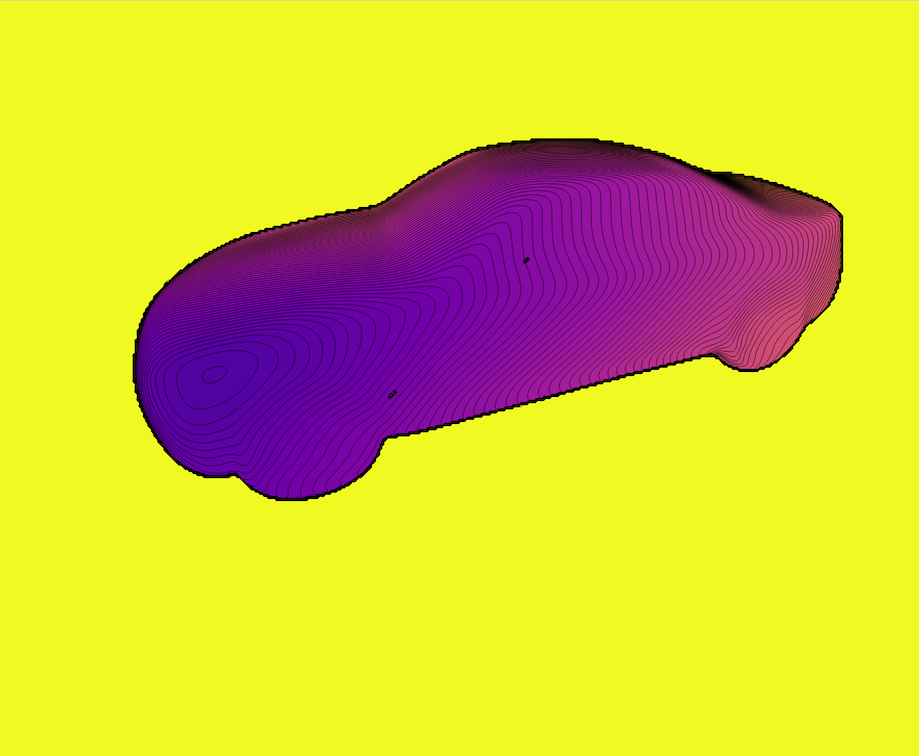}
\includegraphics[width=0.24\linewidth, trim={40mm 90mm 20mm 20mm}, clip]{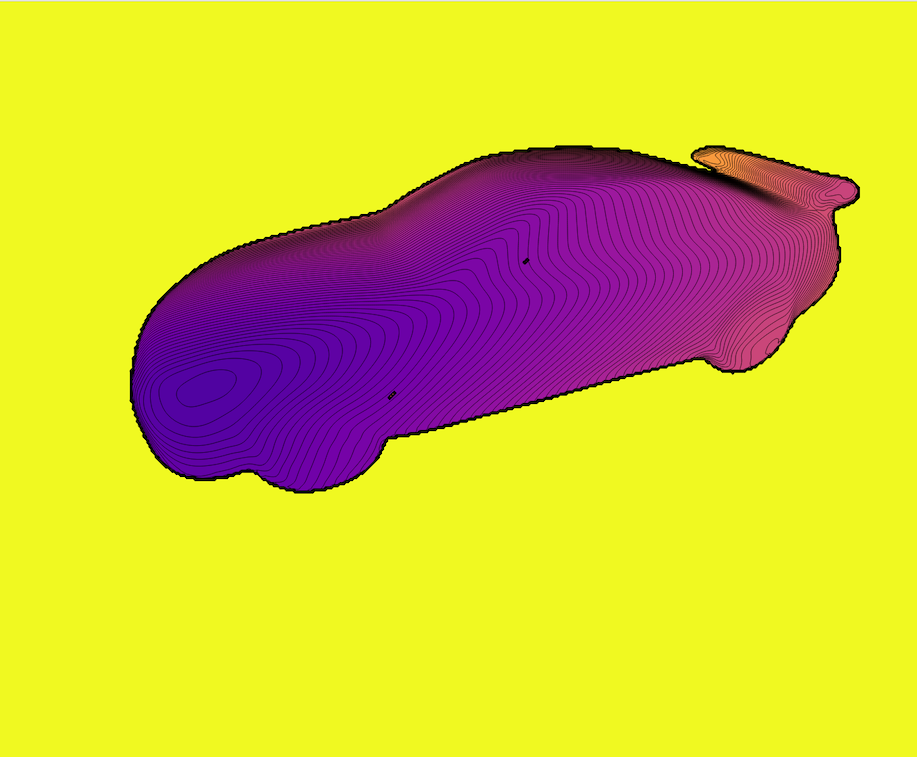}
\includegraphics[width=0.24\linewidth, trim={40mm 90mm 20mm 20mm}, clip]{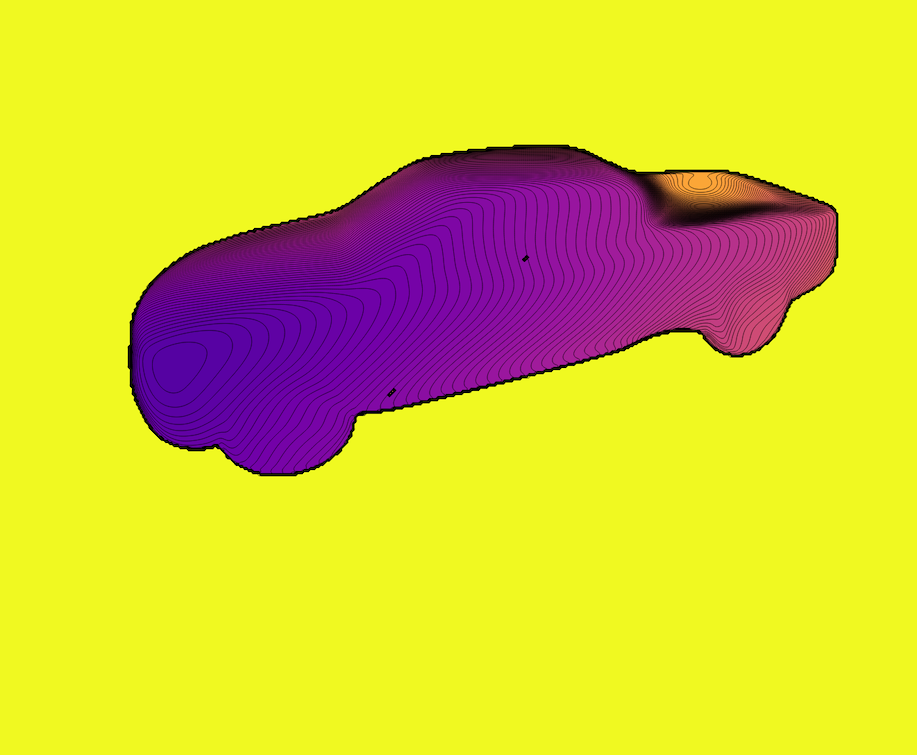}
\includegraphics[width=0.24\linewidth, trim={40mm 90mm 20mm 20mm}, clip]{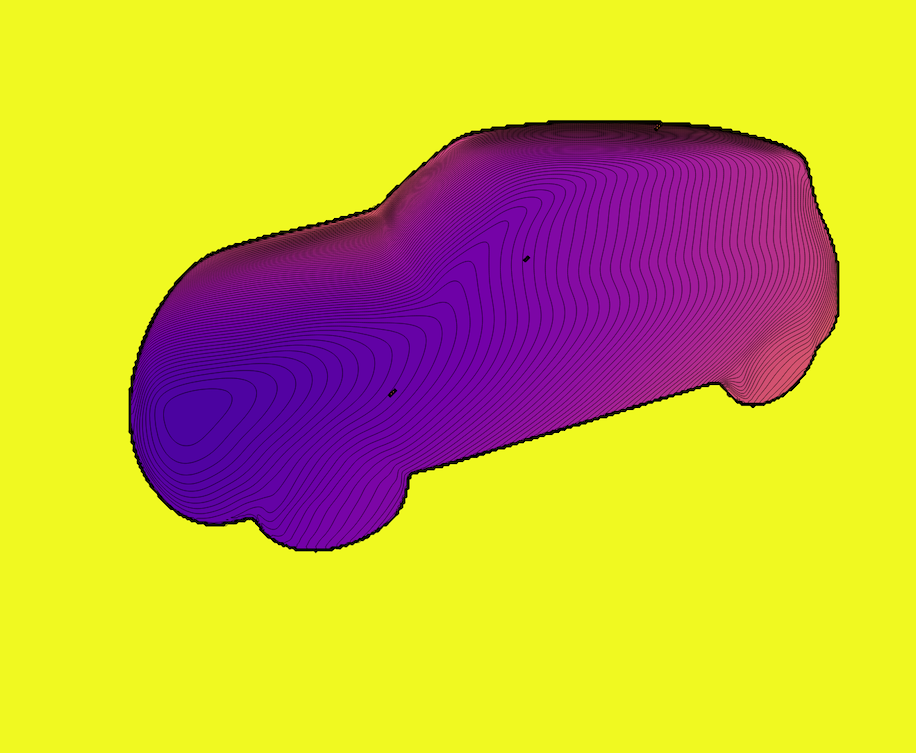}
\includegraphics[width=0.24\linewidth, trim={40mm 90mm 20mm 20mm}, clip]{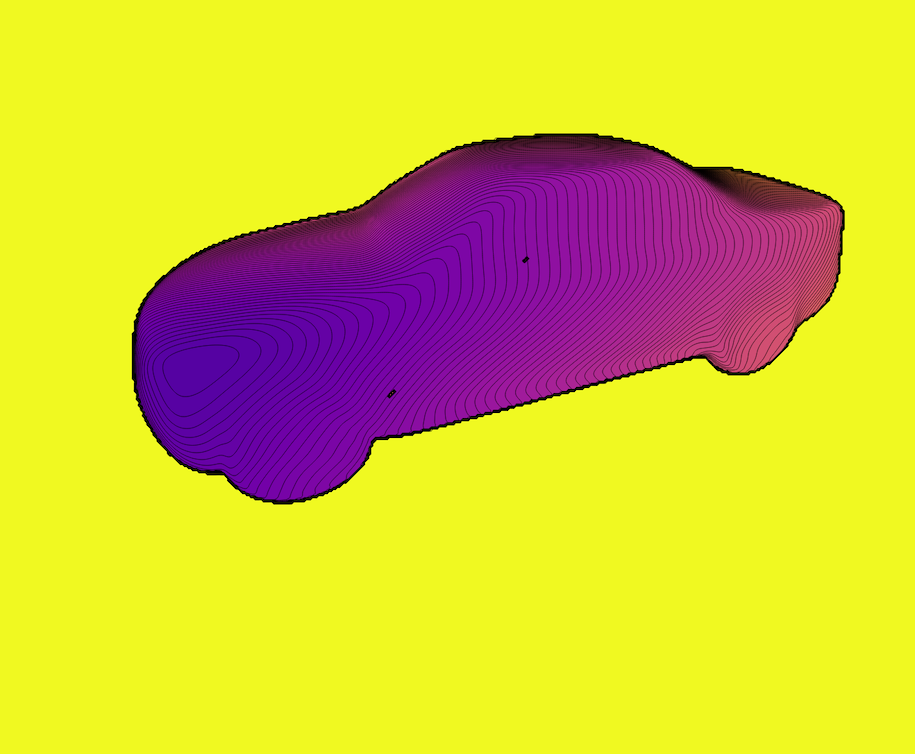}
\includegraphics[width=0.24\linewidth, trim={40mm 90mm 20mm 20mm}, clip]{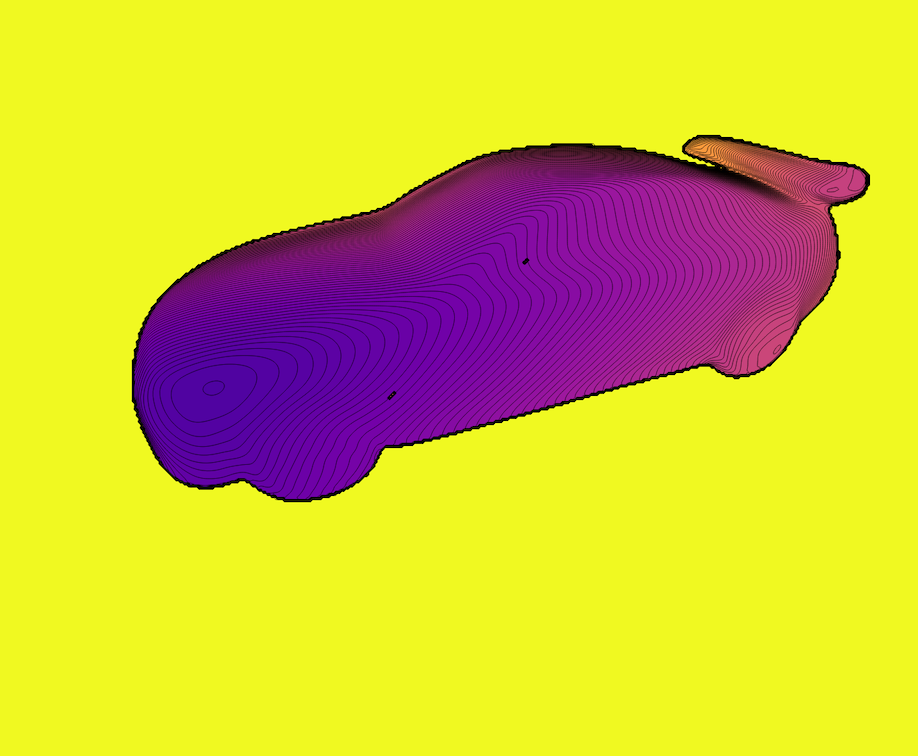}
\includegraphics[width=0.24\linewidth, trim={40mm 90mm 20mm 20mm}, clip]{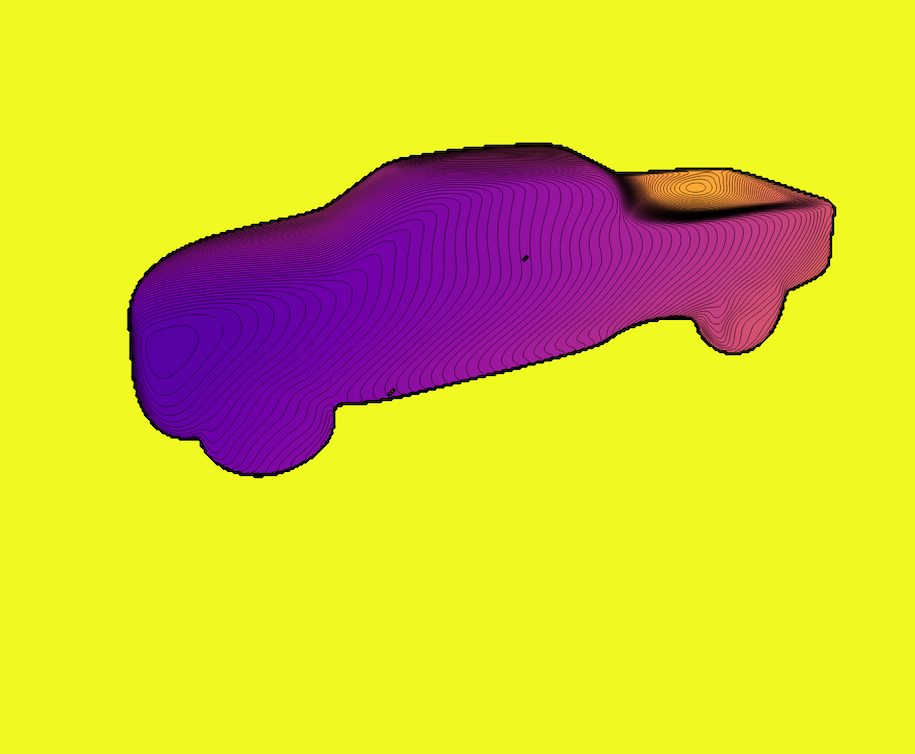}
%
%
\end{minipage}%
\caption{SDDF shape interpolation between two car instances. The first and last row show the SDDF output from the same view for two different instances from the training set. The rows in the middle are generated by using a weighted average of the latent codes of the upper-most and down-most instances as an input to the SDDF network. In each column, from top to bottom, the latent code weights with respect to the upper-most instance are $1$, $0.75$, $0.5$, $0.25$, $0$, respectively. Note how the shapes transform smoothly from top to bottom with intermediate shapes looking like valid cars. This demonstrates that the SDDF model represents the latent shape space continuously and meaningfully.}
\label{fig:carinterpolation}
\end{figure*}
\begin{figure*}[h!]
    \centering
\begin{minipage}{\linewidth}
  \centering
\includegraphics[width=0.24\linewidth, trim={60mm 30mm 60mm 40mm}, clip]{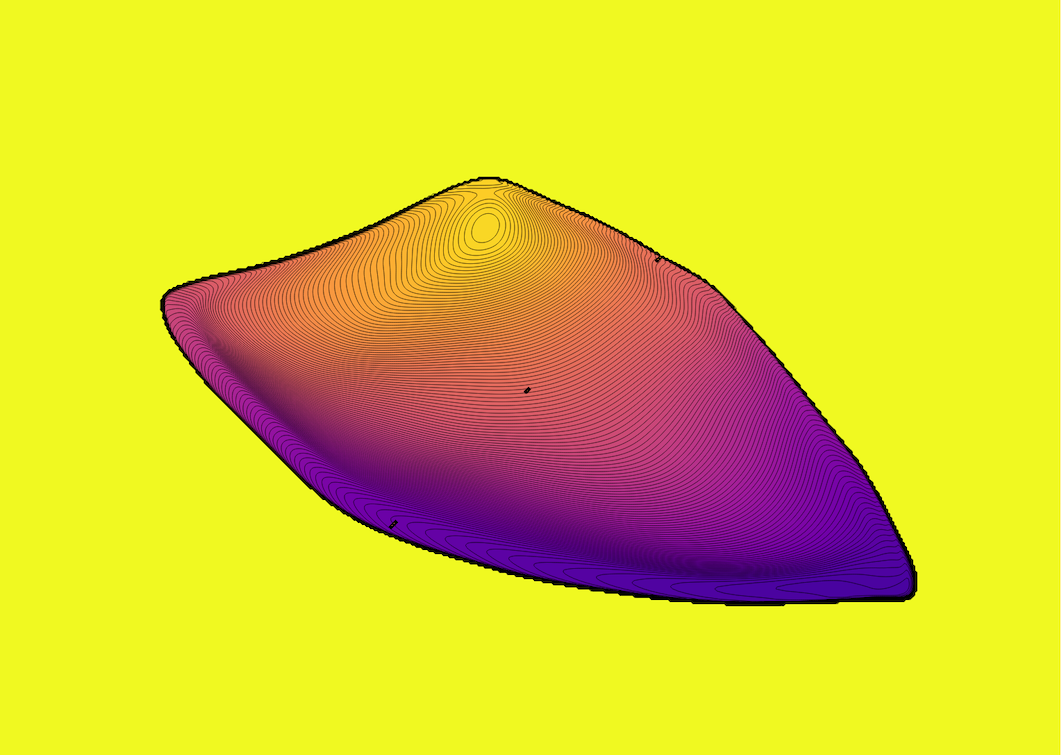}
\includegraphics[width=0.24\linewidth, trim={60mm 30mm 60mm 40mm}, clip]{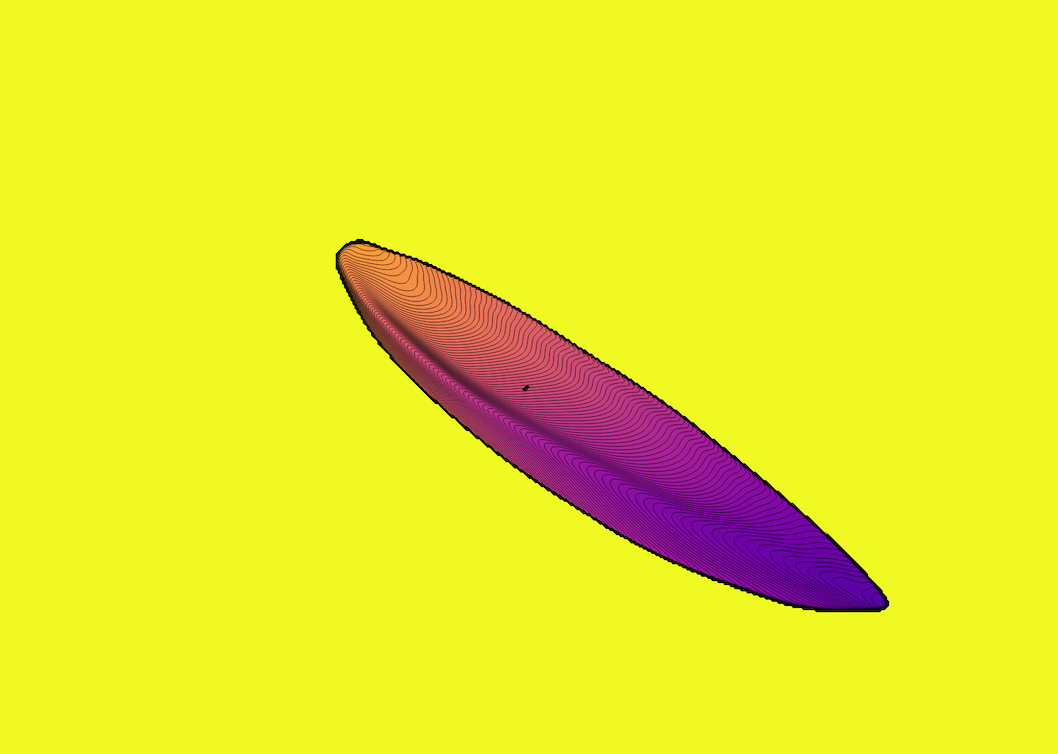}
\includegraphics[width=0.24\linewidth, trim={60mm 30mm 60mm 40mm}, clip]{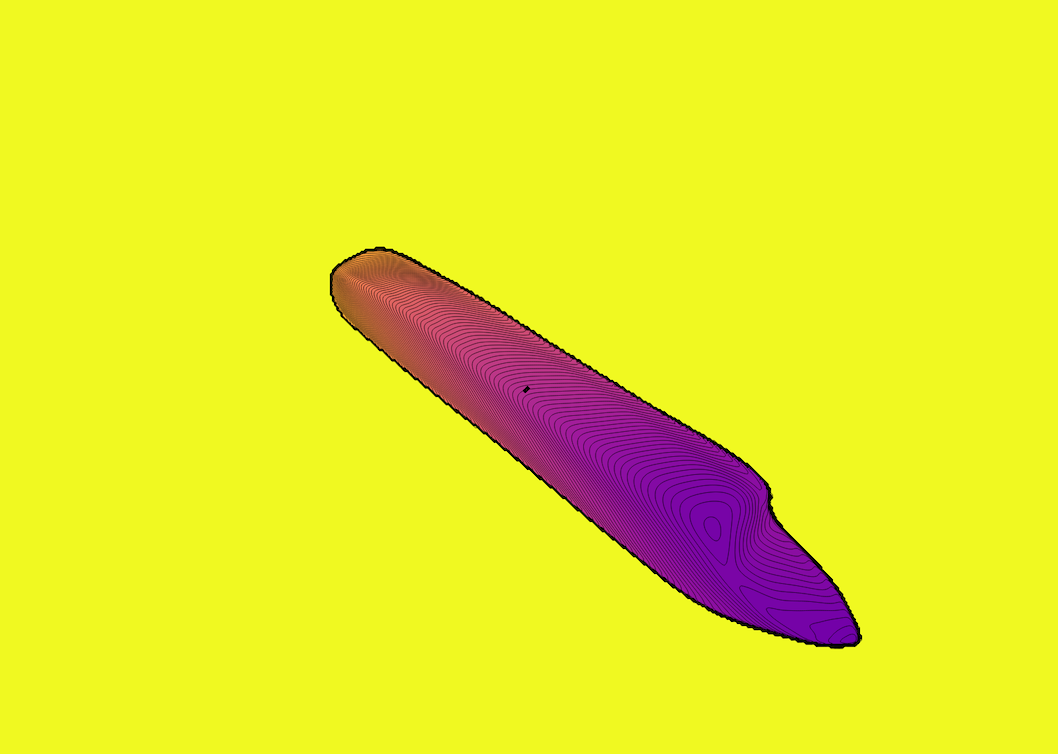}
\includegraphics[width=0.24\linewidth, trim={40mm 10mm 20mm 20mm}, clip]{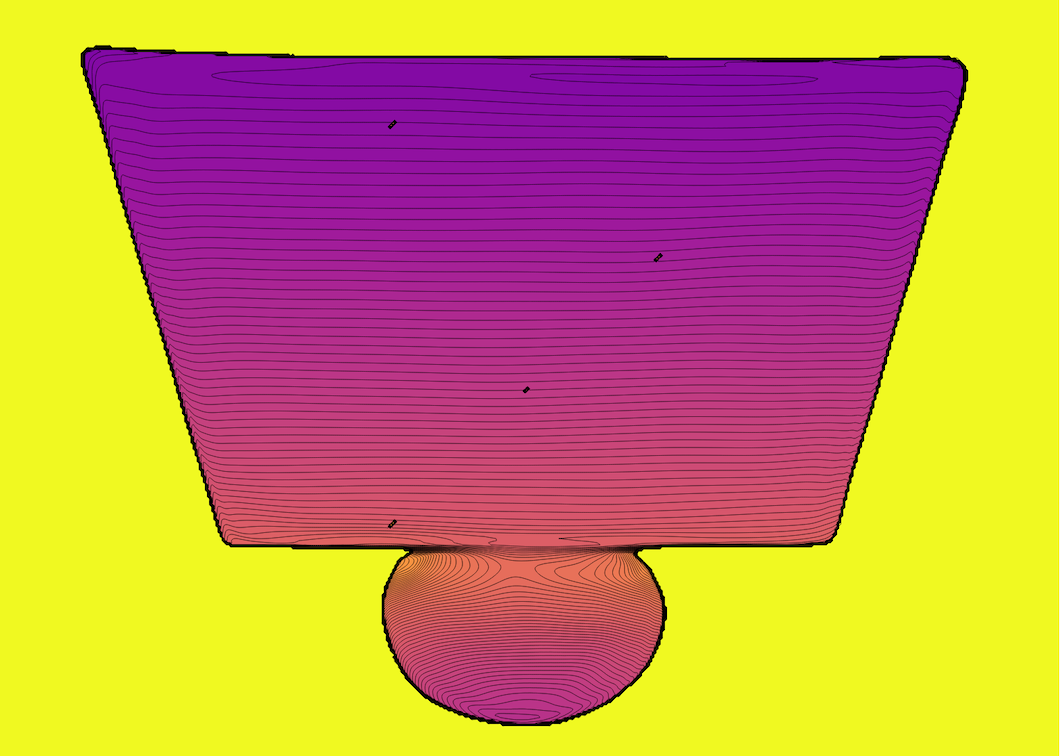}
\includegraphics[width=0.24\linewidth, trim={60mm 30mm 60mm 40mm}, clip]{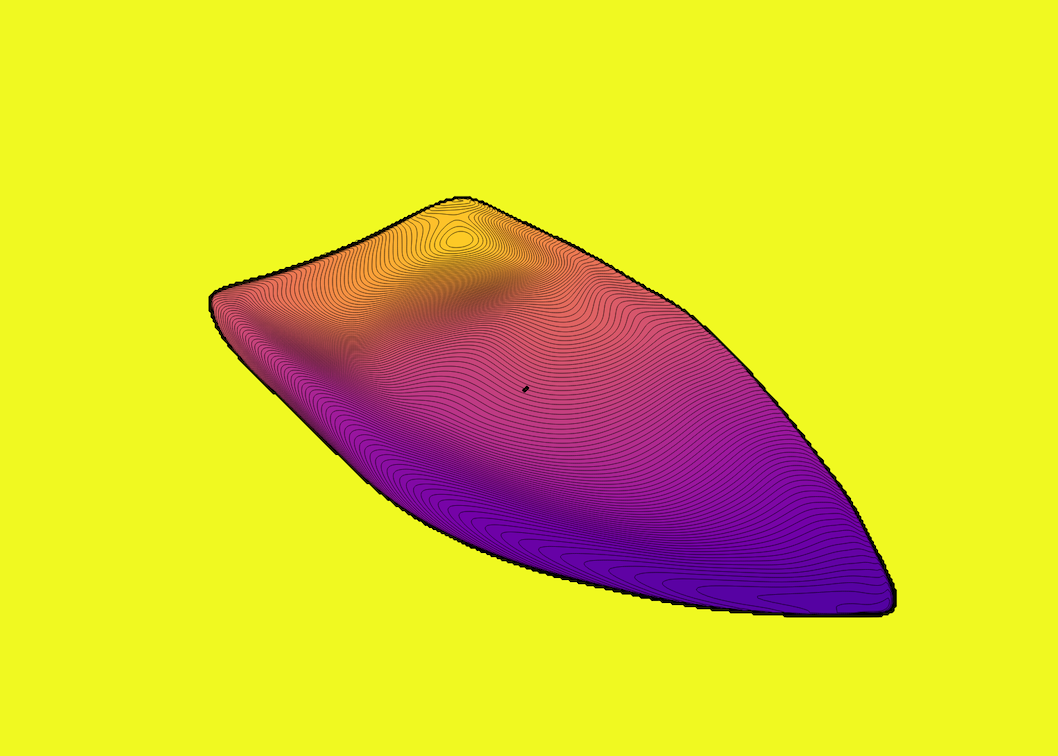}
\includegraphics[width=0.24\linewidth, trim={60mm 30mm 60mm 40mm}, clip]{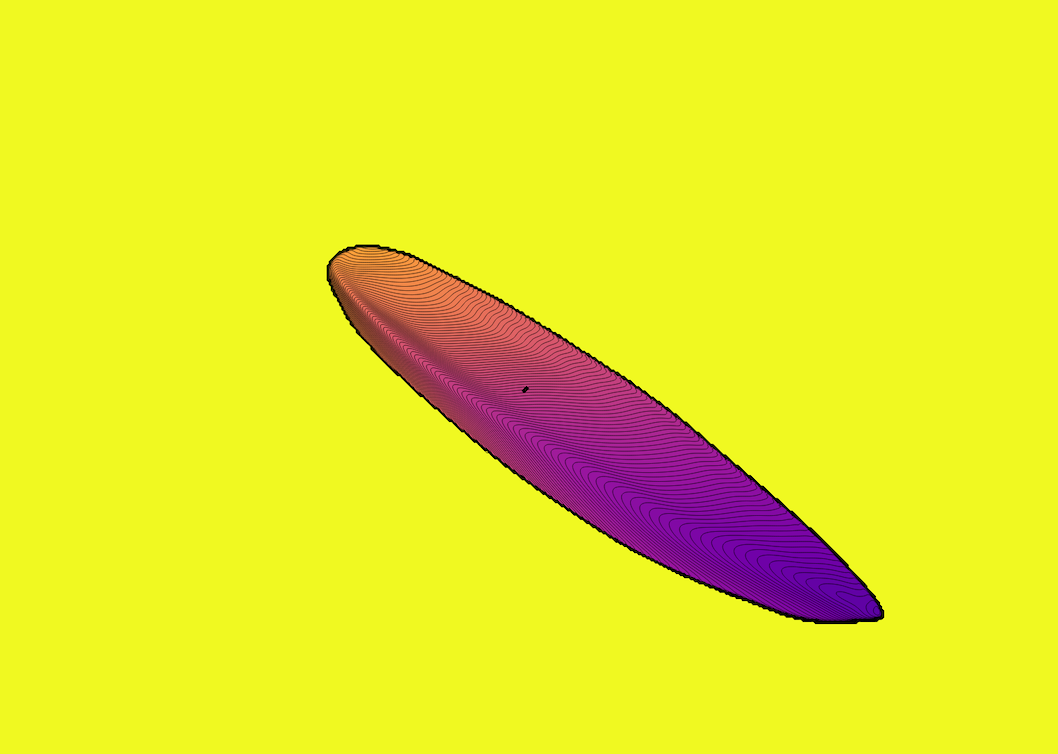}
\includegraphics[width=0.24\linewidth, trim={60mm 30mm 60mm 40mm}, clip]{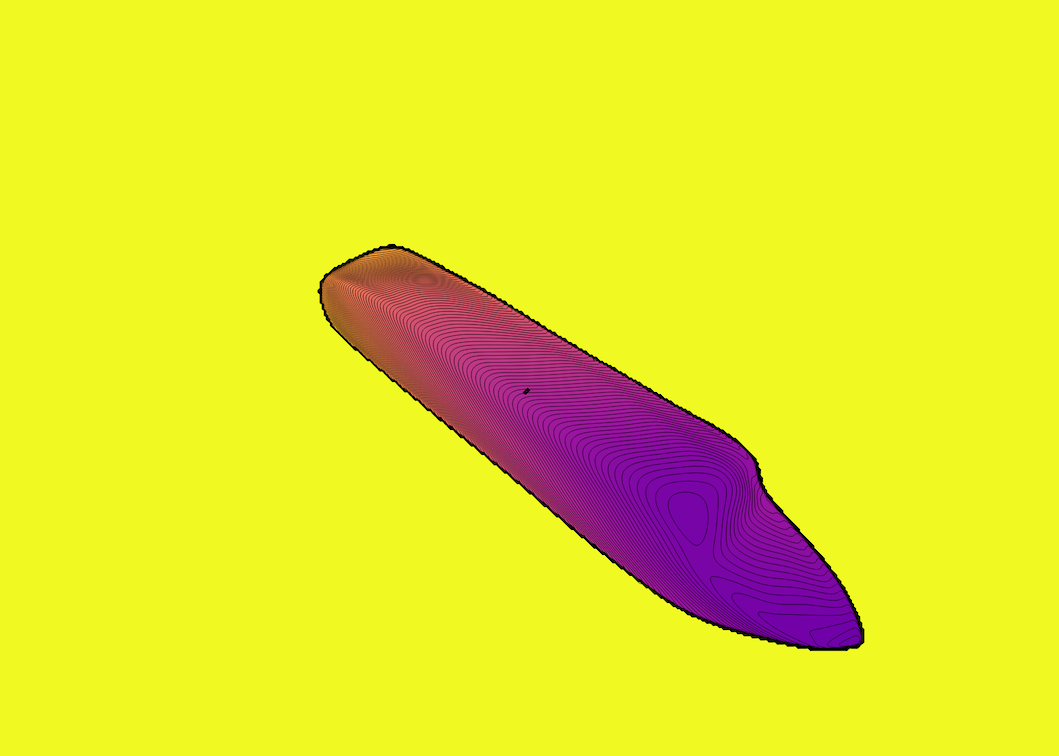}
\includegraphics[width=0.24\linewidth, trim={40mm 10mm 20mm 20mm}, clip]{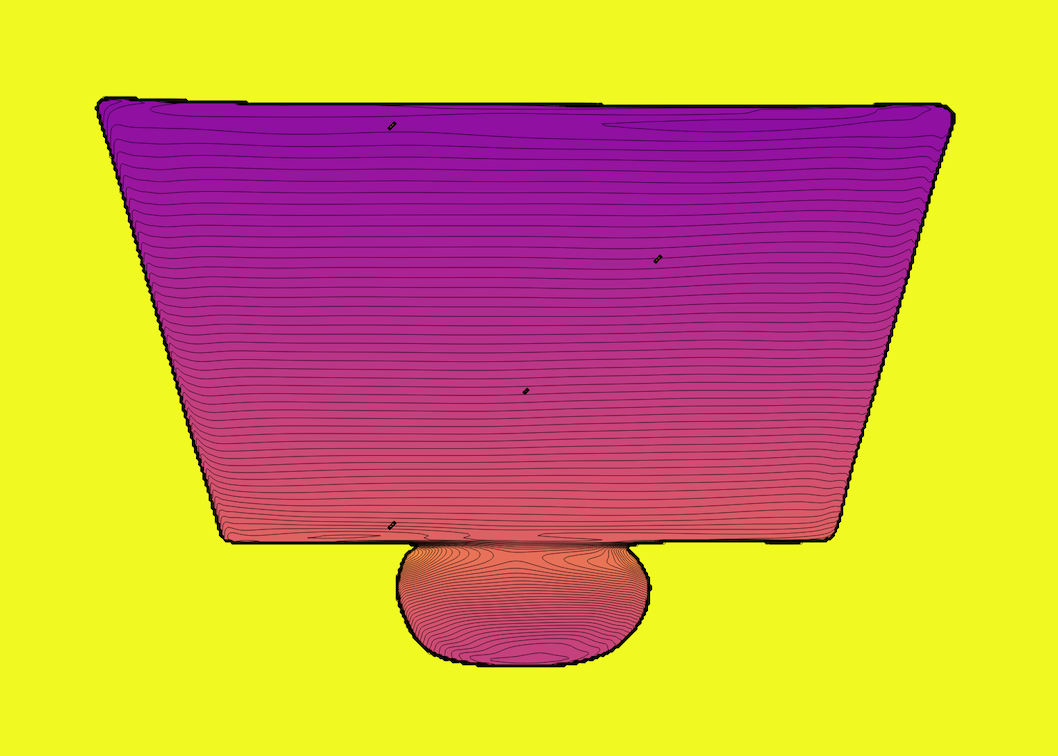}
\includegraphics[width=0.24\linewidth, trim={60mm 30mm 60mm 40mm}, clip]{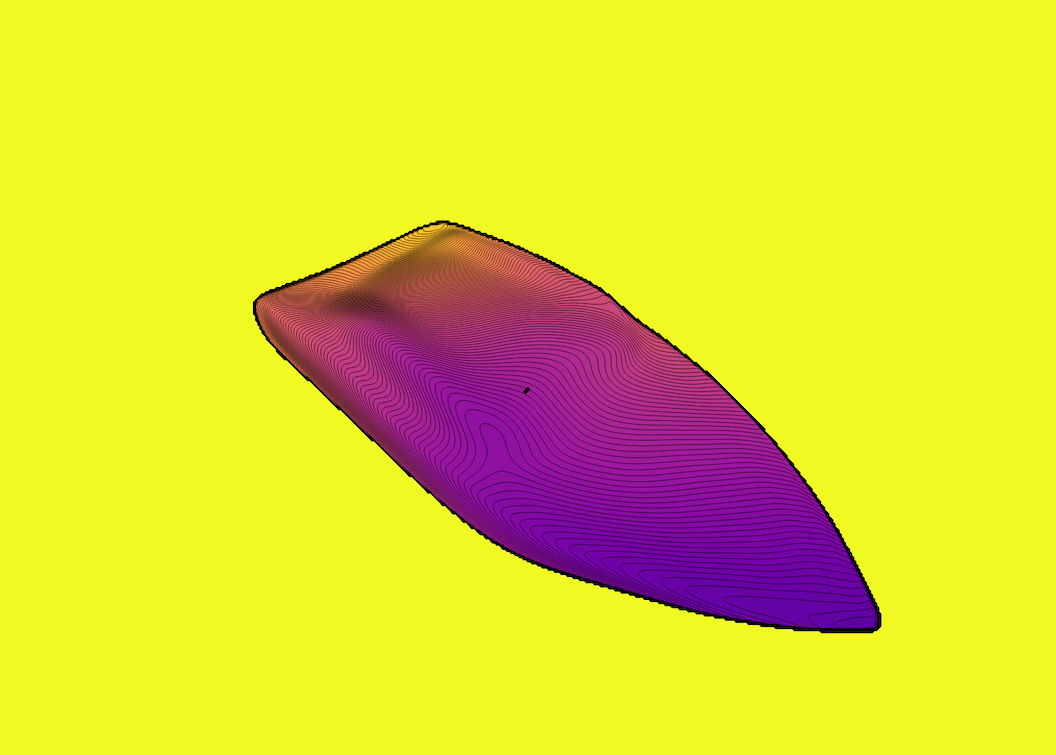}
\includegraphics[width=0.24\linewidth, trim={60mm 30mm 60mm 40mm}, clip]{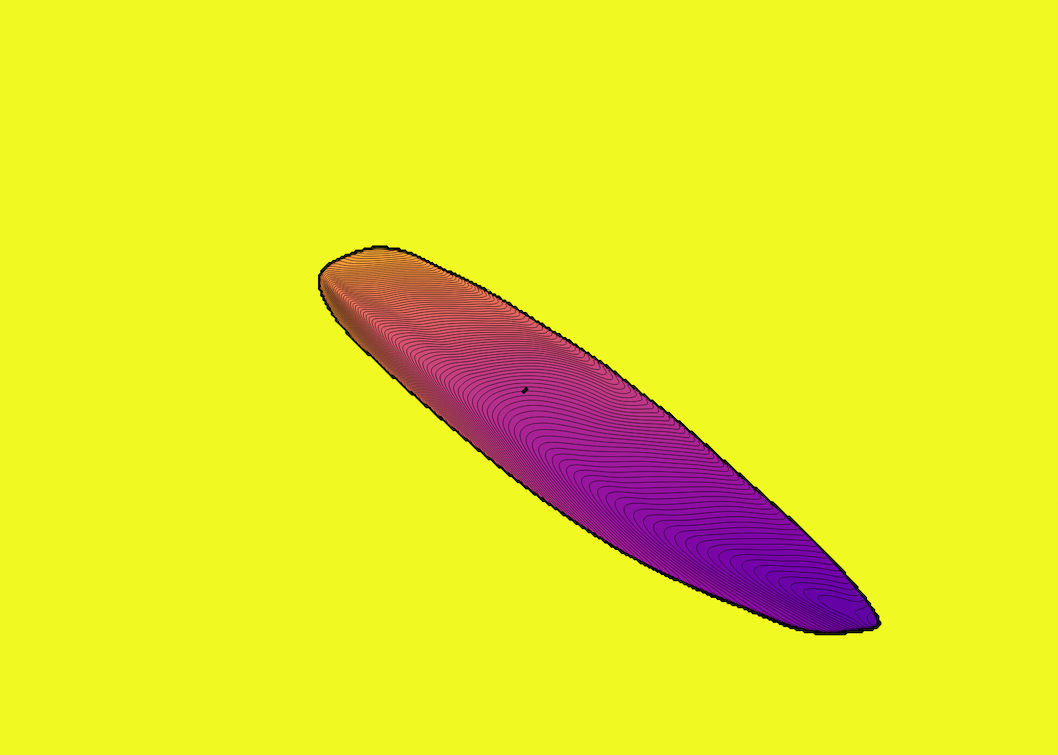}
\includegraphics[width=0.24\linewidth, trim={60mm 30mm 60mm 40mm}, clip]{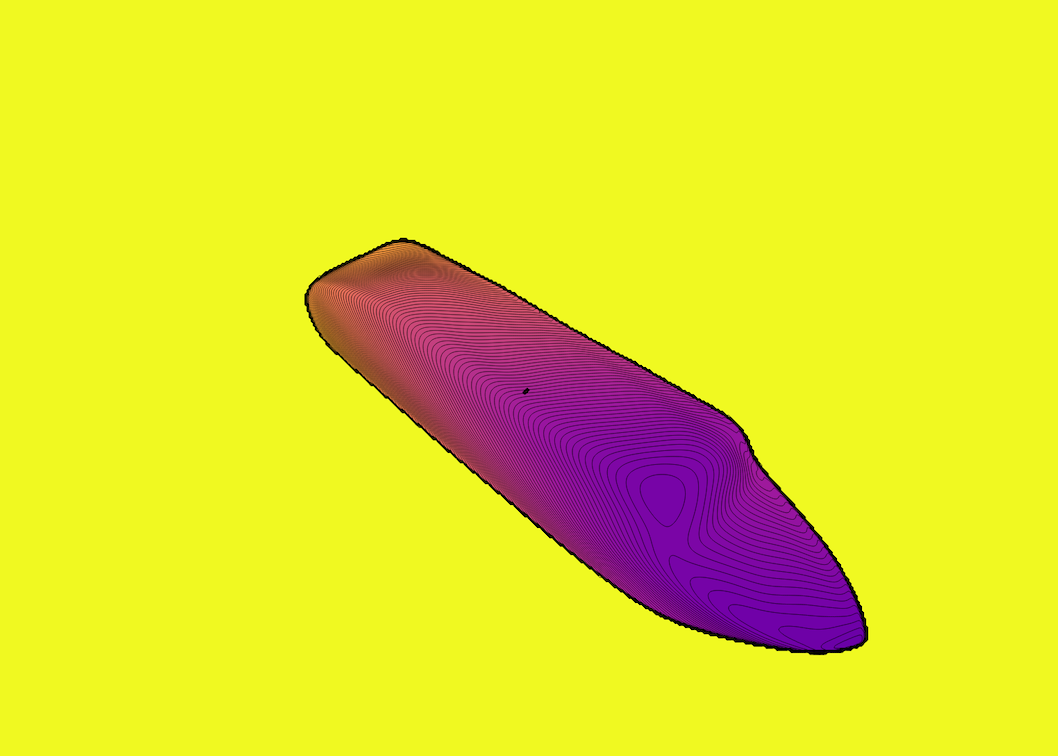}
\includegraphics[width=0.24\linewidth, trim={40mm 10mm 20mm 20mm}, clip]{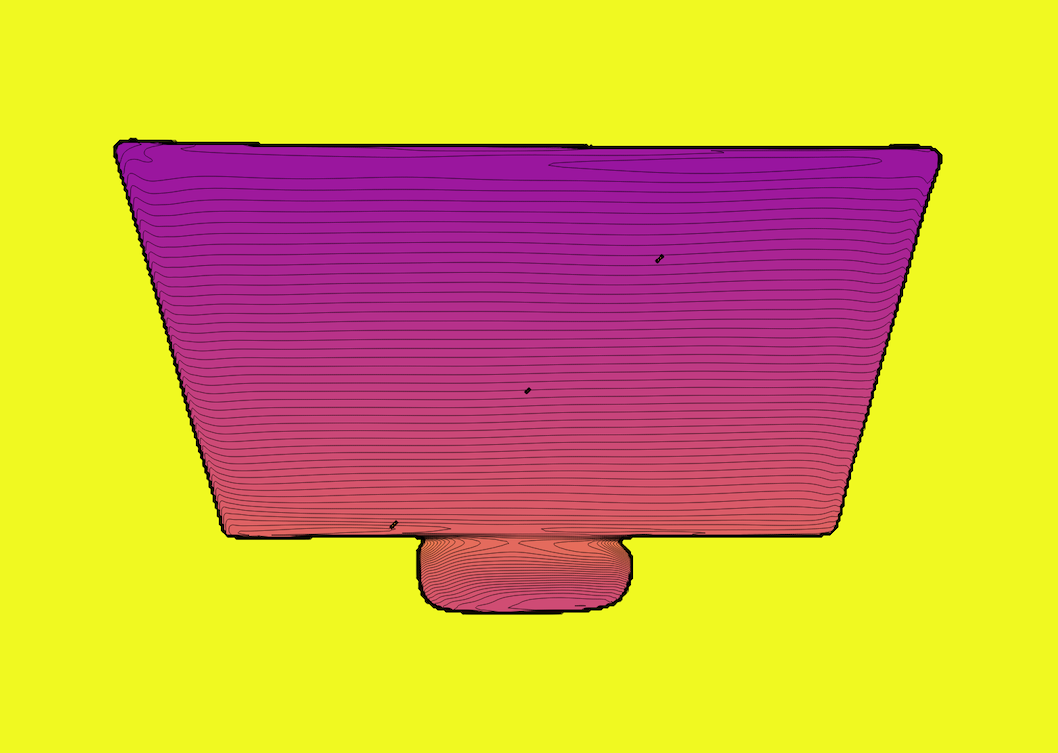}
\includegraphics[width=0.24\linewidth, trim={60mm 30mm 60mm 40mm}, clip]{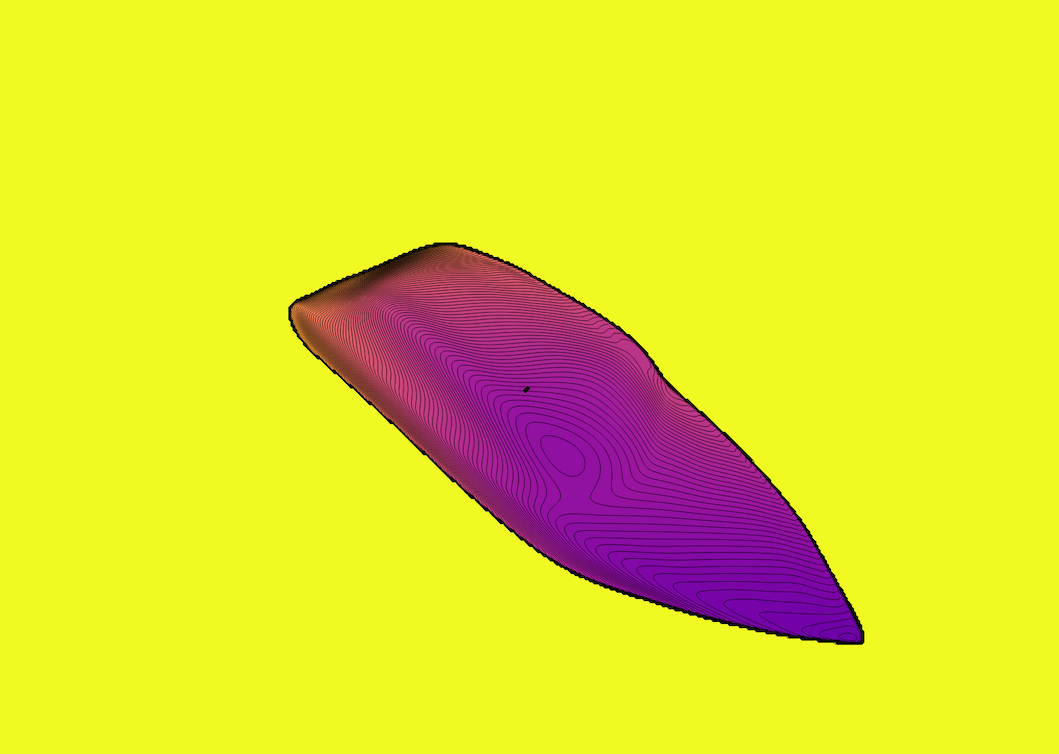}
\includegraphics[width=0.24\linewidth, trim={60mm 30mm 60mm 40mm}, clip]{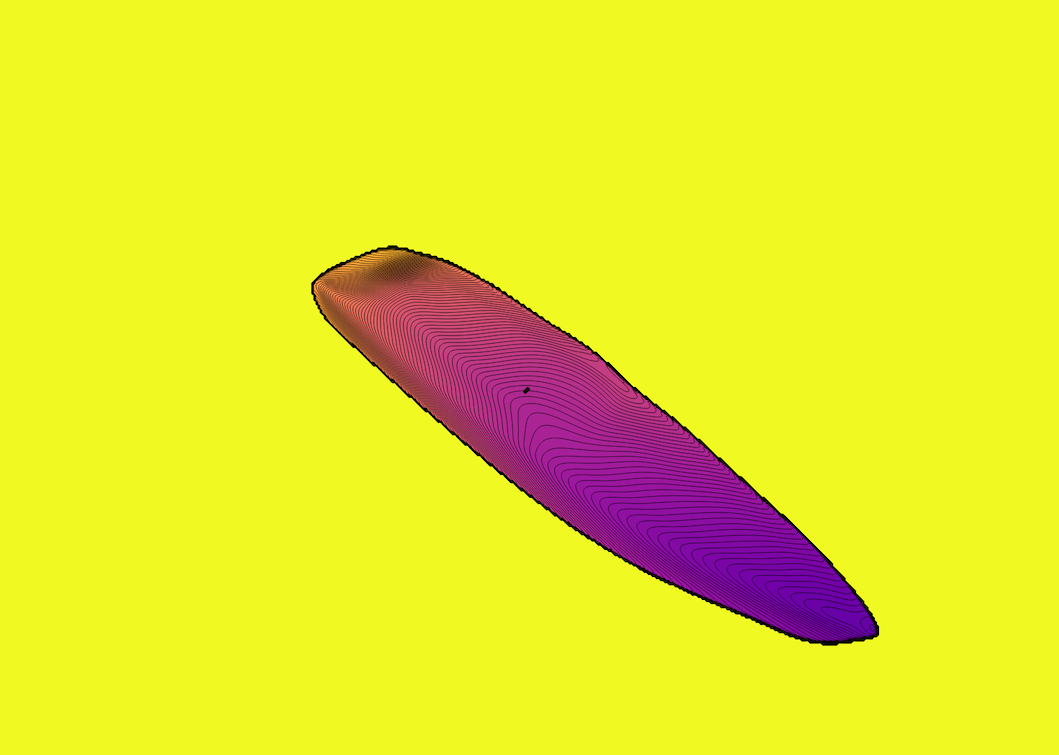}
\includegraphics[width=0.24\linewidth, trim={60mm 30mm 60mm 40mm}, clip]{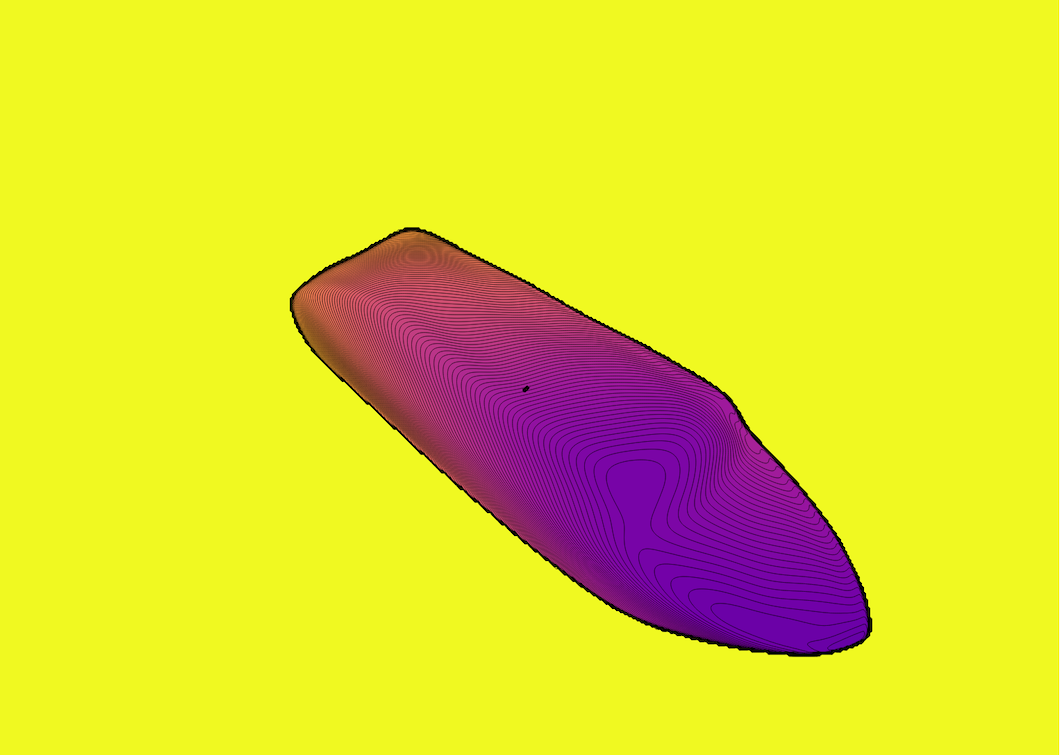}
\includegraphics[width=0.24\linewidth, trim={40mm 10mm 20mm 20mm}, clip]{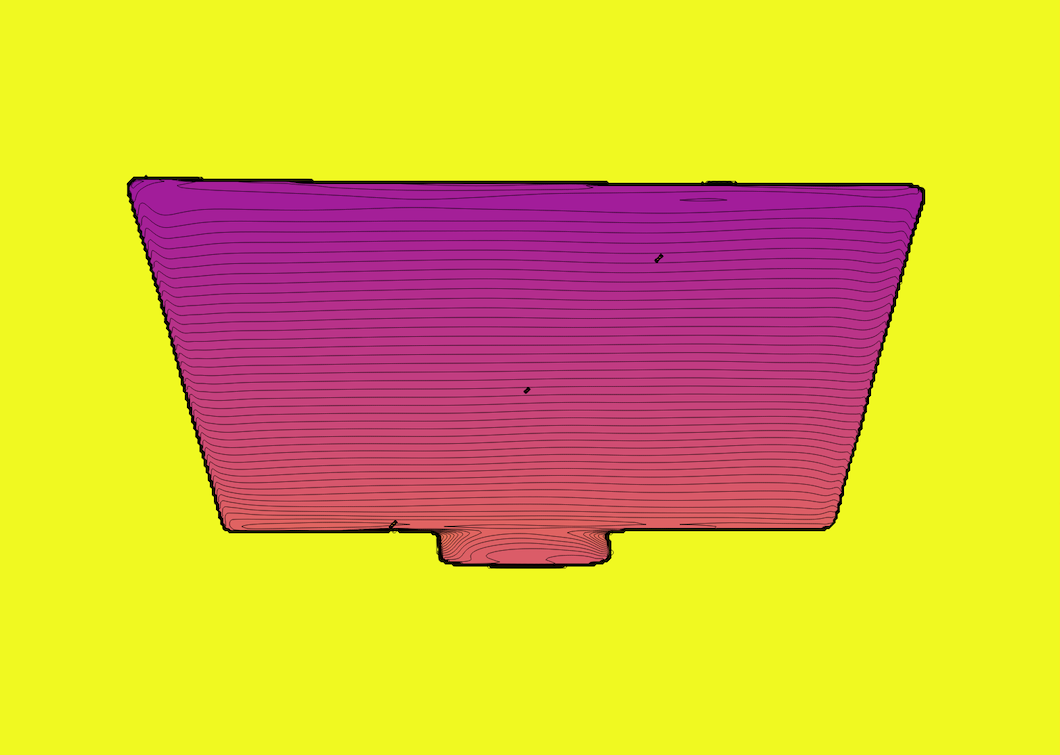}
\includegraphics[width=0.24\linewidth, trim={60mm 30mm 60mm 40mm}, clip]{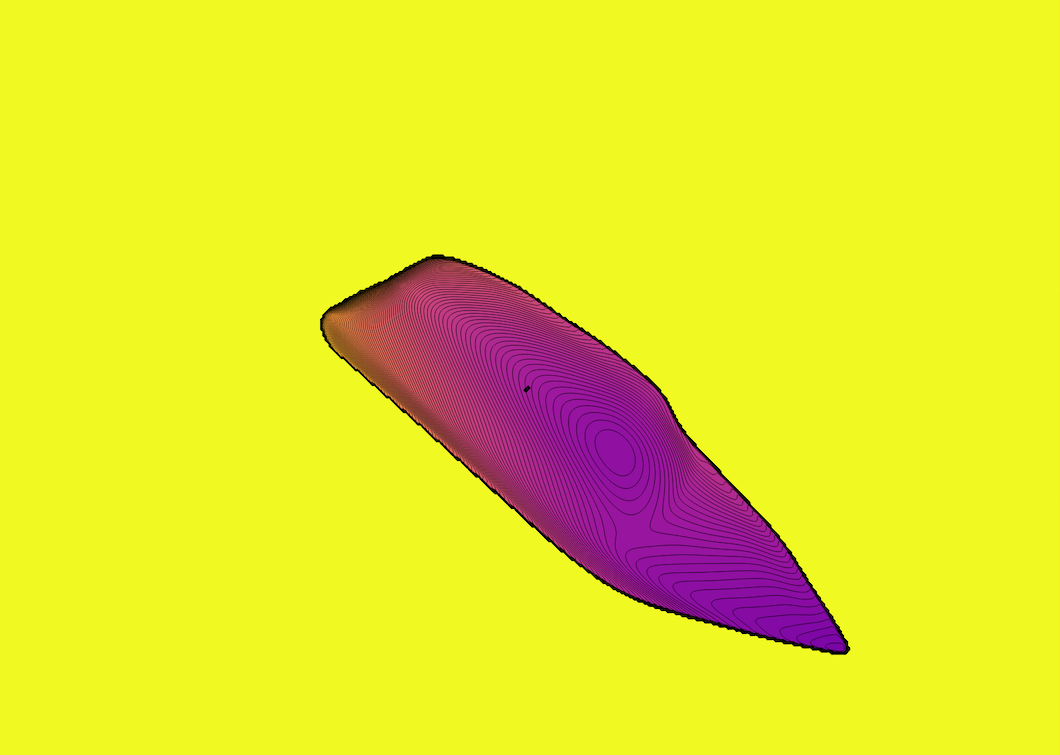}
\includegraphics[width=0.24\linewidth, trim={60mm 30mm 60mm 40mm}, clip]{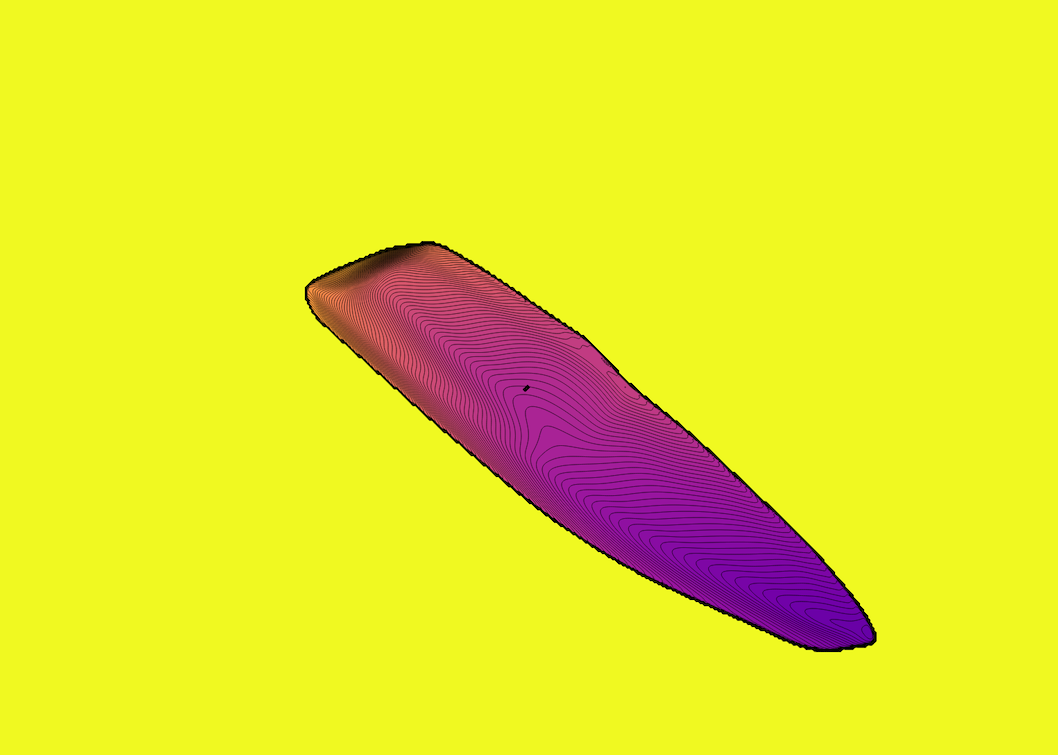}
\includegraphics[width=0.24\linewidth, trim={60mm 30mm 60mm 40mm}, clip]{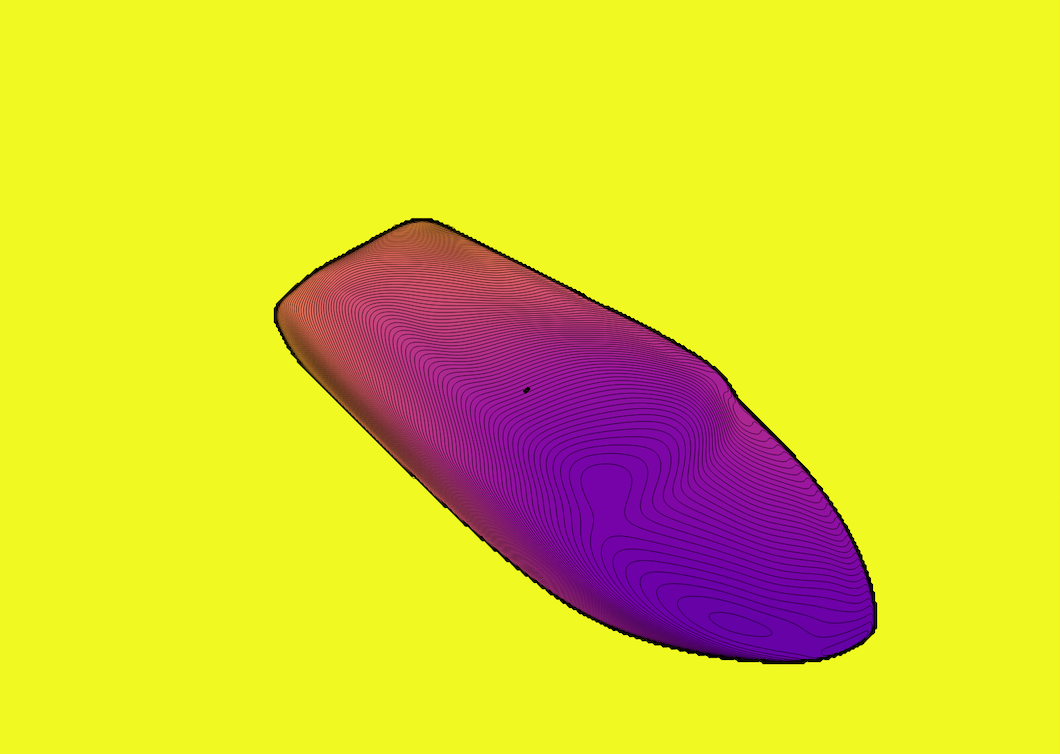}
\includegraphics[width=0.24\linewidth, trim={40mm 10mm 20mm 20mm}, clip]{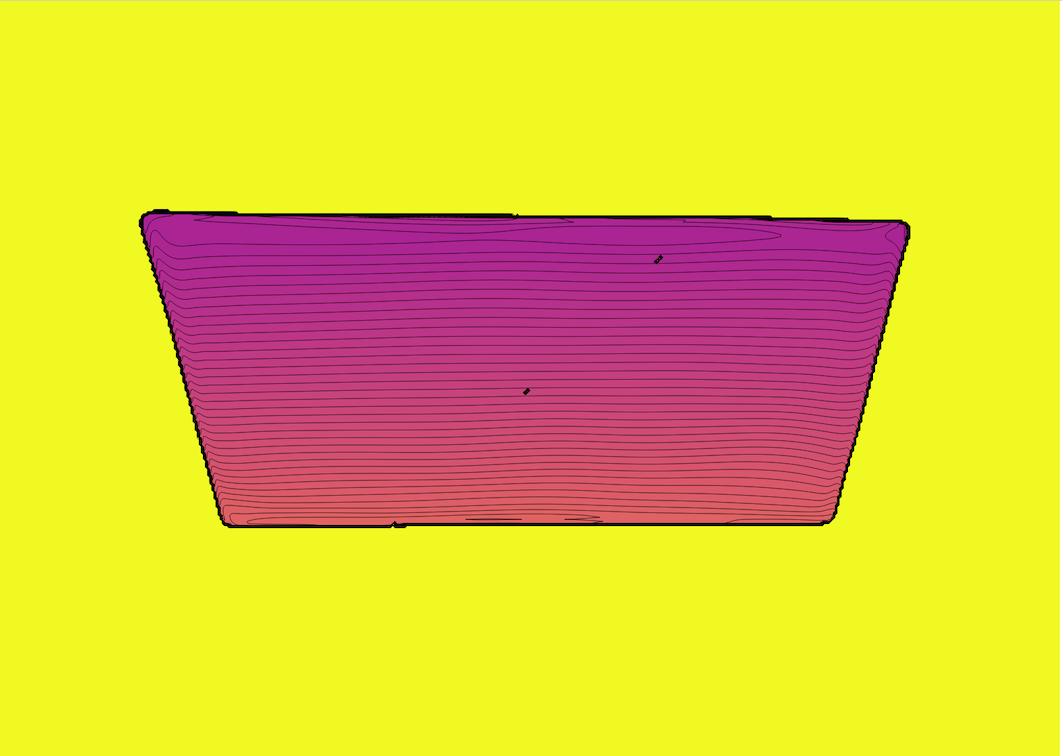}
%
%
%
\end{minipage}%
\caption{SDDF shape interpolation between two watercraft instances in the first three columns and two display instances in the last column. The first and last row show the SDDF output from the same view for two different instances from the training set. The rows in the middle are generated by using a weighted average of the latent codes of the upper-most and downer-most instances as an input to the SDDF network. In each column, from top to bottom, the latent code weights with respect to the upper-most instance are $1$, $0.75$, $0.5$, $0.25$, $0$, respectively. Note how the shapes transform smoothly from top to bottom with intermediate shapes looking like valid instances. This demonstrates that the SDDF model represents the latent shape space continuously and meaningfully.}
\label{fig:watercraft_display_interpolation}
\end{figure*}
%
%
%
%
%
\clearpage
\clearpage

\begin{figure*}[h!]
\centering
\includegraphics[width=0.49\linewidth, trim={0mm 0mm 0mm 0mm}, clip]{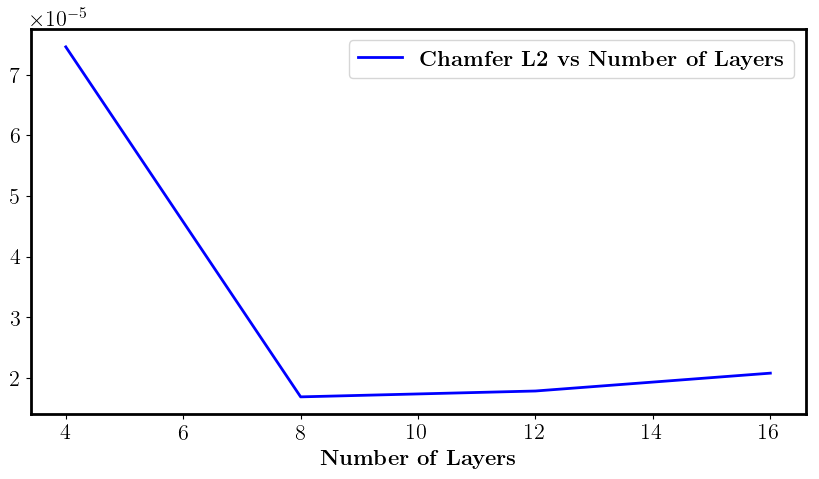}%
\hfill%
\includegraphics[width=0.49\linewidth, trim=0mm 0mm 0mm 0mm, clip]{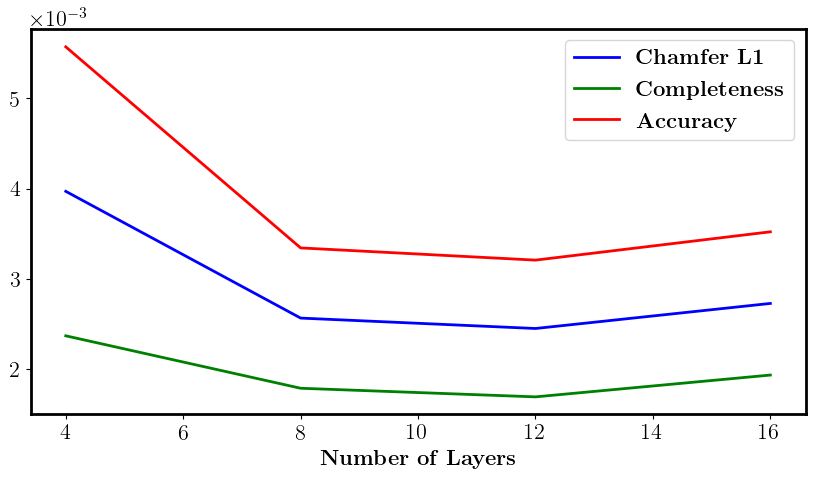}%
\hfill%
\includegraphics[width=0.49\linewidth, trim=0mm 0mm 0mm 0mm, clip]{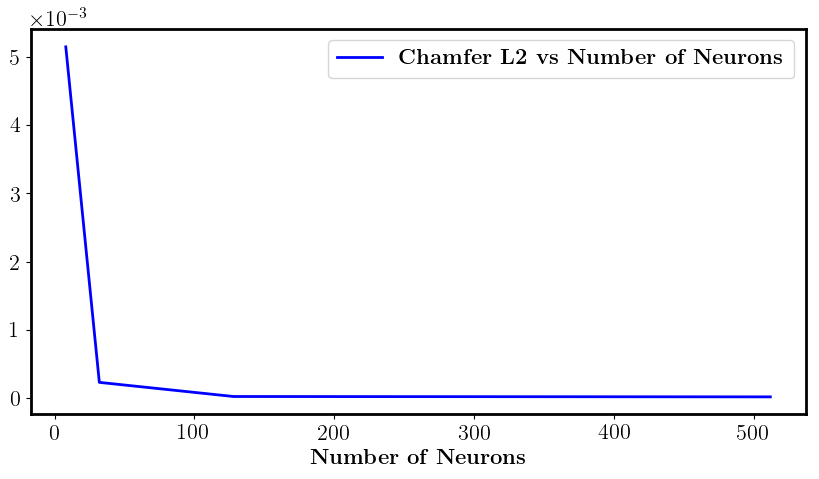}%
\hfill%
\includegraphics[width=0.49\linewidth, trim=0mm 0mm 0mm 0mm, clip]{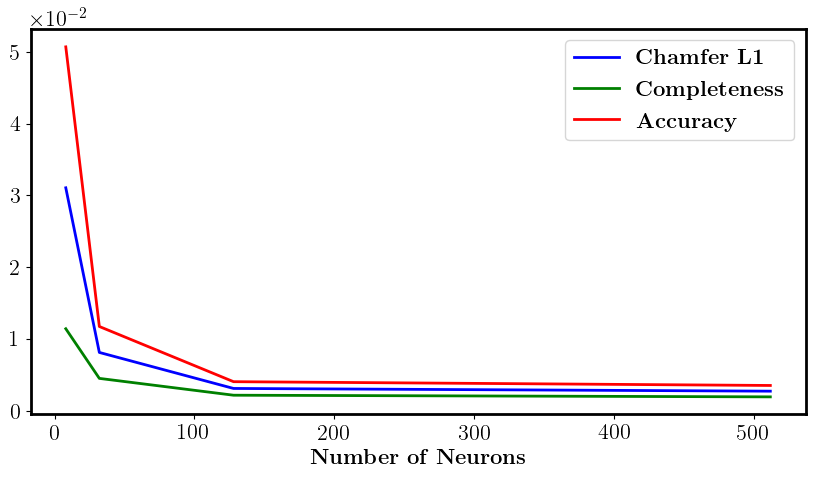}
\caption{The first row shows the error as the number of layers changes between $4,8,12,16$ layers. The second row shows the error as number of neurons changes between $8,32,128,512$. In each row the left figure is Chamfer-$L_2$ distance, and the right one contains Chamfer-$L_1$ distance, Accuracy and Completeness.}
\label{fig:curves}
\end{figure*}

\subsection{Effect of the Network Size on the Performance}

This section evaluates the effect of the number of layers and number of neurons per layer in the SDDF model on the performance qualitatively and quantitatively. The results are obtained using a single sofa instance, shown in Fig.~3 in main paper. We use the same settings as the single-object experiment in Sec.~5.1 with data augmentation from $10k$ random views. The default model in the paper has $16$ layers with $512$ neurons per layer with skip connections every $4$ layers.

First, we varied the number of layers in the neural network, while keeping a skip connection every $4$ layers. We evaluated the SDDF model with $4$ (with out skip connection), $8$, $12$, $16$ layers. Quantitatively, as seen in Fig.~\ref{fig:curves}, at first the error decreases significantly and then it increases slowly. Qualitatively, in Fig.~\ref{fig:newlayer}, with more layers the model can capture finer details about the shape but even with $8$ layers the shape is reconstructed very well.

Second, in the default setup with $16$ layers, we kept an equal number of neurons per layer but varied the number as $8$, $32$, $128$, $512$. As we see in Fig.~\ref{fig:curves}, at first the error decreases significantly and then continues to decrease slowly. In Fig.~\ref{fig:newcell}, we see that with fewer neurons per layer the model cannot capture the shape details very well. In comparison to the changing number of layers experiment, we see that the model is qualitatively more sensitive to the number of neurons.

\begin{figure*}[t]
\begin{minipage}{\linewidth}
\centering
\includegraphics[width=0.24\linewidth, trim={10mm 30mm 10mm 30mm}, clip]{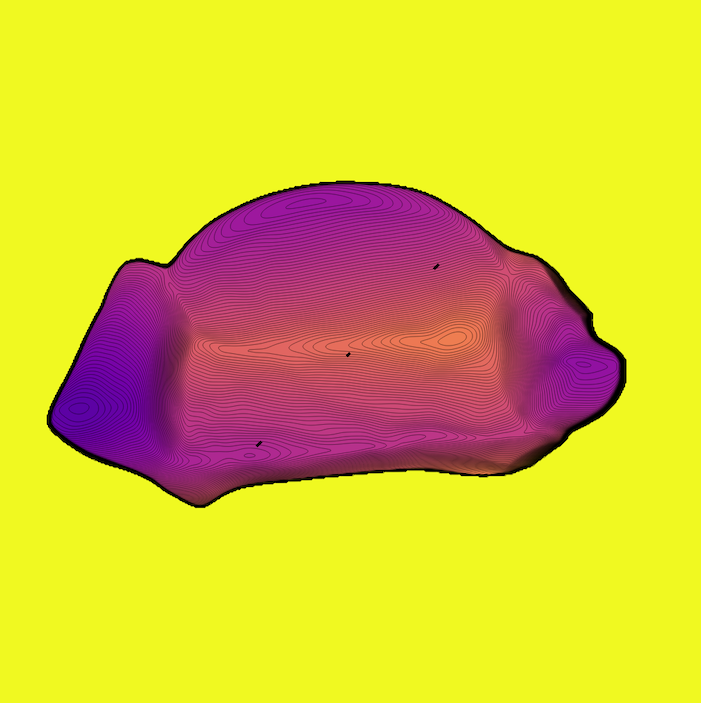}%
\hfill%
\includegraphics[width=0.24\linewidth, trim=10mm 30mm 10mm 30mm, clip]{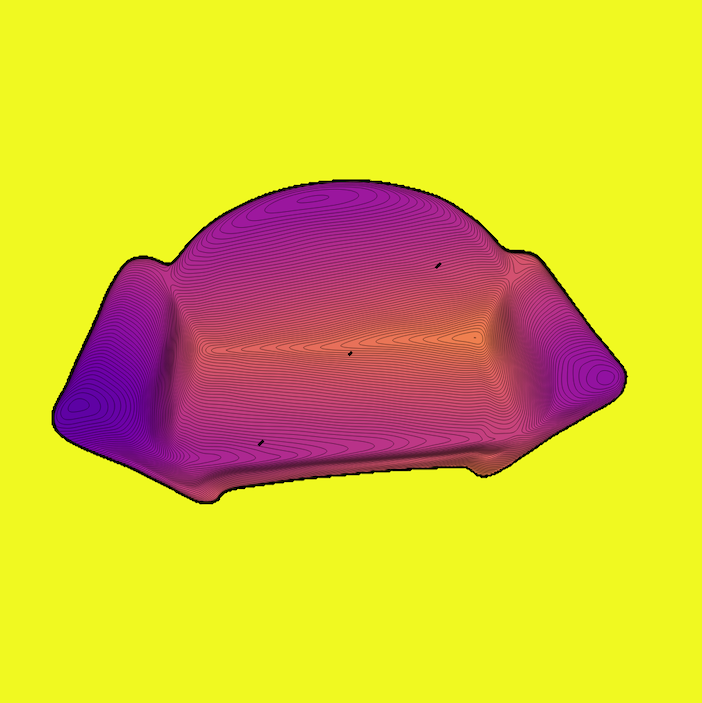}%
\hfill%
\includegraphics[width=0.24\linewidth, trim=10mm 30mm 10mm 30mm, clip]{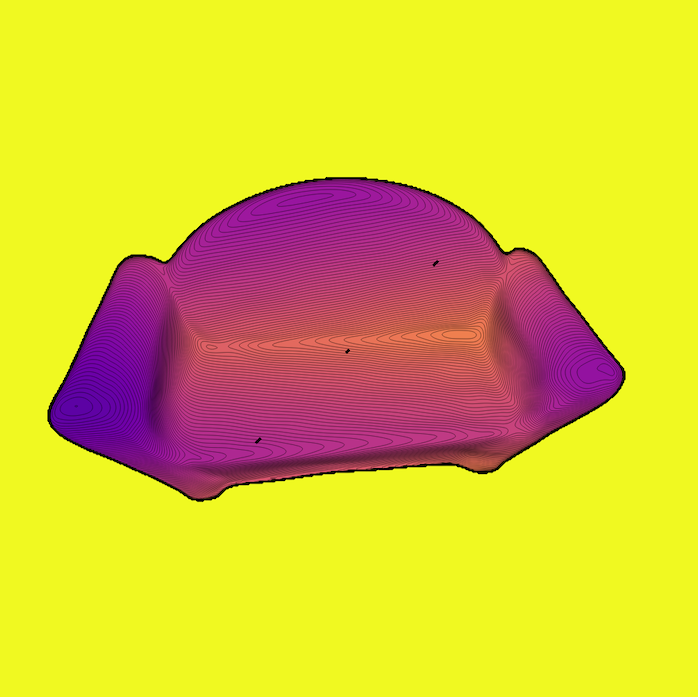}
\hfill%
\includegraphics[width=0.24\linewidth, trim={10mm 30mm 10mm 30mm}, clip]{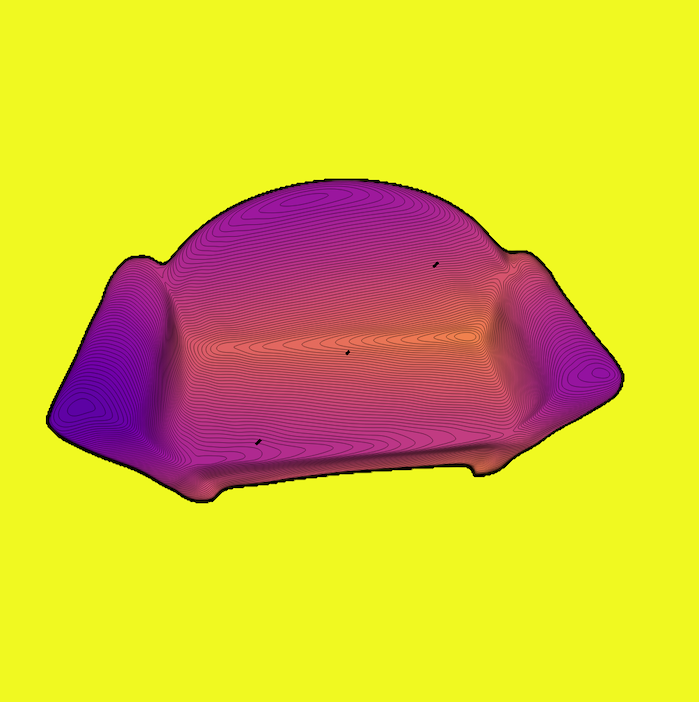}
\caption{Distance views synthesized by our SDDF model trained with the same data with $512$ neurons per layer and different numbers of layers: $4$ (first), $8$ (second), $12$ (third), $16$ (forth).}
\label{fig:newlayer}
\end{minipage}

%
\begin{minipage}{\linewidth}
\centering
\includegraphics[width=0.24\linewidth, trim={10mm 30mm 10mm 30mm}, clip]{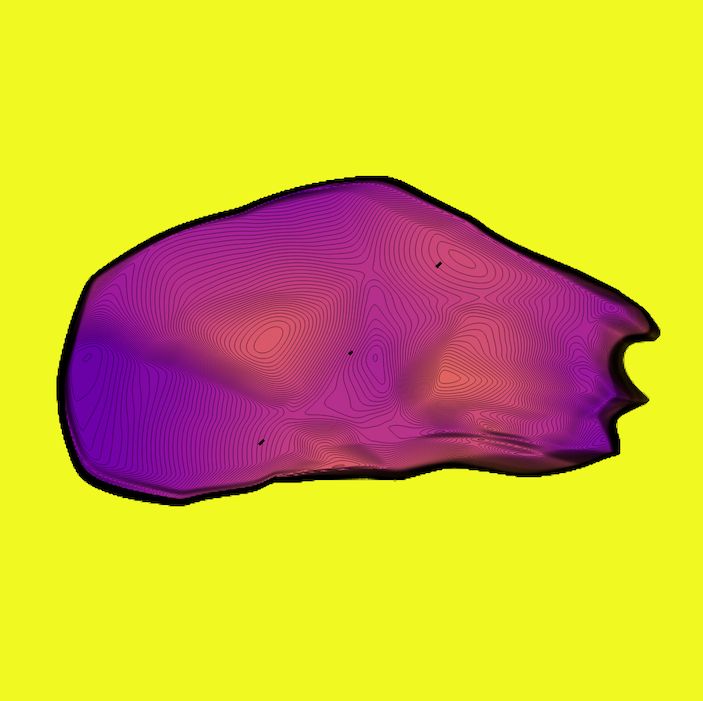}%
\hfill%
\includegraphics[width=0.24\linewidth, trim=10mm 30mm 10mm 30mm, clip]{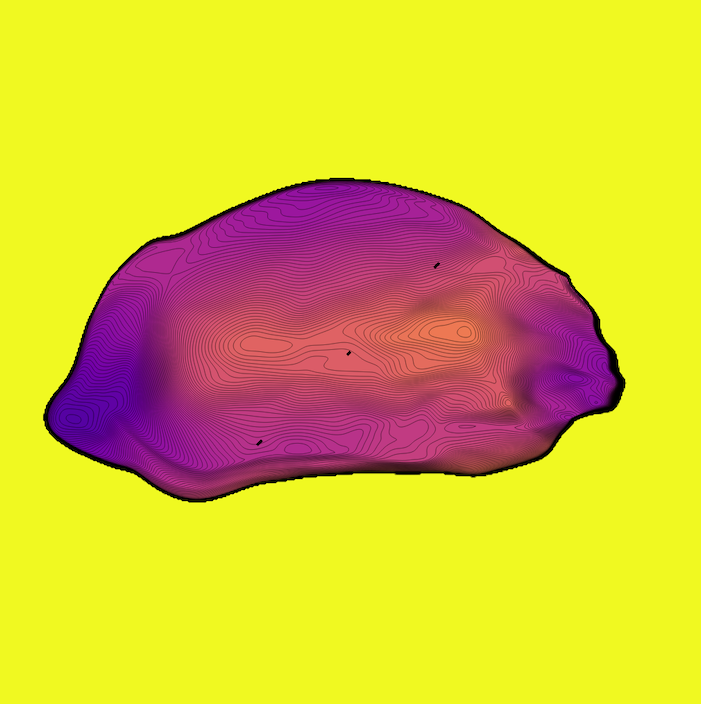}%
\hfill%
\includegraphics[width=0.24\linewidth, trim=10mm 30mm 10mm 30mm, clip]{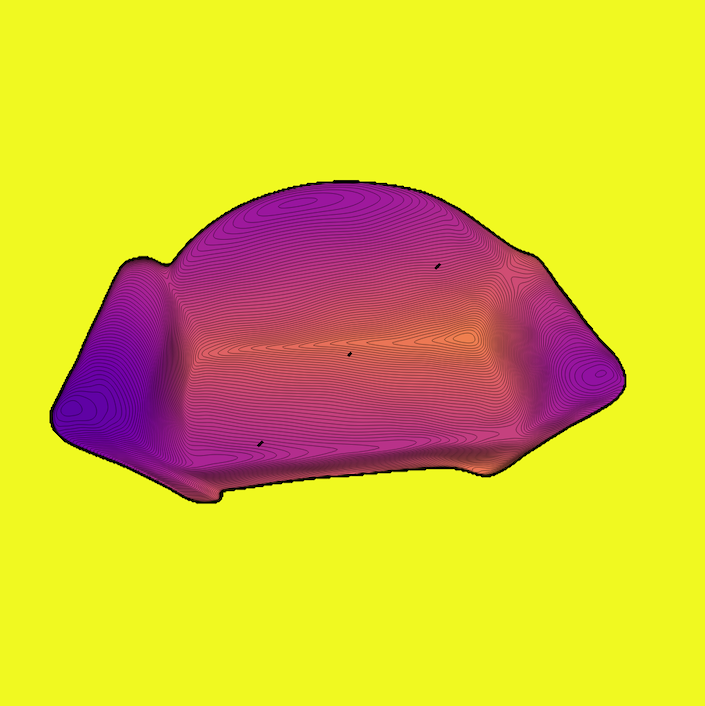}
\hfill%
\includegraphics[width=0.24\linewidth, trim={10mm 30mm 10mm 30mm}, clip]{fig/layerandcells/16layers.png}
\caption{Distance views synthesized by our SDDF model trained with the same data with $16$ layers and different numbers of neurons per layer: $8$ (first), $32$ (second), $128$ (third), $512$ (forth).}
\label{fig:newcell}
\end{minipage}
\end{figure*}

\subsection{Proof of Lemma~5}
This section provides the proof of Lemma~5 and a brief intuitive discussion on difference of the spherical convex hull computation and its discretized approximation.

\setcounter{lemma}{4}
\begin{lemma}\label{lem:cnv3D2D}
Let $\mathcal{P}$ be a set of points on the unit sphere. The boundary points of $\mathcal{P}$ with respect to $\bfe_3$ are points that are adjacent vertices to $\bfe_3$ in the convex hull of $\mathcal{P} \cup \{\bfe_3\}$. Let $m$ be a function that maps a point on the unit sphere to the plane $z=0$ with center $\bfe_3$, i.e., $m([x,y,z]^\top) := [\frac{x}{1-z}, \frac{y}{1-z}]^\top$. A point $\bfu \in \mathcal{P}$ is a boundary point if and only if $m(\bfu)$ is a vertex of the convex hull of $m(\mathcal{P})$.
\end{lemma}

\begin{proof}
For $\bfu = [\bfu_x, \bfu_y, \bfu_z]^\top \in \mathcal{P}$, suppose that $m(\bfu)$ is not a vertex of the convex hull of $m(\mathcal{P})$. Then, there exists a set of points $\{\bfu^i = [\bfu^i_x, \bfu^i_y, \bfu^i_z]^\top\}_{i=0}^n \subset \cal P \setminus \{\bf u\}$ and coefficients $\{\alpha_i\}_{i=0}^n$, $0<\alpha_i<1$, $\sum_{i=0}^n \alpha_i = 1$ such that $m({\bf u}) = \sum_{i=0}^n \alpha_i m({\bf u}^i)$. Let $\beta := \frac{1}{\sum_{i=0}^n\alpha_i\frac{1-\bfu_z}{1-\bfu^i_z}}$, and $\gamma_i = \beta \alpha_i \frac{1-\bfu_z}{1-\bfu^i_z}$, so that $\sum_{i=0}^n\gamma_i=1$. Since all points are on the unit sphere, we have $\gamma_i > 0$, $\beta > 0$, and:
\begin{align*}
    \sum_{i=0}^n\gamma_i {\bf u}^i_x &= \beta (1-{\bf u}_z)\sum_{i=0}^n\alpha_i\frac{{\bf u}^i_x}{1-{\bf u}^i_z} \\
    &= \beta (1-{\bf u}_z) \frac{{\bf u}_x}{1-{\bf u}_z} = \beta {\bf u}_x,\\
    \sum_{i=0}^n\gamma_i {\bf u}^i_z &= \beta (1-{\bf u}_z) \sum_{i=0}^n (\alpha_i \frac{1-(1-{\bf u}^i_z)}{1-{\bf u}^i_z})\\
    &=\beta \sum_{i=0}^n\alpha_i\frac{1-{\bf u}_z}{1-{\bf u}^i_z} - \beta (1-{\bf u}_z) (\sum_{i=0}^n \alpha_i)\\
    &= 1-\beta (1-{\bf u}_z).
\end{align*}
Hence, $\sum_{i=0}^n\gamma_i \bfu^i=[\beta \bfu_x, \beta \bfu_y, 1-\beta (1- \bfu_z)]^\top = (1-\beta) {\bfe_3}+\beta \bfu$. Note that $(1-\beta) {\bfe_3}+\beta{\bf u}$ is on the segment from $\bfe_3$ to $\bf u$, so $0<\beta<1$; otherwise the convex combination of points on the unit sphere ($\sum_{i=0}^n\gamma_i {\bf u}^i$) will be out of unit sphere. This means that the segment between $\bfe_3$ and $\bf u$ intersects with the convex hull of $\mathcal{P}$ at another point $(1-\beta) {\bfe_3}+\beta \bfu$, which implies that $\bfu$ is not a boundary point of $\mathcal{P}$. The converse statement can be proven similary by reversing the steps above. 
\end{proof}

For usual physically plausible objects, a few views (depending on the object shape complexity) with sufficiently high distance image resolution are sufficient to obtain a distance image from any other view using the convex hull approach described in Sec. 4.5 of the main paper. Intuitively, after transforming the object surface to a unit sphere with respect to an arbitrary point on the object surface, the region from which the sphere origin is observable may be decomposed into few convex regions. Hence, considering few distance views with sufficiently high resolutions should be sufficient to encode the real object shape. On the other hand, the discretized approximation of the convex hull is very efficient but its accuracy is limited by the azimuth resolution.


\begin{figure*}[t]
  \centering
  \includegraphics[width=0.24\linewidth, trim=0mm 0mm 0mm 0mm, clip]{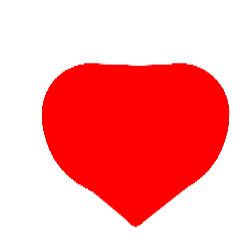}%
  \includegraphics[width=0.24\linewidth, trim=15mm 5mm 25mm 15mm, clip]{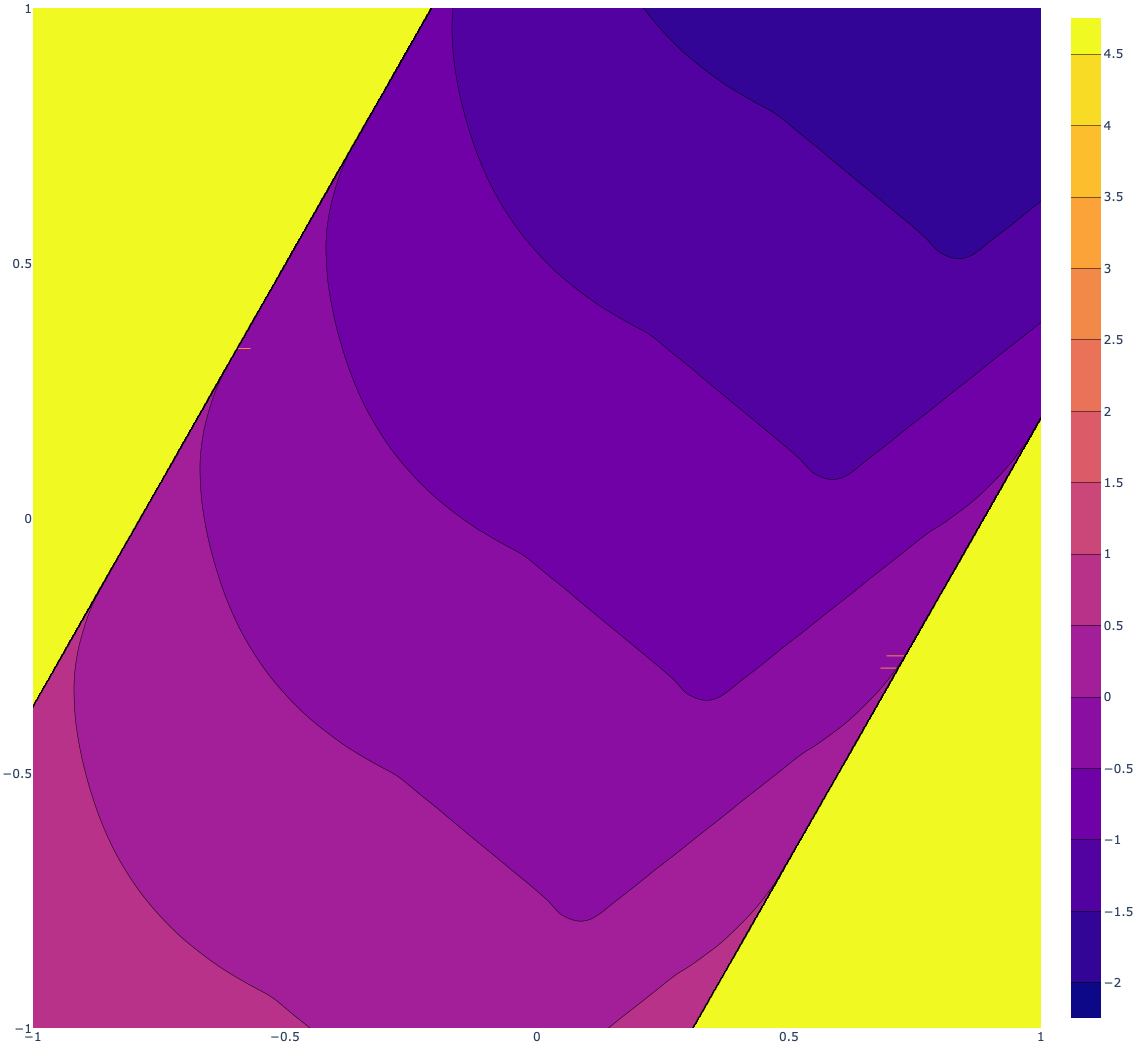}%
  \includegraphics[width=0.24\linewidth, trim=15mm 5mm 25mm 15mm, clip]{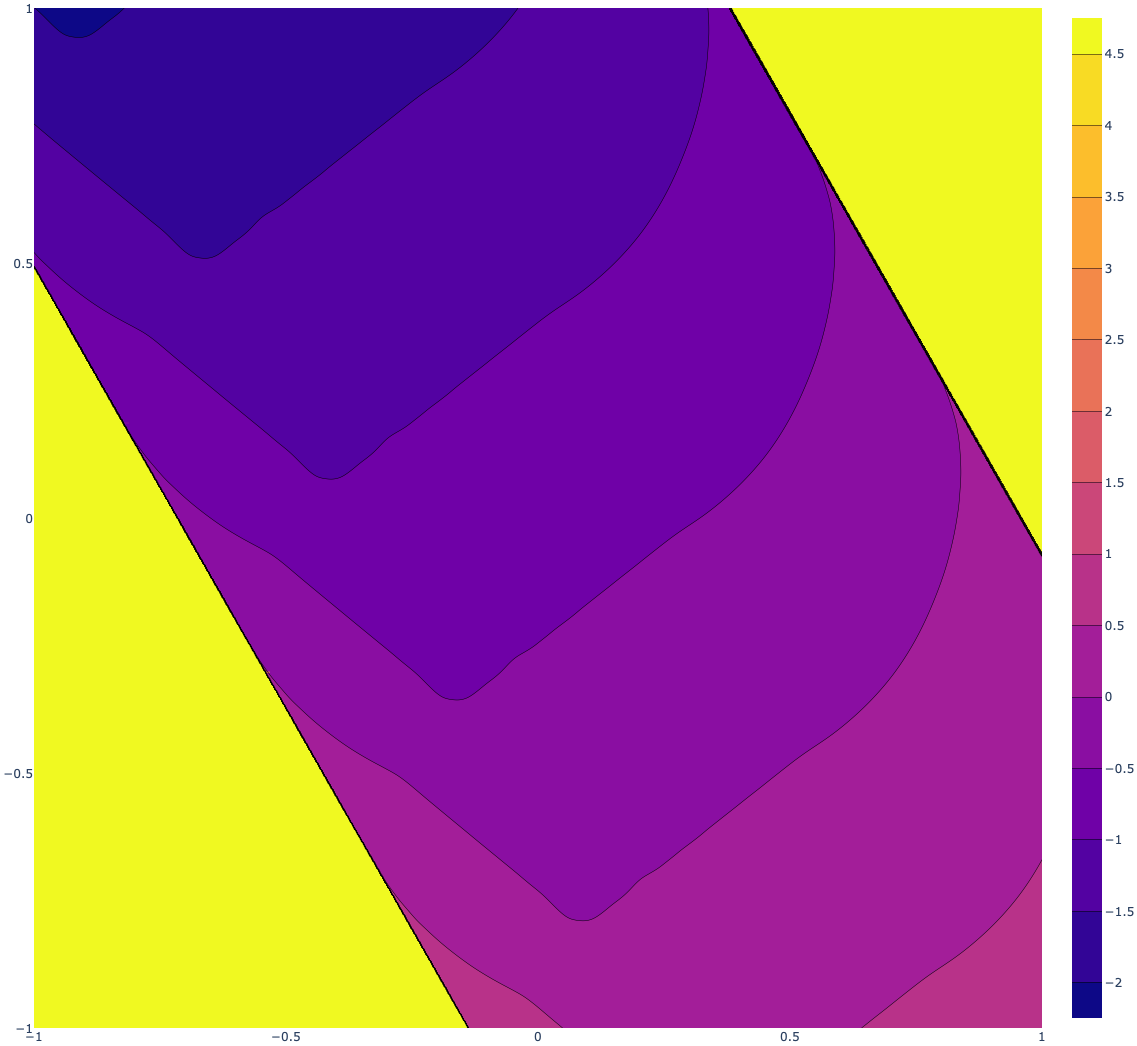}%
  \includegraphics[width=0.24\linewidth, trim=15mm 5mm 25mm 15mm, clip]{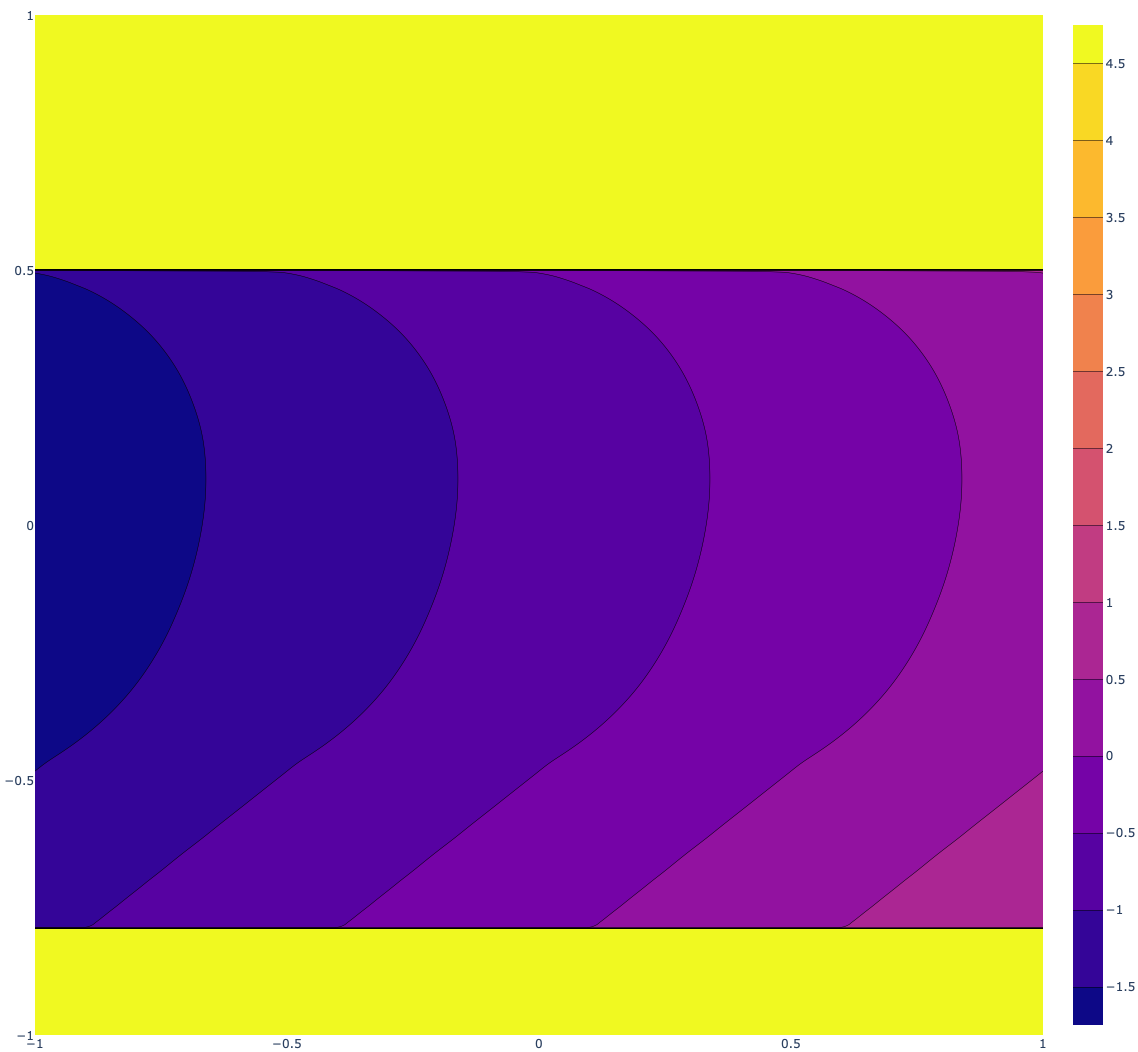}\\
  \includegraphics[width=0.24\linewidth, trim=30mm 30mm 30mm 35mm, clip]{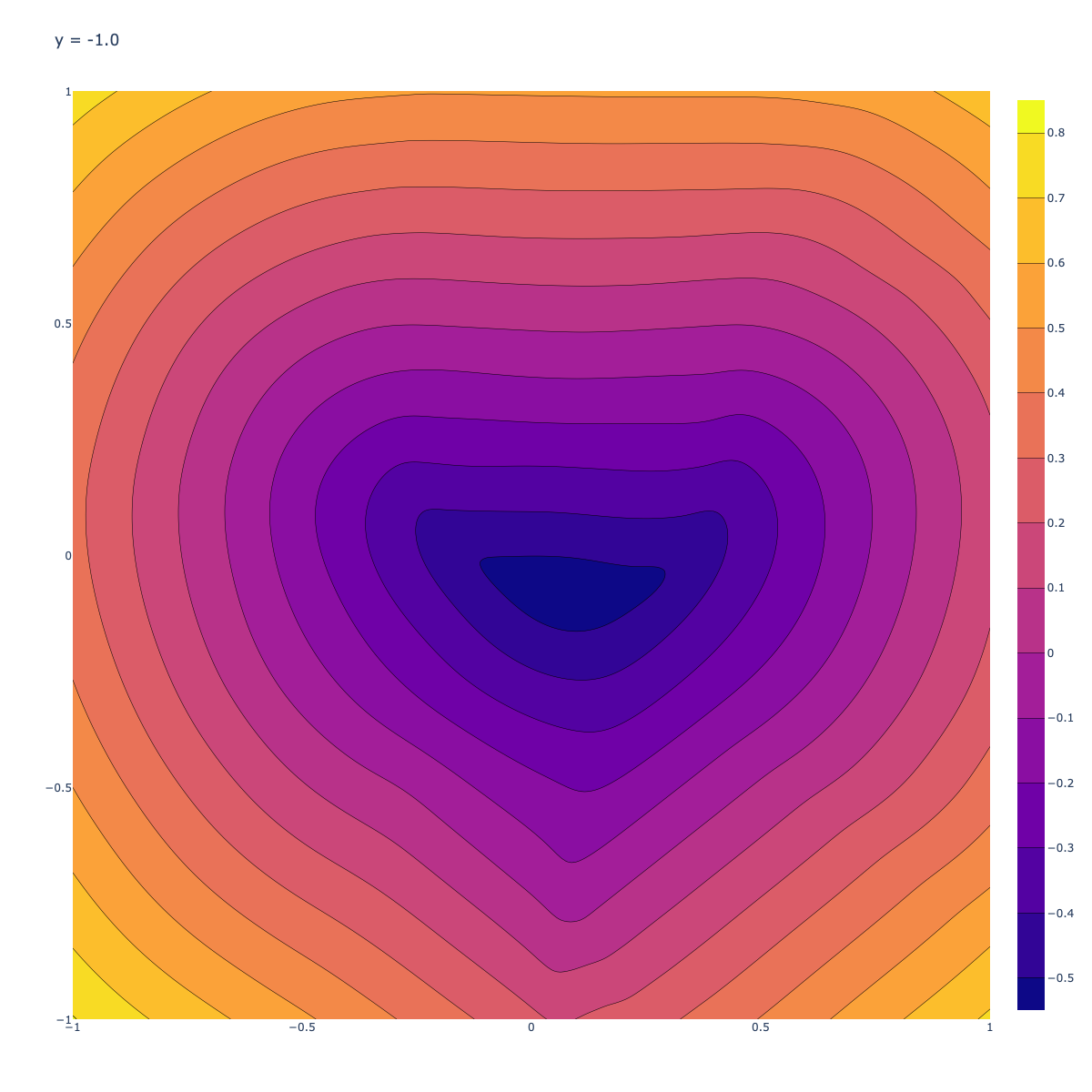}%
  \includegraphics[width=0.24\linewidth, trim=15mm 5mm 25mm 15mm, clip]{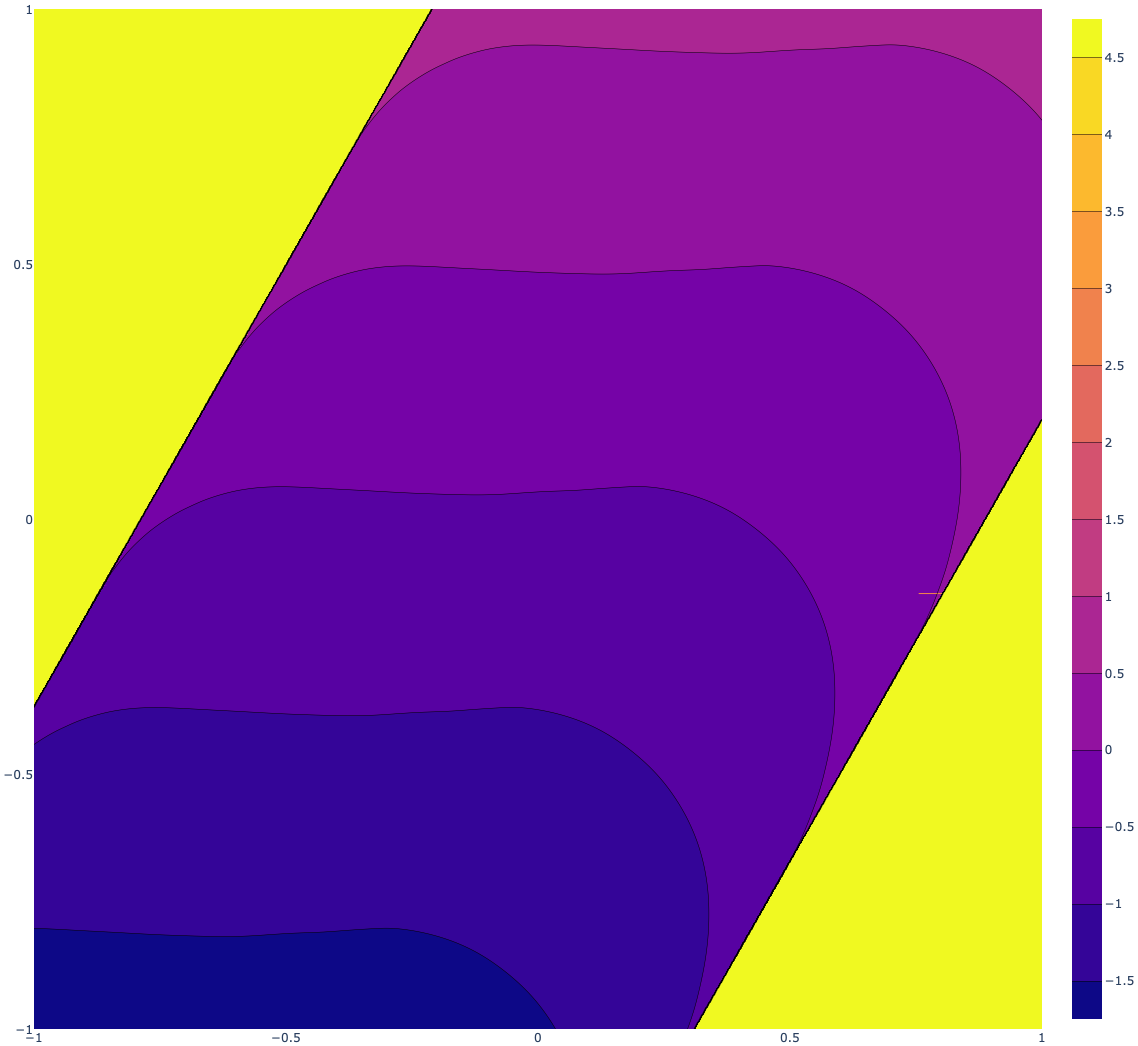}%
  \includegraphics[width=0.24\linewidth, trim=15mm 5mm 25mm 15mm, clip]{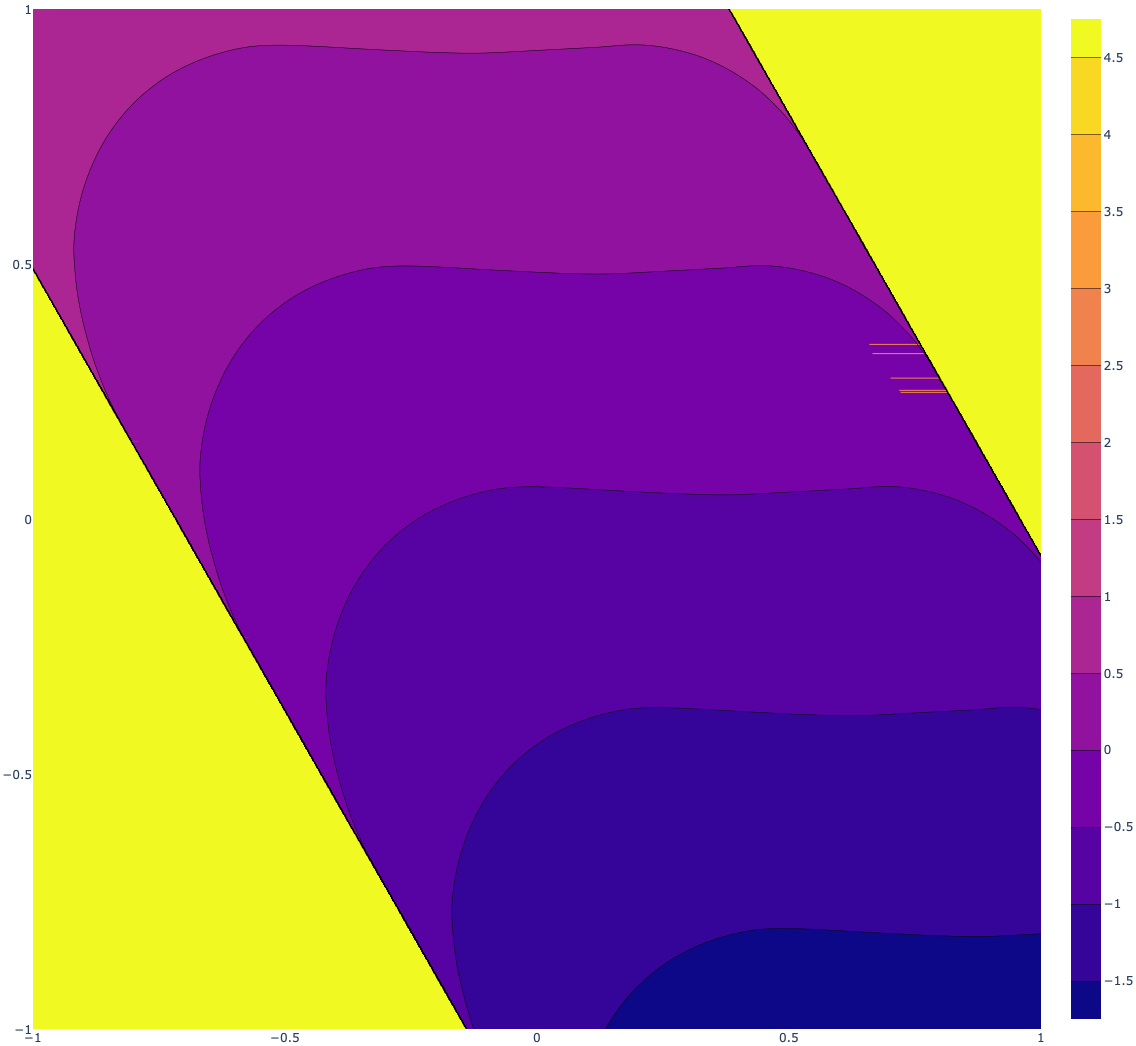}%
  \includegraphics[width=0.24\linewidth, trim=15mm 5mm 25mm 15mm, clip]{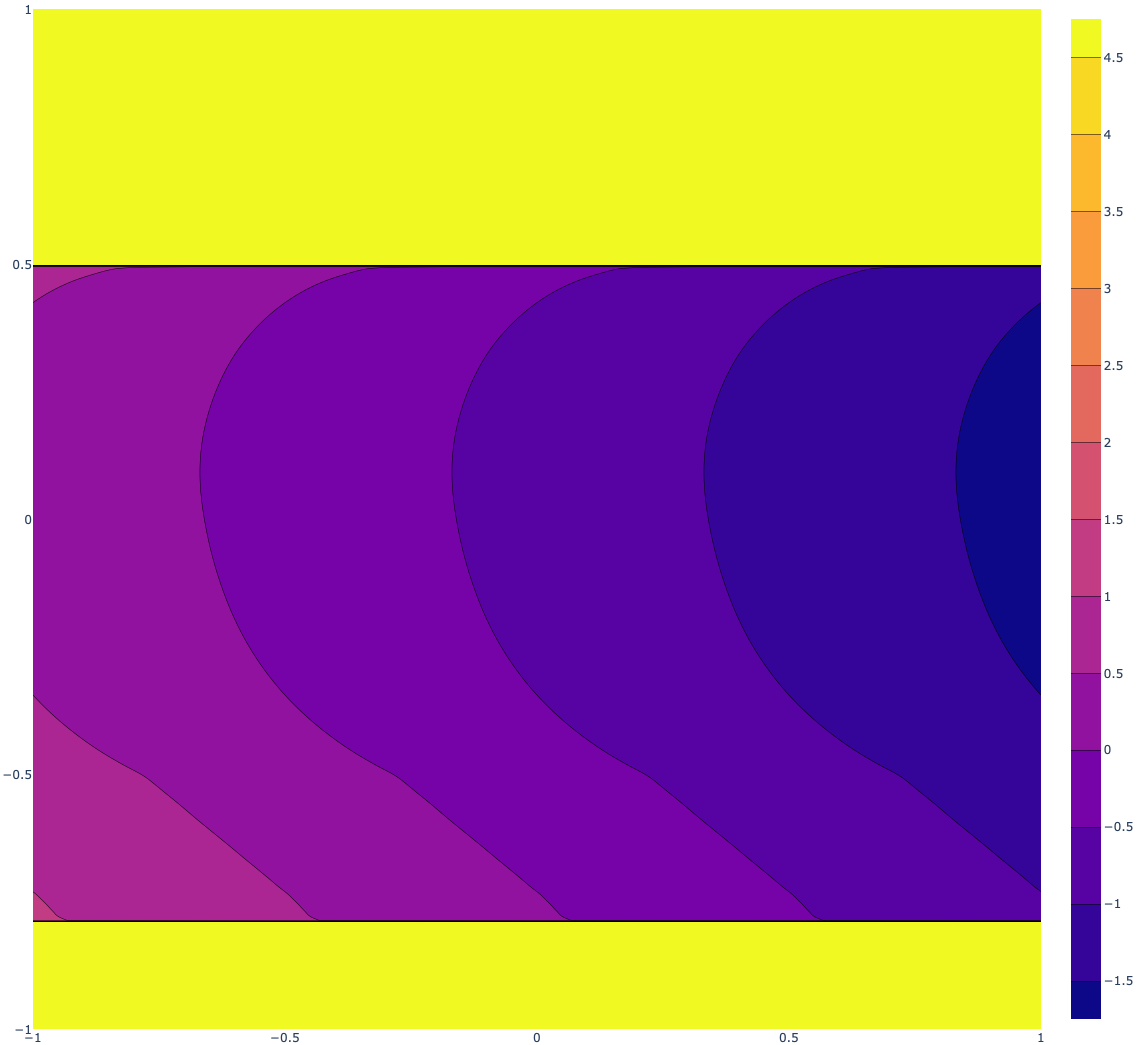}
  \caption{Ground-truth 2D instance (top left) and signed distance field (bottom left) learned by IGR~\cite{gropp2020implicit}. The remaining plots show the output of our SDDF model with a fixed direction at each 2D location in the image. The fixed viewing directions for the six plots from left to right and top to down are $\frac{\pi}{3}$, $\frac{2\pi}{3}$, $\pi$, $\frac{4\pi}{3}$, $\frac{5\pi}{3}$, $2\pi$, respectively. To produce good color contrast, in all images infinite distance values (corresponding to rays in free space) are set to $1$.}
  \label{fig:our2Dheart}
\end{figure*}
%
%

%
\subsection{2D Evaluation}

This section shows that an SDDF model can be used in 2D, e.g., with distance measurements obtained from a LiDAR scanner. We simulated a Hokuyo UTM-30LX Lidar scanner with 1081 rays per scan moving along a manually specified trajectory in an environment containing a 2D shape. See Fig.~\ref{fig:our2Dheart} for an example. The Lidar scans were used as training data for the SDDF and IGR models. After training, the SDDF network can generate distance values to the object contours at any location and in any viewing direction. Fig.~\ref{fig:our2Dheart} visualizes the models learned by IGR and SDDF for a heart-shaped 2D object. The distance predictions of the SDDF model are shown at every 2D location for several fixed viewing directions. We see that the SDDF model recognizes the boundary between free space and the object well. The parallel distance level sets indicate that the condition in Lemma~1 in the main paper indeed holds everywhere.

\subsection{Small Training Set without Synthesized Data}
In this section, we study the performance of the SDDF model when only a small training set is available and the data augmentation technique, described in Sec. 4.5 is not used. The results are generated using $100$ car and $200$ airplane instances from the ShapeNet dataset~\cite{chang2015shapenet}. To generate training data, we use the functions in PyTorch3D~\cite{ravi2020pytorch3d} for ray casting. For each object, we choose $1000$ random locations uniformly distributed on a sphere with orientations facing the object. Each distance image was down-sampled to have at most $5000$ finite rays and $5000$ infinite rays. An SDDF model with 8 fully connected layers, 512 hidden units per layer, and a skip connection from the input to the middle layer is used. All experiments are done on a single GTX $1080$ Ti GPU using PyTorch \cite{paszke2017automatic} and the ADAM optimizer \cite{kingma2014adam} with initial learning rate of $0.005$. 

\paragraph{Single Instance Shape Representation.}
For single-instance shape estimation, we schedule the learning rate to decrease by a factor of $2$ every $500$ steps for $9k$ iterations. In each iteration, we pick a batch of $128^2$ samples randomly from the input data. The result is provided in Fig.~\ref{fig:geometry}.

\paragraph{Shape Completion and Interpolation.}
For Car category-level shape estimation, the network is trained for $500$ epochs with $100$ instances. The learning rate is scheduled to decrease by a factor of $2$ every $50$ epochs. In each iteration, for each shape, $8000$ random samples are picked uniformly out of the training data for that instance. For the Airplane category, the network is trained for $5000$ epochs with $200$ shapes. The learning rate is scheduled to decrease by a factor of $2$ every $250$ epochs. In each iteration, for each instance $10k$ random samples are picked uniformly. Shape completion results are provided in Fig.~\ref{fig:shpcmp}, while shape interpolation results are provided in Fig.~\ref{fig:interpolation}.

\begin{figure*}[t]
\begin{minipage}{0.485\textwidth}
\scalebox{1}{\includegraphics[width=0.49\linewidth, trim={120mm 40mm 95mm 60mm}, clip]{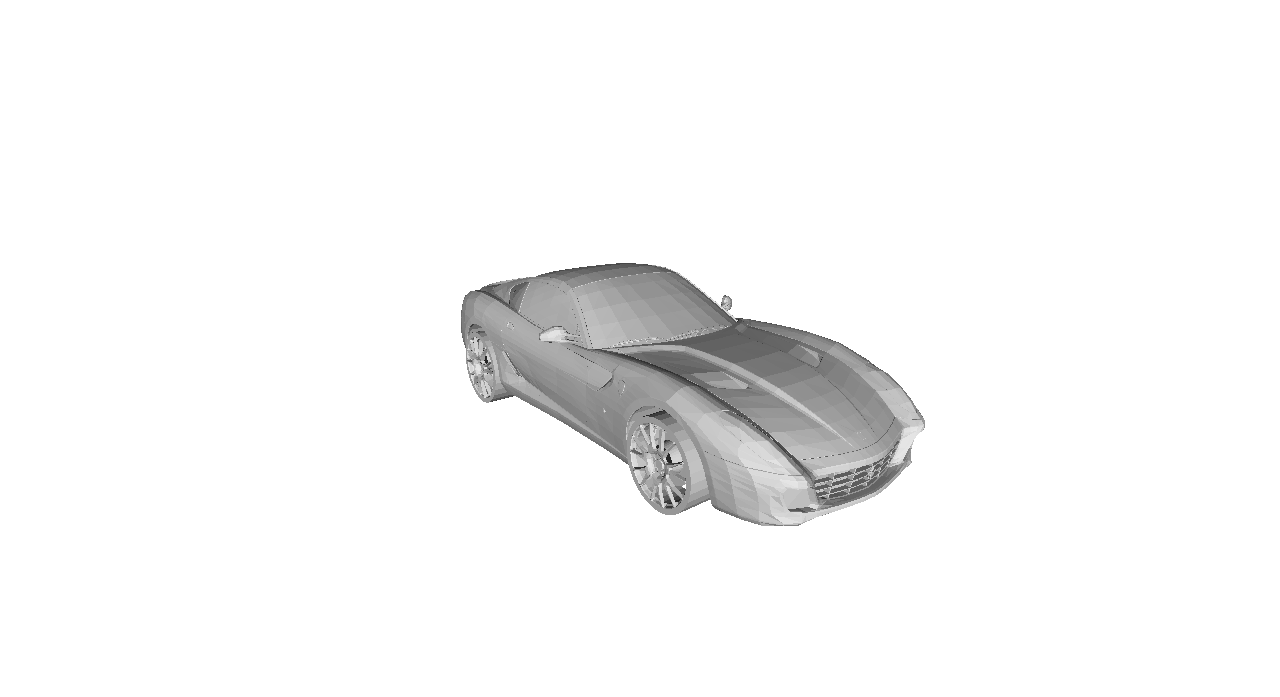}}
\includegraphics[width=0.49\linewidth, trim=60mm 50mm 50mm 130mm, clip]{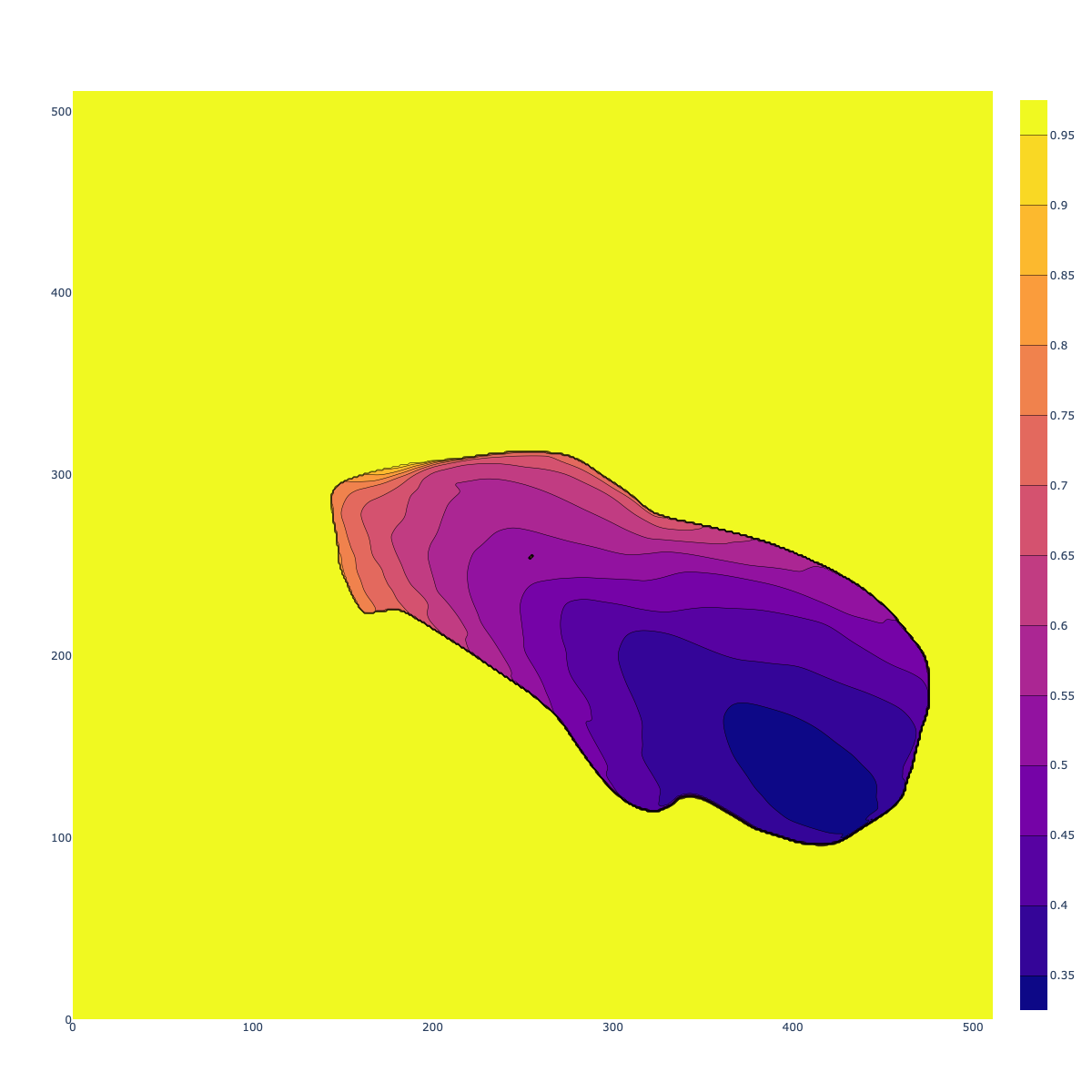}
\includegraphics[width=0.49\linewidth, trim=80mm 50mm 50mm 130mm, clip]{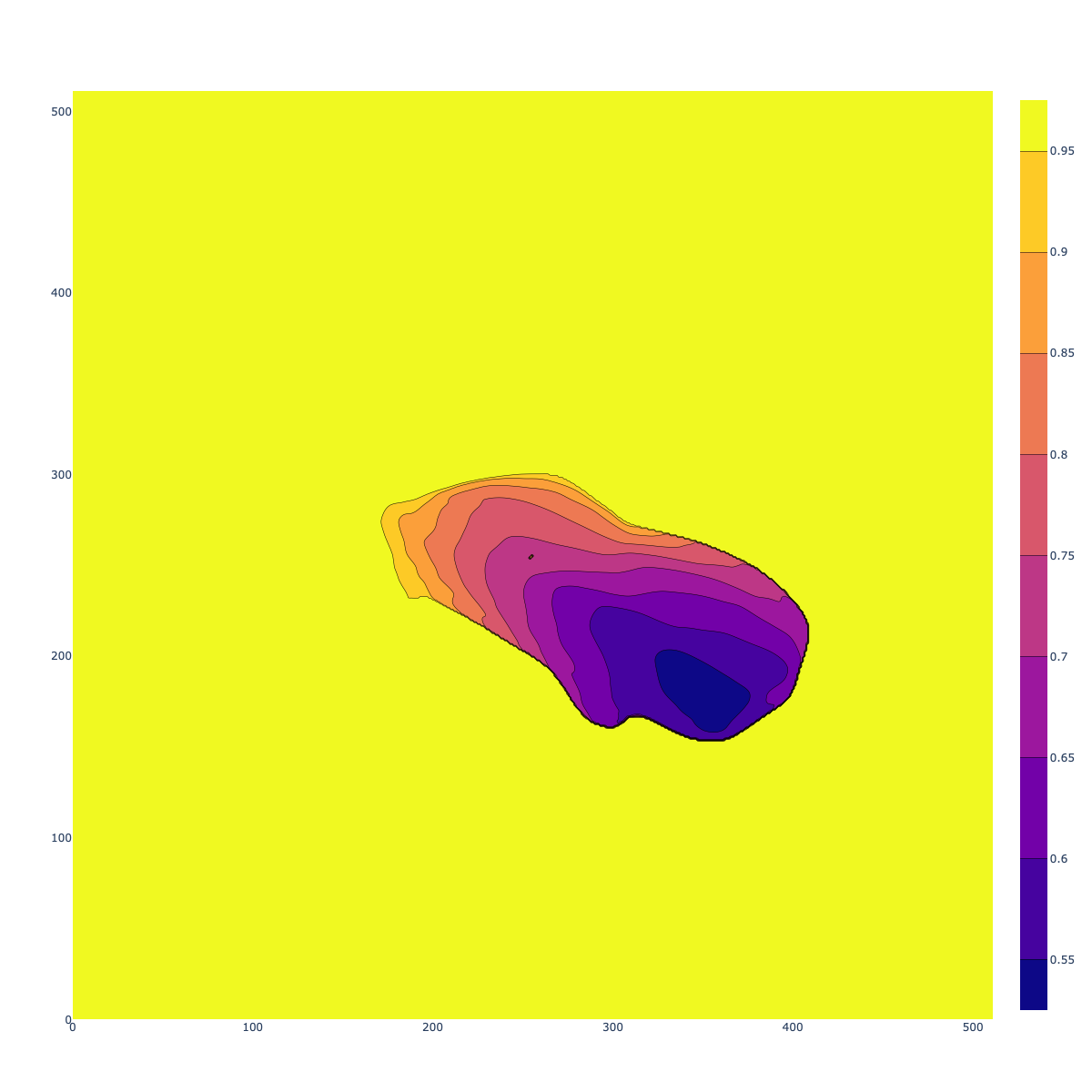}
\includegraphics[width=0.49\linewidth, trim=10mm 20mm 40mm 40mm, clip]{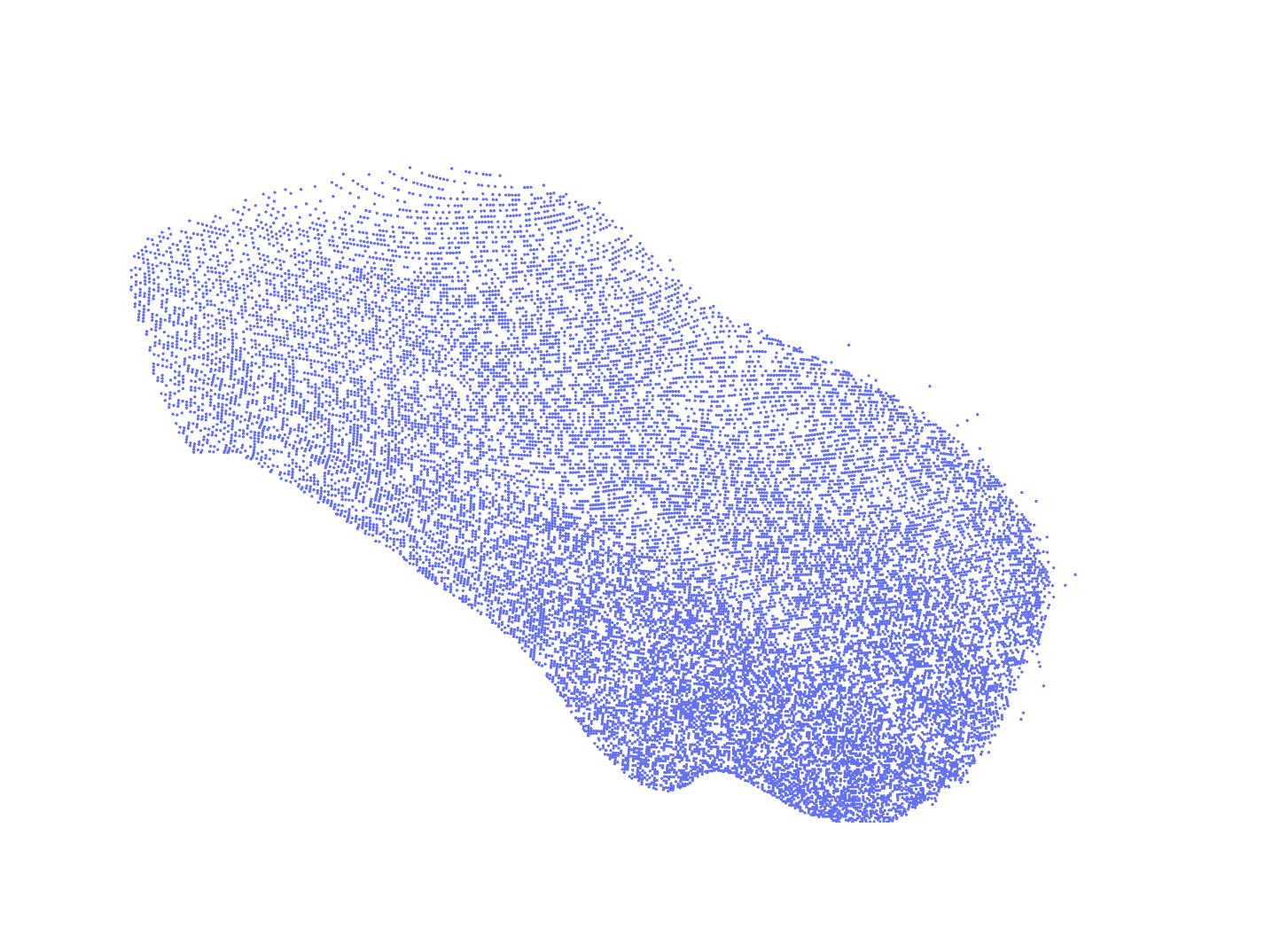}
\includegraphics[width=0.49\linewidth, trim=80mm 50mm 50mm 130mm, clip]{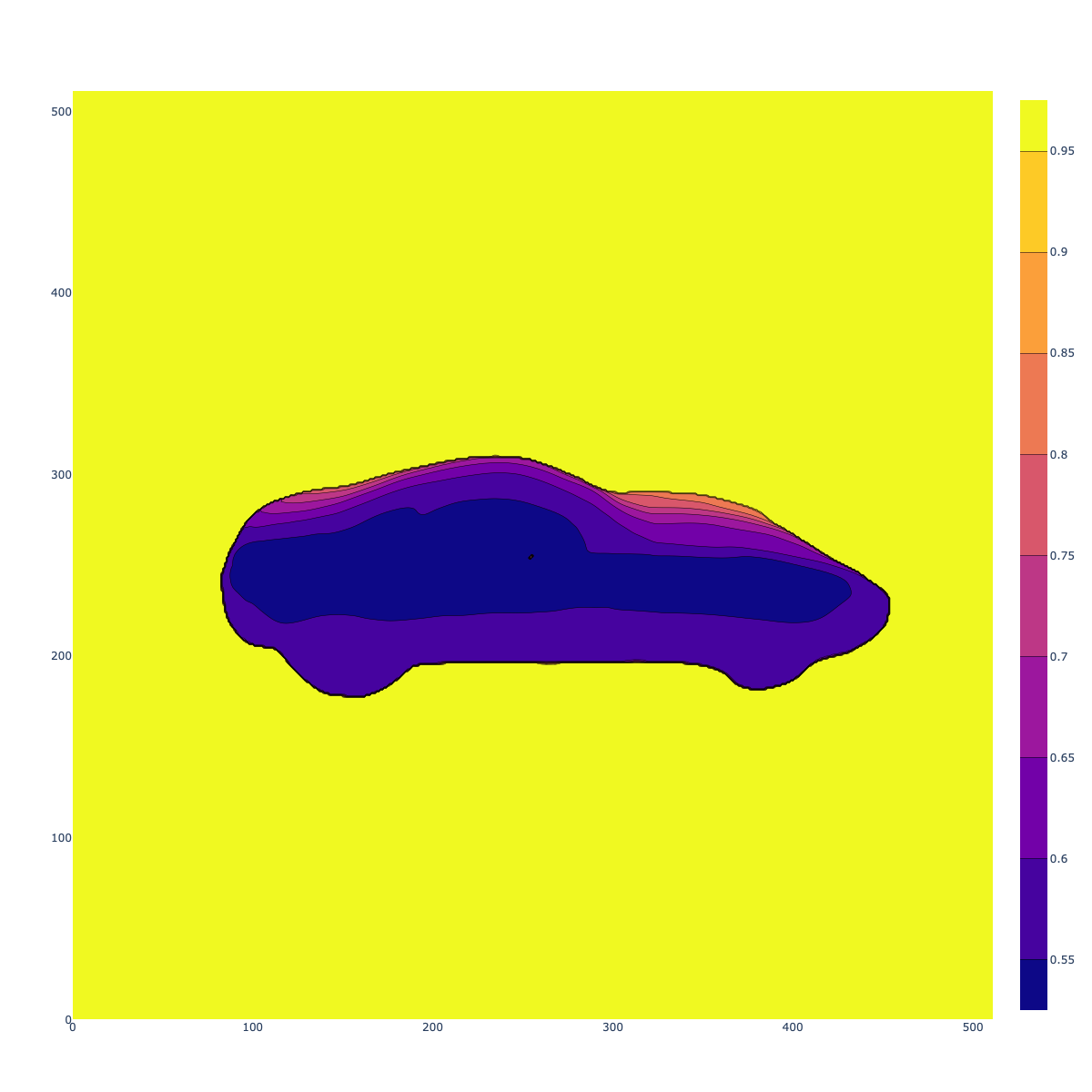}
\includegraphics[width=0.49\linewidth, trim=50mm 50mm 80mm 130mm, clip]{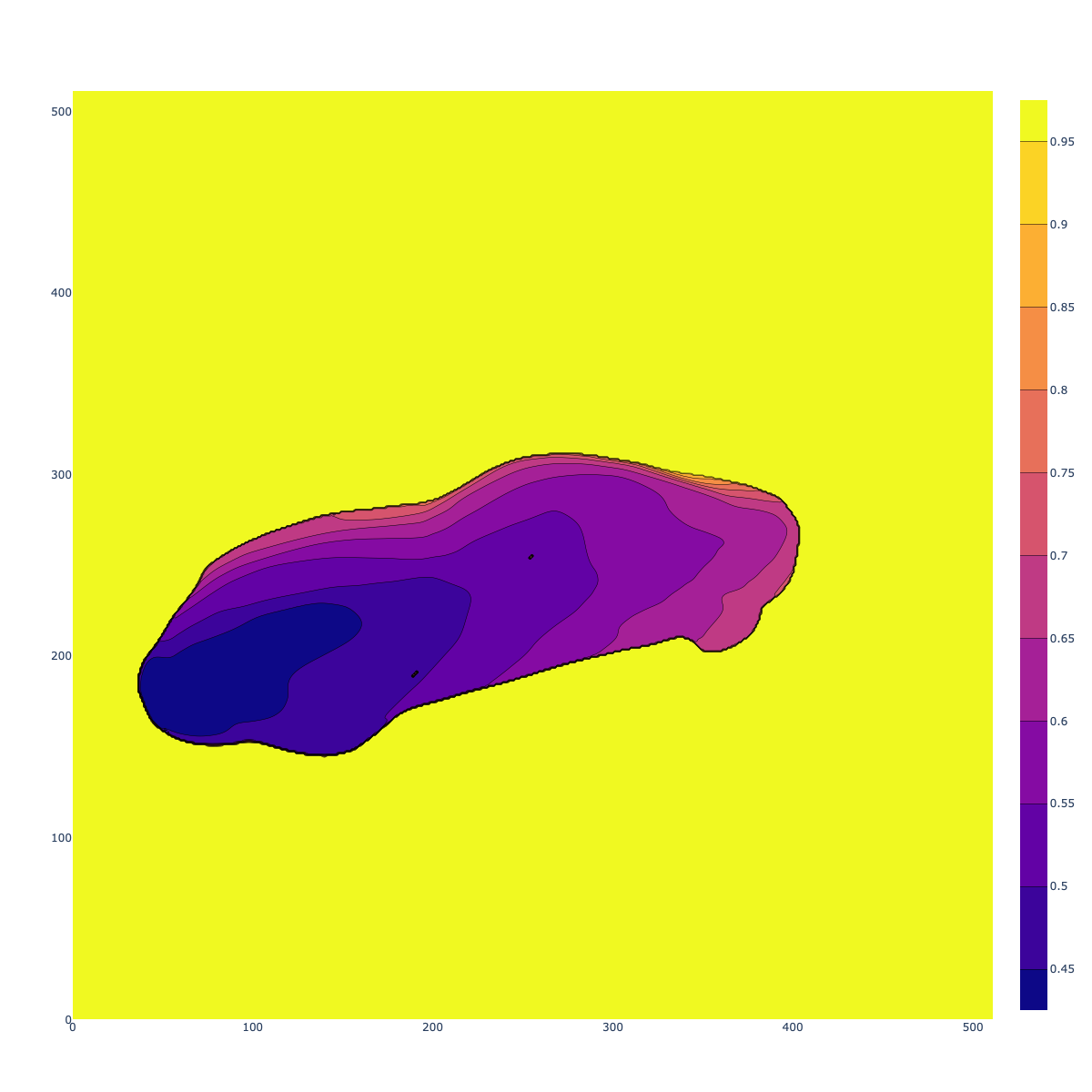}
\end{minipage}%
\hfill%
\begin{minipage}{0.465\textwidth}
  \centering
\scalebox{1}{\includegraphics[width=0.49\linewidth, trim={110mm 45mm 120mm 60mm}, clip]{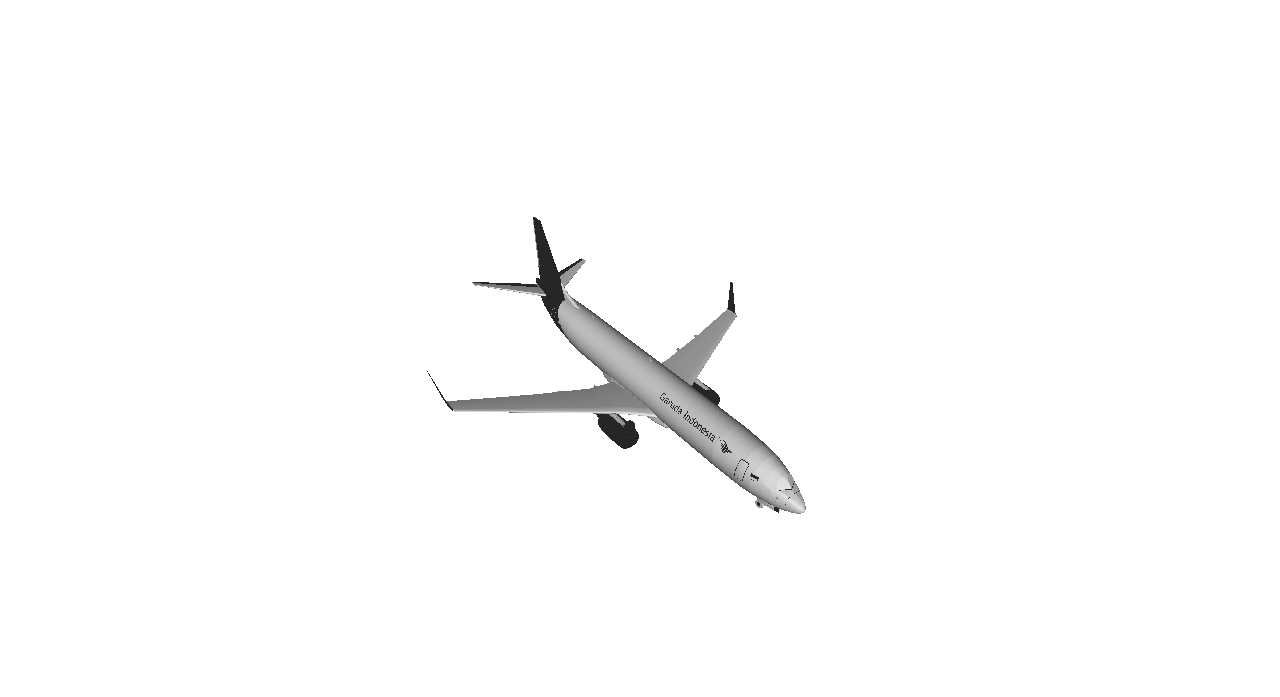}}
\includegraphics[width=0.49\linewidth, trim=75mm 120mm 120mm 115mm, clip]{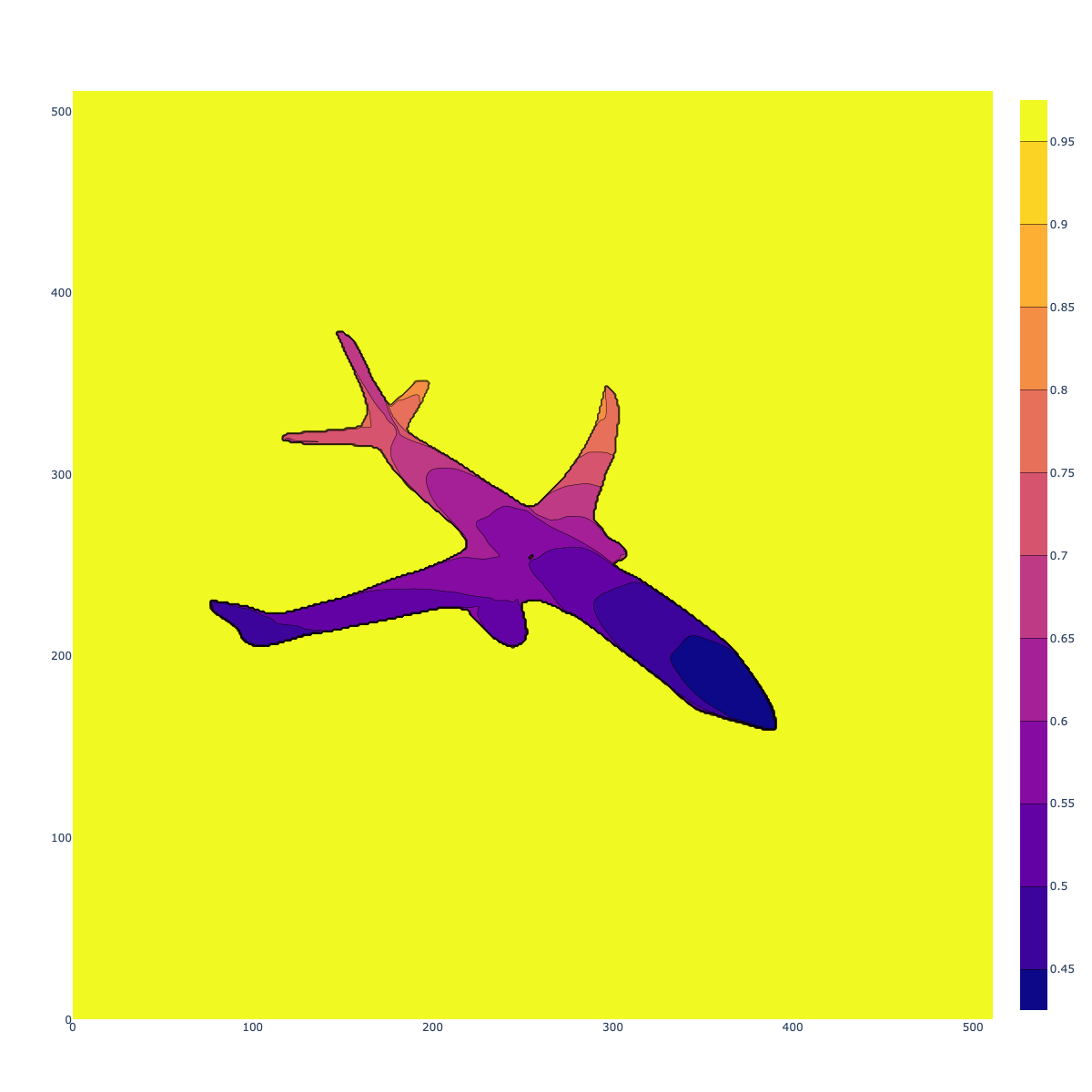}
\includegraphics[width=0.49\linewidth, trim=75mm 120mm 120mm 115mm, clip]{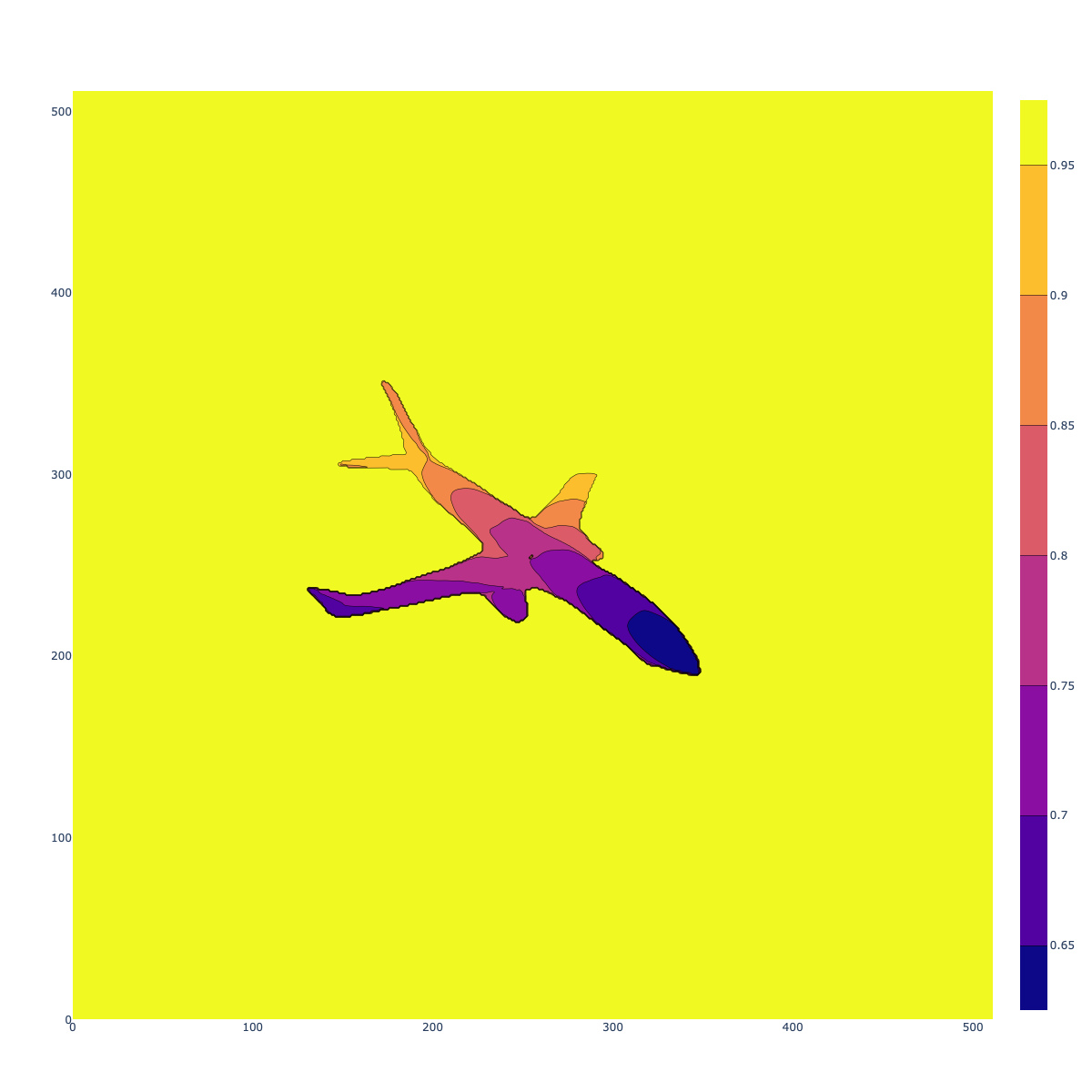}
\includegraphics[width=0.49\linewidth, trim=0mm 0mm 0mm 0mm, clip]{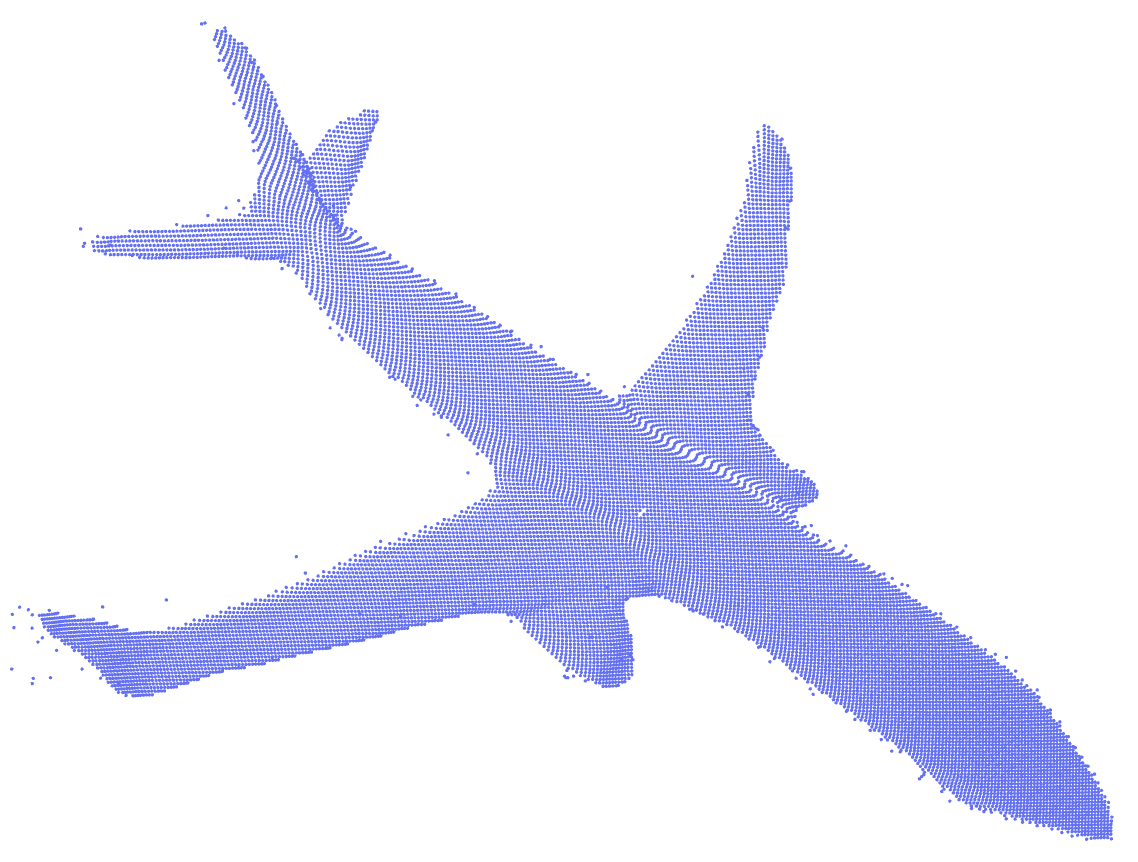}
\includegraphics[width=0.49\linewidth, trim=80mm 60mm 50mm 105mm, clip]{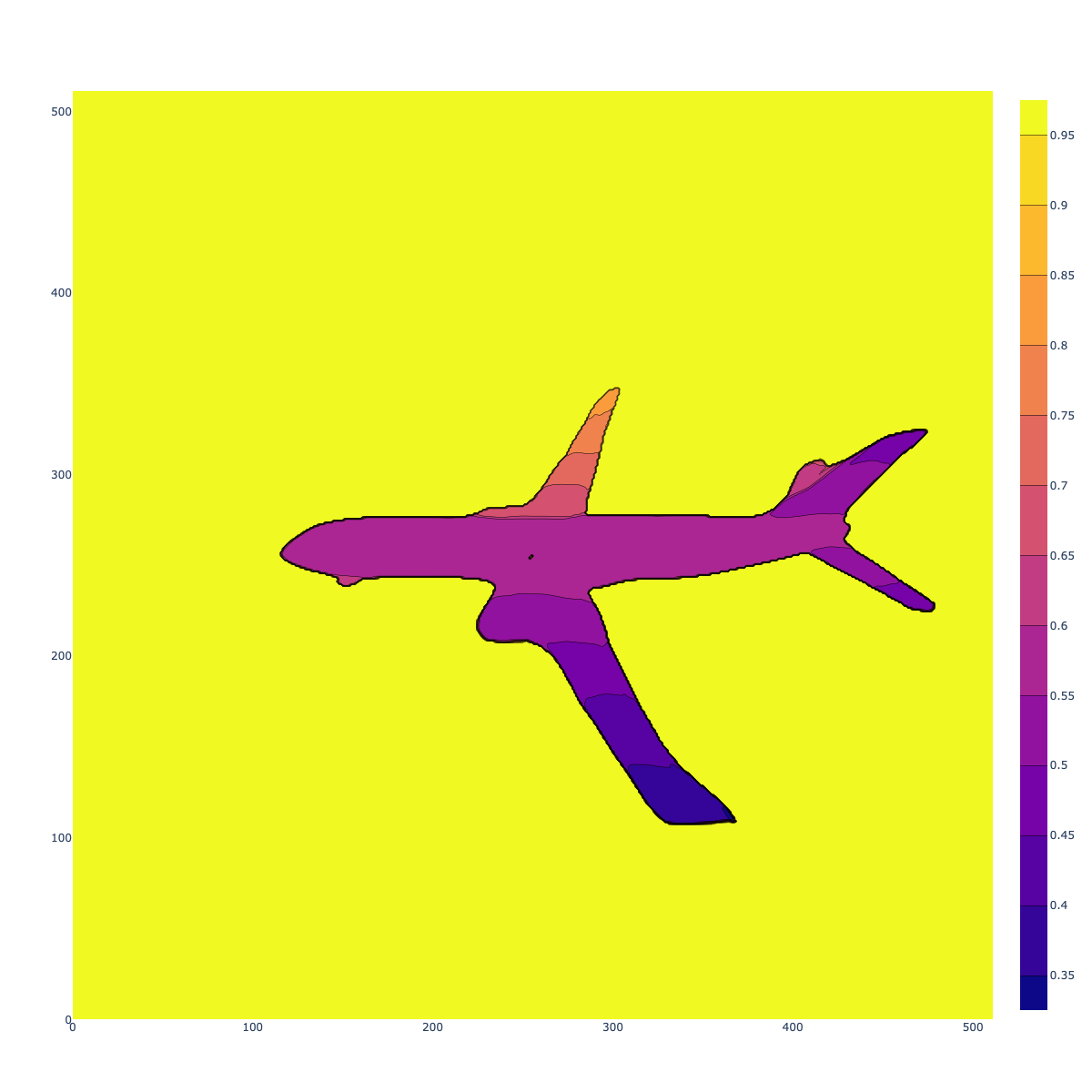}
\includegraphics[width=0.49\linewidth, trim=100mm 110mm 60mm 80mm, clip]{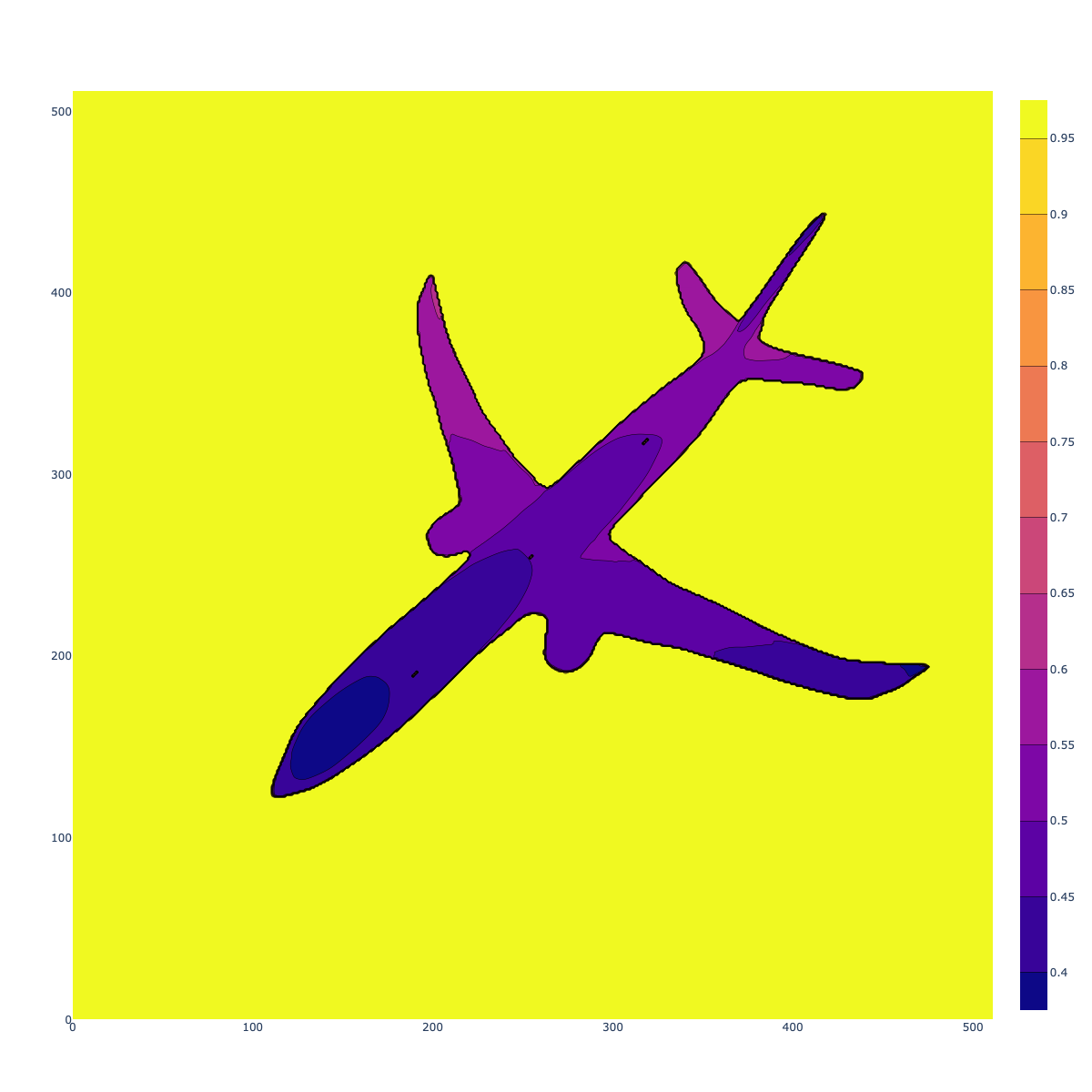}
\end{minipage}
\begin{minipage}{0.036\textwidth}
  \centering
\raisebox{-50.5mm}[0pt][0pt]{\includegraphics[width=\linewidth, trim=0mm 0mm 0mm 0mm, clip]{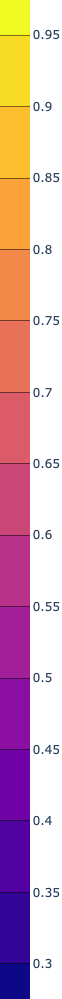}}
\end{minipage}
\caption{SDDF shape representation of a car (left two columns) and airplane (right two columns) instance. In each two columns, the ground-truth model is shown on the top left, a distance image synthesized by the SDDF model is shown on the top right, and a point cloud generated from the distance image is shown in the middle right. The middle left shows a distance image synthesized from the same view but further distance from the object. Note that the level sets in the distant view remain parallel to the close-up view but more yellowish, indicating the distance increase. The third row includes more distance images synthesized by our SDDF model from other views. To produce good color contrast, in all images we set infinite distance values (corresponding to rays in free space) to $1$.}
\label{fig:geometry}
\end{figure*}

\paragraph{Effect of Measurement Noise on the Performance.}
We present qualitative and quantitative results about the effects of noisy distance data and different number of layers and neurons per layer in the model on the performance of the SDDF model. The results are obtained for a single Airplane instance, shown in Fig.~\ref{fig:geometry}. We obtained $500$ finite rays and $500$ infinite rays from $1000$ random locations uniformly distributed on a sphere around the instance with orientations facing the object. The SDDF model was trained for $9k$ iterations in several different settings. The default setting has $8$ layers with $512$ neurons per layer and noise-free distance data for training. Training this model takes about $688$ seconds. A distance view synthesized by the trained SDDF model is shown in Fig.~\ref{fig:times}.

First, keeping the network structure fixed, we varied the standard deviation of zero-mean Gaussian noise added to the distance measurements. To have a sense about the noise magnitude, note that the radius of the sphere on which the camera locations were picked was $0.6$. Qualitatively, as we see in Fig.~\ref{fig:noise}, the more the noise increases, the fewer details the SDDF model can capture. Second, we varied the number of layers (fixing the number of neurons to $512$) and the number of neurons (fixing the number of layers to $8$) in the neural network, keeping a skip connection to the middle layer, to measure the training, as shown in Fig.~\ref{fig:times}.

\begin{figure*}[t]
\begin{minipage}{\linewidth}
\centering
\includegraphics[width=0.24\linewidth, trim={0mm 70mm 70mm 110mm}, clip]{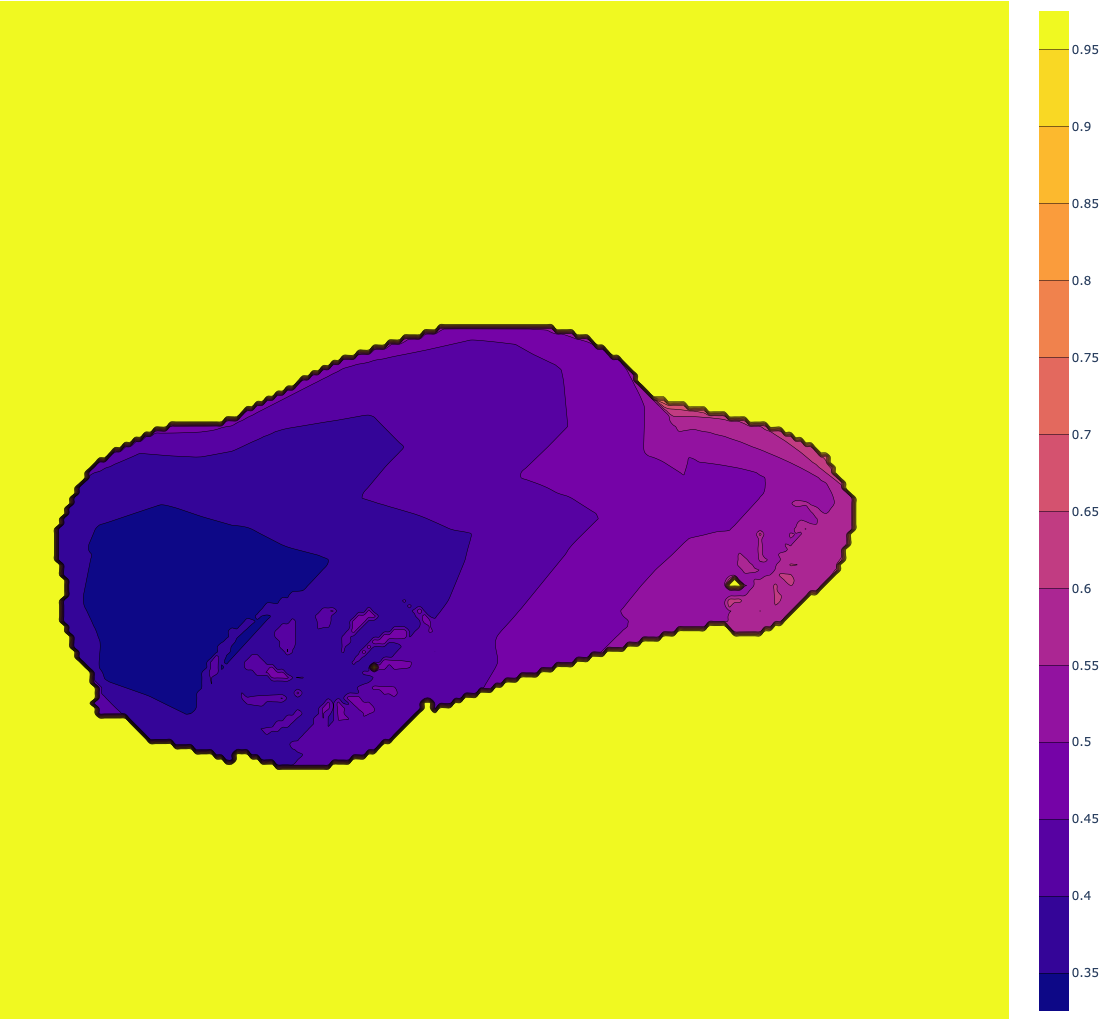}
\includegraphics[width=0.24\linewidth, trim=0mm 30mm 30mm 55mm, clip]{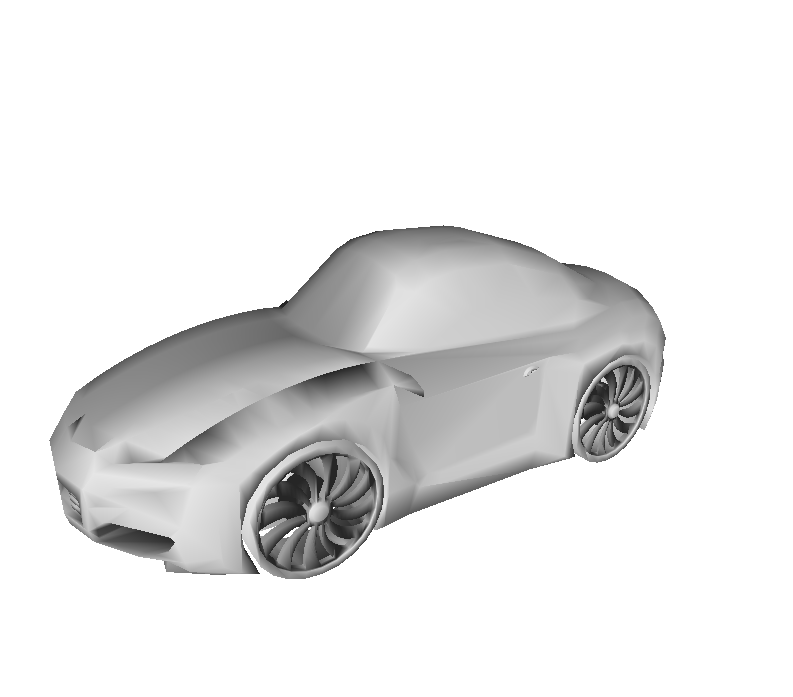}
\includegraphics[width=0.24\linewidth, trim=0mm 70mm 70mm 110mm, clip]{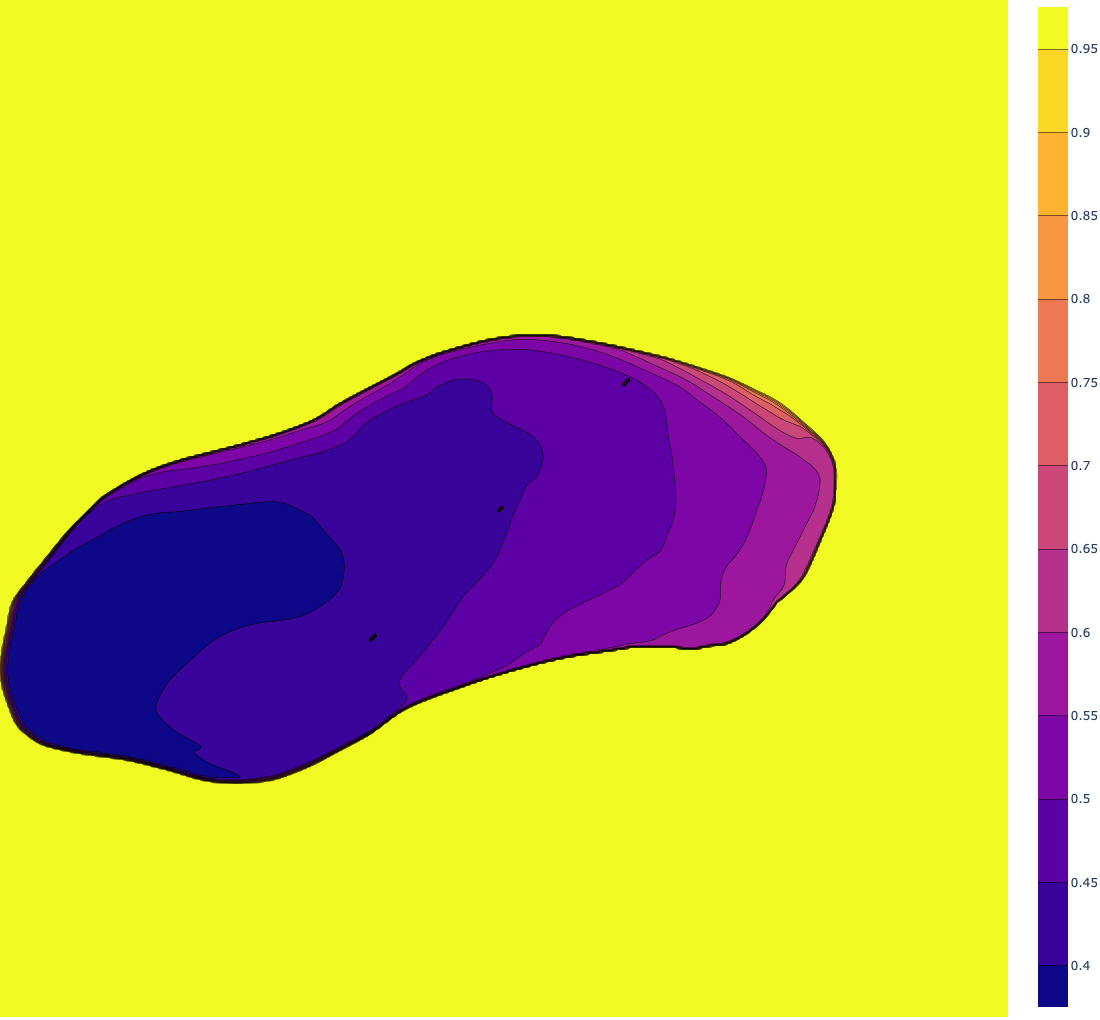}
\includegraphics[width=0.24\linewidth, trim=40mm 105mm 60mm 135mm, clip]{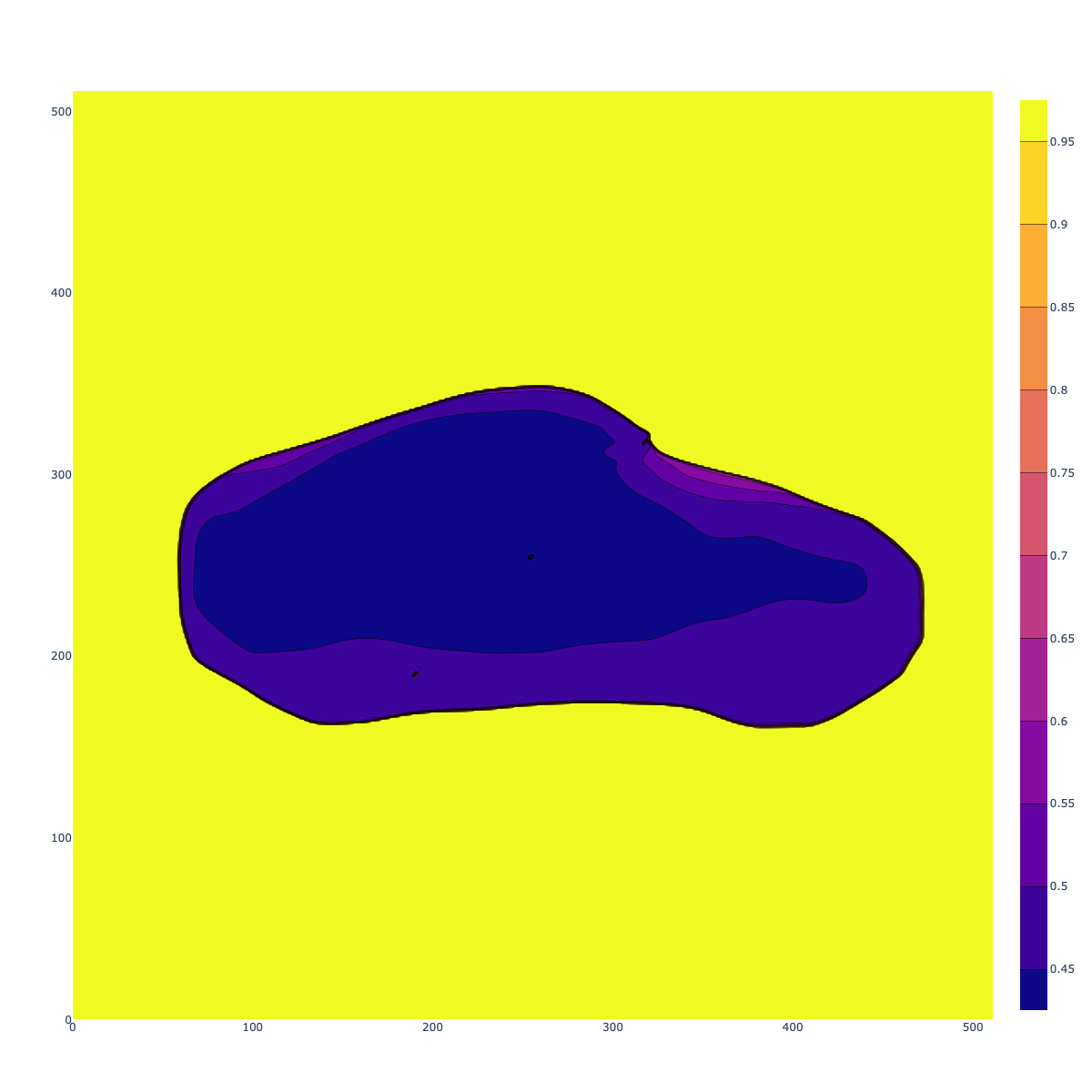}\\
\includegraphics[width=0.24\linewidth, trim={30mm 90mm 60mm 150mm}, clip]{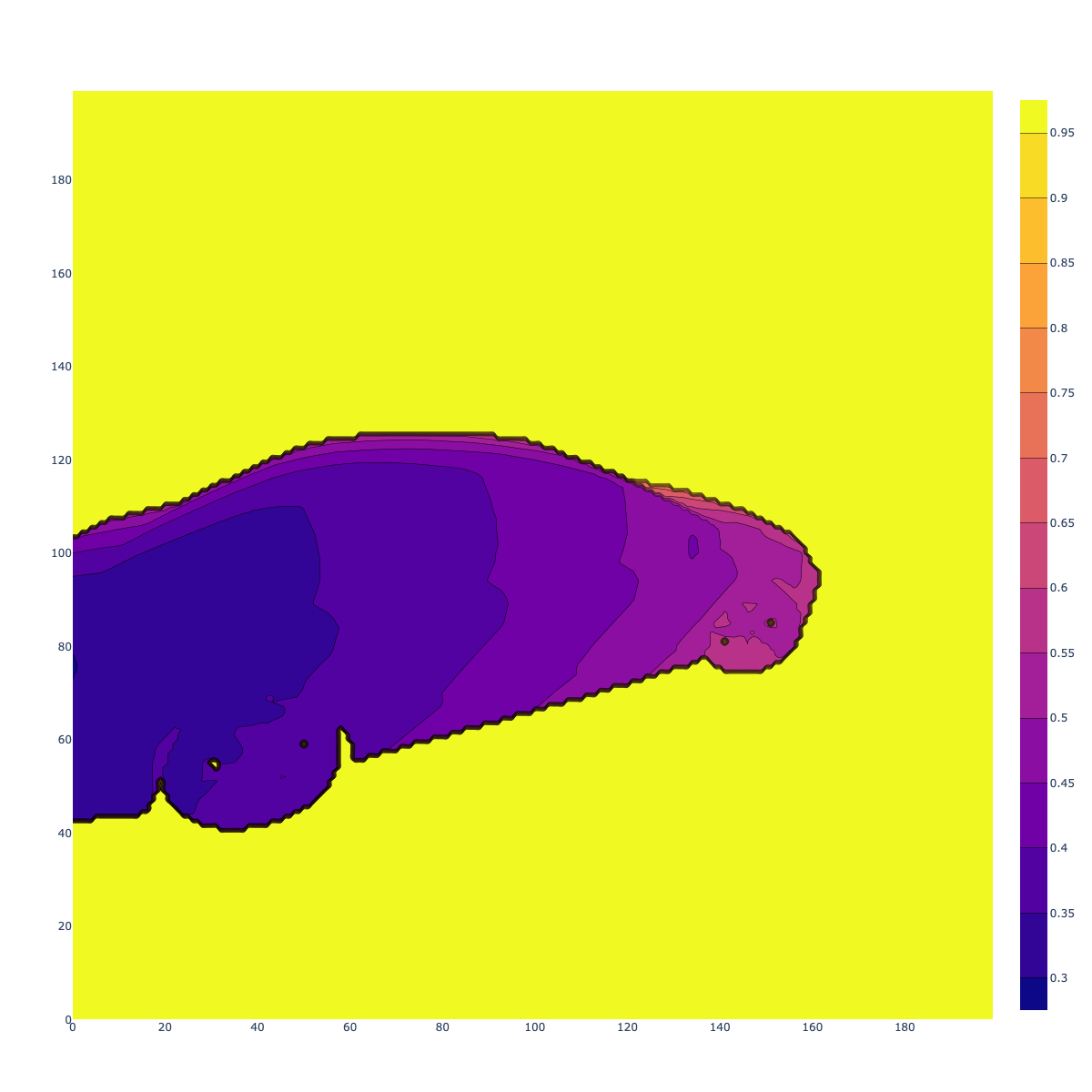}
\includegraphics[width=0.24\linewidth, trim=100mm 40mm 95mm 70mm, clip]{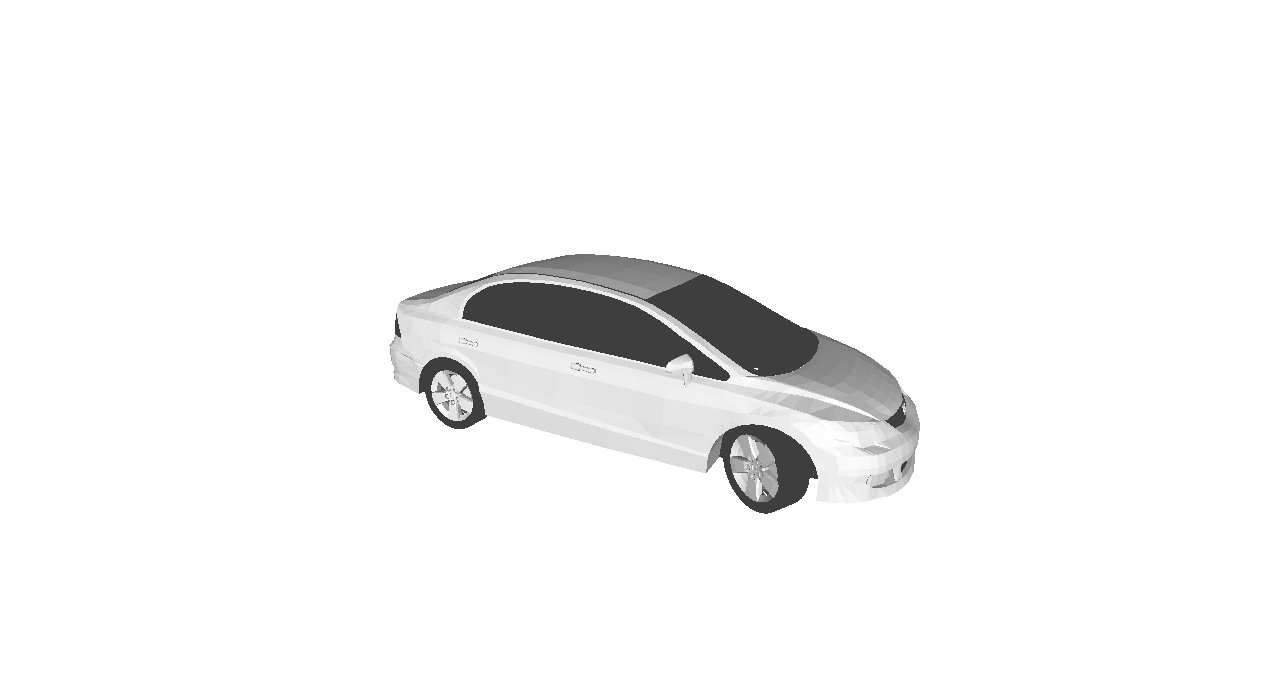}
\includegraphics[width=0.24\linewidth, trim=70mm 120mm 90mm 160mm, clip]{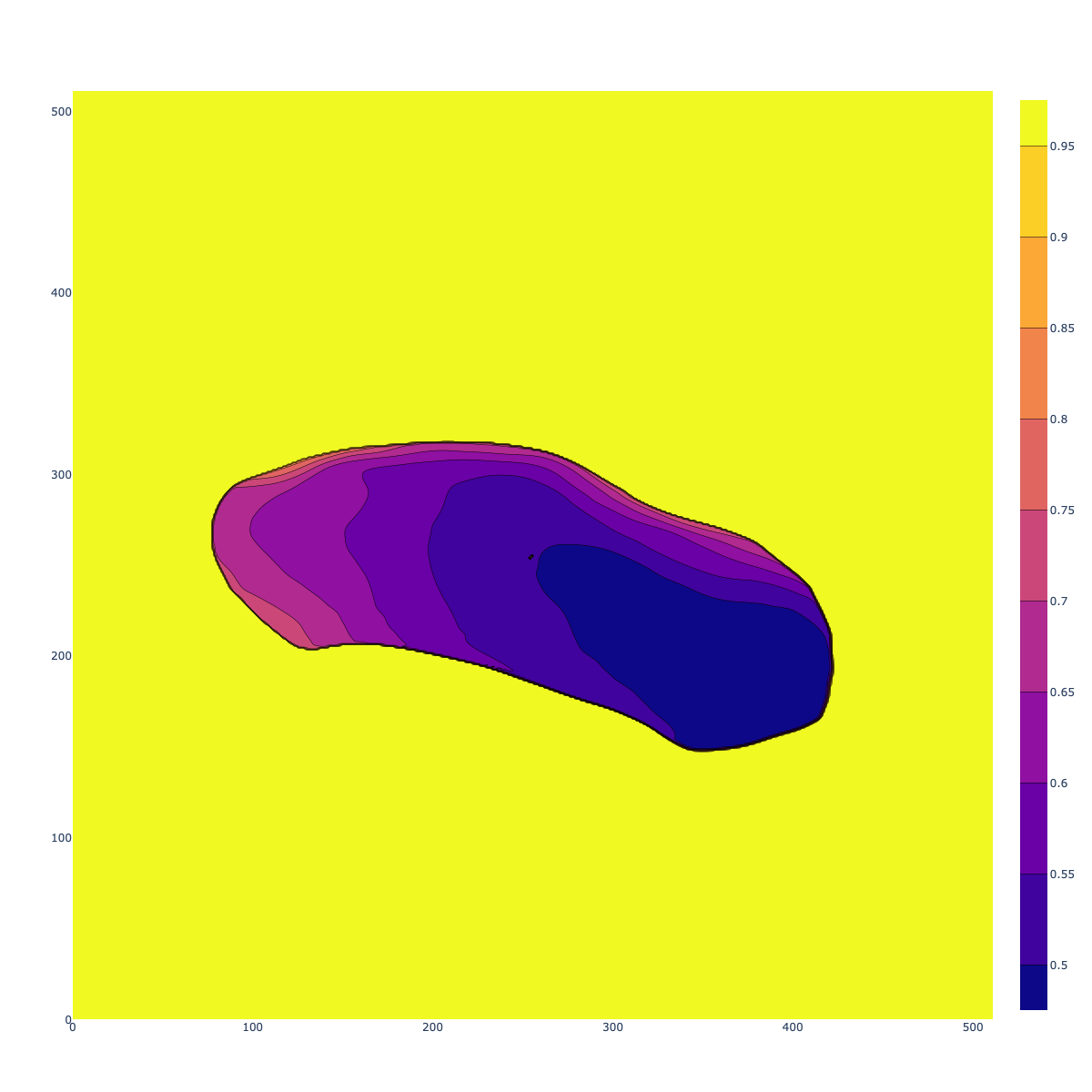}
\includegraphics[width=0.24\linewidth, trim=50mm 115mm 100mm 160mm, clip]{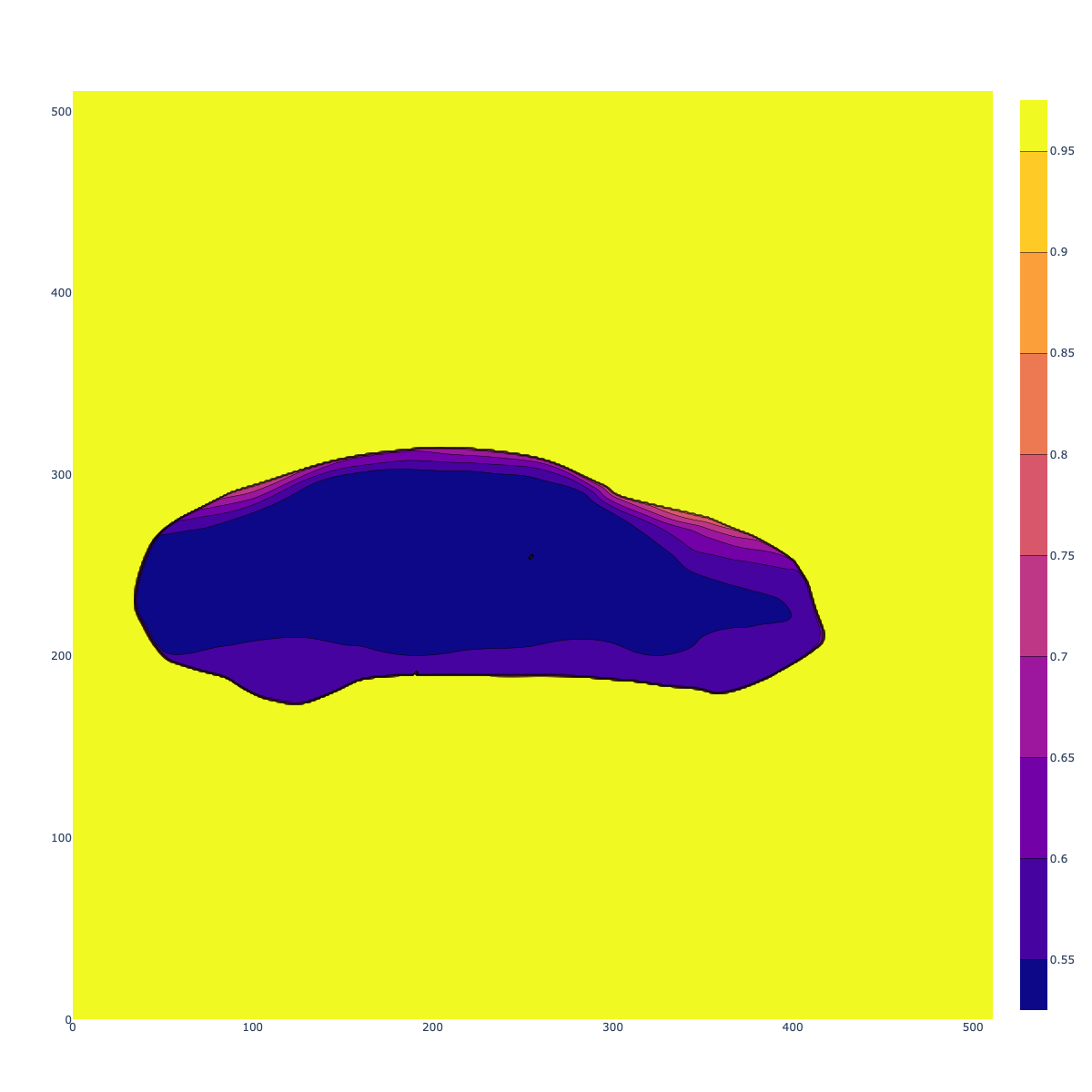}
\includegraphics[width=0.24\linewidth, trim={70mm 100mm 90mm 140mm}, clip]{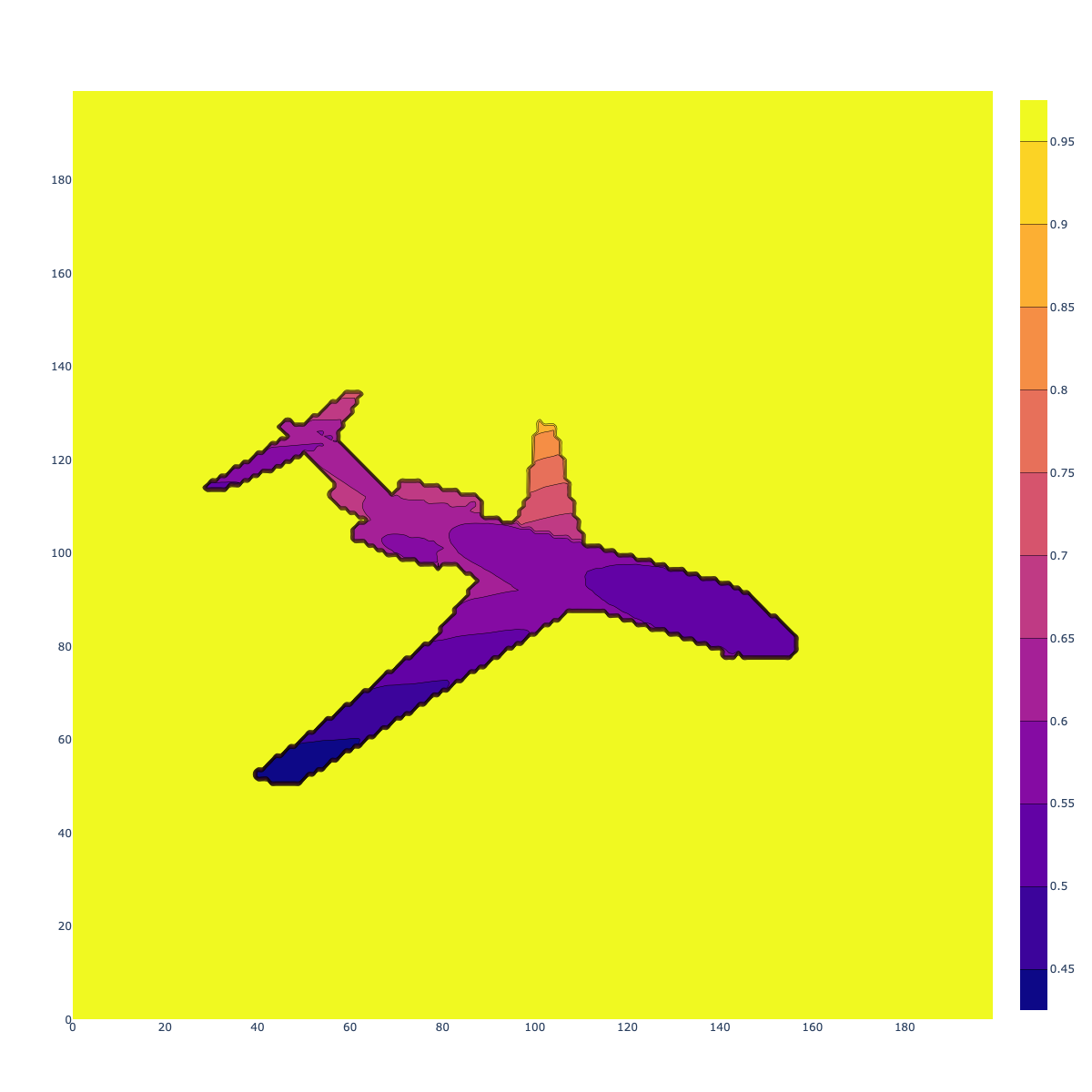}
\includegraphics[width=0.24\linewidth, trim=105mm 40mm 110mm 60mm, clip]{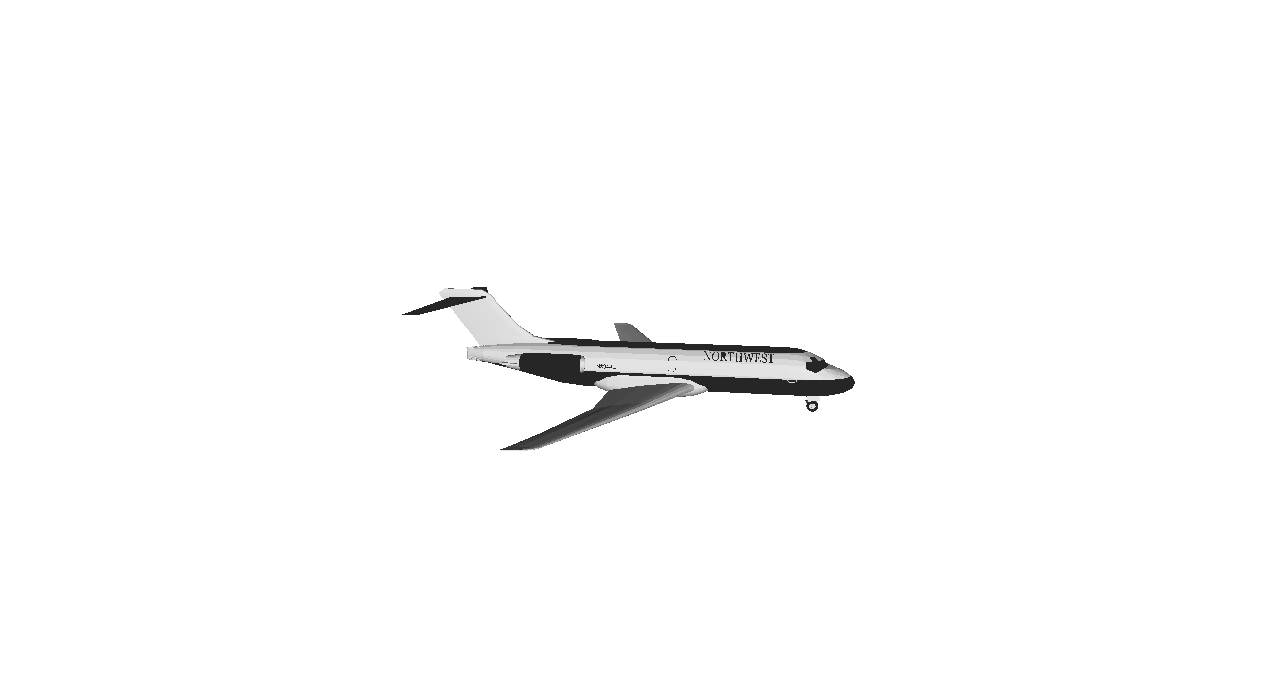}
\includegraphics[width=0.24\linewidth, trim=40mm 90mm 110mm 143mm, clip]{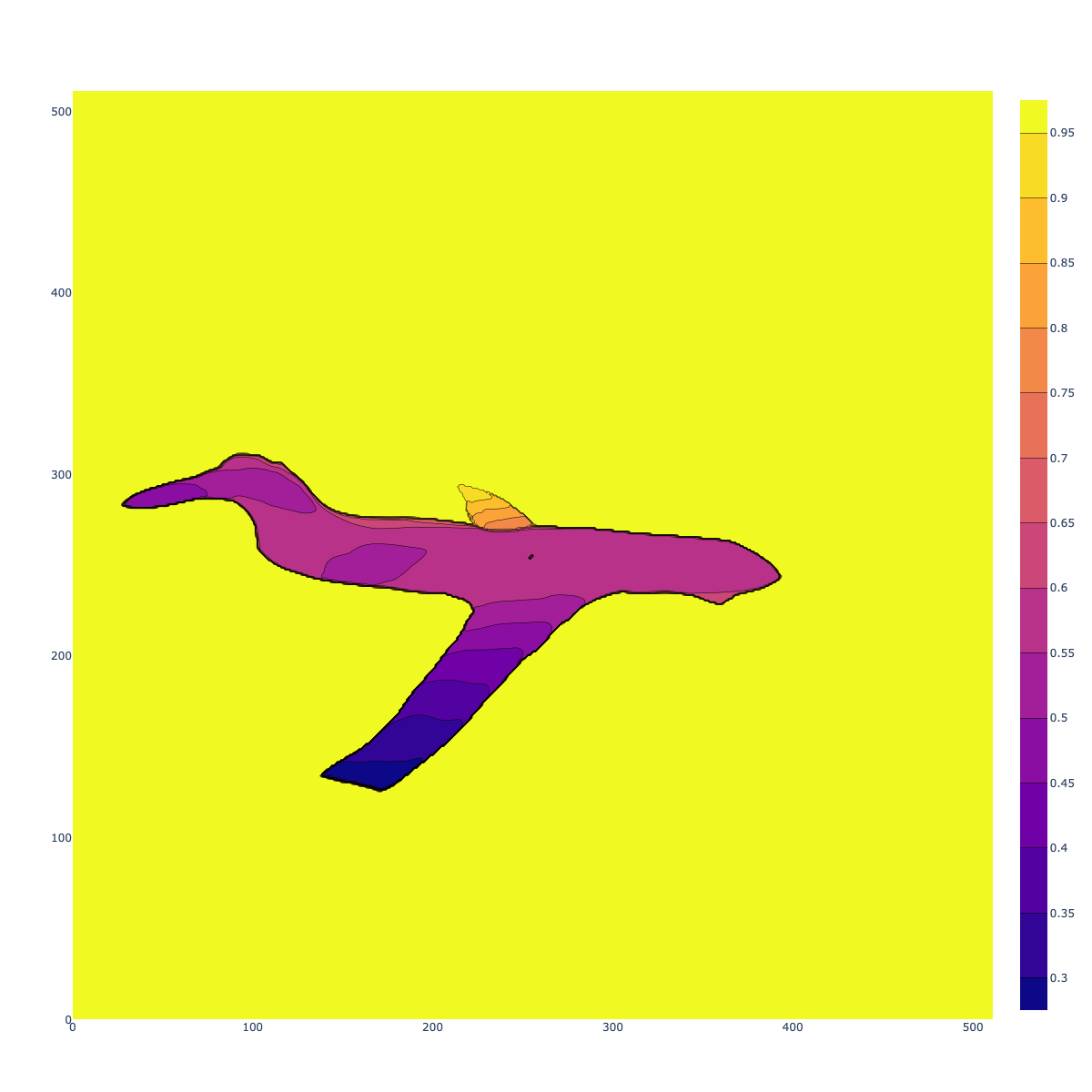}
\includegraphics[width=0.24\linewidth, trim=70mm 100mm 90mm 140mm, clip]{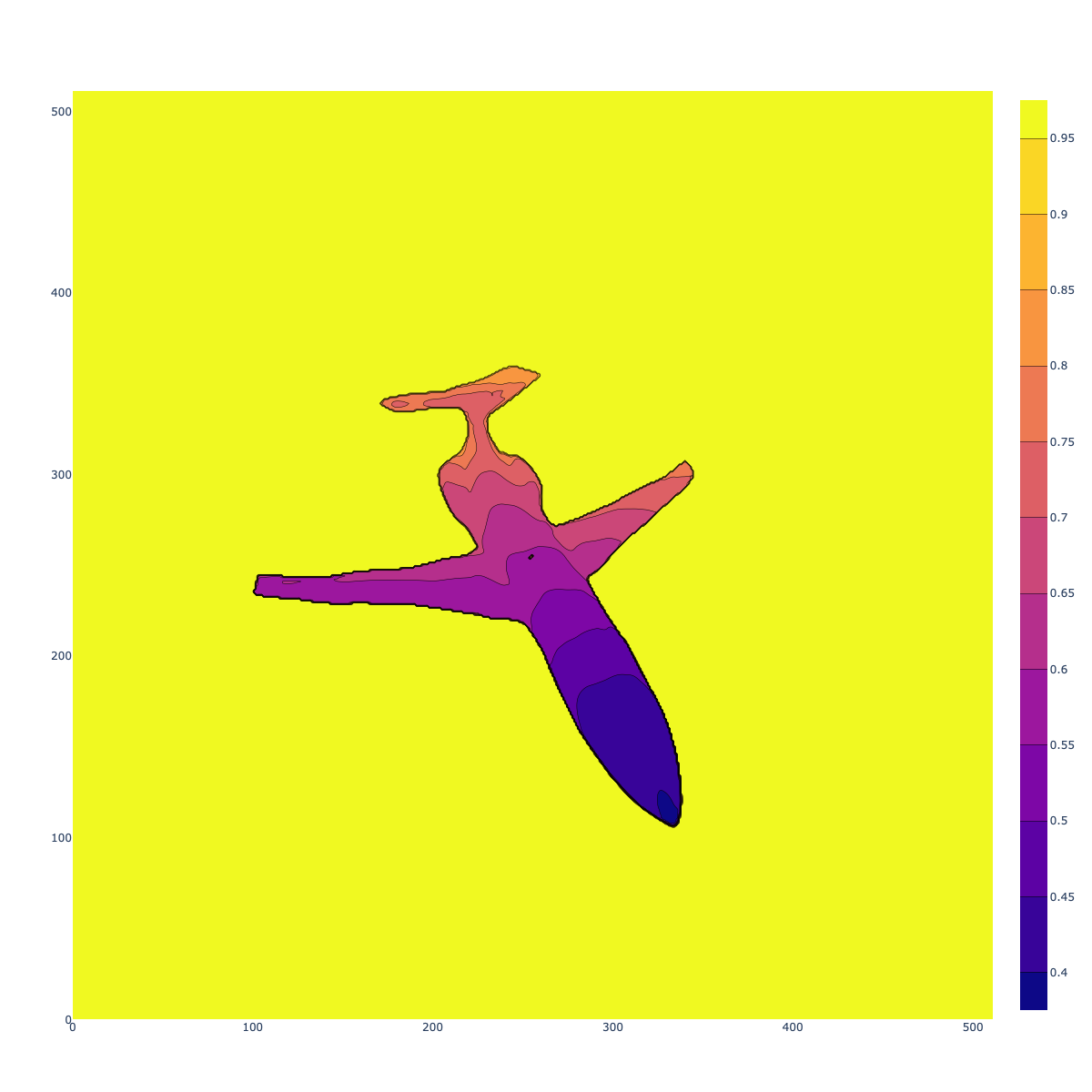}
\includegraphics[width=0.24\linewidth, trim={70mm 120mm 85mm 130mm}, clip]{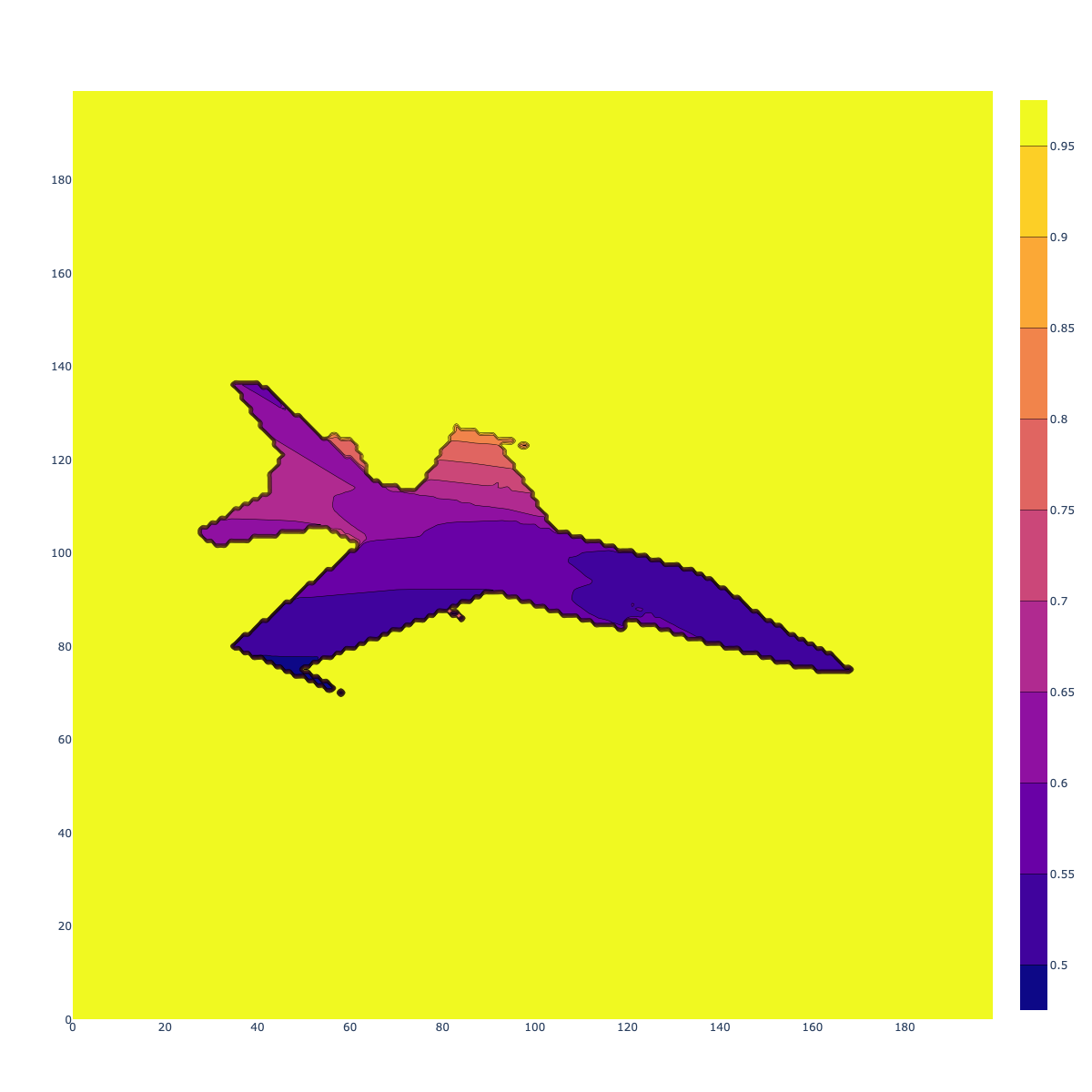}
\includegraphics[width=0.24\linewidth, trim=95mm 40mm 105mm 80mm, clip]{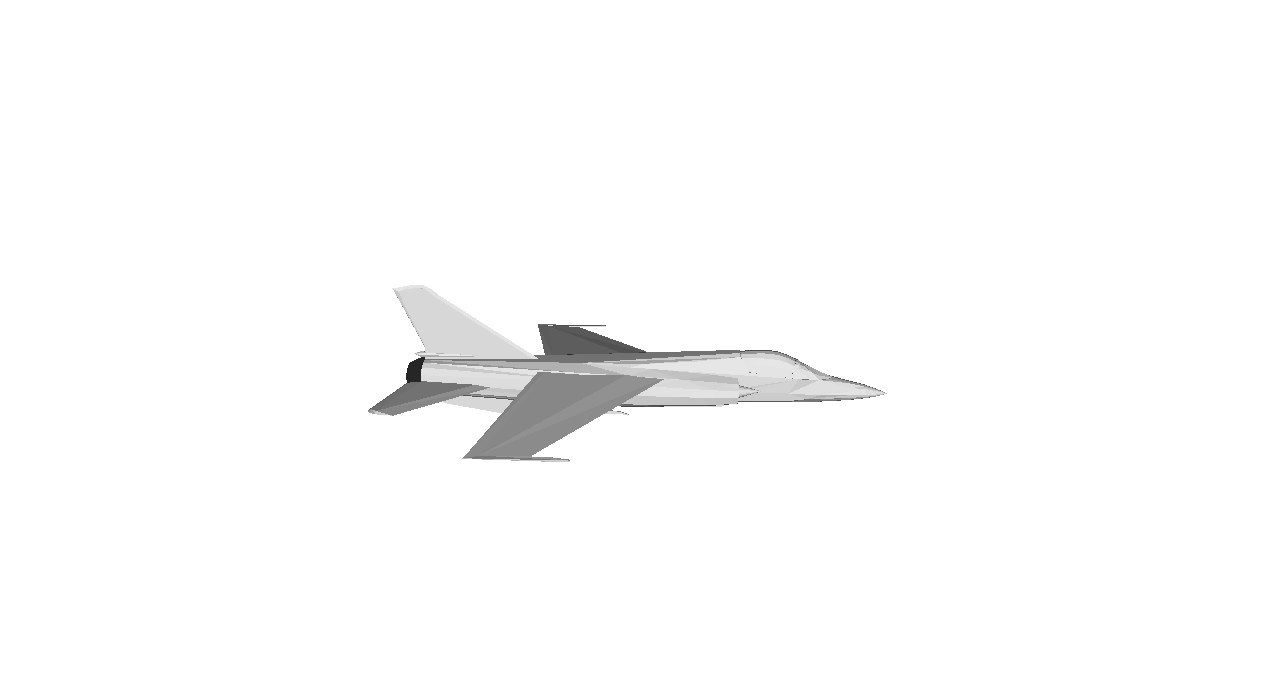}
\includegraphics[width=0.24\linewidth, trim=40mm 90mm 85mm 140mm, clip]{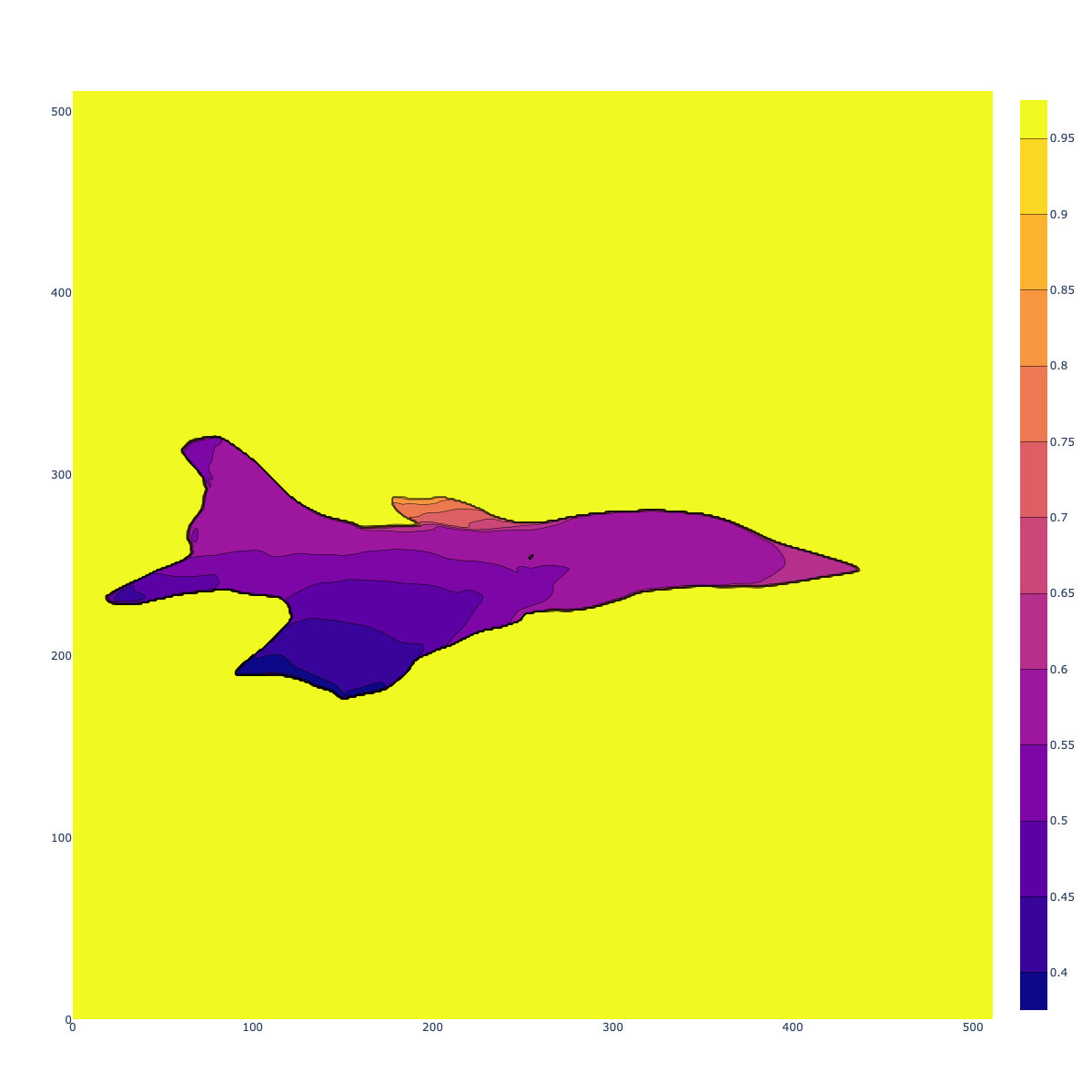}
\includegraphics[width=0.24\linewidth, trim=70mm 120mm 85mm 130mm, clip]{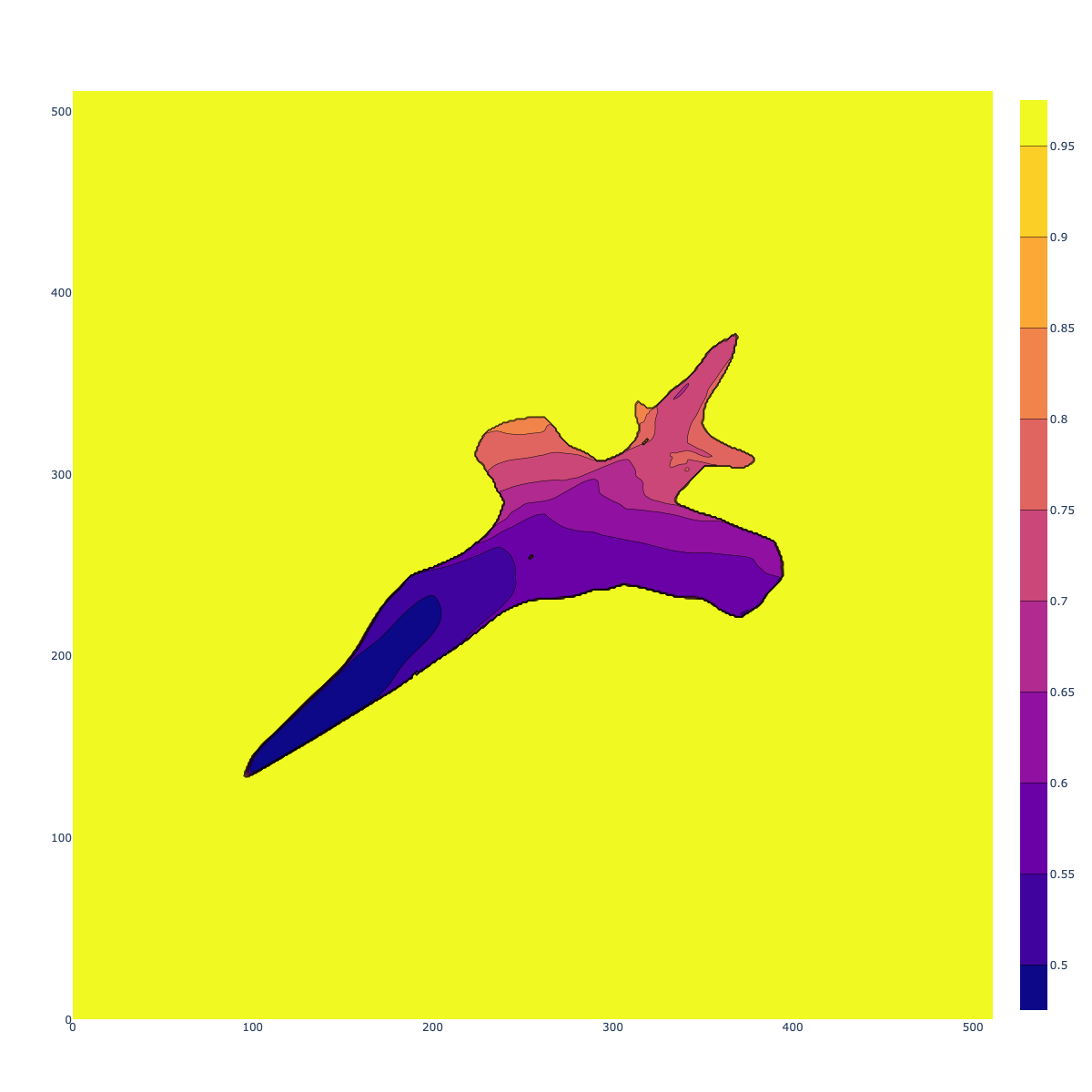}
\end{minipage}
\caption{SDDF shape completion using a distance image (left column) from an unseen object instance (second column). After latent code optimization, the SDDF model can synthesize novel distance views (third and fourth columns).}
\label{fig:shpcmp}
\begin{minipage}{\linewidth}
  \centering
\includegraphics[width=0.19\linewidth, trim={100mm 80mm 40mm 85mm}, clip]{fig/3D/airplain/interpolation/EvalDepthImg5000_phi90_theta90_1.png}
\includegraphics[width=0.19\linewidth, trim=100mm 80mm 40mm 85mm, clip]{fig/3D/airplain/interpolation/EvalDepthImg5000_phi90_theta90_0o75of1_0o25of15.png}
\includegraphics[width=0.19\linewidth, trim=100mm 80mm 40mm 85mm, clip]{fig/3D/airplain/interpolation/EvalDepthImg5000_phi90_theta90_0o5of1_0o5of15.png}
\includegraphics[width=0.19\linewidth, trim=100mm 80mm 40mm 85mm, clip]{fig/3D/airplain/interpolation/EvalDepthImg5000_phi90_theta90_0o25of1_0o75of15.png}
\includegraphics[width=0.19\linewidth, trim={100mm 80mm 40mm 85mm}, clip]{fig/3D/airplain/interpolation/EvalDepthImg5000_phi90_theta90_15.png}
\includegraphics[width=0.19\linewidth, trim=100mm 90mm 40mm 95mm, clip]{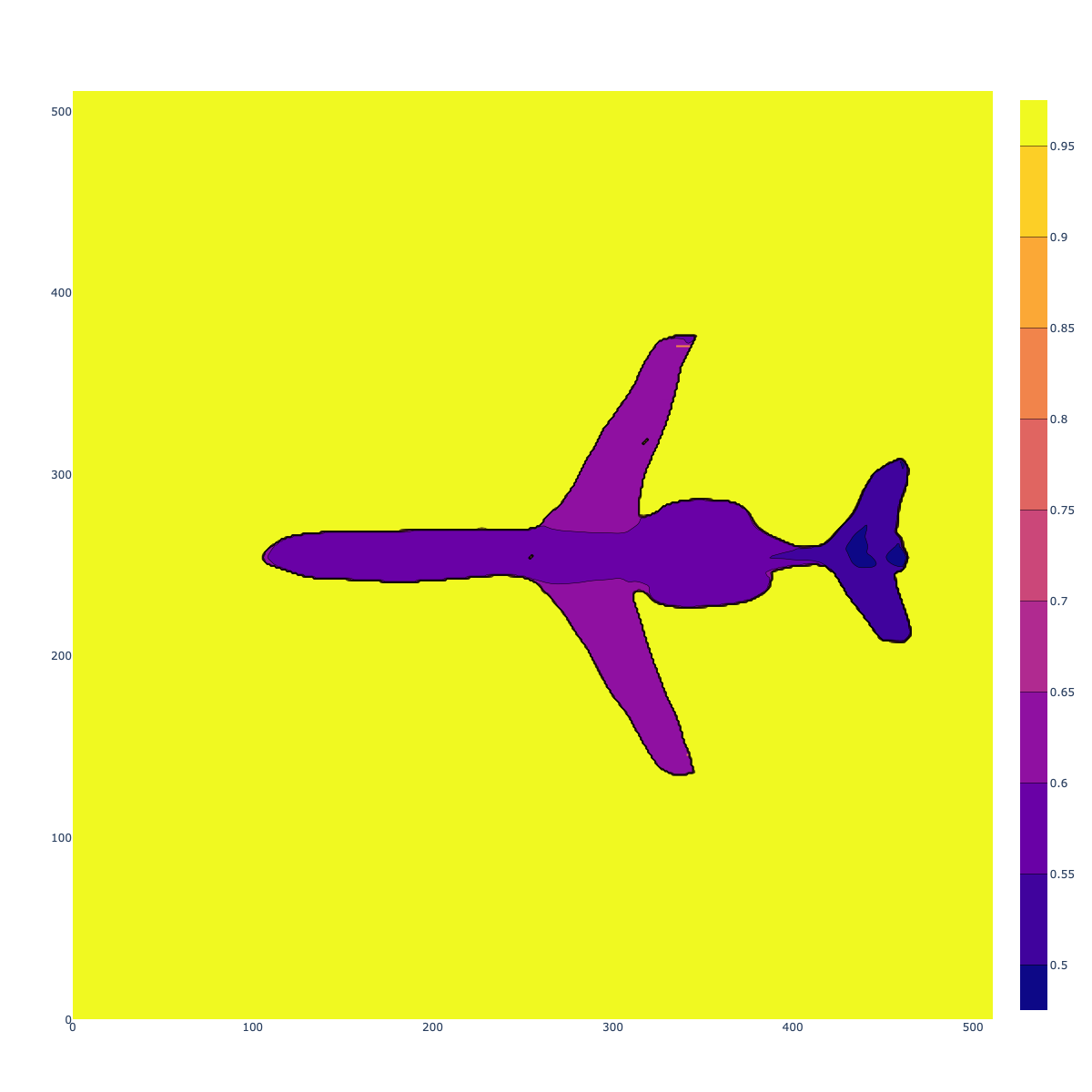}
\includegraphics[width=0.19\linewidth, trim=100mm 90mm 40mm 95mm, clip]{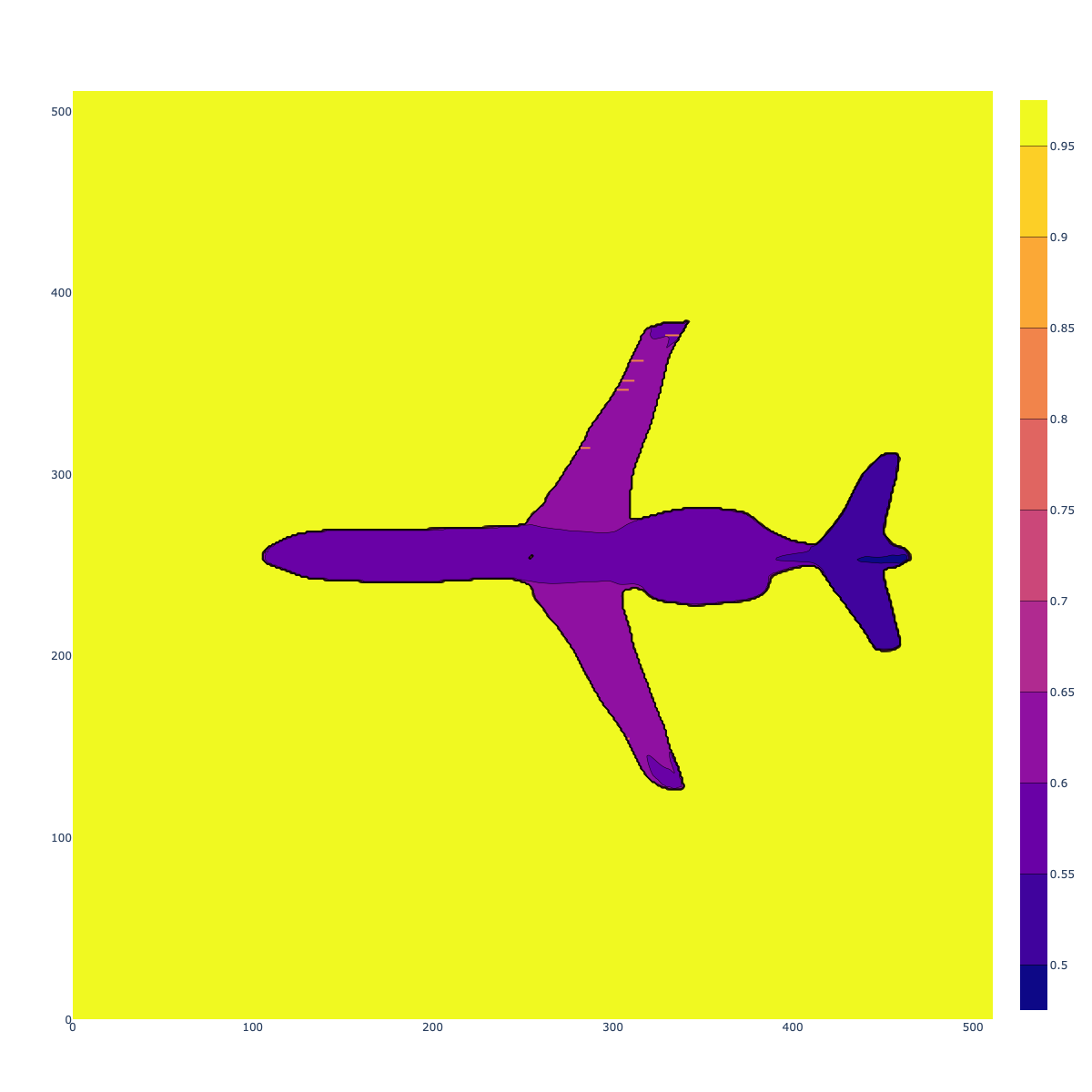}
\includegraphics[width=0.19\linewidth, trim=100mm 90mm 40mm 95mm, clip]{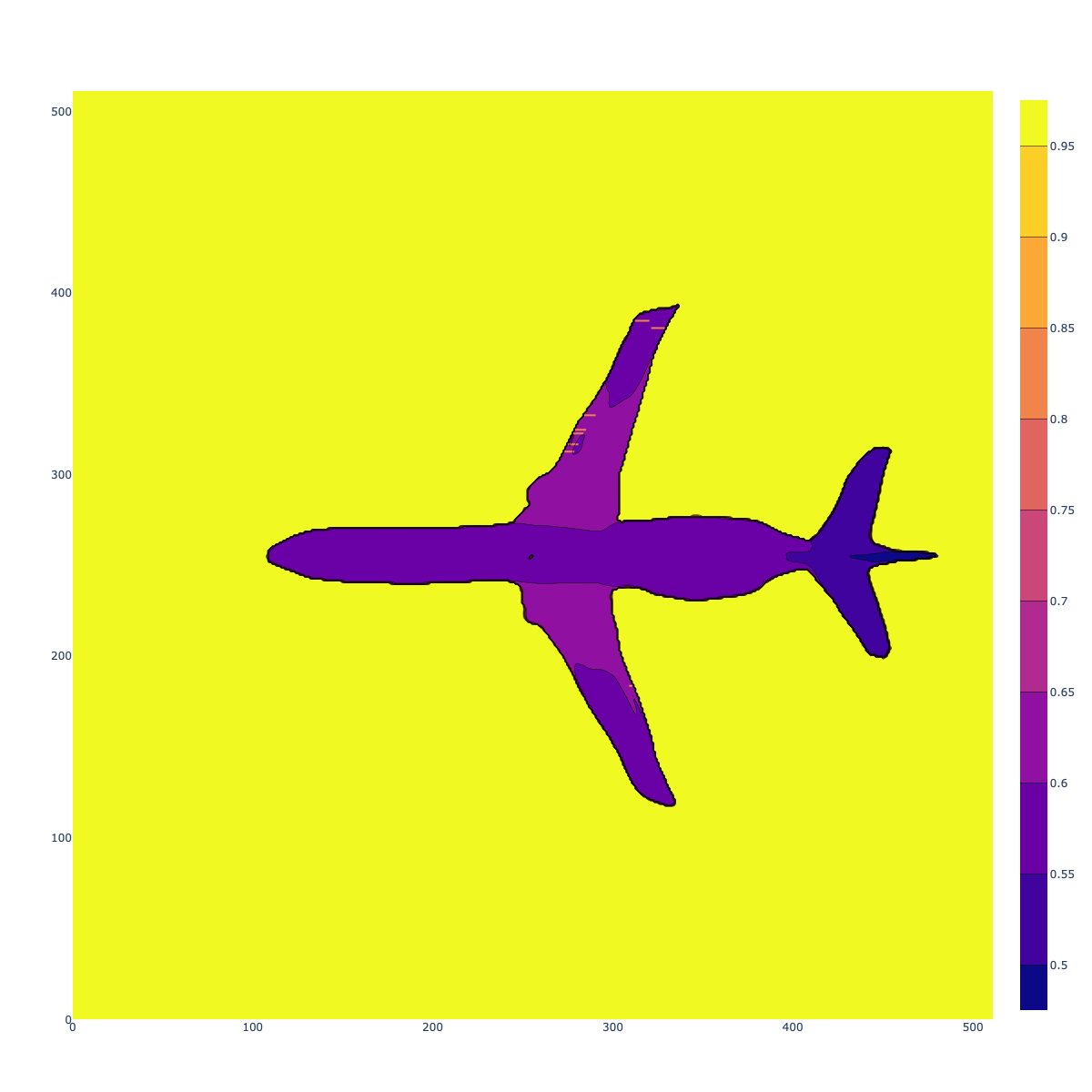}
\includegraphics[width=0.19\linewidth, trim=100mm 90mm 40mm 95mm, clip]{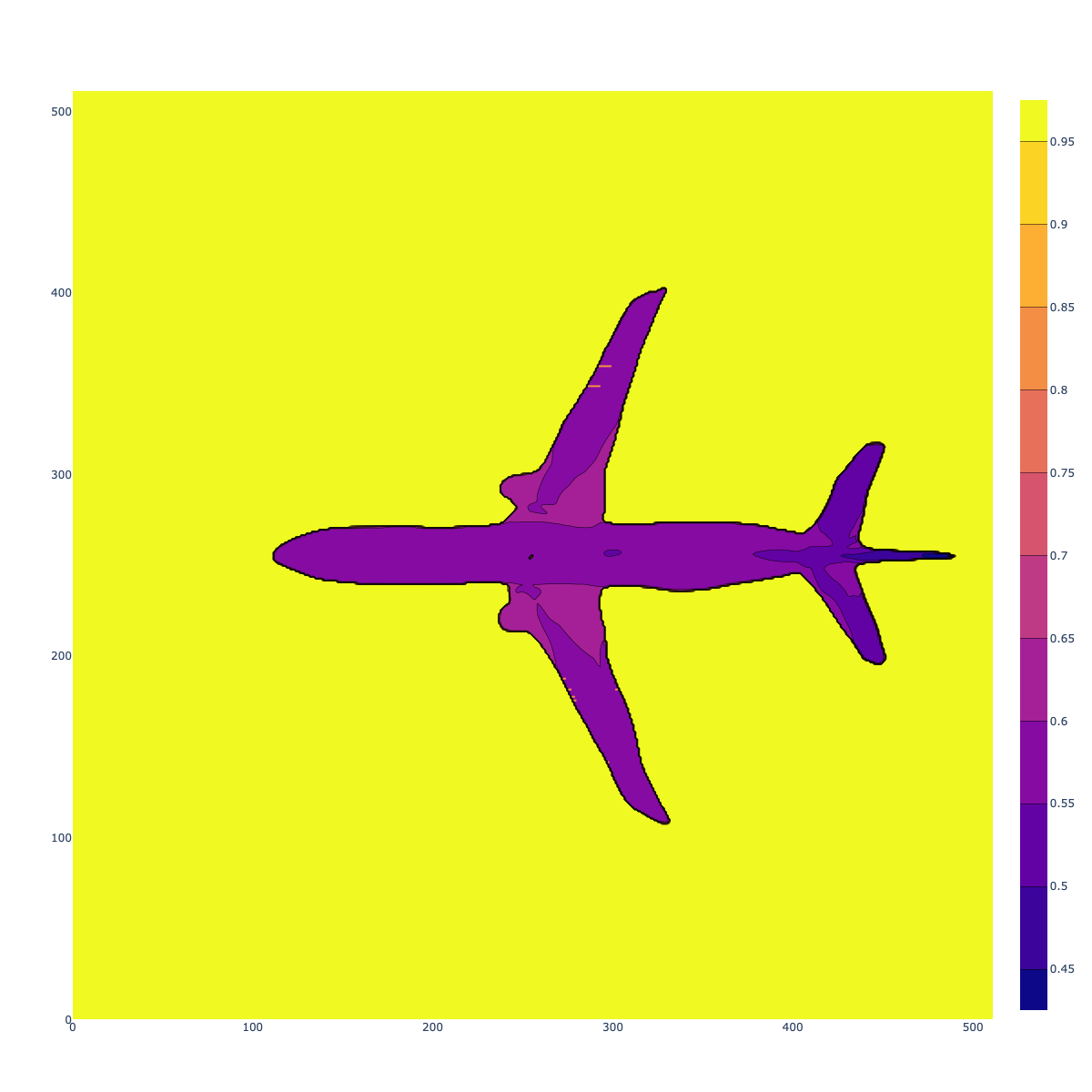}
\includegraphics[width=0.19\linewidth, trim=100mm 90mm 40mm 95mm, clip]{fig/3D/airplain/interpolation/EvalDepthImg5000_phi90_theta90_1.png}
\end{minipage}
\hfill%
\hbox{\hspace{0.7em}\begin{minipage}{\linewidth}
\includegraphics[width=0.19\linewidth, trim={40mm 150mm 105mm 160mm}, clip]{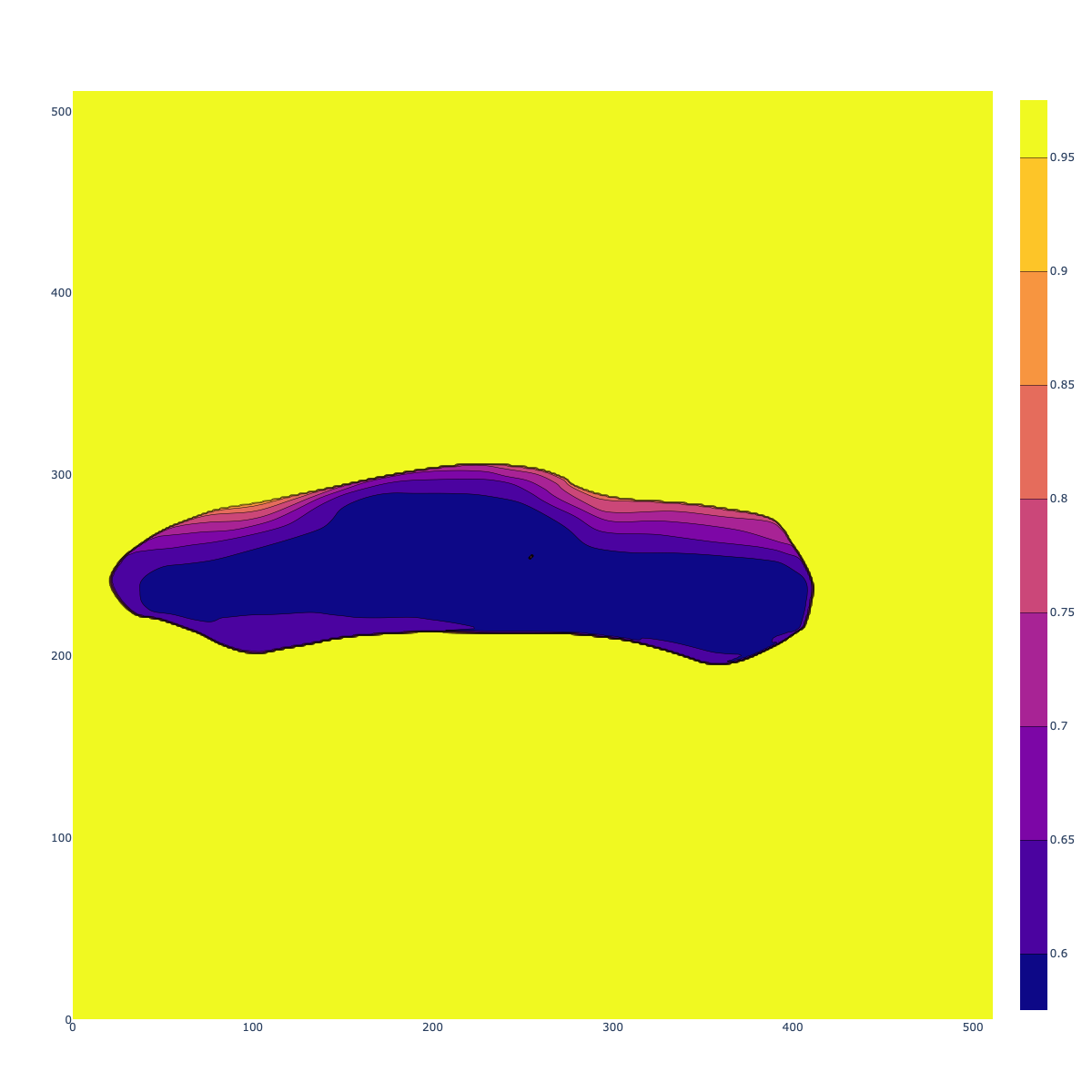}
\includegraphics[width=0.19\linewidth, trim=40mm 150mm 105mm 160mm, clip]{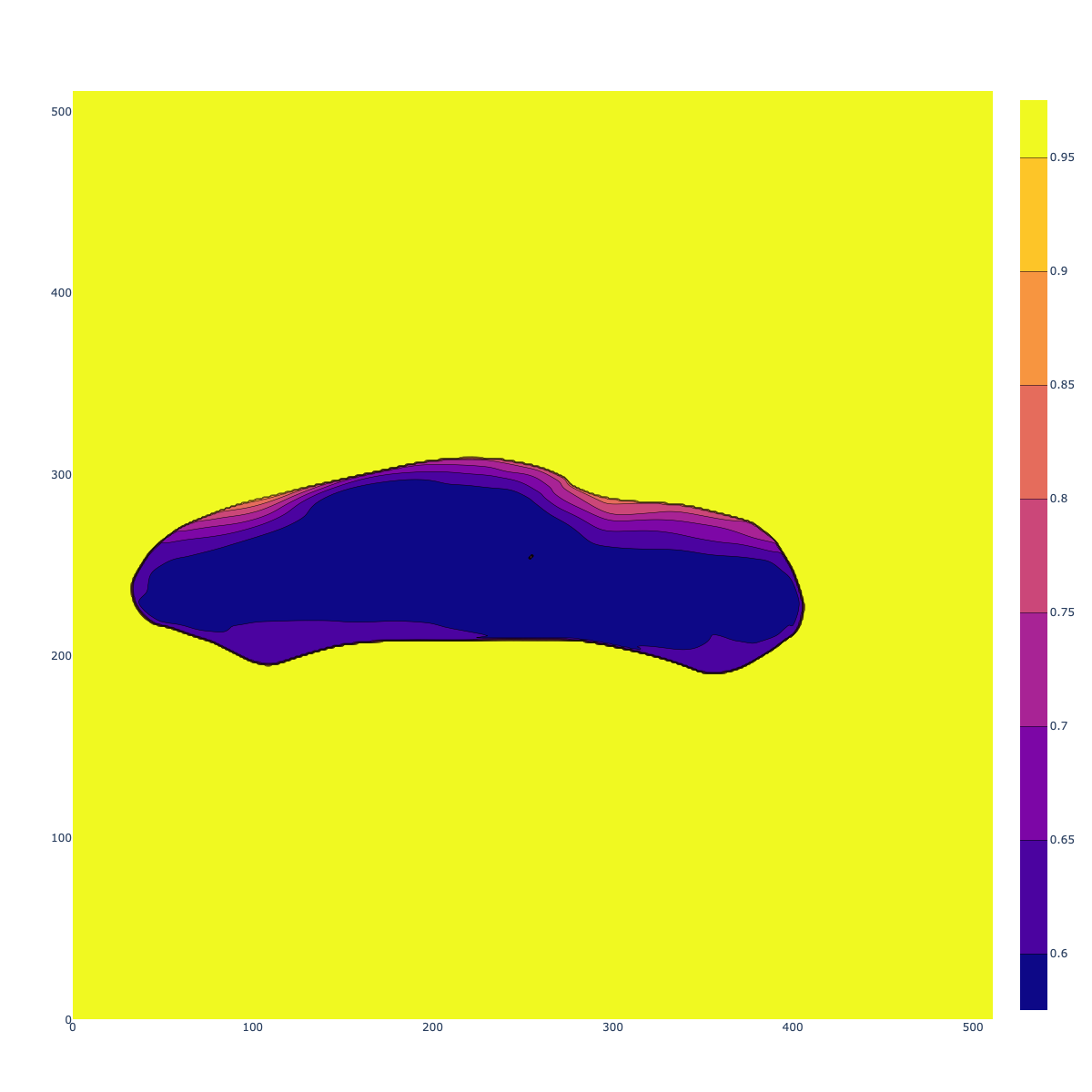}
\includegraphics[width=0.19\linewidth, trim=40mm 150mm 105mm 160mm, clip]{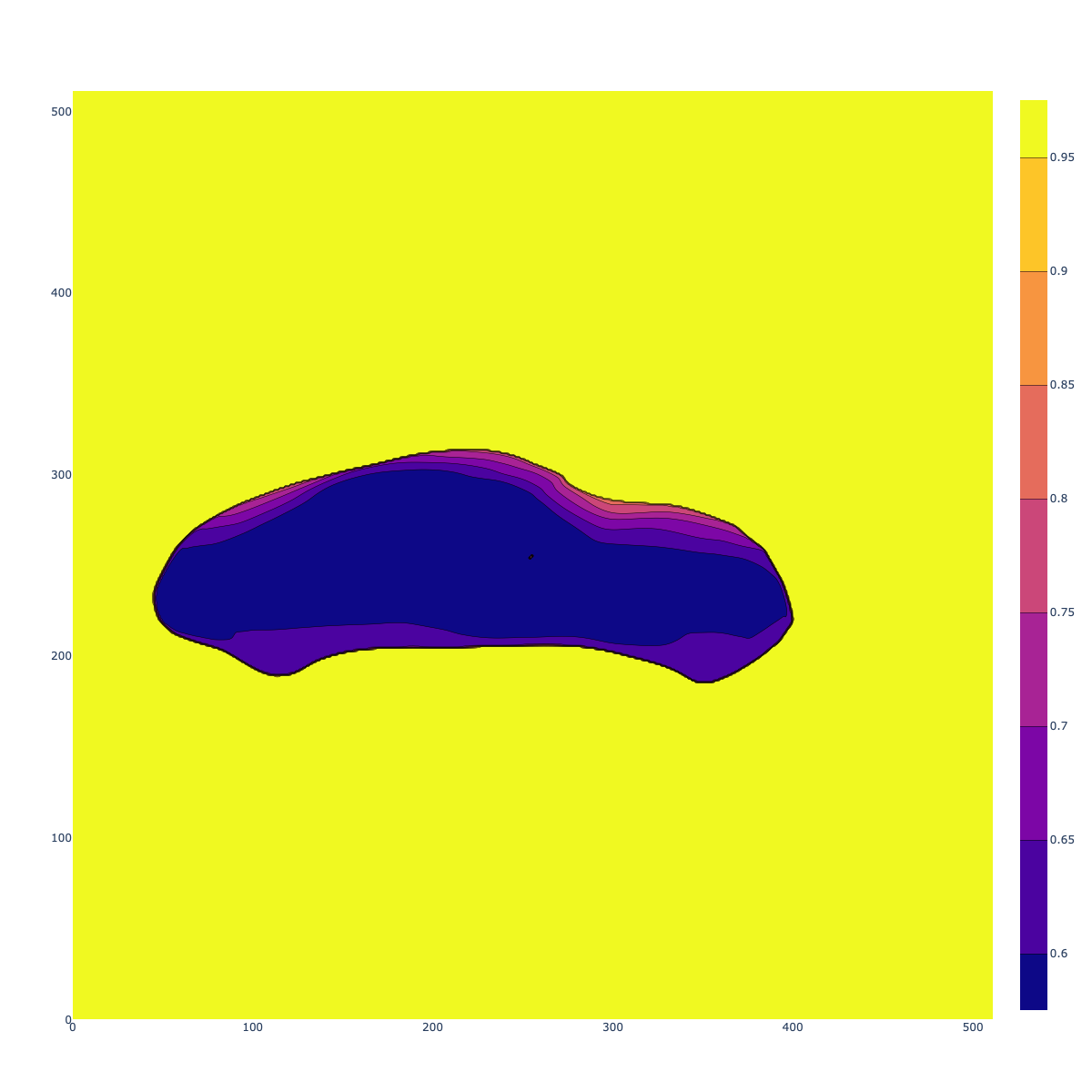}
\includegraphics[width=0.19\linewidth, trim=40mm 150mm 105mm 160mm, clip]{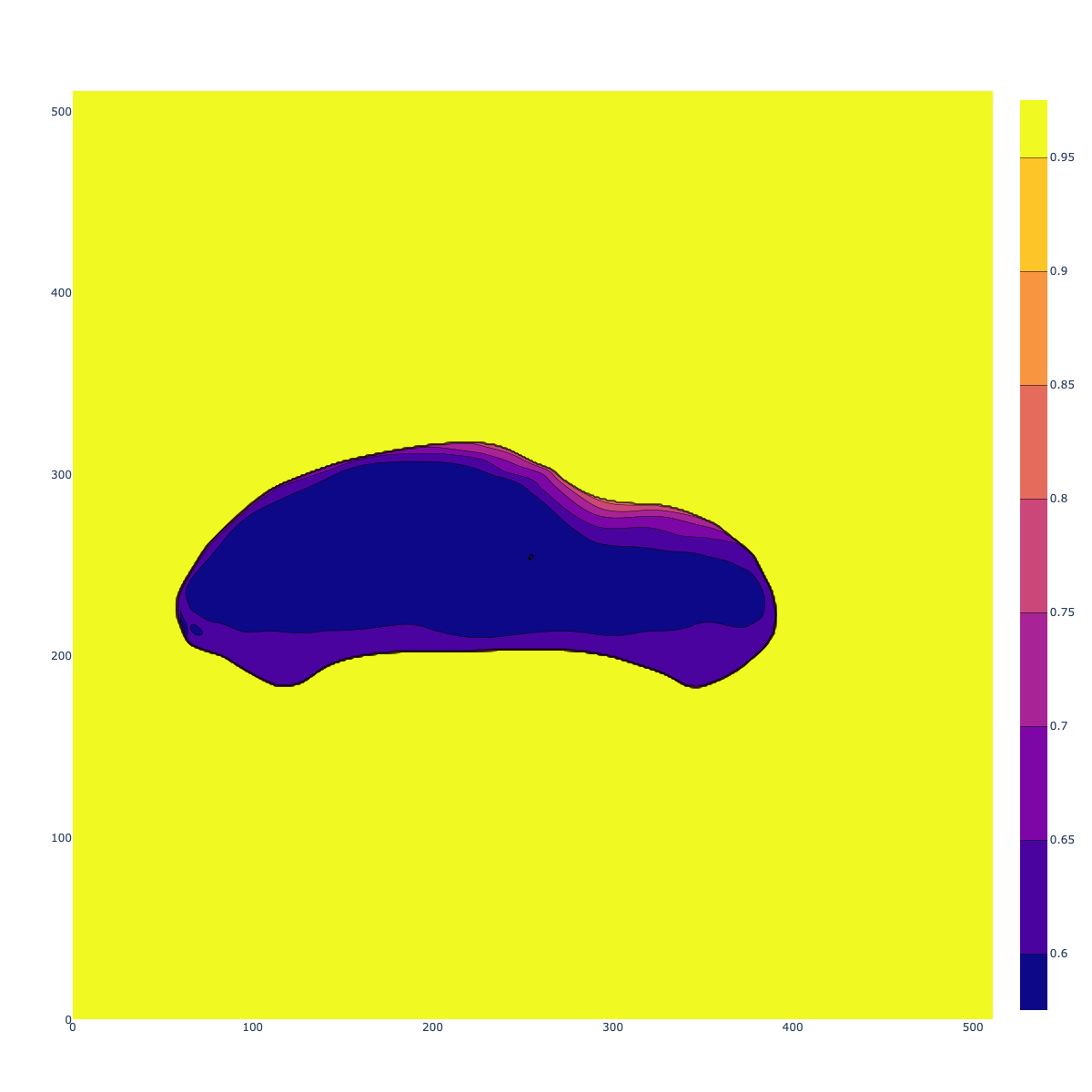}
\includegraphics[width=0.19\linewidth, trim={40mm 150mm 105mm 160mm}, clip]{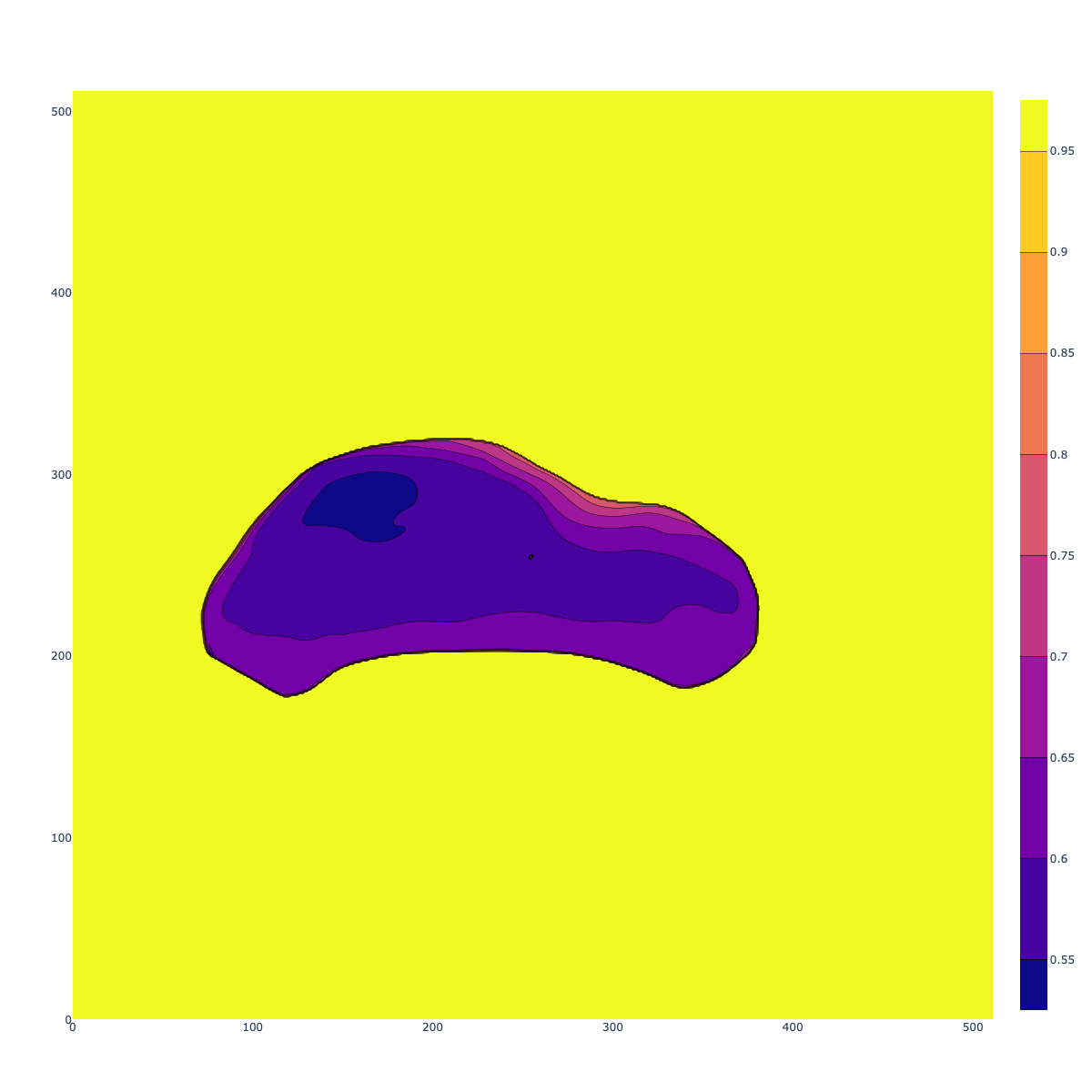}\\
\includegraphics[width=0.19\linewidth, trim=40mm 140mm 100mm 150mm, clip]{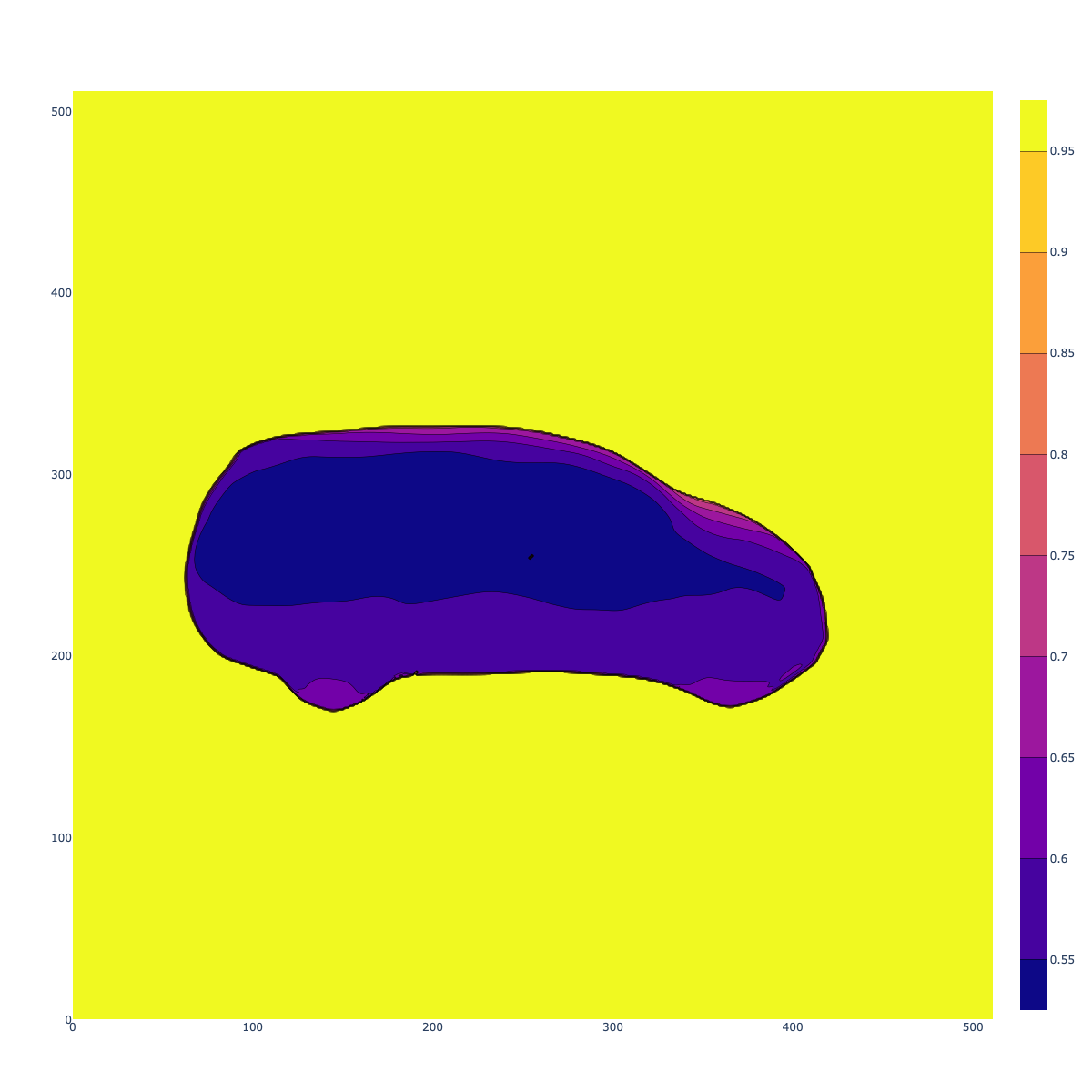}
\includegraphics[width=0.19\linewidth, trim=40mm 140mm 100mm 150mm, clip]{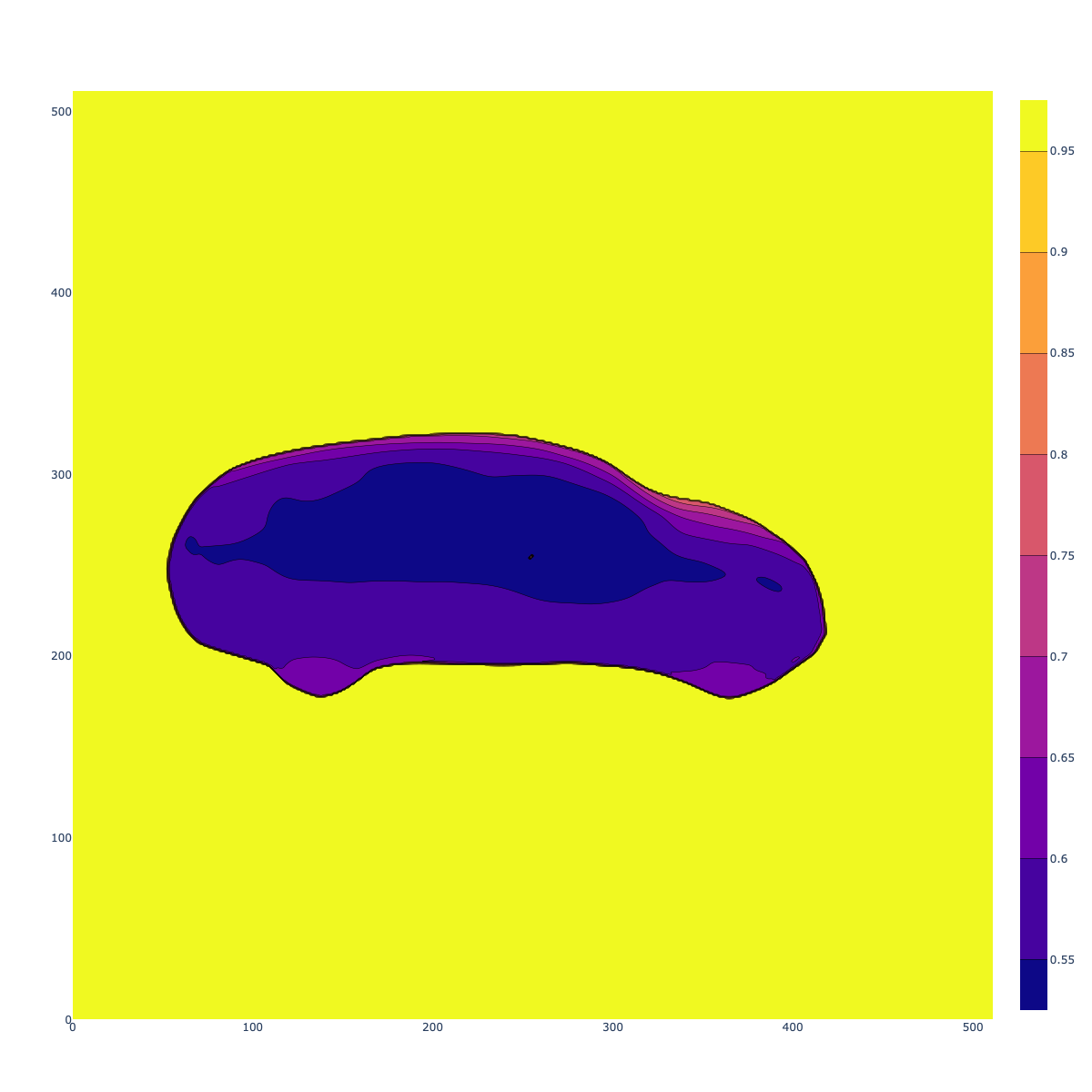}
\includegraphics[width=0.19\linewidth, trim=40mm 140mm 100mm 150mm, clip]{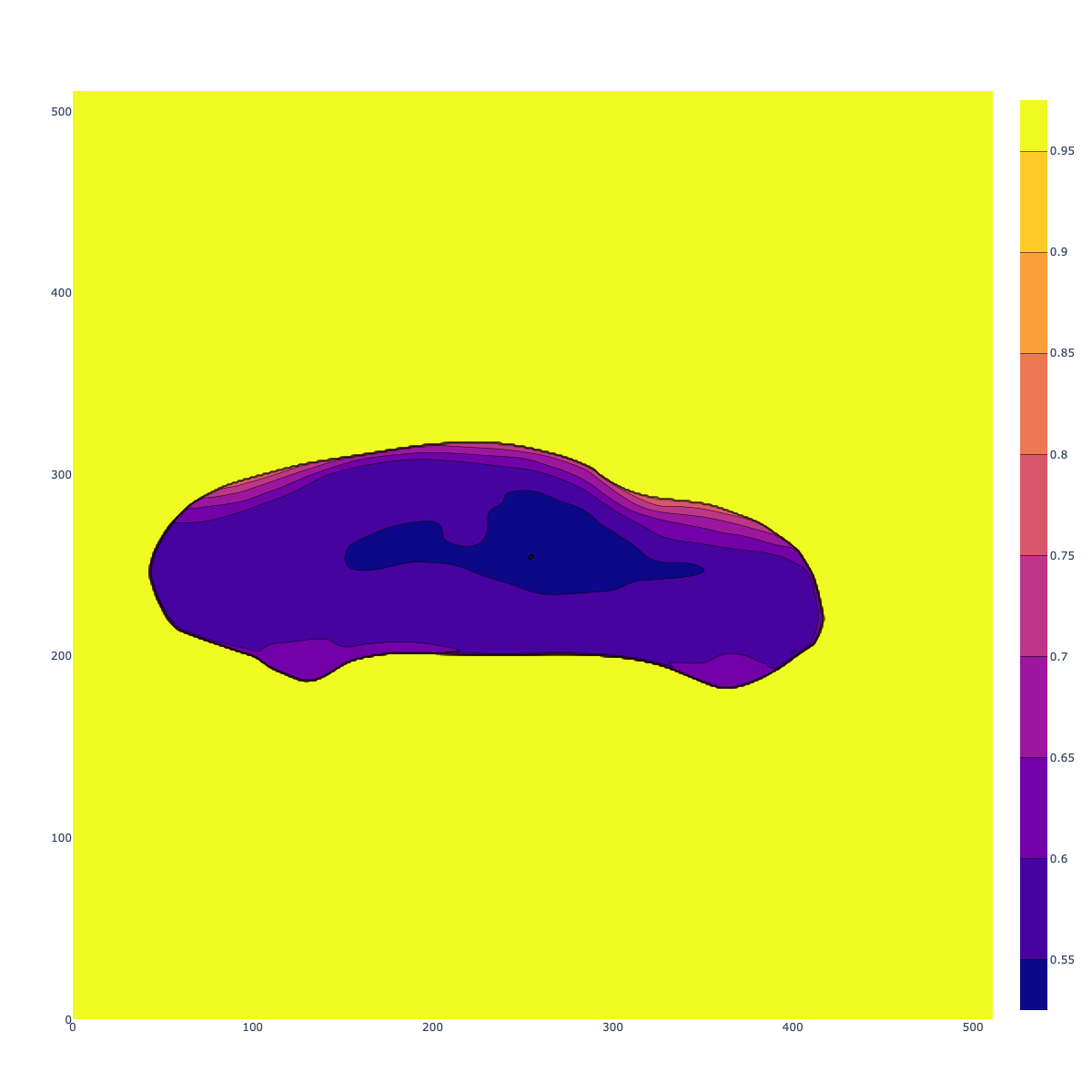}
\includegraphics[width=0.19\linewidth, trim=40mm 140mm 100mm 150mm, clip]{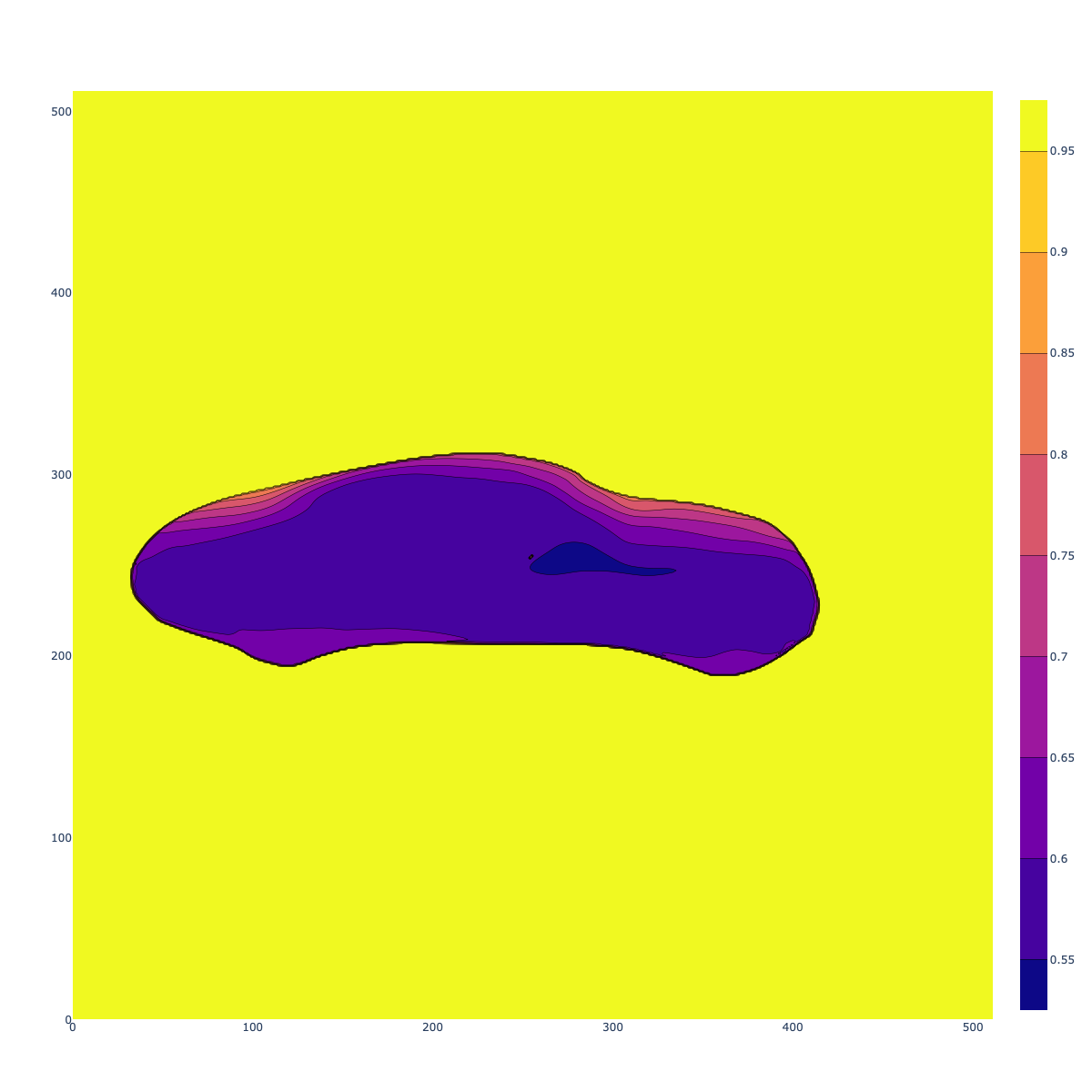}
\includegraphics[width=0.19\linewidth, trim=40mm 140mm 100mm 150mm, clip]{fig/3D/car/interpolation/EvalDepthImg500_phi90_theta22o5_15.png}
\end{minipage}}%
\caption{SDDF shape interpolation between two instances. The left-most and right-most columns show the SDDF output from the same view for two different instances from the training set. The three columns in the middle are generated by using a weighted average of the latent codes of the left-most and right-most instances as an input to the SDDF network. In each row, from left to right, the latent code weights with respect to the left-most instance are $1$, $0.75$, $0.5$, $0.25$, $0$, respectively. Note how the shapes transform smoothly from the left-most to the right-most instances with intermediate shapes looking like valid cars and airplanes. This demonstrates that the SDDF model represents the latent shape space continuously and meaningfully.}
\label{fig:interpolation}
\end{figure*}
%

\begin{figure*}[t]
\centering
\includegraphics[width=0.33\linewidth, trim={75mm 120mm 120mm 115mm}, clip]{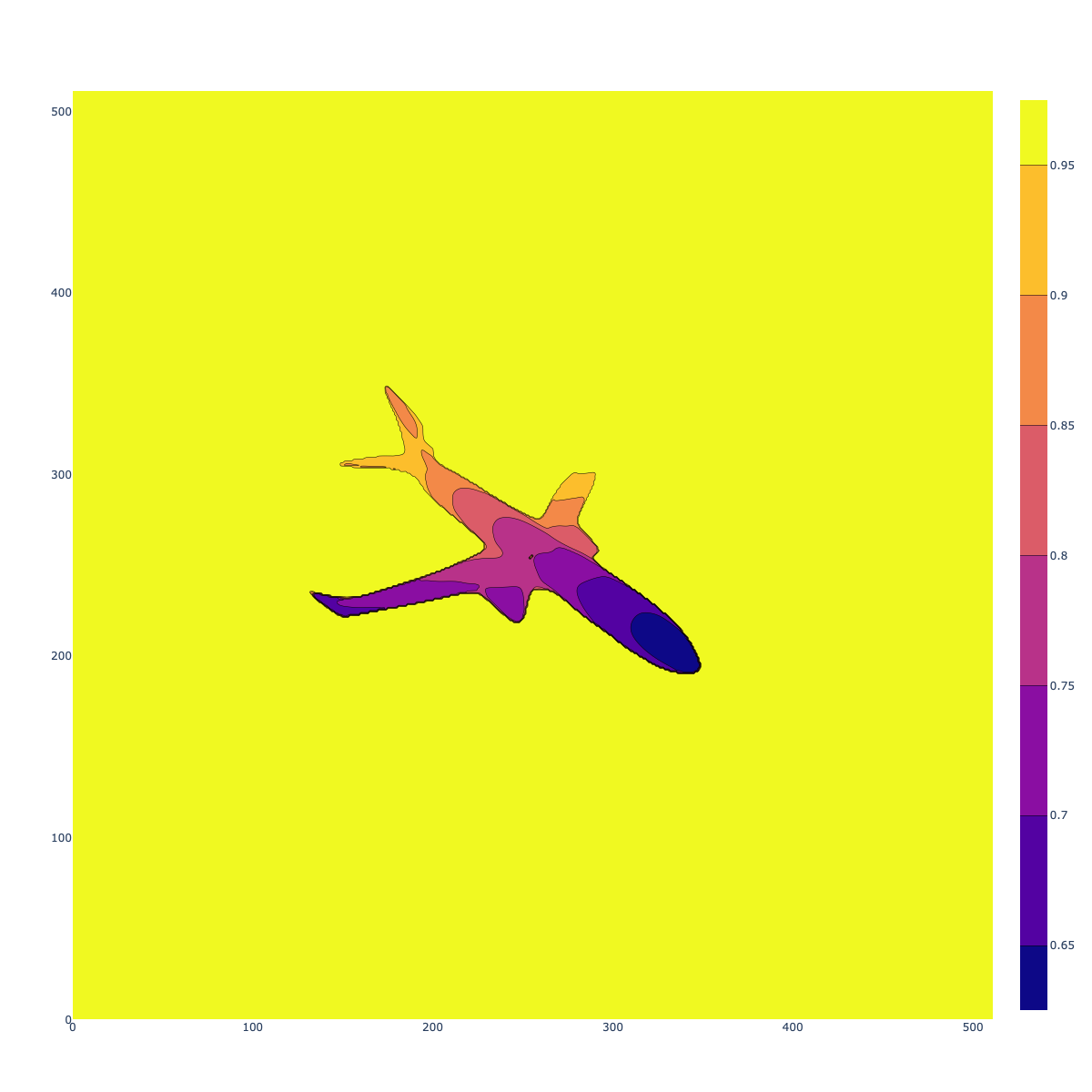}%
\hfill%
\includegraphics[width=0.33\linewidth, trim=0mm 0mm 0mm 0mm, clip]{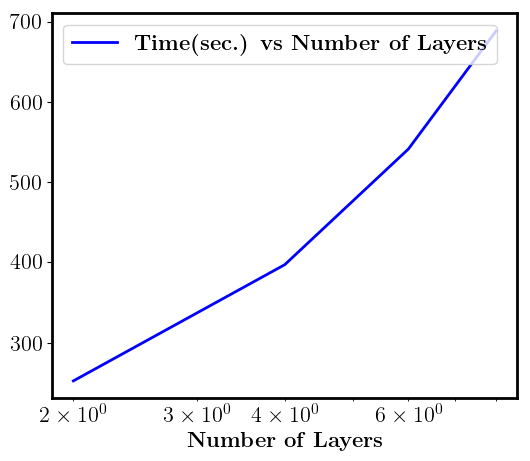}%
\hfill%
\includegraphics[width=0.33\linewidth, trim=0mm 0mm 0mm 0mm, clip]{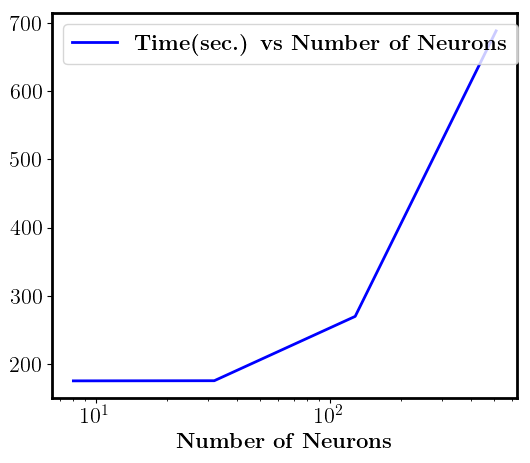}
\caption{Distance image (left) produced by our SDDF model with 8 layers, $512$ neurons per layer, and noiseless training data. The model training time is shown as a function of the number of layers (middle) and number of neurons per layer (right). Note that in the right plot the $x$-axis is in log scale and the number of neurons is equal in all layers. The error of the point cloud, obtained from $100$ distance images produced by our SDDF model at $100$ fixed poses, with respect to the ground-truth instance mesh in different settings: noisy distance data (left), changing number of network layers (middle), and changing numbers of neurons per layer (right). Note that in the right plot the $x$-axis is in log scale and the number of neurons is equal in all layers.}
\label{fig:times}
\end{figure*}
%

\begin{figure*}[t]
\begin{minipage}{\linewidth}
\centering
\includegraphics[width=0.24\linewidth, trim={75mm 120mm 120mm 115mm}, clip]{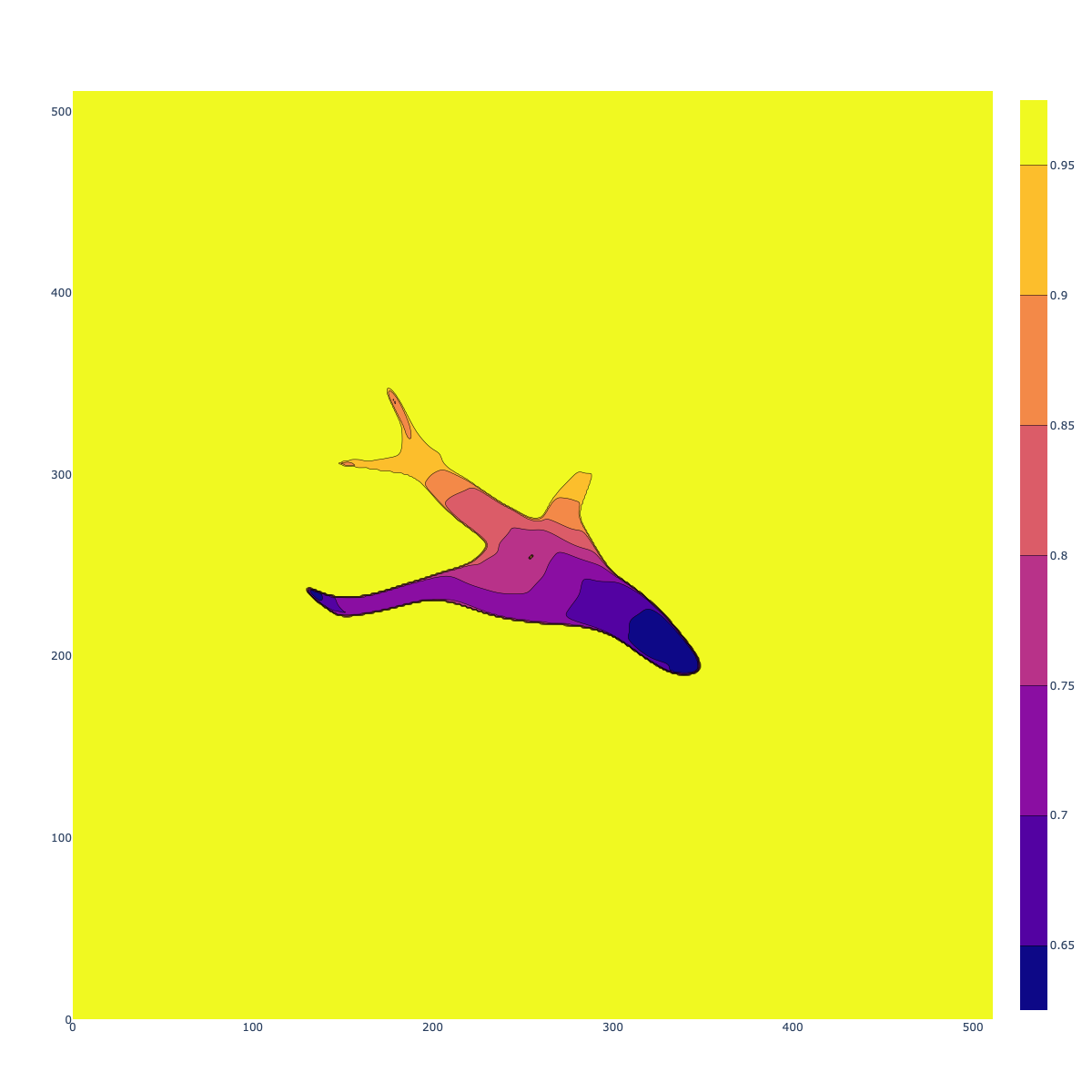}%
\hfill%
\includegraphics[width=0.24\linewidth, trim=75mm 120mm 120mm 115mm, clip]{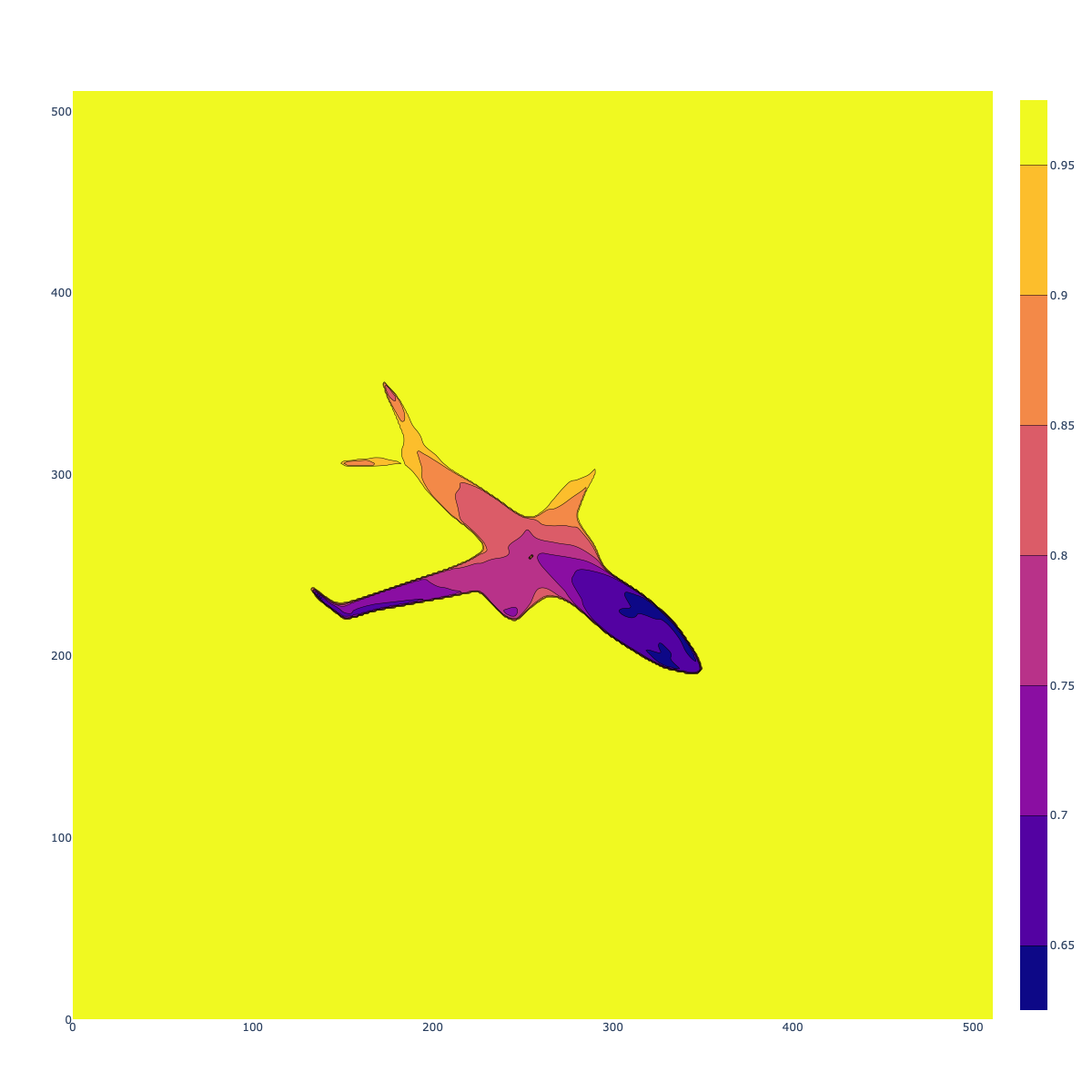}%
\hfill%
\includegraphics[width=0.24\linewidth, trim=75mm 120mm 120mm 115mm, clip]{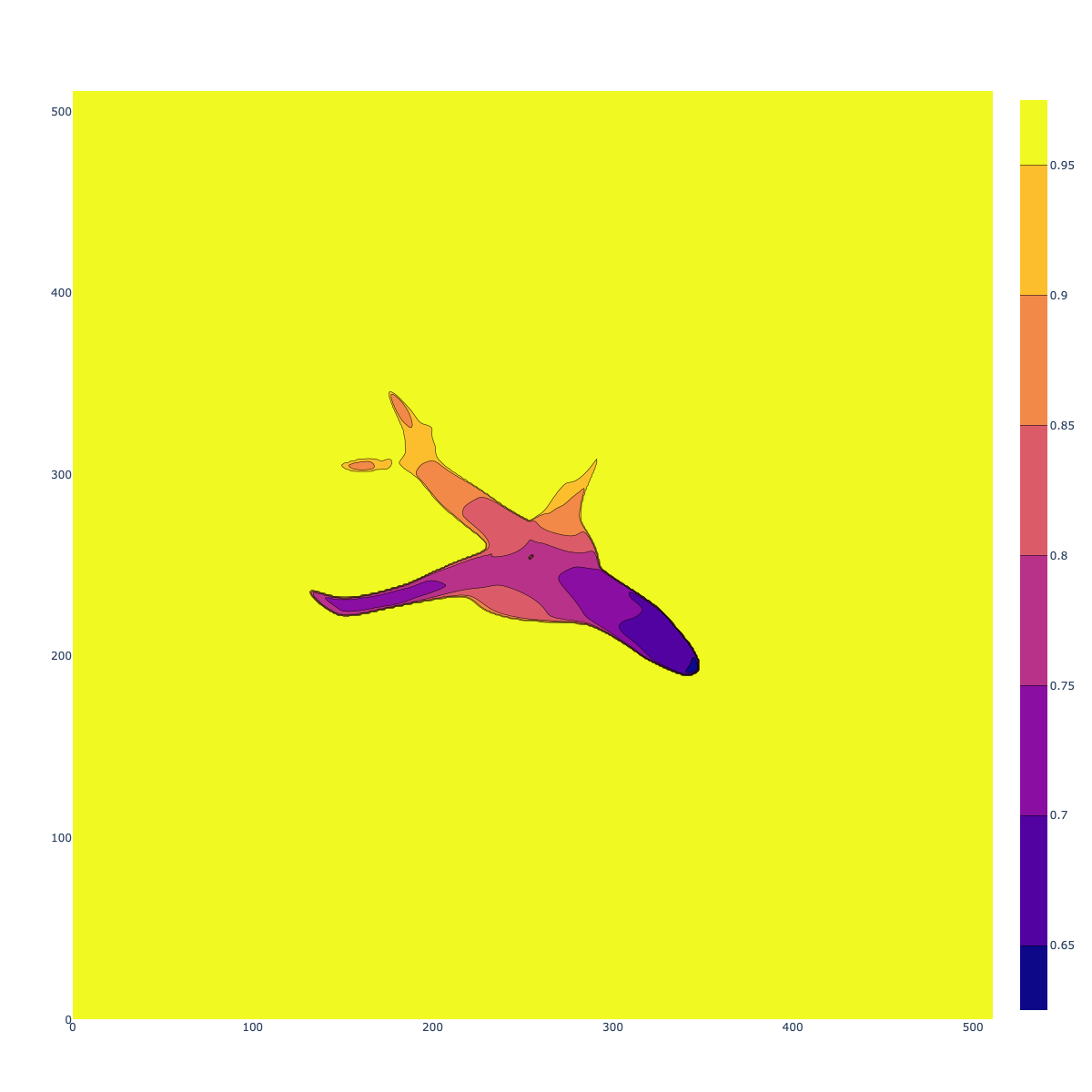}%
\hfill%
\includegraphics[width=0.24\linewidth, trim=75mm 120mm 120mm 115mm, clip]{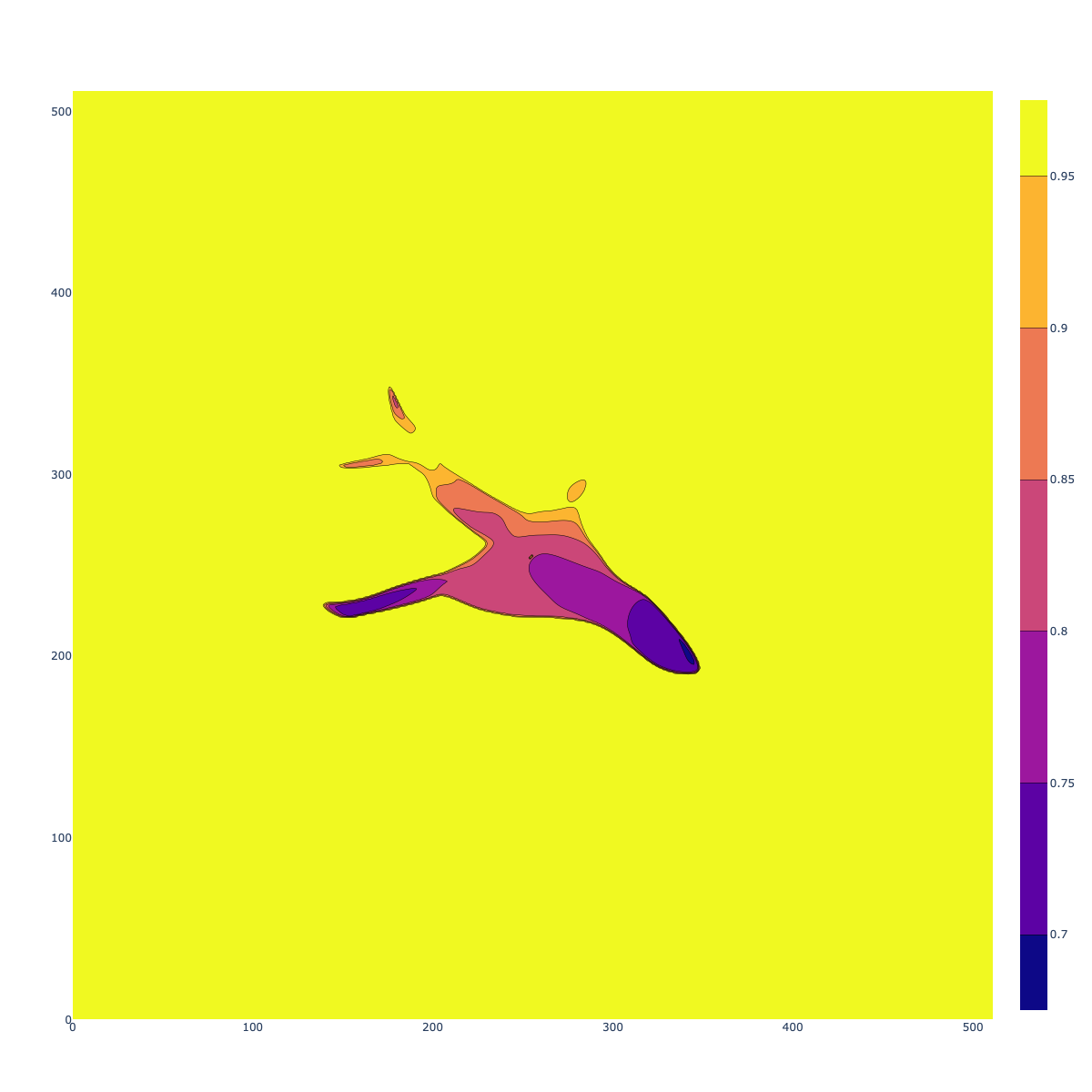}
\caption{Distance views synthesized by our SDDF model with default structure ($8$ layers and $512$ neurons per layer) trained on noisy data. The standard deviation of the Gaussian noise added to the distance data from left to right is $0.1$, $0.2$, $0.3$, $0.4$, respectively.}
\label{fig:noise}
\end{minipage}
\end{figure*}

%
%
%
%

\end{document}